\documentclass[letterpaper,12pt,oneside,final]{book}  

\makeatletter
\def\bstctlcite{\@ifnextchar[{\@bstctlcite}{\@bstctlcite[@auxout]}}
\def\@bstctlcite[#1]#2{\@bsphack
 \@for\@citeb:=#2\do{%
   \edef\@citeb{\expandafter\@firstofone\@citeb}%
   \if@filesw\immediate\write\csname #1\endcsname{\string\citation{\@citeb}}\fi}%
 \@esphack}
\makeatother

\newcommand\Langue{english}            

\usepackage{ifthen}
\usepackage[utf8]{inputenc}

\usepackage{lmodern}      
\usepackage[T1]{fontenc}  

\ifthenelse{\equal{\Langue}{english}}{
	\usepackage[english]{babel}
}{
	\usepackage[english]{babel} 
}

\newenvironment{abstract}%
{%
  \vskip 0.1in%
  \centering
  \begin{minipage}{0.97\textwidth} 
  \centerline%
  {\bfseries Abstract}%
  \vspace{-1mm}%
  \begin{quote}%
  \small
}
{
  \end{quote}%
  \end{minipage}%
  \par%
  \vskip 1ex%
}

\usepackage{silence}
\usepackage{enumitem}
\usepackage{wrapfig}
\usepackage{booktabs}
\usepackage[export]{adjustbox}
\usepackage{colortbl}
\usepackage{multirow}
\usepackage{makecell}
\usepackage{xcolor}
\usepackage{diagbox}

\newcommand{\bo}{\textbf{o}}
\newcommand{\ba}{\textbf{a}}
\newcommand{\br}{\textbf{r}}

\makeatletter
\newcommand{\tocless}[2]{
  \bgroup
  \addtocontents{toc}{\protect\setcounter{tocdepth}{0}}
  #1{#2}
  \addtocontents{toc}{\protect\setcounter{tocdepth}{2}}
  \egroup
}
\makeatother

\usepackage{graphicx}
\usepackage{epstopdf}  
\usepackage{flafter,placeins}
\usepackage{amsmath,color,soulutf8,longtable,colortbl,setspace,xspace,url,pdflscape}
\usepackage[nolist]{acronym}
\usepackage{fancyhdr}
\fancypagestyle{pagenumber}{\fancyhf{}\fancyhead[R]{\thepage}}

\makeatletter
\let\ps@plain=\ps@pagenumber
\makeatother
\usepackage{hyperref}
\usepackage{caption}  
\makeatletter
\providecommand*{\toclevel@compteur}{0}
\makeatother

\usepackage{subcaption} 
\usepackage{siunitx} 
\usepackage{amssymb} 
\usepackage[bottom]{footmisc} 

\RequirePackage[\Langue]{MemoireThese}

\usepackage{amsmath,amsfonts,bm,mathtools}

\def\eqref#1{equation~\ref{#1}}

\def\1{\bm{1}}

\DeclareMathAlphabet{\mathsfit}{\encodingdefault}{\sfdefault}{m}{sl}
\SetMathAlphabet{\mathsfit}{bold}{\encodingdefault}{\sfdefault}{bx}{n}

\def\gA{{\mathcal{A}}}

\def\gS{{\mathcal{S}}}
\def\gT{{\mathcal{T}}}

\def\sR{{\mathbb{R}}}

\newcommand{\E}{\mathbb{E}} 

\newcommand{\KL}{D_{\mathrm{KL}}}
\newcommand{\JS}{D_{\mathrm{JS}}}

\DeclareMathOperator*{\argmax}{arg\,max}
\DeclareMathOperator*{\argmin}{arg\,min}

\DeclareMathOperator{\indep}{\perp}
\DeclareMathOperator{\grad}{\nabla}

\makeatletter
\newcommand{\subalign}[1]{%
  \vcenter{%
    \Let@ \restore@math@cr \default@tag
    \baselineskip\fontdimen10 \scriptfont\tw@
    \advance\baselineskip\fontdimen12 \scriptfont\tw@
    \lineskip\thr@@\fontdimen8 \scriptfont\thr@@
    \lineskiplimit\lineskip
    \ialign{\hfil$\m@th\scriptstyle##$&$\m@th\scriptstyle{}##$\hfil\crcr
      #1\crcr
    }%
  }%
}
\makeatother

\newcommand{\pie}{\pi_{\!_E}}
\newcommand{\pig}{\pi_{\!_G}}

\newcommand{\pe}{p_{\!_E}}
\newcommand{\pg}{p_{\!_G}}

\newcommand{\pei}{p_{\!_{E}, i}}
\newcommand{\pgi}{p_{\!_{G}, i}}

\newcommand{\pigi}{\pi_{\!_{G}, i}}

\newcommand{\tildep}{\tilde{p}}
\newcommand{\tildepi}{\tilde{\pi}}
\newcommand{\tildef}{\tilde{f}}

\usepackage{amsmath, amsthm, amssymb, bm, bbm, cancel, mathtools}
\usepackage{algorithmic}
\usepackage[ruled]{algorithm2e}
\allowdisplaybreaks 
\newtheorem{thm}{Theorem}
\newtheorem{lem}{Lemma}

\newcommand{\vfparams}{\phi}  
\newcommand{\piparams}{\theta}  
\newcommand{\rparams}{\psi}
\newcommand{\pzero}{P_0}  
\newcommand{\ptrans}{P}  
\newcommand{\RF}{R}

\newcommand\CiteStyle{author}

\ifthenelse{\equal{\CiteStyle}{number}}{
	\bibliographystyle{bst_files/IEEEtran}
         
        \newcommand{\citep}{\cite}
        \newcommand{\citet}{\cite}
}{
	\usepackage{natbib}
        \renewcommand{\cite}{\citep}
        \bibliographystyle{apalike}
}

\newcommand\monTitre{Effective Reward Specification in Deep Reinforcement Learning}
\newcommand\monPrenom{Julien}
\newcommand\monNom{Roy}
\newcommand\monDepartement{génie informatique et génie logiciel}  
\newcommand\maDiscipline{Génie informatique}
\newcommand\monDiplome{D}        
\newcommand\anneeDepot{2024}    
\newcommand\moisDepot{Avril}       
\newcommand\PageGarde{N}         
\newcommand\AnnexesPresentes{O}  
\newcommand\mesMotsClef{Liste,de,mot-clés,séparés,par,des,virgules}
\newcommand\monJury{\PresidentJury{M}{Aloise}{Daniel}\\
\DirecteurRecherche{M}{Pal}{Christopher J.}\\
\CoDirecteurRecherche{M}{Bacon}{Pierre-Luc}\\
\MembreJury{F}{Dagdougui}{Hanane}\\
\MembreExterneJury{F}{Vamplew}{Peter}}

\ifthenelse{\equal{\monDiplome}{M}}{
\newcommand\monSujet{Mémoire de maîtrise}
\newcommand\monDipl{Maîtrise ès sciences appliquées}
}{
\newcommand\monSujet{Thèse de doctorat}
\newcommand\monDipl{Philosophi\ae{} Doctor}
}
\hypersetup{
  pdftitle={\monTitre},
  pdfsubject={\monSujet},
  pdfauthor={\monPrenom{} \monNom},
  pdfkeywords={\mesMotsClef},
  bookmarksnumbered,
  pdfstartview={FitV},
  hidelinks,
  linktoc=all
}

\usepackage[titles]{tocloft}

\ifthenelse{\equal{\Langue}{english}}{

}{

}
\begin{document}

\frontmatter
\ifthenelse{\equal{\PageGarde}{O}}{\addtocounter{page}{1}}{}
\thispagestyle{empty}%
\begin{center}%
\vspace*{\stretch{0.1}}
\textbf{POLYTECHNIQUE MONTRÉAL}\\
affiliée à l'Université de Montréal\\
\vspace*{\stretch{1}}
\textbf{\monTitre}\\
\vspace*{\stretch{1}}
\textbf{\MakeUppercase{\monPrenom~\monNom}}\\
Département de~{\monDepartement}\\
\vspace*{\stretch{1}}
\ifthenelse{\equal{\monDiplome}{M}}{Mémoire présenté}{Thèse présentée} en vue de l'obtention du diplôme de~\emph{\monDipl}\\
\maDiscipline\\
\vskip 0.4in
\moisDepot~\anneeDepot
\end{center}%
\vspace*{\stretch{1}}
\copyright~\monPrenom~\monNom, \anneeDepot.
\newpage\thispagestyle{empty}%
\begin{center}%

\vspace*{\stretch{0.1}}
\textbf{POLYTECHNIQUE MONTRÉAL}\\
affiliée à l'Université de Montréal\\
\vspace*{\stretch{2}}
Ce\ifthenelse{\equal{\monDiplome}{M}}{~mémoire intitulé}{tte thèse intitulée} :\\
\vspace*{\stretch{1}}
\textbf{\monTitre}\\
\vspace*{\stretch{1}}
présenté\ifthenelse{\equal{\monDiplome}{M}}{}{e}
par~\textbf{\mbox{\monPrenom~\MakeUppercase{\monNom}}}\\
en vue de l'obtention du diplôme de~\emph{\mbox{\monDipl}}\\
a été dûment accepté\ifthenelse{\equal{\monDiplome}{M}}{}{e} par le jury d'examen constitué de :\end{center}
\vspace*{\stretch{2}}
\monJury
\pagestyle{pagenumber}%

\ifthenelse{\equal{\Langue}{english}}{
	\chapter*{ACKNOWLEDGEMENTS}\thispagestyle{headings}
	\addcontentsline{toc}{compteur}{ACKNOWLEDGEMENTS}
}{
	\chapter*{REMERCIEMENTS}\thispagestyle{headings}
	\addcontentsline{toc}{compteur}{REMERCIEMENTS}
}
First, I wish to express my gratitude to my advisors, Christopher Pal and Pierre-Luc Bacon, for taking a chance on me, introducing me to Mila, this beautiful ecosystem, and granting me the freedom to explore the subjects that caught my interest. I also wish to thank Farida Cheriet, Fantin Girard, Kipton Barros and Nicholas Lubbers who gave me my first opportunities back in undergrad.
Thank you Olivier Delalleau, Joshua Romoff, Pedro Pinheiro, Joseph Viviano, and Emmanuel Bengio for being such great mentors and for inspiring me to be an ever more investigative, efficient, and confident researcher. Thank you to all my collaborators, Derek Nowrouzezahrai, Wonseok Jeon, Joelle Pineau, and Roger Girgis. It has been a great pleasure to work with each and everyone of you.

Thank you, Félix G. Harvey, for first sparking my fascination with these mysterious neural networks. You have been like a big brother to me on this journey, and I might never have embarked on this career path without you. Thank you, Paul Barde, for being such a great companion during our first years in the program, and such a great friend afterwards. I have learned a lot working with you. Thank you, Samuel Lavoie, Maude Lizaire, and David Kanaa, my close friends from the lab. I look forward to many more coffee breaks with you, chatting about movies, chess, and the meaning of life.

Thank you to my family for always encouraging me in my studies. In particular, to my dad, Alain Roy, for making me a curious child, and my mom, Martine Ducharme, for making me a meticulous one! Thanks to my sisters, Laurence and Flavie, for your presence and your colors; you both make me proud. Thank you, Noémie, for your patience, your kindness, and your unwavering support.

Finally, a deep thanks to my dear friends outside the lab. I am extremely lucky to be surrounded by such wonderful human beings. 
Friends are the family we choose, and I am grateful for every one of you.

%
\chapter*{RÉSUMÉ}\thispagestyle{headings}
\addcontentsline{toc}{compteur}{RÉSUMÉ}


Au cours de la dernière décennie, les progrès dans le domaine de l'apprentissage par renforcement profond en ont fait l'un des outils les plus efficaces pour résoudre les problèmes de prise de décision séquentiels. Cette approche combine l'excellence de l'apprentissage profond à traiter des signaux complexes avec l'adaptabilité de l'apprentissage par renforcement (RL) pour s'attaquer à une panoplie de problèmes de contrôle. Lorsqu'il effectue une tâche, un agent de RL reçoît des récompenses ou des pénalités en fonction de ses actions. Cet agent cherche à maximiser la somme de ses récompenses, permettant ainsi aux algorithmes d'IA de découvrir des solutions novatrices dans plusieurs domaines. Cependant, cette focalisation sur la maximisation de la récompense introduit également une difficulté importante: une spécification inappropriée de la fonction de récompense peut considérablement affecter l'efficacité du processus d'apprentissage et entraîner un comportement indésirable de la part de l'agent.

Dans cette thèse, nous présentons des contributions au domaine de la spécification de récompense en apprentissage par renforcement profond sous forme de quatre articles. Nous commençons par explorer l'apprentissage par renforcement inverse, qui modélise la fonction de récompense à partir d'un ensemble de démonstrations d'experts, et proposons un algorithme permettant une implémentation et un un processus d'optimisation efficaces. Ensuite, nous nous penchons sur le domaine de la composition de récompense, visant à construire des fonctions de récompense efficaces à partir de plusieurs composantes. Nous prenons le cas de la coordination multi-agent, et proposons des tâches auxiliaires qui ajoutent des signaux de récompense sous forme de biais inductifs qui permettent de découvrir des politiques performantes dans des environnements coopératifs. Nous investiguons également l'utilisation de l'optimisation sous contrainte et proposons un cadre pour une spécification plus directe et intuitive de la fonction de récompense. Finalement, nous nous tournons vers le problème de l'apprentissage par renforcement pour la découverte de nouveaux médicaments et présentons une approche multi-objectif conditionnée permettant d'explorer tout l'espace des objectifs.

Ci-après, nous commençons par présenter une revue la littérature sur les stratégies de spécification, identifions les limitations de chacune de ces approches et proposons des contributions originales abordant le problème de l'efficacité et de l'alignement en apprentissage par renforcement profond. 
La spécification de récompense représente l'un des aspects les plus difficiles de l'application de l'apprentissage par renforcement dans les domaines réels. Pour le moment, il n'existe pas de solution universelle à ce défi complexe et nuancé; sa résolution nécessite la sélection des outils les plus appropriés pour les exigences spécifiques de chaque application.


\chapter*{ABSTRACT}\thispagestyle{headings}
\addcontentsline{toc}{compteur}{ABSTRACT}
\begin{otherlanguage}{english}


In the last decade, Deep Reinforcement Learning has evolved into a powerful tool for complex sequential decision-making problems. It combines deep learning's proficiency in processing rich input signals with reinforcement learning's adaptability across diverse control tasks. At its core, an RL agent seeks to maximize its cumulative reward, enabling AI algorithms to uncover novel solutions previously unknown to experts. However, this focus on reward maximization also introduces a significant difficulty: improper reward specification can result in unexpected, misaligned agent behavior and inefficient learning. The complexity of accurately specifying the reward function is further amplified by the sequential nature of the task, the sparsity of learning signals, and the multifaceted aspects of the desired behavior.

In this thesis, we present contributions to the field of reward specification in deep reinforcement learning in the form of four articles. We start by exploring inverse reinforcement learning, which models the reward function from a set of expert demonstrations, and introduce an algorithm allowing for an efficient implementation and training procedure. Then, we delve into the realm of reward composition, aiming to construct effective reward functions from various components. We take the case of multi-agent coordination and propose auxiliary tasks that augment the reward signal with inductive biases leading to high-performing policies in cooperative multi-agent environments. We also investigate the use of constrained optimization and propose a framework for direct reward specification when using a specific constraint family. Lastly, we turn our attention to the problem of RL for drug discovery and present a goal-conditioned, multi-objective approach to explore the entire objective space of molecular candidates.

Throughout this document, we survey the literature on effective reward specification strategies, identify core challenges relating to each of these approaches, and propose original contributions addressing the issue of sample efficiency and alignment in deep reinforcement learning. Reward specification represents one of the most challenging aspects of applying reinforcement learning in real-world domains. Our work underscores the absence of a universal solution to this complex and nuanced challenge; solving it requires selecting the most appropriate tools for the specific requirements of each unique application.

\end{otherlanguage}


{\setlength{\parskip}{0pt}
\ifthenelse{\equal{\Langue}{english}}{
	\renewcommand\contentsname{TABLE OF CONTENTS}
}{
	\renewcommand\contentsname{TABLE DES MATIÈRES}
}
\tableofcontents

\ifthenelse{\equal{\Langue}{english}}{
	\renewcommand\listtablename{LIST OF TABLES}
}{
	\renewcommand\listtablename{LISTE DES TABLEAUX}
}\listoftables


\ifthenelse{\equal{\Langue}{english}}{
	\renewcommand\listfigurename{LIST OF FIGURES}
}{
	\renewcommand\listfigurename{LISTE DES FIGURES}
}\listoffigures

}


\ifthenelse{\equal{\Langue}{english}}{
	\chapter*{LIST OF ACRONYMS}
        \addcontentsline{toc}{compteur}{LIST OF ACRONYMS}
}{
}
\pagestyle{pagenumber}
%
%
\begin{longtable}{lp{5in}}
AI      & Artificial Intelligence                     \\
AIL     & Adversarial Imitation Learning               \\
AIRL    & Adversarial Inverse Reinforcement Learning   \\
ASAF    & Adversarial Soft-Advantage Fitting           \\
BCE     & Binary Cross-Entropy                         \\
CMDP    & Constrained Markov Decision Processes        \\
CNN     & Convolutional Neural Network                 \\
CRL     & Constrained Reinforcement Learning           \\
CTDE    & Centralized Training and a Decentralized Execution \\
CUD     & Categorical Uniform Distributions            \\
DAG     & Directed Acyclic Graph                       \\
DDPG    & Deep Deterministic Policy Gradient           \\
DPG     & Deterministic Policy Gradient                \\
DQN     & Deep Q-Networks                              \\
GAN     & Generative Adversarial Network               \\
GFN     & Generative Flow Network                      \\
GNN     & Graph Neural Networks                         \\
GP      & Gradient Penalty                             \\
GVF    & General Value Function                      \\
HiL     & Human-in-the-Loop                            \\
IGD     & Inverted Generational Distance               \\
IL      & Imitation Learning                           \\
IRL     & Inverse Reinforcement Learning               \\
LLM    & Large Language Model                        \\
MADDPG  & Multi-Agent Deep Deterministic Policy Gradient \\
MARL    & Multi-Agent Reinforcement Learning           \\
MC      & Monte Carlo                                  \\
MCMC        & Monte Carlo Markov Chain \\
MDP     & Markov Decision Processes                    \\
ML      & Machine Learning                             \\
MLE     & Maximum Likelihood Estimation                \\
MOMDP  & Multi-Objective MDP                         \\
MOO     & Multi-Objective Optimization                 \\
MORL    & Multi-Objective Reinforcement Learning       \\
MSE     & Mean Squared Error                           \\
NLP     & Natural Language Processing                  \\
NPC     & Non-Player Characters                        \\
PBRS    & Potential-Based Reward Shaping               \\
PC-ent  & Pareto-Clusters Entropy                      \\
PCC     & Pearson Correlation Coefficient              \\
PPO     & Proximal Policy Optimization                 \\
PbRL    & Preference-based Reinforcement Learning      \\
QED     & Quantitative Estimate of Drug-likeness       \\
RL      & Reinforcement Learning                       \\
RLHF    & Reinforcement Learning from Human Feedback   \\
SAC     & Soft Actor-Critic                            \\
SE      & Standard Error                               \\
SGD         & Stochastic Gradient Descent \\
SQIL    & Soft-Q Imitation Learning                    \\
TD      & Temporal Difference                          \\
TRPO    & Trust Region Policy Optimization             \\
Tab-GS  & Tabular Goal-Sampler                         \\
\end{longtable}

\ifthenelse{\equal{\AnnexesPresentes}{O}}{\listofappendices}{}


\mainmatter


\chapter[INTRODUCTION]{\\INTRODUCTION}\label{sec:Introduction}

From the implementation of the first perceptron to modern neural network architectures, Machine Learning (ML) has been at the center of great leaps in artificial intelligence and keeps pushing its boundaries today like never before. Starting with notable achievements in digit recognition and learning to play Backgammon \citep{lecun1998gradient, tesauro1994td}, ML has since asserted its dominance in fields requiring high-dimensional data processing such as computer vision, speech synthesis and natural language processing \citep{chai2021deep, oord2016wavenet, otter2020survey}. These methods are now showing great capability in complex control problems, setting new benchmarks in playing strategic board games, flying stratospheric balloons, and advancing robotics \citep{silver2017mastering, bellemare2020autonomous, akkaya2019solving}, each time pushing the boundaries of what computers can do.

The ML revolution in information processing has been made possible by stepping away from traditional programming to embrace a radically different paradigm. Instead of requiring a person to directly encode a precise sequence of instructions into a computer program, ML combines powerful optimization methods and flexible models to distill vast amounts of data into complex functional behaviors. The ``program'' now consists of a set of connection weights that form an artificial neural network, converting inputs into outputs to achieve our goals. Crucially, these weights are not explicitly designed, but rather \textit{discovered} by an optimization algorithm, and the designers of such systems are now responsible for defining objective functions, curating datasets, and crafting algorithms that guide the learning process \citep{karpathy2017software2}. Not only are such systems easily scalable and can be optimized on dedicated hardware, but more importantly, they allow to address problems of a complexity beyond the reach of conventional logic-based programming.

This new paradigm provides a powerful framework for problem solving, but also presents unique challenges in the realm of Reinforcement Learning (RL), particularly in how agents are instructed to achieve desired outcomes. We use the term \textit{reward specification} to describe the act of providing an agent with a reward function \citep{taylor2023reinforcement, bowling2023settling, abel2021expressivity, singh2009rewards, icarte2022reward}. Although central to the field, reward specification is not a problem that all RL researchers have faced. Most research efforts report results on established benchmarks
using pre-defined reward functions. However, when developing an RL solution for a real-life application, how the reward is specified becomes a major factor pertaining to the success or failure of the project \citep{knox2023reward}.

This thesis aims to address the fundamental question: How can we specify a reward function that effectively captures human intentions, ensuring that an agent can learn efficiently and that its behavior aligns with our objectives? We start by motivating how reinforcement learning is uniquely positioned to help us solve difficult problems in artificial intelligence and highlight two of its own challenges.

\section{The Three Paradigms of Machine Learning}

The field of machine learning encompasses a wide range of approaches, each with distinct methods for directing algorithmic learning. These methods are often categorized into three learning paradigms. Supervised learning, perhaps the most widely used, operates under a strict framework where algorithms learn from labeled data, mapping inputs to known outputs through a clear supervisory signal. This guidance clearly establishes the task that the model should perform. However, its dependency on large volumes of labeled data can be a limitation, as obtaining such data is often costly and labor intensive \citep{mathewson2022brief}.

Unsupervised learning, on the other hand, does not require labels, leaving algorithms to discern patterns and structures within the data autonomously. This approach thus operates under a much lower supervisory burden. However, the absence of explicit guidance in unsupervised learning is both its strength and weakness; it takes advantage of readily available unlabeled data but lacks a definitive direction for problem solving. For this reason, it is often used as a pre-training procedure for other downstream tasks \citep{erhan2010does}.

Reinforcement Learning (RL) strikes a unique balance between the structured guidance of supervised learning and the autonomous nature of unsupervised learning. In RL, an agent learns to perform a task by interacting with its environment, guided by a reward function rather than explicit data labels. This reward function provides a scalar value as feedback for the actions taken, similar to receiving a score in a game, rather than a precise road map to solve the task. The agent's objective is to maximize its cumulative reward, which indirectly shapes its behavior. This form of learning is thus less prescriptive than supervised learning, avoiding the need for exhaustive labeling, yet it is more directed than unsupervised learning since the reward function encodes the task objectives.

By specifying the task using a reward function rather than collecting data, the RL paradigm opens up new possibilities. First, to implement a reward function which captures the desired behavior, a system designer does not need to know the solution to the problem but simply to be able to rate solutions, significantly reducing the amount of expert knowledge that must be held to tackle a particular problem. Second, precisely because one does not need to be able to solve the problem to design its reward function, the performance of the model is unbounded. This, for example, allowed a computer to surpass the best human player at Go for the first time in 2016 \citep{silver2016mastering}. In certain AI applications, such as autonomous driving or object classification, we simply seek to have computers emulate human behavior to free up time and resources to other needs. However, the world is full of challenges such as energy grid management, traffic light optimization, and molecular design, for which the ability to go beyond human performance and reach the best possible solution would be extremely valuable. Its lower supervision requirements and superhuman potential make reinforcement learning uniquely positioned to tackle some of the most important problems that machine learning can face \citep{leike2018scalable}.

\section{Characteristic Challenges of Reinforcement Learning}

The unique freedom of strategy given to the agent to maximize its rewards comes with important challenges. First, there is the problem of \textit{exploration}. The space of behaviors that can be produced in an environment is often exponential in trajectory length, making it intractable to evaluate by performing an exhaustive sweep. How to efficiently explore this environment is a fundamental question in RL \citep{amin2021survey}. To tackle this issue, most algorithms alternate between taking decisions that are already known to lead to good outcomes, to avoid wasting time on completely unfit strategies, and taking random actions to discover whether the current strategy can be improved. This is known as the exploitation-exploration dilemma and no single solution has been found to offer the perfect balance in every environment, leaving practitioners to experiment with different hyperparameters for their specific application.

Secondly, there is the problem of \textit{alignment} \citep{amodei2016concrete, leike2018scalable}. Specifying the task with a reward function is advantageous in reducing the burden of expert knowledge. However, this reward function typically fails to capture all of the aspects of behavior that its designer values. This is because when designing a reward function, a practitioner might envision specific scenarios in which the current function would incentivize the desired behavior, but when exploring vast and complex environments, the agent will inevitably find itself in circumstances that were unforeseen and where this same reward function can become counterproductive \citep{hadfield2017inverse}. Trying to correct this mistake is often complicated since the consequences of modifying the reward function or adding new elements to it can be very difficult to foresee due to the unintuitive nature of the reward accumulation through time, the effect of discounting, and of the competition between reward components \citep{septon2022integrating, booth2023perils}.

To address exploration and alignment, the reward function must allow the agent to both learn efficiently and converge to a behavior that is consistent with the intentions of its designers \citep{sorg2011optimal}. However, these two objectives are often conflictual. In their seminal textbook, \citet{sutton2018rlTextbook} underscore the indirect nature of behavior specification in RL:

\begin{quote}
\textit{``The agent always learns to maximize its reward. If we want it to do something for us, we must provide rewards to it in such a way that in maximizing them the agent will also achieve our goals.''}
\end{quote}

This convoluted mapping from reward to behavior makes the design of an effective reward function a surprisingly arduous task \citep{singh2010separating}. The propensity of RL to develop strategies that unexpectedly exploit its reward function has been reported in several works \citep{randlov1998learning, clark2016faulty, knox2023reward, pan2022effects}, and this challenge may be even more familiar to practitioners outside of academic circles \citep{gupta2022unpacking}. Designing effective reward functions thus represents one of the most challenging aspects of RL \citep{leike2018scalable, daniel2014active, christiano2017deep, hu2020learning, vamplew2022scalar}, and this difficulty appears to increase as RL algorithms become more capable \citep{dewey2014reinforcement, pan2022effects}, making the question of reward specification ever more critical. 

\section{Outline}

In this thesis, we present contributions to different algorithmic families that aim to incorporate human intuition in the task specification process of RL agents in the form of demonstrations, auxiliary tasks, behavioral constraints and goal-conditioning.
The next sections are organized as follows. Chapter~\ref{chap:background} presents the technical background that lays the foundations of deep learning and reinforcement learning. Chapter~\ref{chap:lit_review} presents a review of the relevant literature surrounding the problem of reward specification in RL. Chapters~\ref{chap:preamble}~to~\ref{chap:article4_gcgfn} present four original contributions in the form of peer-reviewed articles. Finally, Chapter~\ref{chap:discussion} presents a discussion of the limitation of our methods, the additional opportunities offered by environment design, agent monitoring, and human interventions, and some fundamental obstacles to effective objective specification.

\chapter[TECHNICAL BACKGROUND]{\\TECHNICAL BACKGROUND}\label{chap:background}

The contributions presented in this thesis focus on various approaches for reward specification in deep reinforcement learning. In this chapter, we present essential elements of deep learning (Section~\ref{sec:deep_learning}) and dive into more details on the foundations of reinforcement learning (Section~\ref{sec:reinforcement_learning}) to establish how these two fields can be brought together to tackle challenging, high-dimensional sequential decision making problems.

\section{Deep Learning}
\label{sec:deep_learning}

Deep learning \citep{goodfellow2016deep} is a subfield of machine learning \citep{murphy2012machine} which focuses on the design of architectures and training algorithms for neural networks. 

\subsection{Neural Networks}
\label{sec:neural_architectures}

Neural networks are a powerful family of parametric models for function approximation. While a plethora of architectures have been developed for different applications, all neural networks at their core are composed of artificial neurons, which simply consist of a weighted linear combination $a(\cdot)$ of their input $x$ followed by a nonlinearity $z(\cdot)$. These units, sometimes called \textit{perceptron} \citep{rosenblatt1958perceptron}, can be assembled into a \textit{layer} of $d_a$ such neurons, giving the model its \textit{width}. Layers can also be composed in sequence, giving the model its \textit{depth}. All but the final layer are called \textit{hidden} layers, and allow to build increasingly more abstract representations of the data \citep{lecun2015deep}. For example, a simple model could be built from one hidden layer followed by a linear output: $f(x) := b^o + w^{o\top} z\big(a(x)\big)$. Here, the variables $x$ and $a$ represent vectors, and nonlinearity functions denoted by $z$ are applied element-wise:
\begin{equation}
    f(x) := b^o + \sum_{i=1}^{d_a} w_i^o z(a_i) \quad , \quad a_{j} := b^a_j + \sum_{i=1}^{d_x} W^a_{ij} x_i
\end{equation}
where $w^o$  (the output weight vector), $W^a$ (the input weight matrix) and $b$ (the bias terms) are the parameters of the model. The output can be augmented with any differentiable function to satisfy the task at hand. For example, binary classification tasks typically suggest the use of the sigmoid output function, while regression tasks often use the network's output as-is. This basic framework has led to the development of several specialized architectures to accommodate different data types. For example, Convolutional Neural Networks (CNNs) are specialized for grid-like data, such as images, and reuse artificial neurons with local connectivity across the entire grid to enforce spatial equivariance of the function it learns \citep{lecun1998gradient}. Graph Neural Networks (GNNs) generalize this idea to node permutations in graph structures \citep{xu2018powerful}.

These building blocks give all neural networks two characteristics crucial to their success. First, neural networks are easily differentiable, allowing the use of the backpropagation algorithm \citep{rumelhart1986learning} to compute parameter updates. This procedure not only allows an efficient and parallelizable way of training large models, but it has become simple to implement with the advent of modern automatic differentiation software \citep{paszke2019pytorch, abadi2016tensorflow}, contributing to its popularity. Second, the capacity of neural networks can be adjusted by increasing the number of parameters (typically the width and depth of the hidden layers). The universal approximation theorem \citep{cybenko1989approximation} states that, given certain assumptions, even a single-layered neural network can, in principle, represent any continuous function arbitrarily precisely given enough capacity. Although this result does not guarantee that such functions can be easily found using first-order optimization processes, it speaks of the scalability of this model family and its ability to represent very complex functions \citep{kaplan2020scaling}. Together, these properties have allowed neural networks to take advantage of the increasing availability of data and computational power to achieve remarkable performance across a wide array of tasks and disciplines.

\subsection{Training Objectives}
\label{sec:training_objectives}

To learn from data, a machine learning model needs a training objective. The foundation of training objectives for deep learning models often begins with the concept of Maximum Likelihood Estimation (MLE) \citep{goodfellow2016deep}. Given a dataset $\mathcal{D}$ of $N$ samples $\mathcal{D}:=\{x^{(i)}\}_{i=1}^N$ from some true data distribution $p_d$, this statistical approach aims to find the parameters that maximize the likelihood of the data under the model distribution $p_\theta$. Under the assumption that the samples $x^{(i)}$ are independently and identically distributed (i.i.d.), the probability over the entire dataset decomposes into a product which can be expressed as the average log-likelihood (the $\log$ function is monotonous and $\frac{1}{N}$ does not depend on $\theta$):
\begin{equation}
\label{eq:maximum_likelihood_objective}
    \theta^*_{\text{MLE}} := \argmax_\theta P\big(\{x^{(i)}\}_{i=1}^N;\theta \big) = \argmax_\theta \prod_{i=1}^N p_\theta(x^{(i)}) = \argmax \frac{1}{N}\sum_{i=1}^N \log p_\theta(x)
\end{equation}
where $\theta$ represents the parameters of the model (e.g. the weights of a neural network). By the law of large numbers, this sum recovers the expectation of the true data distribution as the number of samples increases:
\begin{equation}
    \lim_{N \rightarrow \infty} \frac{1}{N}\sum_{i=1}^N \log p_\theta(x^{(i)}) = \E_{x\sim p_d}\big[ \log p_\theta(x) \big]
\end{equation}

In the context of supervised learning, for each datapoint $x^{(i)}$, a label $y^{(i)}$ is provided, and the goal of the model is to predict the correct output for each sample. In a classification task, maximizing the likelihood of the data is equivalent to minimising the well-known cross-entropy loss:
\begin{equation}
    \theta^*_{\text{MLE}} := \argmax_\theta \E_{(x,y)\sim p_d}\big[ \log p_\theta(y|x) \big] = \argmin_\theta - \E_{(x,y)\sim p_d}\big[ \log p_\theta(y|x) \big]
\end{equation}
For a regression task, it is common to assume a gaussian distribution of the label with fixed standard deviation, and where the model predicts the mean of the distribution. Maximizing the likelihood then uncovers the popular mean-squared error:
\begin{align}
    \theta^*_{\text{MLE}} &:= \argmax_\theta \, \E_{(x,y)\sim p_d}\big[ \log p_\theta(y|x) \big] \quad , \quad \text{with} \quad p_\theta(y|x):= \mathcal{N}\big(y; f_\theta(x), \sigma^2 \big) \\
    &= \argmax_\theta \, \E_{(x,y)\sim p_d}\left[ \log \Big(\frac{1}{\sigma \sqrt{2 \pi}} \Big) -\frac{1}{2\sigma^2} ||y - f_\theta(x)||^2 \right] \\
    &= \argmin_\theta \, \E_{(x,y)\sim p_d}\left[ \, \frac{1}{2} ||y - f_\theta(x)||^2_2 \, \right]
\end{align}
where $\mathcal{N}(y)$ denotes the gaussian probability density function and $||\cdot||_2^2$ is the L2-norm.

In unsupervised learning, we generally seek to learn a model of the data distribution which can be used for a variety of purposes including compression, anomaly detection, or generating new datapoints. Interestingly, the maximum likelihood objective can be seen as minimizing the KL-divergence between our model's distribution $p_\theta$ and the target distribution $p_d$:
\begin{equation}
    \argmin D_{KL}(p_d || p_\theta) := \argmin - \E_{x\sim p_d}\left[ \log \frac{p_\theta(x)}{p_d(x)} \right] = \argmax \E_{x\sim p_d}\big[ \log p_\theta(x) \big] 
\end{equation}
The expectation over $\log p_d(x)$ is independent from $\theta$ and can be ignored, leaving us with the likelihood term alone. Although KL-divergence is not an exact distance measure (it is not reversible), it is still informative of the progress that $p_\theta$ is making towards $p_d$ as its optimum $D_{KL}(p_d||p_\theta)=0$ is reached only when $p_\theta=p_d$. The maximum likelihood objective in Equation~\ref{eq:maximum_likelihood_objective} can thus be optimized directly to model $p_d$.

A possible approach is to parameterize $p_\theta$ as an energy-based model, which can capture arbitrarily complex distributions. However, this family of models does not scale well to high-dimensional data as it requires computing a normalization constant $Z$ that involves an intractable marginalization step. Several methods have thus been developed to avoid computing the partition function. For example, autoregressive models use the chain rule of probability to decompose the joint \citep{van2016conditional}, flow networks leverage the change of variable to model the likelihood as an invertible transformation of a simpler distribution \citep{rezende2015variational}, and variational autoencoders introduce latent variables to optimize a lower bound on the likelihood \citep{kingma2013auto}. Another approach, called implicit modeling, avoids modeling the likelihood function entirely and has led to great success in generating high-quality samples. It consists of learning from a likelihood ratio between the true data distribution and the model's distribution, leading to a min-max game where the loss function $D_\phi$ is learned from binary classification along with a generative model trained to generate plausible samples \citep{goodfellow2014generative, mohamed2016learning}:
\begin{equation}
    \min_\theta \max_\phi \, \, \E_{x\sim p_d} \big[ \log D_\phi(x) \big] + \E_{x\sim p_\theta} \big[ \log (1 - D_\phi(x)) \big]
\end{equation}

Through various applications, deep learning has proven to be an effective and versatile tool in the realm of machine learning, offering a wealth of architectures and training paradigms to suit both supervised and unsupervised settings. Its strength lies in its ability to learn hierarchical representations from data, enabling the extraction of complex patterns and relationships that may not be readily apparent. In the next section, we lay the foundations of reinforcement learning, a task specification paradigm that departs from using data itself as the main objective and instead captures the goal of an agent as a reward function to maximize.  We then discuss how deep learning can be leveraged to provide powerful function approximators for reinforcement learning in vast and complex environments.

\section{Reinforcement Learning}
\label{sec:reinforcement_learning}

Reinforcement Learning (RL) is a very distinct paradigm, naturally suited for sequential decision making problems. From a supervision perspective, it finds itself in between supervised and unsupervised learning. As opposed to unsupervised learning, it has a precisely defined objective through the use of a user-specified reward function. However, contrary to supervised learning, it does not require labeling every datapoint with the correct output, making this approach to task specification much less reliant on expert knowledge. In this section, we describe Markov Decision Processes (Section~\ref{sec:MDPs}), the Bellman equations (Section~\ref{sec:BellmanEquations}) and a variety of foundational approaches to learn optimal policies (Sections~\ref{sec:tabular_rl}~and~\ref{sec:model_free_rl}).

\subsection{Markov Decision Processes and the RL objective}
\label{sec:MDPs}

\begin{figure}[t]
    \centering
    \includegraphics[width=0.9\textwidth]{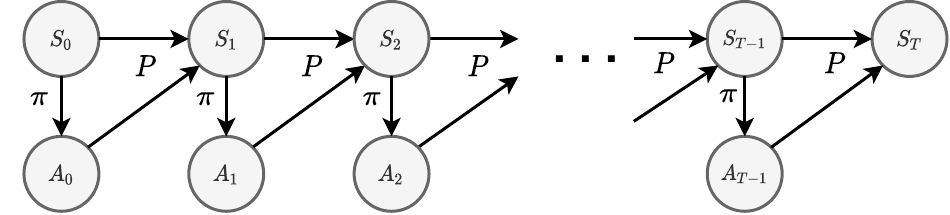}
    \caption[Markov Chain over state-action pairs]{Markov Chain over state-action pairs.}
    \label{fig:markov_chain}
\end{figure}

Markov Decision Processes (MDPs) are a mathematical framework for sequential decision making \citep{puterman1990markov}. An MDP is defined by the tuple $(\ptrans, \pzero, \mathcal{S}, \mathcal{A}, \RF, \gamma)$. The ensemble $\mathcal{S}$ denotes the \textit{state space} and $\mathcal{A}(s)$ denotes the \textit{action space}, both of which can be discrete or continuous. The process initiates by sampling an initial state $s_0$ from the \textit{initial state distribution} $\pzero:\mathcal{S} \rightarrow [0,1]$, $\sum_{s \in \mathcal{S}} \pzero(s) = 1$. At time-step $t$, the agent finds itself in state $s_t \in \mathcal{S}$ and must choose an action $a_t \in \mathcal{A}(s_t)$. The environment then returns the next state $s_{t+1} \in \mathcal{S}$ sampled from the \textit{transition distribution} $\ptrans(\cdot|s_t, a_t)$, with $\ptrans:\mathcal{S} \times \mathcal{S} \times \mathcal{A} \rightarrow [0,1]$ and $\sum_{s' \in \mathcal{S}} \ptrans(s'|s,a)=1$. This transition distribution, also called the ``environment dynamics'', is a central component, as it defines how the agent can influence its environment through the selected action $a$. This interaction between agent and environment induces a Markov chain over state and action pairs that unfolds over time, as depicted in Figure~\ref{fig:markov_chain}, where $S_0, A_0, ..., S_T$ are random variables representing the state and action at each time-step. Note that, for simplicity, we often omit the detailed reference to the random variable itself, such as $S_t$, and instead directly refer to its realization, represented by $s_t$.

Importantly, in an MDP, we assume the \textit{Markov property}, namely, that the environment dynamics $\ptrans$ only depends on the current state and action $(s_t, a_t)$. This implies that we assume the conditional independences $s_{t+1} \indep s_{<t}, a_{<t} \, | \, s_t$ which allow us to safely use a \textit{stationary} policy $\pi(a_t|s_t)$, i.e. a policy that depends only on $s_t$. Thus, a stationary policy represents a mapping from a state to a probability distribution over the possible actions in that state, $\pi: \mathcal{S} \times \mathcal{A} \rightarrow [0,1]$ and $\sum_{a \in \mathcal{A}(s)} \pi(a|s) = 1$.

A given realization of this chain is called a \textit{trajectory} or \textit{rollout}, denoted $\tau=(s_0, a_0, ..., s_T)$. The probability of a trajectory is given by the joint probability of the entire chain under the trajectory distribution $p_\pi$ induced by the policy $\pi$:
\begin{equation}
\label{eq:probability_of_trajectory}
    p_{\pi}(\tau) 
    = p_{\pi}(s_0, a_0, ..., s_{T-1}, a_{T-1}, s_T)
    = \pzero(s_0)\prod_{t=0}^{T-1}\ptrans(s_{t+1}|s_t, a_t)\pi(a_t|s_t)
\end{equation}
It is also convenient to define the marginals $p_{\pi}(s_t)$ and $p_{\pi}(s_t, a_t)$:
\begin{equation}
    p_{\pi, t}(s_t) := \sum_{s_{0:T}\setminus \{s_t\}}\sum_{a_{0:T}} p_{\pi}(\tau)\, , \quad
    p_{\pi, t}(s_t, a_t) := \sum_{s_{0:T}\setminus \{s_t\}}\sum_{a_{0:T}\setminus \{a_t\}} p_{\pi}(\tau)
\end{equation}
which can be interpreted as the probability of the union over all trajectories containing $S_t=s_t$ (or both $S_t=s_t$ and $A_t=a_t$). 

They can also be written as functions of one another using the product rule:
\begin{equation}
    p_{\pi, t}(s_t, a_t)
    = \mathbb{P}(S_t\!=\!s_t, A_t\!=\!a_t|\pi)
    = \mathbb{P}(A_t\!=\!a_t|S_t\!=\!s_t, \pi)\mathbb{P}(S_t\!=\!s_t|\pi)
    = \pi(a_t|s_t) p_{\pi, t}(s_t)
\end{equation}

\paragraph{Supervised learning in MDPs} Describing MDPs sets the table for reasoning about sequential decision making. Before diving into reinforcement learning, it is important to note that RL is not the only method for handling sequential decision making problems. For example, assuming access to a dataset of $N$ trajectories $\mathcal{D}=\{\tau_i\}_{i=1}^N$ solving the task, supervised learning can be applied to learn a policy that captures this behavior. In particular, behavioral cloning is the simplest such method. It consists of searching for the policy $\pi$ that maximizes the likelihood of $\mathcal{D}$ under the induced trajectory distribution $p_\pi$, reducing the problem of learning a policy to a classification problem (for discrete control) or a regression problem (for continuous control). From Equations~\ref{eq:maximum_likelihood_objective}~and~\ref{eq:probability_of_trajectory}, we have:
\begin{align}
    & \argmax_{\pi} \quad \frac{1}{N} \sum_{i=1}^N \log p_{\pi}(\tau_i) \\
    &= \argmax_{\pi} \quad \frac{1}{N} \sum_{i=1}^N \Big( \log \pzero(s_{0,i}) + \sum_{t=0}^{T-1}\log \ptrans(s_{t+1,i}|s_{t,i}, a_{t,i}) + \sum_{t=0}^{T-1} \log \pi(a_{t,i}|s_{t,i}) \Big) \\
    &= \argmax_{\pi} \quad \frac{1}{N}\sum_{i=1}^N \sum_{t=0}^{T-1}\log \pi(a_{t,i}|s_{t,i})
\end{align}
The terms associated with the environment dynamics are independent from $\pi$ and can be removed. We thus seek to maximize the probability of selecting the action $a_i$ associated with the state $s_i$ in the dataset. However, this method requires having access to a set of demonstrations $\mathcal{D}$, which limits its applicability to problems where the target behavior has already been observed. Furthermore, due to its lack of exploration, the method famously suffers from the problem of compounding error \citep{ross2011reduction, laskey2017dart}. It would be much more flexible to be able to define the task without having to demonstrate it, in a way that lets the agent interact with the environment and find the best strategy autonomously.

\paragraph{The RL objective} The reinforcement learning paradigm is based on a drastically different approach for specifying the desired behavior. Instead of using labels $a_i$ for each state $s_i$, we assume that the environment also emits a bounded scalar reward $r_t = \RF(s_t, a_t)$ at each time-step, using the reward function $\RF:\mathcal{S} \times \mathcal{A} \rightarrow [r_{\text{min}}, r_{\text{max}}]$. The reward can also be stochastic or depend on $S_{t+1}$, but for simplicity, we assume here the common case in which it is deterministic and depends only on $(s_t, a_t)$. The sum of rewards $r_1 + r_2 + \dots + r_T$ is called the \textit{return} and the goal in reinforcement learning is to learn a policy that maximizes this sum in expectation. This objective is called the \textit{expected total reward} (or expected return) and can be expressed both under the trajectory distribution or the state-action marginal:
\vspace{-3mm}
\begin{equation}
\label{eq:expected_total_reward}
    J_R(\pi) := \E_{\tau \sim p_\pi} \left[ \sum_{t=0}^T R(s_t, a_t) \right] = \sum_{t=0}^T \E_{(s_t, a_t)\sim p_{\pi, t}} \big[ R(s_t, a_t) \big]
\end{equation}
An \textit{optimal} policy is then defined as the policy $\pi^*$ that maximizes the expected total reward:
\begin{equation}
\label{eq:rl_objective}
    \pi^* = \argmax_{\pi} J_\RF(\pi)
\end{equation}
This objective is appropriate in the \textit{episodic} case, where a trajectory is guaranteed to terminate after $T$ time-steps. However, some problems involve an infinite-horizon for which $T\rightarrow \infty$. In such cases, to keep this infinite sum bounded, we usually introduce a discount factor $\gamma \in [0,1)$, and the RL objective becomes the expected total \textit{discounted} reward:
\begin{equation}
\label{eq:expected_total_discounted_reward}
    J_\RF(\pi) := \E_{\tau \sim p_{\pi}}\left[ \sum_{t=0}^\infty \gamma^t R(s_t, a_t) \right]
    = \sum_{t=0}^\infty \gamma^t \E_{(s_t, a_t)\sim p_{\pi, t}}\big[ R(s_t, a_t) \big]
\end{equation}
Since $\RF(\cdot)$ is bounded by $r_\text{max}$, we can see that $J_\RF(\pi)$ is bounded by a geometric series which evaluates to $\sum_{t=0}^\infty \gamma^t r_{\text{max}} = \frac{r_{\text{max}}}{1 - \gamma} \, \forall \, \gamma \in [0,1)$. Note that in practice, a discount factor is often used even in the episodic case. In addition to keeping the sum bounded, it can be interpreted as a way to trade-off instantaneous and future rewards, with $\gamma=0$ putting weight on the next reward only and $\gamma \rightarrow 1$ considering all rewards equally, or as the probability of transitioning to an \textit{absorbing state}, after which the reward is null forever after.
To gain some intuition about the impact of a given value on the effective horizon of the agent, one can use the rule of thumb that the number of time-steps considered to compute the return is of the order of $\frac{1}{1-\gamma}$, after which the remaining discounted rewards become very small \citep{tallec2019making}. For example $\gamma=0.99$ represents an effective horizon of about 100 time-steps.

\paragraph{Visitation distributions} Finally, another useful distribution is called the state (or state-action) visitation distribution (also called the normalized occupancy measure), defined as:
\begin{equation}
    d_{\pi}(s) := \frac{1}{Z(\gamma, T)}\sum_{t=0}^T\gamma^t p_{\pi, t}(S_t=s)\, , \quad
    d_{\pi}(s,a) := \frac{1}{Z(\gamma, T)}\sum_{t=0}^T\gamma^t p_{\pi, t}(S_t=s, A_t=a)
\end{equation}
where $Z(\gamma, T) = \sum_{t=0}^T \gamma^t$ is a normalizing constant e.g. $Z(1, T) = T$, $Z(\gamma, \infty)=\frac{1}{1-\gamma}$. Much like for the state and state-action marginals, these two distributions can also be written in terms of one another:
\begin{equation}
    d_{\pi}(s)\pi(a|s)
    = \frac{\pi(a|s)}{Z(\gamma, T)}\sum_{t=0}^T\gamma^t p_{\pi, t}(S_t=s)
    = \frac{1}{Z(\gamma, T)}\sum_{t=0}^T\gamma^t\pi(a|s) p_{\pi, t}(S_t=s)
    = d_{\pi}(s,a) 
\end{equation}
The expected total discounted reward can also be written as an expectation over the state-action visitation distribution. Starting from Equation~\ref{eq:expected_total_discounted_reward}, we can use the fact that the state and action space are the same along the trajectory to fully remove the dependency over $t$, yielding:
\begin{align}
\label{eq:expected_total_discounted_reward_from_occupancy_measure}
    J_\RF(\pi)
    &= \sum_{t=0}^T \gamma^t \sum_{s,a}p_{\pi, t}(s,a)\RF(s,a)
    &= \sum_{s,a}\RF(s,a)\underbrace{\sum_{t=0}^T\gamma^t p_{\pi, t}(s,a)}_{Z(\gamma, T) d_{\pi}(s,a)} 
    &= Z(\gamma, T)\mathbb{E}_{(s,a)\sim d_{\pi}}[\RF(s,a)]
\end{align}

\subsection{The Bellman Equations}
\label{sec:BellmanEquations}

Many reinforcement learning algorithms estimate value functions to evaluate an agent's policy and to improve it. Two types of value function are often used: the state value function $v^{\pi}(s)$ and the state-action value function $q^{\pi}(s,a)$. $v^{\pi}$ evaluates how desirable it is for an agent to find itself in state $s$ whereas $q^{\pi}$ evaluates how desirable it is to take action $a$ when finding itself in state $s$ \citep{sutton2018rlTextbook}. More formally, for any state $s \in \mathcal{S}$, $v^{\pi}(s)$ is defined as the expected total discounted reward assuming that we start in state $S_t=s$ and then follow $\pi$ for the rest of the interaction. Similarly, for any state-action pair $s, a \in \mathcal{S} \times \mathcal{A}$, $q^{\pi}(s,a)$ is defined as the expected total discounted reward assuming that we start in state $S_t=s$, take action $A_t=a$, and then continue on by following $\pi$:
\begin{equation}
    \label{eq:value_functions}
    v^{\pi}(s) := \E_{p_{\pi}}\left[ \sum_{k=0}^\infty \gamma^k r_{t+k} \Big| s_t=s \right]\, , \quad
    q^{\pi}(s,a) := \E_{p_{\pi}}\left[ \sum_{k=0}^\infty \gamma^k r_{t+k} \Big| s_t=s, a_t=a \right]
\end{equation}
The objective we seek to maximize, the expected return, can be expressed in terms of $v^\pi$ and the initial state distribution $P_0$:
\vspace{-2mm}
\begin{equation}
    J_R(\pi) = \E_{s\sim P_0}\big[ v^\pi (s) \big]
\end{equation}
A very important property of value functions is that they can be written recursively. By expanding the sum over the rewards $\sum_{k=0}^\infty \gamma^k r_{t+k} = r_t + \gamma \sum_{k=0}^\infty \gamma^k r_{t+k+1}$ and pushing the expectation to the right, we get:
\begin{equation}
\label{eq:bellman_equations}
    v^{\pi}(s) 
    = \E_{p_{\pi}}\big[ \RF(s,a_t) + \gamma v^{\pi}(s_{t+1}) \big] \, , \quad
    q^{\pi}(s,a)
    = \RF(s,a) + \gamma \E_{p_{\pi}}\big[ q^{\pi}(s_{t+1}, a_{t+1}) \big]
\end{equation}
The Equations~\ref{eq:bellman_equations} for $v^{\pi}$ and $q^{\pi}$ are the famous \textit{Bellman equations}. A useful identity is to write them in terms of each other. Starting from their definition, one can again push the expectations to uncover their mixed forms ($v^\pi$ from $q^\pi$ and vice-versa):
\begin{equation}
    v^{\pi}(s) = \E_{p_{\pi}}\big[q^{\pi}(s_t,a_t) | s_t = s\big] \, , \quad
    q^{\pi}(s,a) = \RF(s,a) + \gamma \E_{p_{\pi}}\big[v^{\pi}(s_{t+1})\big]
\end{equation}
The advantage function is another useful quantity defined as the expected gain from taking action $a$ in state $s$ instead of following the policy:
\begin{align}
    a^{\pi}(s,a) 
    &= q^{\pi}(s,a) - v^{\pi}(s) \\
    &= \RF(s,a) + \gamma \E_{p_{\pi}}\big[v^{\pi}(s_{t+1})\big] - v^{\pi}(s) \label{eq:advantage_fct_vf}
\end{align}

We can use value functions to define an ordering on policies \citep{sutton2018rlTextbook}. We say that a policy $\pi'$ is better than a policy $\pi$ if $v^{\pi'}(s) \geq v^{\pi}(s) \, \forall s \in \mathcal{S}$ and strictly superior for at least one state $s$. An \textit{optimal policy} $\pi^*$ is a policy that is better than or equal to all policies. The state and state-action value functions of an optimal policy are called \textit{optimal value functions} $v^{*}$ and $q^{*}$
\begin{equation}
    \label{eq:optimal_value_function}
    v^{*}(s) := \max_{\pi} v^{\pi}(s) \, , \quad
    q^{*}(s,a) := \max_{\pi} q^{\pi}(s,a)
\end{equation}
Note that the optimal value functions can be written independently of any policy, only as a function of the optimal value function at the next state
\begin{align}
    \label{eq:bellman_optimality_equations}
    v^{*}(s) 
    &= \max_{a \in \mathcal{A}(s)} q^{*}(s,a) \\
    &= \max_{a \in \mathcal{A}(s)} \E_{p_{\pi}}\big[ \RF(s,a) + \gamma v^{*}(s_{t+1}) \big] \\ \notag \\
    q^{*}(s,a)
    &= \RF(s,a) + \gamma \E_{p_{\pi}} \big[v^{*}(s_{t+1}) \big] \\
    &= \RF(s,a) + \gamma \E_{p_{\pi}}\big[ \max_{a' \in \mathcal{A}(s_{t+1})} q^{*}(S_{t+1}, a') \big]
\end{align}
This relationship shows that an optimal policy can be very simply expressed as acting greedily on the optimal value function. Since $q^{*}(s,a)$ by definition already accounts for the long-term effect of taking action $a$ in state $s$, a one-step look-ahead on $q^{*}$ yields optimal behavior. While defining an optimal policy is important, in most problems, computing the exact optimal policy is prohibitively expensive. Luckily, the Policy Improvement Theorem tells us how the definition of optimal policy can be used to at least improve the policy that we currently have \citep{sutton2018rlTextbook}. Given an accurate estimate of the value function $q^{\pi}$, one can obtain an improved policy $\pi'$ by increasing in all states the probability of selecting the action $a^{*}$ that yields the highest value, that is, $a^{*}(s) = \argmax_{a\in \mathcal{A}(s)} q^{\pi}(s,a) \, \forall \, s \in \mathcal{S}$. If there is equality (multiple optimal actions), any partitioning of the probability mass among them will yield the same performance, as long as no probability mass is added to suboptimal actions.

The problem of estimating value functions is referred to as \textit{policy evaluation} while the step involving modifying the policy is referred to as \textit{policy improvement}. These two procedures are, in one form or another, at the core of most reinforcement learning algorithms.

\subsection{Tabular RL with known environment dynamics}
\label{sec:tabular_rl}

In small environments for which state values $v^\pi(s)$ can be enumerated in a table, several methods have been developed to build accurate value estimate which cover the entire state space by leveraging known environment dynamics.

\subsubsection*{Solving the System of Equations}

Given perfect knowledge of the dynamics of the environment $\ptrans$ and for sufficiently small MDPs, one can directly solve the system of $|\mathcal{S}|$ equations and $|\mathcal{S}|$ unknowns (one for each $s \in \mathcal{S}$) given by the Bellman equation for $v^{\pi}$ (Equation~\ref{eq:bellman_equations}). This can be better seen in the vector-matrix form. By defining some arbitrary ordering $1, 2, ... \, , |\mathcal{S}|$ over the states $s \in \mathcal{S}$, we can define the vector of unknowns $v^{\pi}$ where $v^{\pi}_i = v^{\pi}(s_i)$. We can also express the reward function in vector form and the environment-policy dynamics in matrix form by marginalizing out the effect of actions in the reward function and transition distribution, respectively:
\begin{equation}
    r^{\pi}(s) 
    := \E_{a \sim \pi(\cdot|s)}\big[ \RF(s, a) \big]
    \, , \quad 
    \ptrans^{\pi}(s'|s) := \E_{a \sim \pi(\cdot|s)}\big[ \ptrans(s'|s,a) \big]
\end{equation}
By expanding the sum from the definition of the value function we can then obtain our system of Bellman equations:
\begin{align}
    v^{\pi} 
    := \sum_{t=0}^{\infty}\left( \gamma \ptrans^{\pi} \right)^t r^{\pi}
    \, \, = \, \,  r^{\pi} + \sum_{t=1}^{\infty}\left( \gamma \ptrans^{\pi} \right)^{t} r^{\pi}
    \, \, = \, \,  r^{\pi} + \gamma \ptrans^{\pi}\sum_{t=0}^{\infty}\left( \gamma \ptrans^{\pi} \right)^{t} r^{\pi}
    \, \, = \, \,  r^{\pi} + \gamma \ptrans^{\pi} v^{\pi}
\end{align}
From there, we can solve for $v^{\pi}$ to get the unique solution of the policy evaluation problem:
\begin{equation}
    v^{\pi} = (I - \gamma \ptrans^{\pi})^{-1}r^{\pi}
\end{equation}

\subsubsection*{Linear Programming}

The problem of finding the optimal value function $v^{*}$ can also be formulated as a Linear Program (LP). Let a $\gamma$-superharmonic vector be any vector $v$ satisfying
\begin{equation}
    v \geq r^{\pi} + \gamma \ptrans^{\pi}v
\end{equation}
One can show that the optimal value function $v^{*}$ is the smallest such $\gamma$-superharmonic vector \citep{kallenberg2011mdps}. This result is at the root of the LP formulation. We want to minimize our value function candidate $v$ as much as possible but such that it remains $\gamma$-superharmonic. Using this constraint, $v^{*}$ can be defined as the solution to the Linear Program:
\begin{equation}
\begin{array}{ll@{}ll}
\text{minimize}  & p_0^{\top} v  &\\
\text{subject to}& v(s) - \gamma \sum_{s' \in \mathcal{S}} \ptrans(s'|s,a)v(s
')&\quad \geq \quad r(s,a),  &\forall \, s,a \in \mathcal{S}\times \mathcal{A}
\end{array}
\end{equation}
where $p_0$ is the vectorized initial state distribution. Any LP solver can thus be used in order to recover the optimal value function by solving this constrained optimization problem. The LP formulation also highlights an interesting relationship between the value function $v^\pi(s)$ and the (unormalized) state-action occupancy measure, since $d_{\pi}(s,a)$ is in fact the solution to the dual problem of this Linear Program. 

\subsubsection*{Dynamic Programming}

A very popular family of methods that scales better than the Linear Programming approach (but which is, however, still computationally expensive) are Dynamic Programming algorithms \citep{bellman1966dynamic, rust2008dynamic}. The main mechanism behind such methods is to use the Bellman equation as an iterative update rule and evaluate it exactly using the known transition distribution $P$:
\begin{equation}
    \label{eq:policy_iteration_v_update}
    v_{k+1}(s) := \sum_{a \in \mathcal{A}(s)} \pi(a|s) \left( \RF(s,a) + \gamma \sum_{s'\in \mathcal{S}} P(s'|s,a)v_{k}(s')\right)
\end{equation}
One can show that starting from any arbitrary estimate $v_0$, the sequence $\{v_k\}_{k=1, 2, \, \dots}$ converges to the true value function of the policy $v^{\pi}$ as $k\rightarrow \infty$. The \textit{Policy Iteration} algorithm uses such updates. Specifically, it alternates between policy evaluation and policy improvement steps until the optimal value function $v^*$ is recovered. During the policy evaluation step, we keep updating the value of each state $s \in \mathcal{S}$ until they converge to the true $v^{\pi}$. Then, the policy improvement step (greedification) consists of greedily improving the deterministic policy by selecting the best-performing action according to the expectation over the next states, i.e. $\pi(s) = \argmax_{a\in\mathcal{A}(s)} \RF(s,a) + \mathbb{E}_{s'\in\mathcal{S}}[v^{\pi}(s')]$. 

The \textit{Value Iteration} algorithm instead fuses these two steps together by using the Bellman optimality equation as an update rule:
\begin{equation}
    \label{eq:value_iteration_v_update}  
    v_{k+1}(s) := \max_{a \in \mathcal{A}} \, \, \left( \RF(s,a) + \gamma \sum_{s'\in \mathcal{S}} P(s'|s,a)v_{k}(s') \right)
\end{equation}
which is equivalent to performing Policy Iteration but interrupting the policy evaluation step after a single sweep over $\mathcal{S}$ instead of waiting until convergence to the true $v^\pi$ at each step. While a more accurate estimate of the value function means better knowledge of how to improve the policy, when considering deterministic policies, only the ordinality of state-action values really matters, and we can speed up the algorithm by improving the policy based on the correct action ordering without necessarily having the exact values of each action in hand. Hence, the policy iteration and value iteration algorithms represent two extreme answers to an interesting question: since a more accurate estimate of the policy value does not always lead to a revised policy improvement step, how long should we keep updating the value estimates before modifying the policy? The former optimizes the value estimate as much as possible, while the latter performs a single update before moving on. For most problems, the best trade-off is somewhere in the middle. \textit{Generalized Policy Iteration} refers the whole spectrum of algorithms that put different levels of emphasis on the value function accuracy before switching to policy improvement.

\subsection{Model-free RL from experience}
\label{sec:model_free_rl}

The methods presented in the previous section are very restrictive in practice, as they require complete knowledge of the environment's transition distribution. They also assume that we can query the reward function directly with any given state and action. While these quantities are known in some simple cases, such as board games, many setups allow only for sampling $\RF(s,a)$ and $P(s'|s,a)$ by interacting with the environment. One approach is to attempt to model these distributions using samples, an often challenging task that leads to a whole family named \textit{model-based RL algorithms}. In this thesis, we focus on \textit{model-free RL algorithms}, the alternative approach that seeks to directly solve the control problem from interaction samples (experience), without learning its dynamics.

\subsubsection*{Monte Carlo methods} 

A foundational family of algorithms is called Monte Carlo (MC) methods. They seek to estimate $q^\pi(s,a)$ or $v^\pi(s)$ directly from Equation~\ref{eq:value_functions} by collecting i.i.d. \textit{ sample episodes} and averaging complete returns \citep{sutton2018rlTextbook}. By the law of large numbers, these unbiased estimates converge to the true values as the number of samples tends to infinity. Although restricted to the episodic case ($T \in \mathbb{N}$), MC methods can be particularly useful when trying to estimate the value of a subset of states only, as they allow each state value estimate to be completely independent (unlike bootstrap methods).

An important consideration that arises for control without a model is the need to maintain sufficient exploration throughout learning. Indeed, to perform policy improvement without a state-transition model, one needs to evaluate $q^\pi$ rather than $v^\pi$. However, when evaluating $q^\pi(s,a)$ from interaction samples collected by a deterministic policy $\pi$, only one action will be chosen for each state $s \in \mathcal{S}$, making it impossible to know whether a different action would have led to higher returns. Two options are available to allow for exploration. The first option is to learn \textit{on-policy} the value function of a \textit{stochastic} policy, often carried out by implementing an $\epsilon$-greedy policy which attributes every action a probability of $\frac{\epsilon}{|\mathcal{A}(s)|}$ to be selected, and the rest of the probability mass to the greedy action. The same principle of generalized policy iteration can be proven to lead to the optimal $\epsilon$-greedy policy \citep{sutton2018rlTextbook} (which is, we hope, very close in performance to the optimal deterministic policy). The second option is to learn \textit{off-policy} the value function of a deterministic \textit{target} policy $\pi$ using samples collected by a different \textit{behavior} policy $\beta$. The use of importance sampling estimators allows us to correct for the fact that the trajectories are now sampled from $p_\beta$ rather than from $p_\pi$. Indeed, to estimate the value $v^\pi(s)$ under the target policy $\pi$, we have:
\begin{equation}
    v^{\pi}(s) := \E_{\tau\sim p_{\pi}}\left[ \sum_{k=0}^{T-t} \gamma^k r_{t+k} \Big| s_t=s \right]
    = \E_{\tau \sim p_{\beta}}\left[ \frac{p_{\pi}(\tau)}{p_{\beta}(\tau)} \sum_{k=0}^{T-t} \gamma^k r_{t+k} \Big| s_t=s \right]
\end{equation}

To compute the importance sampling weights $\frac{p_{\pi}(\tau)}{p_{\beta}(\tau)}$, we assume that the behavior policy $\beta(a|s)$ can be evaluated and is positive in all state-action pairs where the target policy is positive (assumption of coverage). Off-policy methods tend to converge more slowly as the importance weights increase the variance of the value estimates, but are more general as they include the on-policy methods as a special case where $\beta = \pi$.

\subsubsection*{Temporal Difference methods} 

A second family of experience-based algorithms is called Temporal Difference (TD) methods. Instead of learning directly from the expected return definition of Equation~\ref{eq:value_functions} like in the case of Monte Carlo methods, they use the recursive property of the expected return by using an approximation of the Bellman equations of Equation~\ref{eq:bellman_equations} as their update rule. In particular, TD(0) takes a one-sample estimate of the expectation in the Bellman equation for $v^\pi$, and SARSA does the same using the Bellman equation for $q^\pi$. Both MC and TD methods admit online implementations with the following update rules respectively:
\begin{align}
    v_{k+1}^\pi(s) &= v_{k}^\pi(s) + \alpha \big(\hat{q}^{\pi}(s,a) - v_{k}^\pi(s)\big) \tag{MC online update} \\
    v_{k+1}^\pi(s) &= v_{k}^\pi(s) + \alpha \big(\RF(s,a) + \gamma v_{k}^\pi(s') - v_{k}^\pi(s)\big) \tag{TD(0) online update}
\end{align}
where $\alpha$ is a step-size hyperparameter, $a$ and $s'$ are the sampled action and next-state and $\hat{q}^{\pi}(s,a)$ is the discounted return \textit{collected} from $(s,a)$ until the end of the trajectory. It is clear from these equations that while MC methods use the true sampled return as a target to update their value estimate of a given state $s$, TD methods instead use a single sampled reward and then use the current estimate of the return at the next state $v_{k}^\pi(s')$. The term in parentheses $\delta := \RF(s,a) + \gamma v_{k}^\pi(s') - v_{k}^\pi(s)$ is often referred to as the TD error and quantifies how much the current estimate should change to respect the Bellman equation. Learning an estimate of the value at the current state from an estimate of the value at the next state is referred to as \textit{bootstrapping}, and while it might appear ambitious, TD(0) and SARSA have been shown to converge to the true value function assuming that all states (and actions) are visited infinitely many times. 

The most popular TD method for control is the Q-learning algorithm \citep{watkins1992q}. Like SARSA, it estimates the state-action values, however it instead implements the Bellman \textit{optimality} equation as its learning rule:
\begin{equation}
    q_{k+1}^{\pi}(s,a) = q_{k}^{\pi}(s,a) + \alpha \big(\RF(s,a) + \gamma \max_{a' \in \mathcal{A}(s')}q_{k}^\pi(s',a') - q_{k}^{\pi}(s,a) \big)
\end{equation}
Therefore, instead of estimating the values of any given policy $\pi$, the Q-learning algorithm learns the state-action values of the greedy-policy defined over the current value estimates, and has been shown to converge to the optimal value function $q^*$ assuming that all pairs continue to be updated. Because the Bellman optimality equation does not depend on any particular policy, the Q-learning algorithm can be used for off-policy control \textit{as is}, without the need to use importance sampling weights, making it one of the most widely used reinforcement learning algorithms to this day.

Two important points explain why experience-based algorithms scale so much better than dynamic programming algorithms. The first one is statistical. Dynamic programming methods compute full expectations in their update rules (expected updates) which involves considering all successor states, whereas experience-based methods need only a single sample obtained through interaction (sampled updates). The second is computational. Experience-based methods inherently focus on improving their value estimates of regions of the state space that are more visited, therefore avoiding lost computation on improving estimation of states that the agent will never encounter.

Methods based on finding a solution to the Bellman equations, called \textit{value-based} methods, only need to implicitly represent the policy through the learned value function. While these techniques are useful for \textit{discrete} control -- problems for which the size of action space $|\mathcal{A}(s)|$ is finite for all states $s \in \mathcal{S}$ -- they are not well suited for \textit{continuous} control tasks, in which the action space is a continuous domain and consequently the number of possible actions is infinite. One possibility for such cases is to discretize the action space by partitioning it into quantiles. However, this poses a precision vs. complexity trade-off; larger bins prohibit fine control whereas a large number of bins becomes intractable as the number of possible actions grows exponentially with the number of action dimensions.

\subsubsection*{Policy Gradient methods} 

A better suited alternative to continuous control than value-based methods is to learn the policy directly by parameterizing it separately from the value function, a family of algorithms called \textit{policy-based} methods. As long as this policy $\pi_{\piparams}$ is differentiable w.r.t. its parameters $\piparams$, it can be trained to directly maximize the expected total discounted reward using a likelihood ratio estimator of its gradient. This result is known as the Policy Gradient Theorem \citep{sutton2000policy} and algorithms that take this approach are called \textit{policy gradient methods}. Starting with $J_\RF(\pi_{\piparams})$ from Equation~\ref{eq:expected_total_discounted_reward}, we have:
\begin{align}
&\hspace{-2mm}\grad_{\piparams} J_\RF(\pi_{\piparams}) 
    := \grad_{\piparams} \E_{\tau \sim p_{\pi_{\piparams}}}\left[\sum_{t=0}^T \gamma^t r_t \right] \\
    &\hspace{-2mm}= \int_{\tau} \grad_{\piparams} p_{\pi_{\piparams}}(\tau) \Big(\sum_{t=0}^T \gamma^t r_t \Big) d\tau \\
    &\hspace{-2mm}= \int_{\tau} p_{\pi_{\piparams}}(\tau) \grad_{\piparams} \log p_{\pi_{\piparams}}(\tau) \Big(\sum_{t=0}^T \gamma^t r_t \Big) d\tau \\
    &\hspace{-2mm}= \E_{\tau \sim p_{\pi_{\piparams}}} \left[ \Big( \grad_{\piparams} \log \pzero(s_0) + \grad_{\piparams} \sum_{t=0}^{T} \log \ptrans(s_{t+1}|s_t,a_t) + \grad_{\piparams} \sum_{t=0}^{T} \log \pi_{\piparams}(a_t|s_t) \Big) \Big( \sum_{t=0}^T \gamma^t r_t \Big) \right]  \\
    \label{eq:policy_gradient}
    &\hspace{-2mm}= \E_{\tau \sim p_{\pi_{\piparams}}} \left[ \Big( \sum_{t=0}^{T} \grad_{\piparams} \log \pi_{\piparams}(a_t|s_t) \Big) \Big( \sum_{t=0}^T \gamma^t r_t \Big) \right]
\end{align}
This last expectation can be approximated using a Monte Carlo estimator to yield our policy gradient estimate using collected trajectories. We can then use it to update the parameters of our policy through gradient ascent, i.e. $\piparams \leftarrow \piparams + \alpha \grad_{\piparams} J_\RF(\pi_{\piparams})$ where $\alpha$ is the learning rate. The algorithm that uses this particular gradient estimate is known as REINFORCE \citep{williams1992simple}. It uses the actual expected total discounted reward $\sum_{t=0}^T \gamma^t r_t$ to weigh the gradient $\grad_{\piparams} \log \pi_{\piparams}(a_t|s_t)$ at each time-step, making it a Monte Carlo policy gradient method. One advantage of using the actual discounted return is that this policy gradient estimate does not make use of the Markov assumption; we could use it even in the case of partial observability where our policy is conditioned on some observation $o_t$ rather than the true state of the environment $s_t$. However, it also means that this Monte Carlo estimate is only well defined in the episodic case (an episode must terminate before learning can start) and that the estimate generally has a high variance.

Two techniques are commonly used to reduce the variance of this gradient estimator. First, to weigh the gradient term $\grad_{\piparams} \log \pi_{\piparams}(a_t|s_t)$, we can omit all rewards that occurred before time-step $t$ and instead use the discounted return $\hat{q}^{\pi}=\sum_{t'=t}^T \gamma^{t'-t} r_{t'}$ collected from $(s_t, a_t)$. This comes from the fact that, by temporal causality, we have $R_t \indep A_{<t} | (S_t, A_t)$, yielding a simpler version of Equation~\ref{eq:policy_gradient}:
\begin{equation}
    \grad_{\piparams} J_\RF(\pi_{\piparams}) = \E_{\tau \sim p_{\pi_{\piparams}}} \left[ \sum_{t=0}^{T} \grad_{\piparams} \log \pi_{\piparams}(a_t|s_t) \hat{q}^{\pi}(s_t,a_t) \right]
\end{equation}

Second, we can also subtract a state-dependent \textit{baseline} $b(s)$ from this weighting term, yielding:
\begin{equation}
    \grad_{\piparams} J_\RF(\pi_{\piparams}) = \E_{\tau \sim p_{\pi_{\piparams}}} \left[ \sum_{t=0}^{T} \grad_{\piparams} \log \pi_{\piparams}(a_t|s_t) \big(\hat{q}^{\pi}(s_t,a_t) - b(s_t) \big) \right]
\end{equation}
A typical choice is to use an estimate of the state value function $v^{\pi_{\piparams}}$ as baseline. By developing the term with $b(s_t)$, one can show that the use of a baseline does not bias the gradient in expectation:
\begin{equation}
    \sum_{t=0}^{T} \E_{(s_t, a_t) \sim p_{\pi_{\piparams}}} \Big[ \grad_{\piparams} \log \pi_{\piparams}(a_t|s_t)b(s_t)\Big]
    = \sum_{t=0}^{T} \E_{s_t \sim p_{\pi_{\piparams}}} \left[ b(s_t) \grad_{\piparams} \sum_{a \in \mathcal{A}(s_t)} \pi_{\piparams}(a|s_t) \right]
    = 0
\end{equation}

Finally, most policy gradient algorithms do not use the actual return $\hat{q}^{\pi_\theta}$ to weigh the gradient, but instead use a parameterized estimate $q_\phi$ of the true state-action value function $q^{\pi_{\piparams}}$. While the collected return $\hat{q}^{\pi_\theta}$ is already a one-sample estimate of the expected state-action value, it is computed online for each episode. At the price of introducing a bias in the gradient estimate \citep{sutton2000policy}, using a parameterized estimate $q_\phi$ allows to evaluate the return at any time-step without requiring us to wait for the episode to terminate before updating. When combined with $v_{\vfparams}(s_t) = \E_{a_t \sim \pi_{\piparams}}[q_{\vfparams}(s_t, a_t)]$ as baseline, we get the policy gradient computed by the Advantage Actor-Critic (A2C) algorithm, which weighs the gradient with the advantage function $a_{\vfparams}(s_t, a_t)$
\begin{equation}
    \label{eq:a2c_policy_gradient}
    \grad_{\piparams} J_\RF(\pi_{\piparams}) = \E_{\tau \sim p_{\pi_{\piparams}}} \left[ \sum_{t=0}^{T} \grad_{\piparams} \log \pi_{\piparams}(a_t|s_t) a_{\vfparams}(s_t, a_t) \right]
\end{equation}

Finally, note that the policy gradient is on-policy because the expectation is taken w.r.t. the trajectory distribution $p_{\pi_{\piparams}}$ induced by the current policy $\pi_{\piparams}$. Computing the gradient using trajectories collected by a different (or past) policy would yield a biased gradient estimate. In such cases, an unbiased but higher variance estimate can be derived using importance sampling  \citep{degris2012off}.

\section{Deep Reinforcement Learning}
\label{sec:DeepRL}

In the last section, we reviewed foundational RL algorithms which are designed to operate on MDPs with discrete state space $\mathcal{S}$ that contain a small enough number of states to be stored in a table.  Such tasks are useful for testing algorithms and developing the theory. However, most real-life problems involve a number of possible states that is so large that they would not fit in any computer's memory. For such cases, function approximators must be used to represent policies and value functions over these gigantic spaces using a compact number of parameters. Today, neural networks are the most commonly used function approximators in RL due to their representation power and their ability to be trained efficiently using gradient-based methods (see Section~\ref{sec:deep_learning}). Such a combination of deep learning and RL is often referred to as deep Reinforcement Learning (deep RL) algorithms.

\subsection[Deep Q-Learning]{Deep Q-Learning} Neural networks had already been used in the 1990s in successful applications of reinforcement learning with function approximation \citep{tesauro1994td, lin1993reinforcement}. 
In the last decade however, the DQN\footnote{Although the algorithm is called Deep Q-Learning and \textit{DQN} only stands for Deep Q-Networks, members of the community often use the DQN acronym to refer to the algorithm as a whole.} algorithm \citep{mnih2013playing, mnih2015human} stood out by tackling the challenging task of learning to play Atari arcade games directly from raw pixel images \citep{bellemare2013arcade}, thus successfully applying reinforcement learning to a much higher-dimensional input space.

At its core, the DQN algorithm essentially consists of training a deep neural network parameterized by $\vfparams$ that takes a state $s \in \mathcal{S}$ as input and maps it to the q-value estimate $Q_{\vfparams}(s,a)$ of each action $a \in \mathcal{A}(s)$. This model is trained using the Q-Learning algorithm presented in Section~\ref{sec:model_free_rl} by backpropagating through each layer to correct for the TD error. However, the combined use of bootstrapping, off-policy learning, and function approximators (sometimes called the \textit{deadly triad} \citep{sutton2015introduction}) is known to destabilize RL algorithms. The authors of DQN alleviate this issue by making two main changes to the Q-learning algorithm. 

The first and most important is the use of \textit{experience replay}. For each interaction with the environment, a transition $(s,a,r,s')$ is collected and stored in a replay buffer $\mathcal{D}$. After a fixed number of interactions, a minibatch of transitions is uniformly sampled from the buffer and used to compute the stochastic gradient update. Experience replay allows to improve data efficiency as the collected data can be re-used several times and because sampling across the entire buffer effectively decorrelates the samples used for computing the updates as opposed to using consecutive transitions. It also allows the model to maintain its accuracy in estimating the value of long-past states and actions and to avoid oscillations due to drastic changes in the collected data distribution after a parameter update \citep{mnih2013playing}. The second modification is the use of \textit{target networks} in the Q-learning update \citep{mnih2015human}. Denoted $Q_{\bar{\vfparams}}$, the target network is a copy of the main deep Q-network $Q_{\vfparams}$ that is used to compute the value of the next state and actions $(s', a')$ of the target $y = \RF(s,a) + \gamma \max_{a'\in \mathcal{A}(s')} Q_{\bar{\vfparams}}(s',a')$. The interest lies in the fact that $Q_{\bar{\vfparams}}$ always lags behind $Q_{\vfparams}$, making the target relatively constant for a few updates of the parameters $\vfparams$. This is achieved either by \textit{hard updates} $\bar{\vfparams}\leftarrow \vfparams$ after a fixed number of parameter steps, or more frequently using \textit{soft updates} of the form $\bar{\vfparams}\leftarrow \epsilon \vfparams + (1-\epsilon) \bar{\vfparams}$ with $\epsilon \in [0,1]$.

Finally, several follow-up works have provided improvements to the original DQN algorithm. \citet{hessel2018rainbow} provides a detailed evaluation of some of these. Among them, Double-DQN (DDQN) \citep{van2016ddqn} proposes to reduce the overestimation bias in Q-learning by decoupling the action selection and the state-action evaluation when computing the TD target, yielding $y = \RF(s,a) + \gamma Q_{\bar{\vfparams}}\Big(s', \argmax_{a' \in \mathcal{A}(s')} Q_{\vfparams}(s',a')\Big)$ where $Q_{\bar{\vfparams}}$ is the target network presented above. Prioritized Experience Replay \citep{schaul2015prioritized} aims to sample with higher probability transitions from which the agent can learn the most by assigning a sampling probability $p_i$ to each transition $(s,a,r,s')_i$ proportionally to its TD error ($p_i \propto |\delta_i| + \epsilon$) and uses weighted importance sampling \citep{mahmood2014weighted} to correct for the introduced bias. Dueling networks \citep{wang2016dueling} propose a different network architecture that represents the state value estimate and the advantage estimate separately before recombining them into Q-values to allow all action values to quickly benefit from an updated state estimate. Distributional approaches to deep RL \citep{bellemare2017distributional, dabney2018distributional} aim to learn the approximate value distributions of states and actions rather than their expected value alone, allowing more stable learning and an explicit specification of risk aversion within agents.

\subsection{Deep Deterministic Policy Gradients}

Value-based methods are well suited for discrete control because a policy is implicitly defined by taking the $\argmax$ on the Q-values. In continuous action spaces, this operation comes down to an optimization over $\mathcal{A}(s)$ at every time-step, which is generally computationally prohibitive. We must use policy gradient methods instead. 

One option is to represent the policy using an analytic continuous distribution (e.g., Gaussian) and to learn a mapping from the input state to the parameters of that distribution (e.g., mean and variance) using the (stochastic) policy gradient presented in Section~\ref{sec:model_free_rl}. Another approach is to use the deterministic policy gradient (DPG) formulation that is specifically derived for continuous control \citep{silver2014deterministic}.

Intuitively, the DPG moves the policy parameters in the direction that maximizes the action value function $Q_{\vfparams}$ when averaged over all states. The key element on which the deterministic policy gradient is derived lies in the fact that because the action space is continuous, and assuming that the Q-function is parameterized using a differentiable function approximator $\vfparams$, we can backpropagate the signal from  $Q_\vfparams$ through the selected action $a$ to compute the derivative of $J_R$ w.r.t. the parameters $\piparams$ of the continuous policy $\mu_\piparams$:
\begin{align}
    \grad_{\piparams} J_\RF(\mu_{\piparams}) 
    &\approx \E_{p_{\beta}}[\grad_{\piparams}Q_{\vfparams}(s,a)|_{a=\mu_{\piparams}(s)}] \\
    &= \E_{p_{\beta}}[\grad_{\piparams}\mu_{\piparams}(s)\grad_{a}Q_{\vfparams}(s,a)|_{a=\mu_{\piparams}(s)}]
\end{align}
The critic $Q_{\vfparams}$ is learned using the Q-Learning algorithm described above. But rather than computing the explicit $\argmax$ over actions for the action selection of the TD target, the next action $a'$ is computed using the target policy $\mu_{\bar{\piparams}}$, which is trained to maximize the critic's action value. In this sense, the DPG algorithm can be seen as an approximate Q-learning algorithm that uses a parameterized approximate action maximizer to handle large or continuous action spaces. Importantly, because it eliminates the integral over actions, the deterministic policy gradient does not require importance sampling correction ratios when evaluating it using a distinct (stochastic) behavioral policy $\beta(a|s)$ \citep{silver2014deterministic}, making it an off-policy policy gradient algorithm.

\citet{lillicrap2015continuous} essentially extend DPG with the same techniques used by the DQN algorithm \citep{mnih2015human} (replay buffer, target networks, gradient clipping) to improve stability when used with nonlinear function approximators. The original work also uses batch normalization \citep{ioffe2015batch} to eliminate scale differences between state variables and consequently reduce the need to adjust the hyperparameter for every environment. They call this algorithm Deep Deterministic Policy Gradient (DDPG).

TD3 (Twin Delayed DDPG) \citep{fujimoto2018addressing} further improves on DDPG by introducing three modifications that greatly improve performance. First, inspired by the success of Double DQN (DDQN), the authors empirically show that the overestimation bias also affects actor-critic methods. However, contrary to DDQN, their results suggest that in this framework, the target Q-network is too dependent on the main critic to correct for overestimation. Instead, they propose to train two instances of the critic $Q_{\vfparams_1}$ and $Q_{\vfparams_2}$ (both of which also have a corresponding target network $Q_{\bar{\vfparams}_1}$ and $Q_{\bar{\vfparams}_2}$) and to use the smallest Q-value for their target $y = \RF(s,a) + \gamma \min_{i=1,2} Q_{\bar{\vfparams}_i}(s',a')$. Second, they recommend updating the policy $\mu_\piparams$ at a lower frequency than value networks to allow to obtain better value estimates before taking a policy improvement step. Third, they propose to perturb the target action using random noise $a'=\mu_{\bar{\piparams}}(s') + \epsilon \, , \, \epsilon \sim \text{clip}(\mathcal{N}\big(0,\sigma), -c, c\big)$ to smooth out the value function along the action dimensions and allow bootstrapping from similar state-action value estimates.

\subsection{Maximum Entropy Reinforcement Learning}

Although the optimal value function for a finite MDP is unique, there might exist several optimal deterministic policies. In principle, these optimal deterministic policies could be combined into a single stochastic policy that captures many different modes of optimal behavior. In general, there are several reasons why one would prefer learning a stochastic policy. For example, problems with partial observability might only allow for a stochastic optimal policy. Stochastic policies might also be more robust to adapt to a sudden change in the environment. Finally, they allow for a smooth exploration mechanism, as opposed to $\epsilon$-greedy policies. The Maximum Entropy framework for reinforcement learning (MaxEnt RL) \citep{ziebart2010modeling} aims at learning a policy that maximizes both the expected discounted return and the expected discounted entropy of the policy:
\begin{equation}
    \label{eq:maxent_rl_objective}
    J_{\text{MaxEnt}}(\pi) :=
    \sum_{t=0}^\infty \gamma^t \E_{(s_t, a_t)\sim p_{\pi}}\Big[\RF(s_t, a_t) + \alpha \mathcal{H} \big(\pi(\cdot|s_t)\big)\Big]
\end{equation}
 with $\alpha \geq 0$ and where $\alpha \to 0$ recovers the original RL objective (Equation~\ref{eq:expected_total_discounted_reward}). A similar approach which uses entropy maximization to prevent an early collapse of the policy (i.e., maintain exploration) is sometimes used with policy gradient approaches \citep{o2016combining}. Crucially, in the case of MaxEnt RL (Equation~\ref{eq:maxent_rl_objective}), the entropy term is found inside the expectation in the main objective rather than as a regularization term on the policy updates, which will push the policy not only to maximize its entropy in any given state, but to seek and navigate to states in which high entropy is aligned with high return, thus maximizing the entropy of the entire trajectory.
 
 The MaxEnt RL objective allows for a smooth-equivalent of the Bellman Equations often referred to as the soft-Bellman Equations \citep{haarnoja2017reinforcement}:
 \begin{align}
    \label{eq:soft_bellman_equations_for_v}
    v^{*}_{\text{soft}}(s)
    &= \alpha \log \sum_{a \in \mathcal{A}(s)} \exp\left( \frac{1}{\alpha} q^{*}_{\text{soft}}(s,a) \right) \\
    \label{eq:soft_bellman_equations_for_q}
    q^{*}_{\text{soft}}(s,a)
    &= \RF(s,a) + \gamma \E_{p_\pi}[v^{*}_{\text{soft}}(s_{t+1})]
 \end{align}
 where the $\max$ operator over actions for $v^{*}$ has essentially been replaced by a smooth-max operator (log-sum-exp) which approaches a hard-max as $\alpha \to 0$. The optimal policy w.r.t. the MaxEnt RL objective is proportional to the exponential of the soft q-values and $v_{\text{soft}}$ can be seen as the partition function:
 \begin{align}
    \label{eq:optimal_maxent_policy}
    \pi^*_{\text{MaxEnt}}(a|s) = \exp\Big( \frac{1}{\alpha} \big( q^{*}_{\text{soft}}(s,a) - v^{*}_{\text{soft}}(s) \big) \Big)
    &= \frac{\exp \big(\frac{1}{\alpha} q^{*}_{\text{soft}}(s,a)\big)}{\sum_{a' \in \mathcal{A}(s)} \exp \big( \frac{1}{\alpha} q^{*}_{\text{soft}}(s,a')\big)}
 \end{align}
 where $q^{*}_{\text{soft}}(s,a) - v^{*}_{\text{soft}}(s)$ is also referred to as the soft-advantage function. This optimal policy puts equal probability mass on two actions yielding the same expected return and exponentially less mass as the advantage decreases.
 
In the discrete control setting, we see from Equation~\ref{eq:optimal_maxent_policy} that the optimal policy can simply be represented by a softmax over the optimal soft q-function. Thus, we can learn a policy by parameterizing the soft q-function only using a function approximator $Q_{\vfparams}^{\text{soft}}(s,a)$ and evaluating $V_{\vfparams}^{\text{soft}}(s)$ exactly using Equation~\ref{eq:soft_bellman_equations_for_v} to produce a one-sample estimate of the target from Equation~\ref{eq:soft_bellman_equations_for_q}. This model can then be learned by minimizing the error in a DQN-like fashion, an approach that can be referred to as Soft Q-Learning \citep{haarnoja2017reinforcement}. In the continuous control case, the summation of actions in Equations~\ref{eq:soft_bellman_equations_for_v}~and~\ref{eq:optimal_maxent_policy} turn into integrals and one cannot simply represent the policy using a q-function only. In this case, the Soft Actor-Critic (SAC) algorithm \citep{haarnoja2018soft} can be used. Using this approach, the policy $\pi_\piparams$, the soft q-function $Q_{\vfparams}^{\text{soft}}$ and the soft value function $V_{\psi}^{\text{soft}}$ are parameterized separately. $Q_{\vfparams}^{\text{soft}}$ and $V_{\psi}^{\text{soft}}$ are trained by minimizing Bellman residuals, whereas $\pi_\piparams$ is trained to maximize a DDPG-like policy gradient using the reparameterization trick. Note that since the trade-off between maximizing entropy and return depends on the scale of the reward function, \citet{haarnoja2018soft} also propose an entropy temperature adjustment method for learning $\alpha$ automatically in order to avoid having to tune this hyperparameter when applying the algorithm to a different task.

\subsection{Generative Flow Networks}

Generative Flow Networks (GFlowNets, GFNs) are a class of generative models originally designed for compositional object generation \citep{bengio2021flow}. In this setting, the generated objects are assembled step by step by taking actions corresponding to adding new elements to the current state. Although closely linked to energy-based models and Monte Carlo Markov Chain sampling methods \citep{bengio2023gflownet}, the framework learns to model a distribution from a reward function, which also positions it as a suitable approach for some reinforcement learning problems.

GFlowNets generally operate on a particular subclass of MDP defined by some key characteristics. First, the action space is discrete, allowing for a finite set of actions in each state. Second, the state-space forms a Directed Acyclic Graph (DAG), meaning that all states must always eventually lead to a terminal state, without the possibility for any loop, and the transition function is deterministic (i.e., each action $a_t$ leads to only one successor state $s_{t+1}$). Finally, the reward function is positive in terminal states $R(s_T) > 0$ and null on all other states $R(s)=0$. Note that several tasks of interest such as molecular generation \citep{bengio2021flow}, causal discovery \citep{deleu2022bayesian} and sequence generation \citep{jain2022biological} present this type of MDP. 

The central property of GFlowNets lies in the terminal state distribution that characterises the target policy. Here, the model seeks to learn a policy $\pi$ such that the probability that a trajectory ends in a particular terminal state $s_T$ is proportional to its reward:
\begin{equation}
\label{eq:gfn_property}
    \pi^*_{\text{GFN}} \quad : \quad p_{\pi^*_{\text{GFN}}}(s_T) \propto R(s_T)
\end{equation}
Similarly to MaxEnt RL, this behavior is desirable as it allows to capture \textit{all} of the modes defined by the reward function rather than uncovering a single high-performing behavior using traditional RL methods. However GFlowNets operate from a different framework. Here, we model the problem as learning a flow of probability particles that operates in the state space $\mathcal{S}$. The flow starts from a unique initial state $s_0$ and spreads itself across the transitions $s\rightarrow s'$ until reaching a terminal state $s_T$. It is constrained by border conditions which state that the amount of flow to a terminal state must equal its reward $R(s_T)$, and that the total flow $Z$ in the network (departing from $s_0$) must equal the sum of rewards $Z:=\sum_{s_T}R(s_T)$. Enforcing these constraints alongside the principle of \textit{conservation of flow} throughout the network allows to recover a policy that samples terminal states proportionally to their reward. 

Several objectives have been shown to be sufficient to meet these conditions \citep{madan2023learning}. The most common is the Trajectory-Balance objective \citep{malkin2022trajectory}, which states that any given trajectory $\tau$ should yield the same probability when going forward as when going backward in the MDP which introduces the corresponding forward and backward policies $\pi^{F}$ and $\pi^{B}$. Recalling the MDP is deterministic, from Equation~\ref{eq:probability_of_trajectory} we have:
\begin{align}
    P_0(s_0)&p_{\pi^F}(\tau|s_0) = \mathbb{P}(s_T)p_{\pi^B}(\tau|s_T) \\
    &\Leftrightarrow \quad \prod_{t=0}^{T-1} \pi^{F}(s_{t+1}|s_t) = \frac{R(s_T)}{Z} \prod_{t=0}^{T-1} \pi^{B}(s_t|s_{t+1})
\end{align}
with $P_0(s_0)=1$. Concretely, both $\pi^F$ and $\pi^B$ can be parameterized by neural networks while $Z_\theta$ can be kept as a free parameter, and this equation can be turned into a loss function $L_{TB}$ using a squared log-ratio:
\begin{equation}
    L_{TB}(\theta) := \left(\log \frac{Z_\theta \, \prod_{t=0}^{T-1} \pi_\theta^{F}(s_t|s_{t+1})}{R(s_T) \prod_{t=0}^{T-1} \pi_\theta^{B}(s_{t+1}|s_t)}\right)^2
\end{equation}
As in MaxEnt RL, the result is a stochastic policy which seeks to capture all modes of the reward function.  However, both methods lead to different behaviors as the likelihood of sampling each mode is not the same. MaxEnt RL converges to a policy which samples trajectories proportionally to their exponential return along the entire path, whereas GFlowNets samples trajectories proportionally to their final reward. These behaviors are equivalent when the state space is represented as a tree (i.e., there is only one path leading to any terminal state) but differ in general DAGs where multiple paths may lead to the same terminal state. In this case, the MaxEnt RL objective of maximizing the \textit{diversity of trajectories} will favor final states $s_T$ that can be reached from many paths, while a GFlowNet will sample terminal states strictly based on their terminal rewards, thus enforcing a \textit{diversity of outcomes}. Whether diversity of trajectories or diversity of outcomes is preferable will depend on the context. For example, a policy for character control in a video game might benefit from trajectory diversity to embody all possible styles of locomotion. In contrast, in a drug discovery application where only the quality of the finished molecule matters, the diversity of outcomes would be preferred. All in all, GFlowNets add another tool to a growing collection of deep reinforcement learning methods for complex sequential decision making in the real world.

\chapter[LITERATURE REVIEW]{\\LITERATURE REVIEW}\label{chap:lit_review}

The reward function is a central component of reinforcement learning algorithms. It defines the task to be solved in a given MDP. Designing a reward function is traditionally a manual and iterative process, guided by the user's intuition and refined through trial and error \citep{knox2023reward, booth2023perils, hayes2022practical}. Since each iteration requires training an RL agent to convergence, this process is slow and costly, posing a major challenge for RL deployment in real-world applications. Numerous studies have focused on the issue of reward specification, creating a range of paradigms aimed at guiding exploration and ensuring alignment through advanced design techniques for the reward function. In this chapter, we review the most established areas of reinforcement learning that address this challenge. We group these methods into two distinct categories. First, in Section~\ref{sec:reward_composition}, we present Reward Composition approaches, which take account of the multi-faceted nature of the reward function and provide it as a set of reward components to the learning algorithm. Then, in Section~\ref{sec:reward_modeling}, we cover Reward Modeling methods, which leverage various supervision signals to learn the reward function instead of explicitly specifying it.

\section{Reward Composition}
\label{sec:reward_composition}

\begin{figure}
    \centering
    \includegraphics[width=0.9\textwidth]{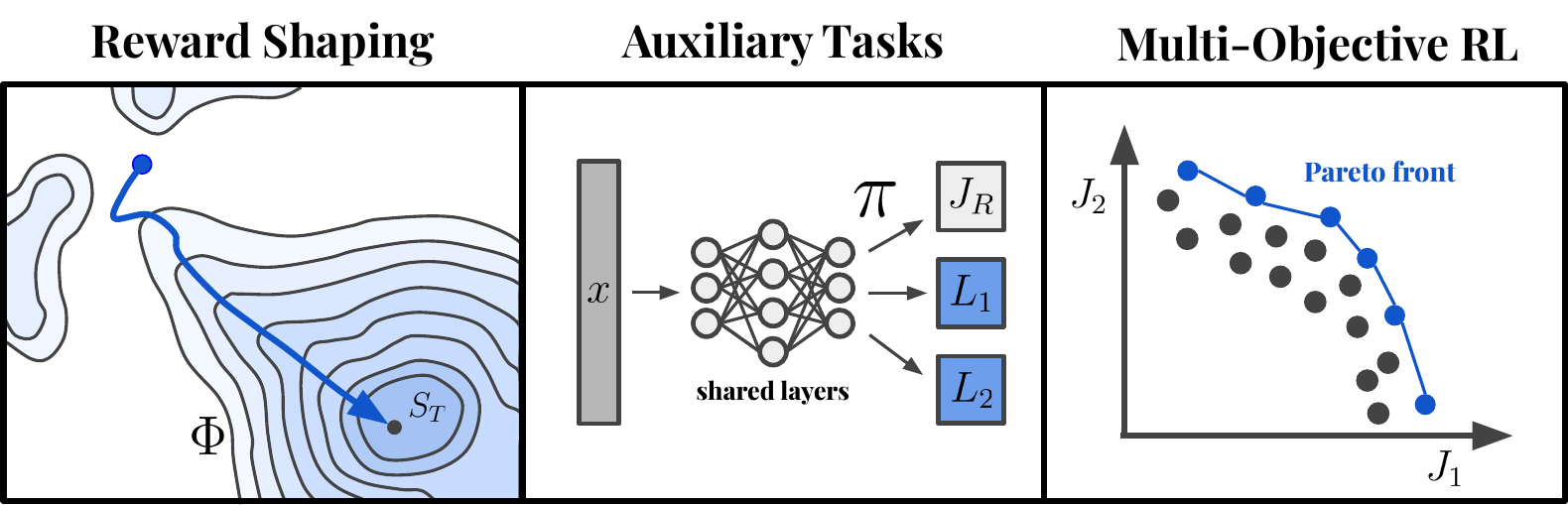}
    \caption[Overview of reward composition strategies]{Overview of reward composition strategies. Reward composition consists in integrating different components into a reward function. (Left) Potential-based reward shaping aims at hinting the agent toward its goal by augmenting its sparse reward function with a denser signal defined from a potential function $\Phi$ on the state space. (Middle) Auxiliary tasks (here $L_1$ and $L_2$) are optimised simultaneously with the main objective $J_R$ to help the agent learn a useful representation of the input $x$. (Right) Multi-objective RL treats each reward component as a distinct objective and uses various techniques to select for different solutions on the Pareto front.}
    \label{fig:reward_composition}
\end{figure}

The Reward Hypothesis, posited by \citet{sutton2018rlTextbook}, states that any goals that we wish an agent to accomplish can be thought of as the maximization of the expected total reward.
In other words, a sufficiently complex scalar reward function \textit{could}, in principle, be used to specify any conceivable task. While the exactitude of this hypothesis relies on careful theoretical assumptions \citep{bowling2023settling}, the large number of RL benchmark environments developed for research in the last decade \citep{todorov2012mujoco, bellemare2013arcade, brockman2016openai, juliani2018unity, dosovitskiy2017carla, vinyals2017starcraft, wydmuch2018vizdoom} and the numerous applications of RL in the real world \citep{deepmind2016cooling, yu2019deep, bellemare2020autonomous, shahidinejad2020joint} represent strong evidence that rewards can indeed capture a very wide variety of tasks. Nevertheless, despite their scalar nature, most reward functions usually encompass a number of distinct elements that are simply weighted against each other before being aggregated into a single number. For example, a self-driving car should aim to reach its destination while respecting traffic laws and avoiding obstacles. A grasping robot should learn to move objects around while minimizing its energy consumption and without damaging its environment. 

\textit{Reward composition} here refers to the idea of assembling a reward function from multiple distinct components. In addition to making the agent's behavior more interpretable \citep{anderson2019explaining}, a component-centric treatment of the reward function allows the utilization of these components for various purposes, such as aiding exploration, enhancing alignment, or balancing diverse goals. In the next sections, we survey prior work that frame these components either as shaping rewards, auxiliary tasks, or competing objectives.


\subsection{Potential-Based Reward Shaping}
\label{sec:reward_shaping}

The challenge of specifying good reward functions is as old as reinforcement learning. One of its central dilemmas is that of sparse vs. dense rewards. Sparse rewards are often natural to specify but difficult to learn from. For example, we could reward an agent only when reaching its final destination in a navigation task. However, when learning from a sparse reward only, an agent needs to explore its environment in the total absence of feedback until it happens to stumble on the solution. Only then can it learn from this experience and reinforce the successful behavior. A denser reward signal could potentially help guide the agent towards the solution. For example, we could give a smaller reward when the agent navigates closer to its destination. While this approach can dramatically improve the sample efficiency of the algorithm, a naive implementation of such breadcrumbs can also lead to degenerate solutions and reward hacking by exploiting cycles in the reward function \citep{amodei2016concrete}. For example, the agent could learn to run in circles to repeatedly get rewarded for ``making progress'' without ever actually reaching the goal \citep{randlov1998learning}.

\textit{Potential-Based Reward Shaping} (PBRS) \citep{ng1999policy} is a reward composition strategy allowing to augment a sparse reward function $R$ with a denser signal $R_{\text{shaping}}$ without changing the set of optimal policies. It eliminates the risk of reward cycles by restricting the form of the shaping reward to the discounted difference between the potential of the next state $\Phi(s')$ and that of the current state $\Phi(s)$:
\begin{equation}
      R' := R + R_{\text{shaping}} \, \, , \quad R_{\text{shaping}} := \gamma \Phi(s') - \Phi(s) \, \, , \quad \argmax_{\pi \in \Pi} \, \, J_{R'}(\pi) = \argmax_{\pi \in \Pi} \, \, J_{R}(\pi)
\end{equation}
The potential function $\Phi$ can be thought of as a topographic map guiding the agent towards higher peaks of the reward function (see Figure~\ref{fig:reward_composition}). Since it does not depend on actions, the cumulative discounted return for $R_{\text{shaping}}$ is independent of $\pi$. Consequently, $R_{\text{shaping}}$ does not affect the ordering of the Q-values for $R$ and the set of optimal policies under $J_R$, but allows the agent to reach the critical regions of the MDP with fewer exploration steps \citep{laud2003influence}. Empirically, this approach has been shown to greatly accelerate learning and is still being used in challenging high-dimensional environments \citep{berner2019dota}.

In subsequent work, \citet{wiewiora2003principled} extend potential-based shaping to potential functions of both states and actions $\Phi(s,a)$ and shows that using it for shaping reward is equivalent to initializing the Q-values of the main reward $R$ to this potential. \citet{devlin2012dynamic} instead augment potential functions with a time dependence $\Phi(s,t)$ to allow dynamic adaptation of the shaping reward. \citet{harutyunyan2015expressing} proposes to use the value function of any arbitrary shaping signal as a potential over states to be used for potential-based reward shaping, allowing the PBRS framework to be extended to a larger family of shaping functions. Finally, \citet{hu2020learning} cast dynamic reward shaping as a bilevel optimization problem in which they automatically learn a weighting coefficient to reduce the impact of harmful reward shaping once detected.


\subsection{Auxiliary Tasks in RL}
\label{sec:auxiliary_tasks}

Reinforcement learning is a very general approach to artificial intelligence. \citet{silver2021reward}  even argue that the act of maximizing a reward signal by trial and error in a rich environment could be sufficient to develop any attribute of intelligence that is required to solve a particular task. However, this process can be highly inefficient due to the sparsity of the reward function and the informational complexity of the environment \citep{yu2018towards}. Although potential-based reward shaping offers a principled solution to the problem of reward sparsity, it is limited to specific functional forms \citep{harutyunyan2015expressing}. A common alternative consists of using \textit{auxiliary tasks} to leverage additional learning signal from the environment. An auxiliary task often involves predicting quantities that are useful to complete the main control task. One or even several such tasks can be optimized simultaneously with the RL objective $J_R$ by scalarizing them using fixed weighting coefficients $\lambda_k$ for each of the $K$ auxiliary losses $L_k$:
\begin{equation}
    J_{tot}(\pi) := J_R(\pi) - \sum_{k=1}^K \lambda_k L_{k}(\pi)
\end{equation}
This is typically implemented by training a single neural network model sharing the first layers across all tasks while being equipped with different heads for the policy and the auxiliary task predictions (see Figure~\ref{fig:reward_composition}). The shared encoding layers thus benefit from the learning signal of both objective types, allowing to perform informative updates even in the absence of external reward. Contrary to model-based approaches \citep{moerland2023model}, which seek to learn the transition distribution of the environment to perform planning, auxiliary tasks are only applicable with deep parameterisations and seek to improve the representations learned by an agent \citep{vincent2008extracting}, 
or to accelerate the optimization process by helping avoid large portions of the policy space \citep{gupta2022unpacking}.

One of the first uses of auxiliary tasks is presented in the work of \citet{suddarth1990rule}, which used so-called \textit{hints} to accelerate the learning of neural networks trained to solve simple logical problems. In reinforcement learning, \citet{sutton2011horde} introduces generalized value functions (GVFs), which extend the concept of state-action value function to measuring different properties of a policy $\pi$ beyond its return $J_R(\pi)$ on the main task. The idea has since been applied more generally to deep reinforcement learning in a variety of contexts, often achieving important improvements in learning speed and performance. For example, in addition to its main task, \citet{jaderberg2016reinforcement} train an agent to correctly predict incoming rewards and maximize perceptual changes in its environment. \citet{shelhamer2016loss} and \citet{laskin2020curl} use self-supervised successive-states prediction and contrastive losses to learn useful representations for the policy and value function. \citet{mirowski2016learning} use depth and loop-closure prediction as auxiliary objectives for navigation tasks. \citet{lample2017playing} task the agent to predict the presence of enemies or weapons when learning to play a video game. \citet{kartal2019terminal} train the agent to predict whether it is close to the end of the episode. \citet{fedus2019hyperbolic} predict the return for multiple time horizons. \citet{hernandez2019agent} use action prediction as an auxiliary task in multi-agent settings. \citet{song2021multimodal} perform velocity estimation to improve the representations learned by a policy controlling mobile indoor robots.

Interestingly, auxiliary tasks seem to provide benefits that exceed the information contained in the target of such task. For example, \citet{mirowski2016learning} compare predicting depth as an auxiliary task with simply providing the agent with a depth-map as additional input, and report significantly improved performance with the former approach, supporting the wider effect of representation learning provided by the act of solving these tasks, as opposed to being provided with their answer. A growing body of work focuses on uncovering the mechanisms by which auxiliary tasks are so effective in improving sample efficiency in RL, investigating their role in preventing overfitting \citep{dabney2021value} and representation collapse \citep{lyle2021effect}.

Finally, while scalarization and parameter sharing are often used to integrate auxiliary tasks in the training pipeline, alternative approaches have been developed to address the challenges of competing gradient updates and learning instabilities \citep{teh2017distral, yu2020gradient, rosenbaum2017routing}. In particular, some works investigate the automatic adaptation of task coefficients to automatically detect and tune down the influence of auxiliary tasks that would become harmful to the main objective \citep{du2018adapting, lin2019adaptive}. 


\subsection{Multi-Objective RL}

As seen previously, some problems such as playing chess can be broken down into a main sparse objective, e.g. winning the game, and denser rewards which are often correlated with the main objective and meant to guide the agent towards successful policies, e.g. capturing opponent pieces. However, for other tasks, the target behavior encapsulates truly distinct or even conflicting objectives which are equally important and should not be freely traded-off by the agent \citep{vamplew2022scalar}. For example, avoiding obstacles in a navigation task is not enforced merely to guide the agent towards reaching its goal; both preserving the integrity of its surroundings and reaching its target location are nonnegotiable criteria for a successful policy. The question of how to best handle multiple objectives has been identified as one of the main limitations preventing RL from being applied in more real world domains \citep{dulac2021challenges}, and multi-objective RL (MORL) algorithms are designed specifically to address this question. 

At their core, multi-objective approaches treat each reward component as a distinct criterion. They are formalized as Multi-Objective MDPs (MOMDPs) \citep{roijers2013survey}, defined by the tuple $(\ptrans, \pzero, \mathcal{S}, \mathcal{A}, \{\RF\}_{k=1}^K, \gamma)$. For the most part, they are identical to their regular MDP counterpart, defined in Section~\ref{sec:MDPs}. However, instead of a single reward function, we now have a set of $K$ reward functions $\{\RF_k\}_{k=1}^K$ that map a state-action pair to a \textit{vector} of rewards. MOMDPs thus generalize regular MDPs since $K=1$ brings us back to the single-objective case. Each component is defined as $\RF_k:\mathcal{S} \times \mathcal{A} \rightarrow [r_{k\text{-min}}, r_{k\text{-max}}]$ and $J_{\RF_k}(\pi)$ represents the expected total discounted reward for the reward component $k$ as defined in Equation~\ref{eq:expected_total_discounted_reward}. The problem now becomes:
\begin{equation}
    \argmax_{\pi \in \Pi} \quad (J_{\RF_1}, \dots, J_{\RF_K})
\end{equation}
Importantly, with $K > 1$, the notion of \textit{optimality} is not directly transferable to the multi-objective case. In single-reward MDPs, the RL objective induces a total order over policies. However, with multiple objectives, a policy $\pi'$ may perform better than $\pi$ on $J_{\RF_1}$ but worse on $J_{\RF_2}$, thus requiring a different criterion to compare potential solutions. A policy $\pi'$ is said to \textit{dominate} $\pi$ if it is superior on at least one objective and at least equal on the others:
\begin{equation}
    \pi' \succ \pi \quad \Leftrightarrow \quad
\begin{cases}
  \forall \, k: & J_{\RF_k}(\pi') \geq J_{\RF_k}(\pi) \\
  \exists \, k*: & J_{\RF_{k*}}(\pi') > J_{\RF_{k*}}(\pi)
\end{cases}
\end{equation}
A solution $\pi$ is said to be \textit{Pareto-optimal} if there are no other solution that dominates it, and the set of Pareto-optimal solutions forms the \textit{Pareto front} in objective space (see Figure~\ref{fig:reward_composition}). In most problems, some of the objectives will conflict and an optimal solution $\pi^*$ which dominates all other policies will not exist; we will have to settle for a solution that strikes an acceptable trade-off between the objectives. MORL algorithms can generally be categorized into \textit{single-policy} and \textit{multi-policy} methods \citep{vamplew2011empirical} depending on whether they seek to generate a single policy or an approximation of the entire Pareto front. 

Single-policy approaches seek to find the policy that best captures the desired trade-off. They typically assume that the \textit{utility function} $u$ of the user is known and use it to combine the objectives and recover a total order over the policies in $\Pi$ \citep{hayes2022practical}. 
A simple and popular approach to utility-based MORL is to extend existing RL algorithms by learning a set of value functions $\{Q_k^\pi\}_{k=1}^K$ and adapting the action selection process by taking the action that maximizes utility. \citet{aissani2008efficient} employ this technique with the SARSA algorithm and a linear utility function that leads to an action selection of the form: $\argmax_a \sum_k Q_k^\pi(s,a)$. To capture preferences on some of the objectives, linear utility functions can employ different weighting coefficients $\{w_k\}_{k=1}^K \text{with} \sum_k w_k=1$ \citep{castelletti2002reinforcement}. However, the effective coverage of this approach is limited to only the convex parts of the Pareto fronts \citep{das1997closer, vamplew2008limitations}. To overcome this limitation, other methods employ nonlinear utility functions \citep{van2013hypervolume, van2013scalarized}, but nonlinear scalarization is difficult to combine with value-based RL algorithms, as they break the additivity property required for Q-decomposition \citep{russell2003q}. Therefore, these approaches have been extended to policy-gradient actor-critic architectures \citep{siddique2020learning, reymond2023actor}.

In the absence of a known utility function, other methods recover a total order over the policies in $\Pi$ by assuming a hierarchy of priority among the objectives to optimize \citep{gabor1998multi}. A \textit{lexicographic ordering} implies that the ordinality of the objectives now reflects their rank in importance. In the same spirit as Asimov's famous ``laws of robotics'', this means that an optimal policy $\pi^*$ should maximize $J_{R_1}$ in priority, \textit{then} $J_{R_2}$, \textit{then} $J_{R_3}$, and so on. 
Let us define $\Pi_k^*$ the set of optimal policies w.r.t. $J_{R_1}$ through $J_{R_k}$. 
A lexicographically optimal policy is defined as a policy that maximizes each objective in lexicographic order without worsening the previous ones:
\begin{equation}
    \pi^* := \argmax_{\pi \in \Pi_{K-1}^*} J_{R_{K}}(\pi) \quad , \quad \Pi_{k}^* := \left\{ \pi: J_{R_{k}}(\pi) = \max_{\pi' \in \Pi_{k-1}^* } J_{R_{k}}(\pi')\right\} \quad , \quad \Pi_0^* := \Pi
\end{equation}
where the recursion occurs in the constraint established by the optimal policy-sets $\Pi_k^*$. Since this hierarchy severely restricts the set of remaining optimal policies at each step, a strong lexicographic ordering would not scale to a large number of conflicting objectives. Instead, it is common to relax this ordering by introducing slack variables to specify how much we can deviate from $J_{R_k}^*$ to improve $J_{R_{k+1}}$ \citep{wray2015multi, skalse2022lexicographic, pineda2015revisiting}. This method has been applied in various contexts, including autonomous driving \citep{li2018urban} and multi-agent RL \citep{hayes2020dynamic}.

Another related single-policy approach called Constrained Reinforcement Learning (CRL) picks one of the objectives as the main reward to optimize and defines a set of constraints using the other reward functions. While several types of constraints have been used (probabilistic, instantaneous) \citep{liu2021policy}, the most common approach is to define a cumulative constraint on each objective $J_{R_k}$ for $k>1$ by specifying thresholds $d_k$. Formally, the problem becomes:
\begin{align}
    \pi^* := \argmax_{\pi \in \Pi} J_{R_1}(\pi) \quad \text{s.t.} \quad  J_{R_k}(\pi) \geq d_k \, , \, \, k = 2 , \dots , K
\end{align}
A policy that satisfies the constraint set is said to be a \textit{feasible} policy, and we seek to find the best-performing feasible policy $\pi \in \Pi_C$ over $J_{R_1}$. Note that, while equivalent, the constraints are typically re-labeled as cost functions $C_k$ and their threshold revsersed i.e. $J_{C_k}(\pi) \leq d_k$ \citep{altman1999constrained}. A popular approach to solving such problems is to use a Lagrangian relaxation \citep{bertsekas1997nonlinear} to fold the constraints and $J_{R_1}$ into a single objective addressed by a bilevel optimization process \citep{tessler2018reward, chow2017risk, liang2018accelerated, stooke2020responsive, bohez2019value}. Other alternatives based on trust-regions \citep{achiam2017constrained, yang2020projection}, Lyapunov functions \citep{chow2018lyapunov, chow2019lyapunov} and the interior-point method \citep{liu2020ipo} have also been proposed for Safe RL applications in which respecting the constraints is important throughout the entire search.

In general, the choice of approach will depend on which formulation best captures the true preferences of the user. Some problems are naturally formulated with a hierarchy of importance over objectives, while for others, a distinction between objectives to maximize and constraints to satisfy makes the most sense. Finally, when the user is unsure about their own preferences, multi-policy methods can be employed. These approaches focus on generating a variety of solutions striking different trade-offs on the Pareto front to allow the user to select one of these options later on \citep{hayes2022practical}. Inner-loop methods aim at maintaining a set of nondominated policies to learn about all Pareto-optimal solutions simultaneously \citep{barrett2008learning, iima2014multi, van2014multi, reymond2019pareto, li2020deep}. In contrast, outer-loop methods generate multiple solutions by simply running single-policy algorithms in sequence while varying the parameters of the utility function, the ordering of the objectives, or the thresholds of the constraints \citep{parisi2014policy, mossalam2016multi, xu2020prediction}. Some methods also condition the policy with these preferences to have access to all the learned policies and share their representation in a single model \citep{abels2019dynamic}. Ultimately, multi-policy approaches remain more computationally costly but offer additional flexibility in terms of the proposed solutions.


\section{Reward Modeling}
\label{sec:reward_modeling}

Despite the variety of tools and approaches presented in Section~\ref{sec:reward_composition} that help practitioners define safe and efficient reward functions for RL, many take the position that, for complex tasks and environments, manually specified reward functions are doomed to be incomplete and underspecified \citep{ibarz2018reward, hadfield2017inverse, leike2018scalable}. Instead, they advocate for \textit{learning} the reward function from human supervision, a radically different strategy which can be referred to as \textit{reward modeling}.

\begin{figure}
    \centering
    \includegraphics[width=0.9\textwidth]{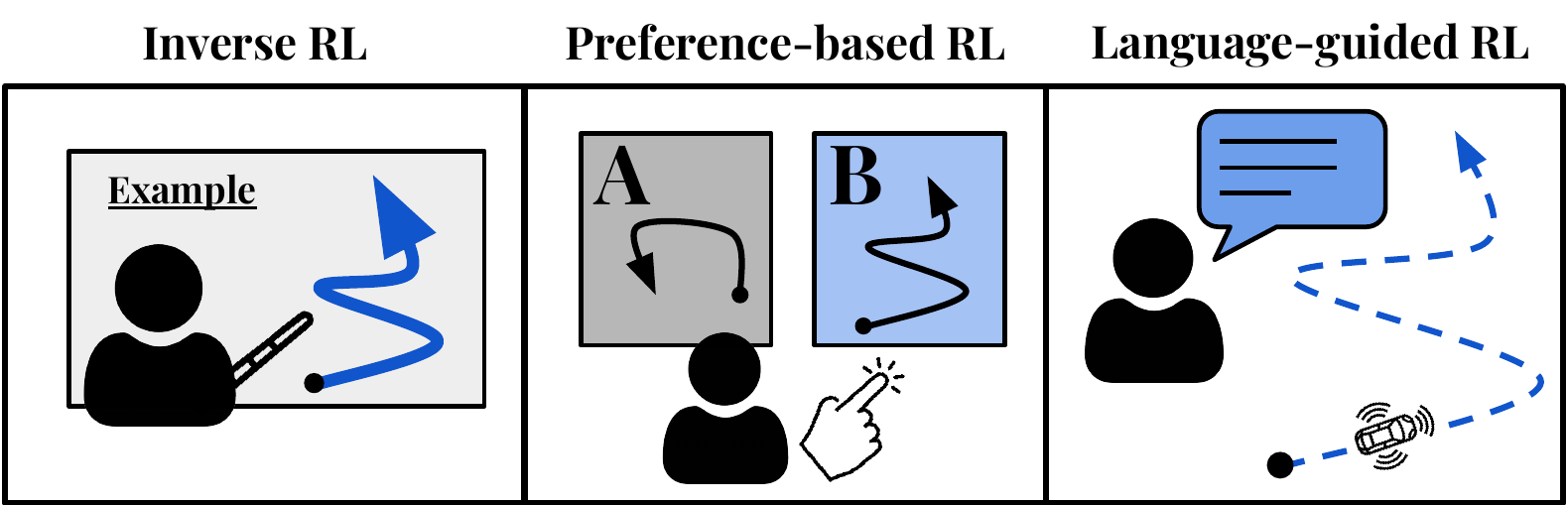}
    \caption[Overview of reward modeling strategies]{Overview of reward modeling strategies. Reward modeling consists in learning a model of the reward function from human supervision in an attempt to capture the true intentions of the task designer. (Left) In Inverse RL, a model of the reward function is learned from expert demonstrations. (Middle) In preference-based RL, the reward model is learned from human preferences over pairs of examples. (Right) In language-guided RL, a reward model can be trained to predict whether an agent's behavior is in accordance with a language command.}
    \label{fig:reward_modeling}
\end{figure}

The main benefit of such approaches is that they enable a definition of the task that is not subject to arbitrary scaling choices for the numerical value of the reward, the shaping signal, and the other components describing the desired behavior. Instead, reward modeling seeks to leverage different forms of human input to model a parameterized reward function $r_\rparams$, which will then lead a downstream RL algorithm to learn the intended behavior \citep{jeon2020reward}. In the next sections, we review the literature on reward modeling from three types of supervision signal: expert demonstrations, human preferences, and natural language (see Figure~\ref{fig:reward_modeling}).


\subsection{Inverse RL}

Reinforcement learning (RL) aims to learn a behavior policy from a reward function. Conversely, \textit{inverse} reinforcement learning (IRL) aims to learn a reward function from a demonstrated behavior. 
More specifically, starting from a finite set of $N$ expert trajectories $\mathcal{D}_E:=\{\tau^{(i)}\}_{i=1}^N$, the goal is to learn a parameterized reward function $r_\rparams$ for which the expert is uniquely optimal, and then to use it to train an agent to behave like the expert by running a reinforcement learning algorithm on this learned reward function.

The IRL problem was first introduced by \citet{russell1998learning}. An algorithm for this problem generally follows this general procedure: starting from a first estimate of the true reward function, we need to iteratively (1) solve for a policy $\pi$ which is optimal under $J_{r_\rparams}$ and (2) modify our reward estimate $r_\rparams$ to minimize the distance between the learned policy's behavior and expert behavior inferred from the demonstration set $\mathcal{D}_E$. This paradigm presents two important challenges. First, to obtain $\pi$ from our current estimate $r_\rparams$, one needs to solve a complete reinforcement learning problem which must be repeated at each step of the IRL algorithm. Second, there are generally a large number of reward functions that could explain the demonstrated behavior (including degenerate solutions such as $r_\rparams(s,a)=0 \, \, \forall \, s,a$). Foundational works in IRL handle the ambiguity of the solution set in different ways, and most can be categorized either as maximum margin methods, maximum entropy methods, and Bayesian approaches \citep{adams2022survey, arora2021survey}.

Maximum-margin methods aim at solving the ambiguity problem by maximizing a margin which describes how well the learned reward $r_\rparams$ explains the demonstrated behavior compared to any other policy \citep{ng2000algorithms, ratliff2009learning, silver2008high, ratliff2006maximum}. In other words, it seeks to make the expert as \textit{uniquely} optimal as possible. An example of such margin could be, for each state, the difference between the value of the action from the demonstration set $a^*$ and any other action $a$:
\begin{equation}
    \argmax_{r_\rparams} \quad \sum_{s \in \mathcal{S}} \Big( \hat{Q}(s,a^*) - \max_{a \in \mathcal{A}(s) \setminus \{a^*\}} \hat{Q}(s,a) \Big)
\end{equation}
Apprenticeship learning \citep{abbeel2004apprenticeship, abbeel2005exploration} is a particularly influential method for maximum margin optimisation which seeks to match the feature expectation of the demonstrated behavior. 

Another approach to tackle the ambiguity of the rewards is the maximum entropy IRL framework \citep{ziebart2008maximum}, which states that the probability of a trajectory should be proportional to its cumulative rewards:
\begin{equation}
    \mathbb{P}(\tau|r_\rparams) \propto \exp \left(\sum_{(s,a)\in \tau} r_\rparams(s,a)\right)
\end{equation}
In other words, trajectories generating the same return under $r_\rparams$ should be equally likely under $\pi$, whereas a trajectory with higher return should be exponentially more likely, with the trajectories from the demonstration set $\mathcal{D}_E$ be the most probable of all.
The goal here is to obtain a policy that acts as randomly as possible, while maximizing the reward estimate $r_\rparams$ so that the solution commits as little as possible to any one possible reward function. To achieve this property, the learned reward function is often extended with an entropy term when training the policy $\pi$. Entropy over trajectories is often used, but extensions of this framework include the use of causal entropy \citep{ziebart2010modeling} or relative entropy with a baseline policy \citep{boularias2011relative}.

Finally, bayesian IRL \citep{ramachandran2007bayesian} consists in capturing the distribution over all possible candidate reward functions which explain the expert behavior:
\begin{equation}
    \mathbb{P}(r_\rparams|\tau) \propto \mathbb{P}(\tau|r_\rparams) \mathbb{P}(r_\rparams) \quad , \quad \tau \sim \mathcal{D}_E
\end{equation}
While allowing to handle the reward ambiguity issue in a principled way, the prior distribution $\mathbb{P}(r_\rparams)$ must be carefully selected based on the expected properties of the true reward function \citep{arora2021survey}. Different parameterisations have been explored for likelihood $\mathbb{P}(\tau|r_\rparams)$, such as Boltzmann distributions \citep{choi2011map}, which require computationally expensive MCMC sampling methods to estimate the normalization constant or Gaussian processes \citep{levine2011nonlinear} that limit the expressivity of the model by depending on engineered features.

The reliance of early methods on linear combinations of predefined reward features or on simple nonlinear parameterisations prevented their applicability to higher-dimensional problems. More recent approaches \citep{wulfmeier2015maximum, finn2016guided} leverage deep function approximators to model complex nonlinear relations between the state features and the reward model. They also circumvent the need to fully optimize the policy in the inner loop of the reward optimization by swapping the two optimisation loops or by considering partial updates of the policy for each reward iteration. In particular, Adversarial Inverse Reinforcement Learning (AIRL) \citep{fu2017learning} take inspiration from both the adversarial imitation learning framework from \citet{ho2016generative} and the potential-based reward shaping from \citet{ng1999policy} to learn robust reward functions from expert demonstrations. The idea is to train a discriminator $D_{\rparams}$ to classify whether a given transition has been generated by the agent policy $\pi_\piparams$ or the expert policy $\pi_E$, while parameterizing it so that the learned reward implements a potential function $h_\omega$ over states:
\begin{align}
\label{eq:gail_objective}
    \min_{\piparams} \,  \, \max_{\rparams} \, \, &\E_{d_{\pi_E}}\big[\log D_{\rparams}(s,a,s')\big] + \E_{d_{\pi_{\piparams}}}\big[\log(1 - D_{\rparams}(s,a,s'))\big] \\
    & \text{with} \quad D_{\rparams}(s,a,s') := \frac{\exp\big(f_\rparams(s,a,s')\big)}{\exp\big(f_\rparams(s,a,s')\big) + \pi_\piparams(a|s)} \\
    & \text{and} \quad f_\rparams(s,a,s') := \underbrace{r_\rparams(s,a)}_{\text{main reward}} + \underbrace{\gamma h_\omega(s') - h_\omega(s)}_{\text{shaping reward}}
\end{align}
These approaches seek to allow for more expressive reward modeling capabilities and improve the generalisation properties of these algorithms by providing negative examples to the learned policy. \citet{ghasemipour2020divergence} provides a divergence-based classification of recent progress in this field.


\subsection{Preference-based RL}

Preference-based RL (PbRL) is a paradigm for learning a policy from non-numerical feedback using reinforcement learning \citep{wirth2017survey}. The idea of learning a reward function from qualitative feedback stems from the observation that it is often easier to judge the quality of a solution than to produce it ourselves; a concept that has been fundamental to the field of computational complexity \citep{fortnow2009status}. Indeed, since expert demonstrations are expected to be optimal, providing them may be expensive and requires the human supervisor to have great proficiency at that task. Instead, preference-based RL takes a different approach by asking a human to simply compare different outcomes and using this ranking as a supervision signal, thus providing feedback to the agent without having to fully specify a reward function (regular RL) or to produce a complete set of demonstrations (inverse RL).

The simplest approach for collecting human preferences consists of presenting the examiner with a set (or even a single pair) of example trajectories and asking them to provide a ranking from the least adequate to the best example. 
Despite its simplicity, qualitative feedback in the form of rankings has been found to be sufficient to specify a variety of control problems. Theoretical analysis of task specification in RL even pinpoint the very generic idea of ``goals'' as a ``a binary preference relation expressing preference over one outcome over another'' \citep{bowling2023settling}, suggesting that enumerating preferences between pairs of candidates in the set of all possible solutions may be sufficient to specify any task \citep{abel2021expressivity}. 
The relationship $\tau_i \succ \tau_j$ indicates that the trajectory $\tau_i$ is strictly preferred over $\tau_j$. In PbRL, the goal is to find a policy $\pi^*$ that maximizes the difference in likelihood for trajectories $\tau_i, \tau_j$ for all the preferences $i, j$ available:
\begin{equation}
    \tau_i \succ \tau_j \quad \Rightarrow \quad \pi^* = \argmax_{\pi \in \Pi} \, p_{\pi}(\tau_i) - p_{\pi}(\tau_j)
\end{equation}

The use of rankings as a learning signal in reinforcement learning dates back to \citet{cheng2011preference} and \citet{akrour2011preference}, who, respectively, propose algorithms that exploit the signal from user preferences over actions or trajectories to derive improved policies. This approach has been the subject of growing interest with several follow-up works \citep{wirth2017survey} and has since been scaled up in the deep RL framework by \citet{christiano2017deep}. They propose to train a reward model $r_{\rparams}$ parameterized by a neural network $\rparams$ from pair comparisons of partial trajectories and to use it to learn a policy $\pi_\piparams$. Crucially, reward modeling, policy optimization and preference rating all run in parallel in an asynchronous fashion. The policy is randomly initialized and starts collecting trajectories $\tau$. These trajectories are broken into trajectory fragments $\sigma$ and sent to a human evaluator in pairs $(\sigma_i, \sigma_j)$. The evaluator indicates whether a fragment is better, that they are equal, or refuses to include them in the database. The preference label $y := I(\sigma_i \succ \sigma_j)$ is recorded (or $y=\frac{1}{2}$ if they are equal), and this relationship is leveraged to train $r_{\rparams}$ by assuming that the preference grows exponentially with the sum of rewards to form a probabilistic model of human preferences:
\begin{equation}
    P_\rparams(\sigma_i \succ \sigma_j) := \frac{\exp\left(\sum_{(s_t, a_t) \in \sigma_i} r_{\rparams}(s_t, a_t)\right)}{\exp\left(\sum_{(s_t, a_t) \in \sigma_i} r_{\rparams}(s_t, a_t)\right) + \exp\left(\sum_{(s_t, a_t) \in \sigma_j} r_{\rparams}(s_t, a_t)\right)}
\end{equation}
It can then be optimized by minimizing the Binary Cross-Entropy loss (BCE) between the predicted likelihood of the preference ordering and the true preference label  \citep{david1963method}:
\begin{equation}
    \text{BCE}(P_\rparams, y) = - y \log P_\rparams(\sigma_i \succ \sigma_j) - (1 - y) \log P_\rparams(\sigma_i \prec \sigma_j) \quad \forall \quad (i,j, y) \in \mathcal{D}
\end{equation}
The policy $\pi_\piparams$ can then be trained to maximize the return over $r_\psi$ using any deep RL algorithm. To prioritize the order in which trajectories should be queried for evaluation, the authors train an ensemble of reward models and query the trajectory-fragment pair which shows the highest variance between to maximally reduce the uncertainty across the ensemble. The results show that preference-based RL can scale to complex problems without having access to the true reward function, and that it can even outperform learning from the true reward function in some cases where human evaluations lead to a better shaped reward model $r_{\rparams}$ than the true reward function $R$. Several extensions of this work have been explored, for example, to enable off-policy learning by relabeling past experiences when the reward model is updated \citep{lee2021pebble} or to assess the benefits of modeling the dynamics of the environment alongside preference-based policy iteration \citep{liu2023efficient}.


Preference-based learning can also be used on a fixed set of demonstrations. \citet{brown2019extrapolating} uses human rankings in the context of imitation learning where the demonstrations may be suboptimal. They require the dataset to provide a total order over the trajectories, which implies a number of constraints that grows quadratically in the number of samples and allows learning the reward function entirely off-line.
In cases where rankings are not provided but the demonstrations are deemed optimal, \citet{reddy2019sqil} use a method which considers all the trajectories in the demonstration set as being of equal value and preferred to any other generated trajectory, which can be interpreted either as inferring preferences from a set of demonstrations or as performing IRL with a discretized reward.
\citet{jain2013learning} experiment with a more subtle way of combining the information from human demonstrations and preferences by offering the evaluator the option to provide a trajectory that merely improves upon the last demonstration rather than providing a near-optimal example. Finally, demonstrations and preferences can be used in combination to improve the sample efficiency of preference-based approaches while outperforming the demonstrations \citep{ibarz2018reward, palan2019learning}.

More recently, with the advent of Large Language Models (LLMs), preference-based reinforcement learning (also known as RL from Human Feedback, RLHF) has experienced a surge in popularity in an effort to improve the behavior of large pre-trained models \citep{fernandes2023bridging, ziegler2019fine, stiennon2020learning, jaques2019way, kreutzer2018can, bai2022training, askell2021general}. In particular, \citet{ouyang2022training} experiment with fine-tuning LLMs using RLHF and show that 
incorporating human preferences allows to obtain better aligned models even with a fraction of the original capacity. 
\citet{glaese2022improving} break down human preferences into several distinct aspects of desirable behavior to make better use of the provided feedback. These methods showcase the remarkable flexibility of preference-based reward specification, which allows to capture the essence of very nuanced and personalized objectives such as insuring that personal assistants provide harmless, accurate and useful advice.


\subsection{Language-guided RL}

Finally, natural language could represent a source of supervision which is less demanding than complete demonstrations, but more sample efficient than preference rankings. Humans use natural language extensively to give feedback to each other, and recent advances in natural language processing \citep{treviso2023efficient} open up the opportunity for leveraging the flexibility of language instructions for task specification in RL. This approach to reward specification is particularly appealing because language naturally captures object relations and elements of compositionality of the agent's environment. Moreover, it provides a user interface which does not require a technical background to specify objectives to or modify the behavior of an agent. This source of supervision could be particularly useful in areas where data efficiency is required or where human priors can be helpful \citep{luketina2019survey}.

Most approaches for language-guided RL aim to learn a dense reward function $\RF_{\text{NLP}}$ indicating the partial completion of a task and use it to augment the external reward function in a way similar to reward shaping (Section \ref{sec:reward_shaping}) and auxiliary tasks (Section \ref{sec:auxiliary_tasks}):
\begin{equation}
    \RF' = \RF + \lambda \RF_{\text{NLP}}
\end{equation}
Early attempts to use language for task specification used an object-oriented definition of the environment to specify an agent's reward function \citep{macglashan2017interactive, arumugam2017accurately, williams2018learning, chevalier2018babyai}. These methods are constrained to a certain set of pre-defined concepts and object relations, limiting their scalability to the complexity of the target behavior and the training environment. \citet{goyal2019using, goyal2021pixl2r} propose a more flexible language-to-reward model which takes as input the embeddings of both a language command and a trajectory in the environment and estimates whether they are related. Others instead learn a mapping from language commands to state embeddings to specify goals through natural language \citep{kaplan2017beating, waytowich2019narration}. \citet{fu2019language} learn a language-conditional reward function from demonstrations in the aim to be reused for different tasks. \citet{sumers2021learning} propose to use sentiment inference on the provided commands to avoid the need to collect an explicit dataset. \citet{bahdanau2018learning} take inspiration from IRL and use an adversarial loss to learn a reward function that connects commands and goals.

More recently, advances in Large Language Models (LLMs) \citep{zhao2023survey} have enabled important steps towards scalable language-based reward specification. \citet{kwon2023reward} explore the idea of prompting a language model with the desired behavior and then using it to provide a numerical reward to the agent in a text-based negotiation game. \citet{klissarov2023motif} use LLMs to provide preferences over pairs of events in a captioned environment to guide the agent towards interesting states. Another approach is to use LLMs to automatically break down a language description of a high-level task into a set of low-level subgoals which can be executed by an underlying language-conditioned control algorithm \citep{huang2022language, ahn2022can}. Finally, a more end-to-end perspective consists in using LLMs and iterative prompting to translate a language description of the desired behavior into a code implementation of the reward function, which can then be optimized by any RL algorithm \citep{yu2023language, ma2023eureka}. These successes demonstrate that the vast amount of knowledge about the world distilled in LLMs has the potential to help bridge the gap between human intentions and effective reward specifications.


\chapter[PREAMBLE TO TECHNICAL CONTRIBUTIONS]{\\PREAMBLE TO TECHNICAL CONTRIBUTIONS}
\label{chap:preamble}

Chapters~\ref{chap:article1_asaf}~to~\ref{chap:article4_gcgfn} present four original contributions related to the field of reward specification in deep reinforcement learning. Respectively, we present methods that make use of inverse reinforcement learning (article~1), auxiliary tasks (article~2), and multi-objective RL (articles~3 and~4) which are intended to support policy learning both in terms of efficiency and alignment. In all of these works, I (Julien Roy) made significant contributions and was deeply involved in designing the algorithms, surveying the relevant literature, carrying out and analyzing experiments, and writing the manuscript. Here is a brief overview of each article.

\paragraph{Article 1} is titled \textit{Adversarial Soft Advantage Fitting:
Imitation Learning without Policy Optimization}. It presents improvements in adversarial imitation learning. Imitation learning takes a supervised learning approach to sequential decision making. The agent is provided with a set of demonstrations that are deemed optimal and seeks to learn a policy that emulates this behavior. We propose a novel architecture for adversarial IL which allows to significantly simplify the implementation and accelerate the training of these methods. In combination with RL, efficient IL algorithms could be used to aid exploration and further ground the behavior using a small set of demonstrations.

\paragraph{Article 2} is titled \textit{Promoting Coordination through Policy Regularization in Multi-Agent Deep Reinforcement Learning}. It focuses on the use of auxiliary tasks in RL. We investigate the case of multi-agent cooperative tasks and propose different auxiliary objectives which are optimized simultaneously with the main task to allow the agents to discover effective cooperative behaviors faster. Our approach is an example of how intuitions about the properties of effective strategies can be incorporated in the optimization process and serve as useful inductive biases for RL agents.

\paragraph{Article 3} is titled \textit{Direct Behavior Specification via Constrained Reinforcement Learning}. It presents a general framework to easily incorporate hard constraints on the agent behavior. This framework separates the main task that the agent is asked to perform from additional requirements that it should abide to, and allows the system designer to specify all of these constraints in a foreseeable and intuitive way. The resulting paradigm allows to efficiently design new tasks, maintain a clear monitoring on the ability of the agent to meet these constraints, and insure that the final behavior is aligned with the designer's intentions.

\paragraph{Article 4} is titled \textit{Goal-conditioned GFlowNets for Controllable Multi-Objective Molecular Design}. It presents an application of goal-conditioned reinforcement learning to the problem of multi-objective molecular design for computer-based drug discovery. We formulate the goals as subregions of the objective space and train a discrete generative model to target specific trade-offs on the Pareto front in order to widen its solution coverage. Our method allows for a controllable generative process that can then be leveraged to explore the molecular space in an intentional manner.

\chapter[ARTICLE 1: ADVERSARIAL SOFT ADVANTAGE FITTING:\\IMITATION LEARNING WITHOUT POLICY OPTIMISATION]{\\ARTICLE 1: ADVERSARIAL SOFT ADVANTAGE FITTING: IMITATION LEARNING WITHOUT POLICY OPTIMISATION}\label{chap:article1_asaf}

\vspace{-3mm}
\begin{center}
Co-authors\\Paul Barde, Wonseok Jeon, Joelle Pineau, Christopher Pal \& Derek Nowrouzezahrai
\end{center}
\vspace{-5mm}
\begin{center}
Published in\\Advances in Neural Information Processing Systems, December 12, 2020
\end{center}

\begin{abstract}
Adversarial Imitation Learning alternates between learning a discriminator -- which tells apart expert's demonstrations from generated ones -- and a generator's policy to produce trajectories that can fool this discriminator. This alternated optimization is known to be delicate in practice since it compounds unstable adversarial training with brittle and sample-inefficient reinforcement learning. We propose to remove the burden of the policy optimization steps by leveraging a novel discriminator formulation. Specifically, our discriminator is explicitly conditioned on two policies: the one from the previous generator's iteration and a learnable policy. When optimized, this discriminator directly learns the optimal generator's policy. Consequently, our discriminator's update solves the generator's optimization problem for free: learning a policy that imitates the expert does not require an additional optimization loop. This formulation effectively cuts by half the implementation and computational burden of Adversarial Imitation Learning algorithms by removing the Reinforcement Learning phase altogether. We show on a variety of tasks that our simpler approach is competitive to prevalent Imitation Learning methods.
\end{abstract}

\section{Introduction}

Imitation Learning (IL) treats the task of learning a policy from a set of expert demonstrations. IL is effective on control problems that are challenging for traditional Reinforcement Learning (RL) methods, either due to reward function design challenges or the inherent difficulty of the task itself \cite{abbeel2004apprenticeship,ross2011reduction}. Most IL work can be divided into two branches: Behavioral Cloning and Inverse Reinforcement Learning. Behavioral Cloning casts IL as a supervised learning objective and seeks to imitate the expert's actions using the provided demonstrations as a fixed dataset \cite{pomerleau1991efficient}. Thus, Behavioral Cloning usually requires a lot of expert data and results in agents that struggle to generalize. As an agent deviates from the demonstrated behaviors -- straying outside the state distribution on which it was trained -- the risks of making additional errors increase, a problem known as compounding error \cite{ross2011reduction}. 

Inverse Reinforcement Learning aims to reduce compounding error by learning a reward function under which the expert policy is optimal \cite{abbeel2004apprenticeship}. Once learned, an agent can be trained (with any RL algorithm) to learn how to act at any given state of the environment. Early methods were prohibitively expensive on large environments because they required training the RL agent to convergence at each learning step of the reward function \citep{ziebart2008maximum, abbeel2004apprenticeship}. Recent approaches instead apply an adversarial formulation (Adversarial Imitation Learning, AIL) in which a discriminator learns to distinguish between expert and agent behaviors to learn the reward optimized by the expert. AIL methods allow for the use of function approximators and can in practice be used with only a few policy improvement steps for each discriminator update \citep{ho2016generative, fu2017learning, finn2016connection}.

While these advances have allowed Imitation Learning to tackle bigger and more complex environments \cite{kuefler2017imitating, ding2019goal}, they have also significantly complexified the implementation and learning dynamics of Imitation Learning algorithms. It is worth asking how much of this complexity is actually mandated. For example, in recent work, \citet{reddy2019sqil} have shown that competitive performance can be obtained by hard-coding a very simple reward function to incentivize expert-like behaviors and manage to imitate it through off-policy direct RL. \citet{reddy2019sqil} therefore remove the reward learning component of AIL and focus on the RL loop, yielding a regularized version of Behavioral Cloning. 
Motivated by these results, we also seek to simplify the AIL framework but following the opposite direction: keeping the reward learning module and removing the policy improvement loop.

We propose a simpler yet competitive AIL framework. Motivated by \citet{finn2016connection} who use the optimal discriminator form, we propose a structured discriminator that estimates the probability of demonstrated and generated behavior using a single parameterized maximum entropy policy. Discriminator learning and policy learning therefore occur simultaneously, rendering seamless generator updates: once the discriminator has been trained for a few epochs, we simply use its policy model to generate new rollouts. We call this approach Adversarial Soft Advantage Fitting (ASAF).

\newpage
We make the following contributions:
\begin{itemize}
\setlength\itemsep{-1mm} 
\item \textbf{Algorithmic}: we present a novel algorithm (ASAF) designed to imitate expert demonstrations without any Reinforcement Learning step.
\item \textbf{Theoretical}: we show that our method retrieves the expert policy when trained to optimality.
\item \textbf{Empirical}: we show that ASAF outperforms prevalent IL algorithms on a variety of discrete and continuous control tasks. We also show that, in practice, ASAF can be easily modified to account for different trajectory lengths.
\end{itemize}

\section{Background}
\paragraph{Markov Decision Processes (MDPs)}
We use \citet{hazan2018provably}'s notation and consider the classic $T$-horizon $\gamma$-discounted MDP $\mathcal{M}=\langle \mathcal{S}, \mathcal{A}, \mathcal{P}, \mathcal{P}_0, \gamma, r, T \rangle$. For simplicity, we assume that $\mathcal{S}$ and $\mathcal{A}$ are finite. Successor states are given by the transition distribution $\mathcal{P}(s'|s,a)\in [0,1]$, and the initial state $s_0$ is drawn from $\mathcal{P}_0(s)\in[0,1]$. Transitions are rewarded with $r(s,a)\in\mathbb{R}$ with $r$ being bounded. The discount factor and the episode horizon are $\gamma \in [0,1]$ and $T\in \mathbb{N}\cup \{\infty\}$, where $T<\infty$ for $\gamma=1$. Finally, we consider stationary stochastic policies $\pi \in \Pi : \mathcal{S}\times\mathcal{A} \rightarrow ]0,1[$ that produce trajectories $\tau = (s_0, a_0, s_1, a_1, ..., s_{T-1}, a_{T-1}, s_T)$ when executed on $\mathcal{M}$.

The probability of trajectory $\tau$ under policy $\pi$ is 
$
    P_\pi(\tau)
    \triangleq 
    \mathcal{P}_0(s_0)\prod_{t=0}^{T-1}\pi(a_t|s_t)\mathcal{P}(s_{t+1}|s_t, a_t)
$
and the corresponding marginals are defined as 
$
    d_{t, \pi}(s)
    \triangleq 
    \sum_{\tau:s_t=s}P_\pi(\tau)
$
and  
$
    d_{t, \pi}(s, a)
    \triangleq 
    \sum_{\tau:s_t=s, a_t=a}P_\pi(\tau)
    =
    d_{t, \pi}(s)\pi(a|s)
$, respectively.
With these marginals, we define the normalized discounted  state and state-action occupancy measures as
$
    d_{\pi}(s)
    \triangleq
    \frac{1}{Z(\gamma, T)}\sum_{t=0}^{T-1}\gamma^td_{t, \pi}(s)
$
and 
$
    d_{\pi}(s, a)
    \triangleq
    \frac{1}{Z(\gamma, T)}\sum_{t=0}^{T-1}\gamma^td_{t, \pi}(s, a)
    =
    d_{\pi}(s)\pi(a|s)
$
where the partition function
$Z(\gamma, T)$ is equal to $\sum_{t=0}^{T-1}\gamma^t$.
Intuitively, the state (or state-action) occupancy measure can be interpreted as the discounted visitation distribution of the states (or state-action pairs) that the agent encounters when navigating with policy $\pi$.
The expected sum of discounted rewards can be expressed in term of the occupancy measures as follows:
\begin{equation*}
\begin{aligned}
    J_{\pi}[r(s, a)] 
    \triangleq
    \E_{\tau\sim P_\pi}\left[\textstyle\sum_{t=0}^{T-1}\gamma^t \, r(s_t, a_t)\right]
    =
    Z(\gamma, T) \, \E_{(s,a)\sim d_{\pi}}[r(s, a)]
\end{aligned}
\end{equation*}
In the entropy-regularized Reinforcement Learning framework \cite{haarnoja2018soft},
the optimal policy maximizes its entropy at each visited state in addition to the standard RL objective:

\begin{equation*}
    \pi^*
    \triangleq
    \argmax_\pi J_\pi[r(s,a)+\alpha \mathcal{H}(\pi(\cdot|s))]\,, \quad \mathcal{H}(\pi(\cdot|s)) = \E_{a\sim \pi(\cdot|s)}[-\log(\pi(a|s))]
\end{equation*}
As shown in \cite{ziebart2010modeling, haarnoja2017reinforcement} the corresponding optimal policy is:
\begin{align}
    \label{eq:max_ent_policy}
    \pi_{\mathrm{soft}}^*(a|s) =\exp \left({\alpha}^{-1} \, A^*_{\mathrm{soft}}(s,a)\right) \quad \text{with} \quad
    &A^*_{\mathrm{soft}}(s,a)
    \triangleq
    Q^*_{\mathrm{soft}}(s,a)- V^*_{\mathrm{soft}}(s)
    \\
    V^*_{\mathrm{soft}}(s)
    =
    \alpha \log\sum_{a\in\mathcal{A}}\exp \left({\alpha}^{-1} \, Q^*_{\mathrm{soft}}(s,a)\right), \,\,
    &Q^*_{\mathrm{soft}}(s,a)
    =
    r(s,a) + \gamma\mathbb{E}_{s'\sim\mathcal{P}(\cdot|s, a)}\left[ V^*_{\mathrm{soft}}(s')\right]
\end{align}

\paragraph{Maximum Causal Entropy Inverse Reinforcement Learning}
In the problem of Inverse Reinforcement Learning (IRL), it is assumed that the MDP's reward function is unknown but that demonstrations from using expert's policy $\pie$ are provided. 
Maximum causal entropy IRL \cite{ziebart2008maximum} proposes to fit a reward function $r$ from a set $\mathcal{R}$ of reward functions and retrieve the corresponding optimal policy by solving the optimization problem:
\begin{equation}
\label{eq:irl_problem}
    \min_{r\in\mathcal{R}}\left(\max_{\pi} J_\pi[r(s, a) + \mathcal{H}(\pi(\cdot|s))]  \right)-J_{\pie}[r(s, a)]
\end{equation}
In brief, the problem reduces to finding a reward function $r$ for which the expert policy is optimal. In order to do so, the optimization procedure searches high entropy policies that are optimal with respect to $r$ and minimizes the difference between their returns and the return of the expert policy, eventually reaching a policy $\pi$ that approaches $\pie$. 
Most of the proposed solutions \cite{abbeel2004apprenticeship,ziebart2010modeling,ho2016generative} transpose IRL to the problem of distribution matching; \citet{abbeel2004apprenticeship} and \citet{ziebart2008maximum} used linear function approximation and proposed to match the feature expectation; \citet{ho2016generative} proposed to cast Equation~\ref{eq:irl_problem} with a convex reward function regularizer into the problem of minimizing the Jensen-Shannon divergence between the state-action occupancy measures:
\begin{equation}
\label{eq:irl_feature_matching}
    \min_{\pi} D_{\text{JS}}(d_\pi, d_{\pie}) - J_\pi[\mathcal{H}(\pi(\cdot|s))]
\end{equation}

\paragraph{Connections between Generative Adversarial Networks (GANs) and IRL}

For the data distribution $\pe$ and the generator distribution $\pg$ defined on the domain $\mathcal{X}$, the GAN objective \cite{goodfellow2014generative} is
\begin{align}
    \label{eq:GAN_obj}
    \min_{\pg} \max_D L(D, \pg) \, , \quad L(D, \pg)
    \triangleq
    \mathbb{E}_{x\sim \pe}[\log D(x)]
    +
    \mathbb{E}_{x\sim \pg}[\log(1-D(x))]
\end{align}

In \citet{goodfellow2014generative}, the maximizer of the inner problem in Equation~\ref{eq:GAN_obj} is shown to be:

\begin{align}
\label{eq:optimal_D}
    D_{\pg}^*
    \triangleq \argmax_D L(D, \pg)
    =\frac{\pe}{\pe + \pg}
\end{align}
and the optimizer for Equation~\ref{eq:GAN_obj} is
$
     \argmin_{\pg}
     \max_D L(D, \pg)
     = \argmin_{\pg}L(D_{\pg}^*, \pg)=\pe
$. Later, \citet{finn2016connection} and \citet{ho2016generative}  concurrently proposed connections between GANs and IRL. The Generative Adversarial Imitation Learning (GAIL) formulation in \citet{ho2016generative} is based on matching state-action occupancy measures, while \citet{finn2016connection} considered matching trajectory distributions. Our work is inspired by the discriminator proposed and used by \citet{finn2016connection}:
\begin{align}
    \label{eq:irl_D}
    D_{\theta}(\tau)
    \triangleq
    \frac{
        p_\theta(\tau)
    }{
        p_\theta(\tau)+q(\tau)
    }
\end{align}

where $p_\theta(\tau)\propto \exp r_\theta(\tau)$ with reward approximator $r_\theta$ motivated by maximum causal entropy IRL. 
Note that Equation~\ref{eq:irl_D} matches the form of the optimal discriminator in Equation~\ref{eq:optimal_D}. Although \citet{finn2016connection} do not empirically support the effectiveness of their method, the Adversarial IRL approach of \citet{fu2017learning} (AIRL) successfully used a similar discriminator for state-action occupancy measure matching.

\section{Imitation Learning without Policy Optimization}

In this section, we derive Adversarial Soft Advantage Fitting (ASAF), our novel Adversarial Imitation Learning approach. Specifically, in Section~\ref{sec:asaf_theory}, we present the theoretical foundations for ASAF to perform Imitation Learning on full-length trajectories. Intuitively, our method is based on the use of such \textit{structured discriminators} -- that match the optimal discriminator form -- to fit the trajectory distribution induced by the expert policy. This approach requires being able to evaluate and sample from the learned policy and allows us to learn that policy and train the discriminator simultaneously, thus drastically simplifying the training procedure. We present in Section~\ref{sec:policy_class} parametrization options that satisfy these requirements. Finally, in Section~\ref{sec:asaf_practical_algorithm}, we explain how to implement a practical algorithm that can be used for arbitrary trajectory-lengths, including the transition-wise case.

\subsection{Adversarial Soft Advantage Fitting -- Theoretical setting}\label{sec:asaf_theory}
Before introducing our method, we derive GAN training with a structured discriminator.

\paragraph{GAN with structured discriminator}
Suppose that we have a generator distribution $\pg$ and some arbitrary distribution $\tildep$ and that both can be evaluated efficiently, e.g., categorical distribution or probability density with normalizing flows \cite{rezende2015variational}. 
We call a \textit{structured discriminator} a function $D_{\tildep, \pg}:\mathcal{X}\rightarrow [0,1]$ of the form
$
    D_{\tildep, \pg}(x)
    =
    {\tildep(x)}\big/({\tildep(x)+\pg(x)})
$
which matches the optimal discriminator form for Equation~\ref{eq:optimal_D}. Considering our new GAN objective, we get:
\begin{align}
    \label{eq:structured_GAN_obj}
    \min_{\pg}\max_{\tildep}
    L(\tildep, \pg)\, , \quad 
    L(\tildep, \pg)
    \triangleq
    \E_{x\sim \pe}
    [\log D_{\tildep, \pg}(x)]
    +
    \E_{x\sim \pg}
    [\log (1-D_{\tildep, \pg}(x))]
\end{align}

While the unstructured discriminator $D$ from Equation~\ref{eq:GAN_obj} learns a mapping from $x$ to a Bernoulli distribution, we now learn a mapping from $x$ to an arbitrary distribution $\tildep$ from which we can analytically compute $D_{\tildep, \pg}(x)$. One can therefore say that $D_{\tildep, \pg}$ is \emph{parameterized} by $\tildep$. 
For the optimization problem of Equation~\ref{eq:structured_GAN_obj}, we have the following optima:

\begin{lem}
\label{lem:optim_D_gan}
The optimal discriminator parameter for any generator $\pg$ in Equation~\ref{eq:structured_GAN_obj} is equal to the expert's distribution,  
$\tildep^* 
    \triangleq \argmax_{\tildep} L(\tildep, \pg) 
    = \pe$
, and the optimal discriminator parameter is also the optimal generator, i.e.,
$$
    \pg^* 
    \triangleq \argmin_{\pg} \max_{\tildep} L(\tildep, \pg)
    = \argmin_{\pg} L(\pe, \pg)
    = \pe = \tildep^*
$$
\end{lem}
\vspace{-15pt}
\proof{See Appendix~\ref{app:ASAF:proof_of_optim_D_gan}}

Intuitively, Lemma~\ref{lem:optim_D_gan} shows that the optimal discriminator parameter is also the target data distribution of our optimization problem (i.e., the optimal generator). In other words, solving the inner optimization yields the solution of the outer optimization. In practice, we update $\tildep$ to minimize the discriminator objective and use it directly as $\pg$ to sample new data.

\paragraph{Matching trajectory distributions with structured discriminator}
\label{subsec:asaf_discriminator_theory}
Motivated by the GAN with structured discriminator, we consider the trajectory distribution matching problem in IL. Here, we optimise Equation~\ref{eq:structured_GAN_obj} with 
$
x=\tau, 
\mathcal{X}=\gT, 
\pe=P_{\pie},
\pg=P_{\pig},
$ which yields the following objective:
\begin{align}
    \label{eq:TASAF_obj}
    \min_{\pig}\max_{\tildepi}
    L(\tildepi, \pig)
    \,,\quad L(\tildepi, \pig)
    \triangleq
    \E_{\tau\sim P_{\pie}}
    [\log D_{\tildepi, \pig}(\tau)]
    +
    \E_{\tau\sim P_{\pig}}
    [\log (1-D_{\tildepi, \pig}(\tau))],
\end{align}
with the structured discriminator:
\begin{align}
    \label{eq:simplified_trajectory_optimal_discriminator}
    D_{\tildepi, \pig}(\tau)
    = \frac{P_{\tildepi}(\tau)}{P_{\tildepi}(\tau)+P_{\pig}(\tau)}
    = \frac{q_{\tildepi}(\tau)}{q_{\tildepi}(\tau)+q_{\pig}(\tau)}
\end{align}

Here we used the fact that $P_{\pi}(\tau)$ decomposes into two distinct products: $q_\pi(\tau)\triangleq\prod_{t=0}^{T-1}\pi(a_t|s_t)$ which depends on the stationary policy $\pi$ and $\xi(\tau)\triangleq\mathcal{P}_0(s_0)\prod_{t=0}^{T-1}\mathcal{P}(s_{t+1}|s_t, a_t)$ which accounts for the environment dynamics. Crucially, $\xi(\tau)$ cancels out in the numerator and denominator leaving $\tildepi$ as the sole parameter of this structured discriminator. In this way, $D_{\tildepi, \pig}(\tau)$ can evaluate the probability of a trajectory being generated by the expert policy simply by evaluating products of stationary policy distributions $\tildepi$ and $\pig$. With this form, we can get the following result:
\begin{thm}
\label{thm:traj_GAN}
The optimal discriminator parameter for any generator policy $\pig$ in Equation~\ref{eq:TASAF_obj} $\tildepi^*
\triangleq
\argmax_{\tildepi}L(\tildepi, \pig)
$ is such that
$
q_{\tildepi^*}
=
q_{\pie}
$, and using generator policy $\tildepi^*$ minimizes $L(\tildepi^*, \pig)$, i.e.,
$$
\tildepi^*
\in
\argmin_{\pig}\max_{\tildepi}L(\tildepi, \pig)
=
\argmin_{\pig}L(\tildepi^*, \pig)
$$
\end{thm}
\vspace{-15pt}
\proof{See Appendix~\ref{app:ASAF:proof_of_thm_traj_GAN}}

Theorem~\ref{thm:traj_GAN}'s benefits are similar to the ones from Lemma~\ref{lem:optim_D_gan}: we can use a discriminator of the form of Equation~\ref{eq:simplified_trajectory_optimal_discriminator} to fit to the expert demonstrations a policy $\tildepi^*$ that simultaneously yields the optimal generator's policy and produces the same trajectory distribution as the expert policy.

\subsection{A Specific Policy Class}
\label{sec:policy_class}
The derivations of Section~\ref{sec:asaf_theory} rely on the use of a learnable policy that can both be evaluated and sampled from in order to fit the expert policy. A number of parameterization options that satisfy these conditions are available.

First of all, we observe that since $\pie$ is independent of $r$ and $\pi$, we can add the entropy of the expert policy $\mathcal{H}(\pie(\cdot|s))$ to the MaxEnt IRL objective of Eq.~(\ref{eq:irl_problem}) without modifying the solution to the optimization problem:
\begin{equation}
    \min_{r\in\mathcal{R}}
    \left(\max_{\pi \in \Pi} J_{\pi}[r(s,a)+\mathcal{H}(\pi(\cdot|s))]
    \right)
    - J_{\pie}[r(s,a)+\mathcal{H}(\pie(\cdot|s))]
\end{equation}
The max over policies implies that when optimising $r$, $\pi$ has already been made optimal with respect to the causal entropy augmented reward function $r'(s,a| \pi) = r(s,a) + \mathcal{H}(\pi(\cdot|s))$ and therefore it must be of the form presented in Eq.~(\ref{eq:max_ent_policy}). Moreover, since $\pi$ is optimal w.r.t. $r'$ the difference in performance $J_{\pi}[r'(s,a| \pi)]-J_{\pie}[r'(s,a|\pie)]$ is always non-negative and its minimum of 0 is only reached when $\pie$ is also optimal w.r.t. $r'$, in which case $\pie$ must also be of the form of Eq.~(\ref{eq:max_ent_policy}). 

With discrete action spaces we propose to parameterize the MaxEnt policy defined in Equation~\ref{eq:max_ent_policy} with the following categorical distribution:
\begin{equation}
   \tildepi(a|s) = \exp\left(Q_\theta(s,a) - \log\sum_{a'}\exp Q_\theta(s,a') \right) 
\end{equation}
where $Q_\theta$ is a model parameterized by $\theta$ that approximates $\frac{1}{\alpha} Q^*_{\text{soft}}$.

With continuous action spaces, the soft value function involves an intractable integral over the action domain. Therefore, we approximate the MaxEnt distribution with a Normal distribution with diagonal covariance matrix like it is commonly done in the literature \cite{ haarnoja2018soft,nachum2018trustpcl}. By parameterizing the mean and variance we get a learnable density function that can be easily evaluated and sampled from. 

\subsection{Adversarial Soft Advantage Fitting (ASAF) -- practical algorithm}
\label{sec:asaf_practical_algorithm}
Section~\ref{subsec:asaf_discriminator_theory} shows that assuming $\tildepi$ can be evaluated and sampled from, we can use the structured discriminator of Equation~\ref{eq:simplified_trajectory_optimal_discriminator} to learn a policy $\tildepi$ that matches the expert's trajectory distribution. Section~\ref{sec:policy_class} proposes parameterizations for discrete and continuous action spaces that satisfy those assumptions. In practice, as with GANs \cite{goodfellow2014generative}, we do not train the discriminator to convergence as gradient-based optimisation cannot be expected to find the global optimum of non-convex problems. Instead, Adversarial Soft Advantage Fitting (ASAF) alternates between two simple steps: (1) training $D_{\tildepi, \pig}$ by minimizing the binary cross-entropy loss,
\begin{equation}
\label{eq:BCE_loss}
\begin{aligned}
    &\mathcal{L}_{BCE}(\mathcal{D}_E, \mathcal{D}_G, \tildepi) \approx -\frac{1}{n_E} \sum_{i=1}^{n_E} \log D_{\tildepi, \pig}(\tau_i^{(E)}) - \frac{1}{n_G} \sum_{i=1}^{n_G} \log \left(1 - D_{\tildepi, \pig}(\tau_i^{(G)})\right) \\
    &\text{where  }\quad \tau_i^{(E)}\sim \mathcal{D}_E \text{ ,  } \tau_i^{(G)}\sim \mathcal{D}_G \, \text{ and  } \, D_{\tildepi, \pig}(\tau) = \frac{\prod_{t=0}^{T-1}\tildepi(a_t|s_t)}{\prod_{t=0}^{T-1}\tildepi(a_t|s_t)+\prod_{t=0}^{T-1}\pig(a_t|s_t)}
\end{aligned}
\end{equation}
with minibatch sizes $n_E = n_G$, and (2) updating the generator's policy as $\pig \leftarrow \tildepi$ to minimize Equation~\ref{eq:TASAF_obj} (see Algorithm~\ref{alg:asaf}).

We derived ASAF considering full trajectories, yet it might be preferable in practice to split full trajectories into smaller chunks. This is particularly true in environments where trajectory length varies a lot or tends to infinity. 
 
To investigate whether the practical benefits of using partial trajectories hurt ASAF's performance, we also consider a variation, ASAF-\textit{w}, where we treat trajectory-windows of size \textit{w} as if they were full trajectories. Note that considering windows as full trajectories results in approximating that the initial state of these sub-trajectories have equal probability under the expert's and the generator's policy (this is easily seen when deriving Equation~\ref{eq:simplified_trajectory_optimal_discriminator}).

\begin{tabular}{@{}p{0.44\textwidth}p{0.52\textwidth}@{}}
In the limit, ASAF-1 (window-size of 1) becomes a transition-wise algorithm which can be desirable if one wants to collect rollouts asynchronously or has only access to unsequential expert data. While ASAF-1 may work well in practice it essentially assumes that the expert's and the generator's policies have the same state occupancy measure, which is incorrect until actually recovering the true expert policy.
&
\vspace{-0.25cm}
\begin{algorithm}[H]
    \begin{algorithmic}
        \caption{\label{alg:asaf}ASAF}
        \REQUIRE expert trajectories $\mathcal{D}_E = \{\tau_i\}_{i=1}^{N_E}$
        \STATE Randomly initialize $\tildepi$ and set $\pig \leftarrow \tildepi$
        \FOR{steps $m=0$ to $M$}
            \STATE Collect trajectories $\mathcal{D}_G = \{\tau_i\}_{i=1}^{N_G}$ using $\pig$
            \STATE Update $\tildepi$ by minimizing Equation~\ref{eq:BCE_loss}
            \STATE Set $\pig \leftarrow \tildepi$
        \ENDFOR
    \end{algorithmic}
\end{algorithm}
\end{tabular}

Finally, to offer a complete family of algorithms based on the structured discriminator approach, we show in Appendix~\ref{app:ASAF:ASQF} that this assumption is not mandatory and derive a transition-wise algorithm based on Soft Q-function Fitting (rather than soft advantages) that also gets rid of the RL loop. We call this algorithm ASQF. While theoretically sound, we found that in practice, ASQF is outperformed by ASAF-1 in more complex environments (see Section~\ref{sec:results_and_discussion}).

\section{Related works}
\citet{ziebart2008maximum} first proposed MaxEnt IRL, the foundation of modern IL. \citet{ziebart2010modeling} further elaborated MaxEnt IRL as well as deriving the optimal form of the MaxEnt policy at the core of our methods. \citet{finn2016connection} proposed a GAN formulation to IRL that leveraged the energy based models of \citet{ziebart2010modeling}. \citet{finn2016guided}'s implementation of this method, however, relied on processing full trajectories with Linear Quadratic Regulator and on optimizing with guided policy search, to manage the high variance of trajectory costs. To retrieve robust rewards, \citet{fu2017learning} proposed a straightforward transposition of \cite{finn2016connection} to state-action transitions. In doing so, they had to however do away with a GAN objective during policy optimization, consequently minimizing the Kullback–Leibler divergence from the expert occupancy measure to the policy occupancy measure (instead of the Jensen-Shannon divergence) \cite{ghasemipour2019divergence}. 

Later works \cite{sasaki2018sample, Kostrikov2020Imitation} move away from the Generative Adversarial formulation. To do so, \citet{sasaki2018sample} directly express the expectation of the Jensen-Shannon divergence between the occupancy measures in term of the agent's Q-function, which can then be used to optimize the agent's policy with off-policy Actor-Critic \cite{degris2012off}. Similarly, \citet{Kostrikov2020Imitation} use Dual Stationary Distribution Correction Estimation \cite{nachum2019dualdice} to approximate the Q-function on the expert's demonstrations before optimizing the agent's policy under the initial state distribution using the reparametrization trick \cite{haarnoja2018soft}. While \cite{sasaki2018sample,Kostrikov2020Imitation} are related to our methods in their interests in learning directly the value function, they differ in their goal and thus in the resulting algorithmic complexity. Indeed, they aim at improving the sample efficiency in terms of environment interaction and therefore move away from the algorithmically simple Generative Adversarial formulation towards more complicated divergence minimization methods. In doing so, they further complicate the Imitation Learning methods while still requiring to explicitly learn a policy. Yet, simply using the Generative Adversarial formulation with an Experience Replay Buffer can significantly improve the sample efficiency \cite{kostrikov2018discriminatoractorcritic}. 
For these reasons, and since our aim is to propose efficient yet simple methods, we focus on the Generative Adversarial formulation.

While \citet{reddy2019sqil} share our interest for simpler IL methods, they pursue an opposite approach to ours. They propose to eliminate the reward learning steps of IRL by simply hard-coding a reward of 1 for expert's transitions and of 0 for agent's transitions. They then use Soft Q-learning \cite{haarnoja2017reinforcement} to learn a value function by sampling transitions in equal proportion from the expert's and agent's buffers. Unfortunately, once the learner accurately mimics the expert, it collects expert-like transitions that are labeled with a reward of 0 since they are generated and not coming from the demonstrations. This effectively causes the reward of expert-like behavior to decay as the agent improves and can severely destabilize learning to a point where early-stopping becomes required \cite{reddy2019sqil}.

Our work builds on \cite{finn2016connection}, yet its novelty is to explicitly express the probability of a trajectory in terms of the policy in order to directly learn this latter when training the discriminator. In contrast, \cite{fu2017learning} considers a transition-wise discriminator with un-normalized probabilities which makes it closer to ASQF (Appendix~\ref{app:ASAF:ASQF}) than to ASAF-1. Additionally, AIRL \cite{fu2017learning} minimizes the Kullback-Leiber Divergence \cite{ghasemipour2019divergence} between occupancy measures whereas ASAF minimizes the Jensen-Shanon Divergence between trajectory distributions.

Finally, Behavioral Cloning uses the loss function from supervised learning (classification or regression) to match expert's actions given expert's states and suffers from compounding error due to co-variate shift \cite{ross2010efficient} since its data is limited to the demonstrated state-action pairs without environment interaction. Contrarily, ASAF-1 uses the binary cross entropy loss in Equation~\ref{eq:BCE_loss} and does not suffer from compounding error as it learns on both generated and expert's trajectories.

\section{Results and discussion}

We evaluate our methods on a variety of discrete and continuous control tasks. Our results show that, in addition to drastically simplifying the adversarial IRL framework, our methods perform on par or better than previous approaches on all but one environment. 
When trajectory length is really long or drastically varies across episodes (see MuJoCo experiments Section~\ref{sec:mujoco_results}), we find that using sub-trajectories with fixed window-size (ASAF-\textit{w} or ASAF-1) significantly outperforms its full trajectory counterpart ASAF.

\subsection{Experimental setup}
\label{sec:results_and_discussion}
We compare our algorithms ASAF, ASAF-\textit{w} and ASAF-1 against GAIL \cite{ho2016generative}, the predominant Adversarial Imitation Learning algorithm in the literature, and AIRL \cite{fu2017learning}, one of its variations that also leverages the access to the generator's policy distribution. Additionally, we compare against SQIL \cite{reddy2019sqil}, a recent Reinforcement Learning-only approach to Imitation Learning that proved successful on high-dimensional tasks. Our implementations of GAIL and AIRL use PPO \cite{schulman2017proximal} instead of TRPO \cite{schulman2015trust} as it has been shown to improve performance \cite{kostrikov2018discriminatoractorcritic}. Finally, to be consistent with \cite{ho2016generative}, we do not use causal entropy regularization.  

For all tasks except MuJoCo, we selected the best performing hyperparameters through a random search of equal budget for each algorithm-environment pair (see Appendix~\ref{app:ASAF:hyperparameters}) and the best configuration is retrained on ten random seeds. For the MuJoCo experiments, GAIL required extensive tuning (through random searches) of both its RL and IRL components to achieve satisfactory performances. Our methods, ASAF-\textit{w} and ASAF-1, on the other hand showed much more stable and robust to hyperparameterization, which is likely due to their simplicity.  
SQIL used the same SAC\cite{haarnoja2018soft} implementation  and hyperparameters that were used to generate the expert demonstrations.

Finally for each task, all algorithms use the same neural network architectures for their policy and/or discriminator (see full description in Appendix~\ref{app:ASAF:hyperparameters}). 
Expert demonstrations are either generated by hand (mountaincar), using open-source bots (Pommerman) or from our implementations of SAC and PPO (all remaining). More details are given in Appendix~\ref{app:ASAF:environments}.

\begin{figure}[t]
    \centering
    \includegraphics[width=1.\textwidth]{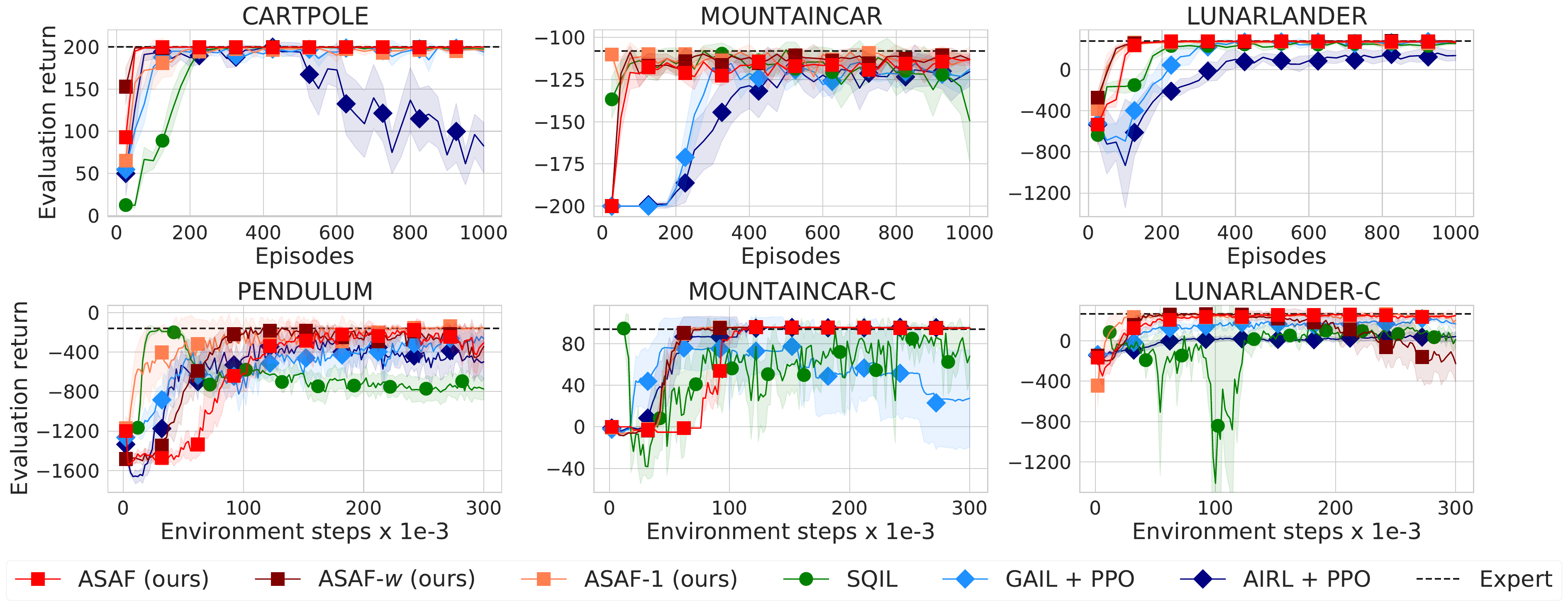}
    \caption[Results on classic control and Box2D tasks]{Results on classic control and Box2D tasks for 10 expert demonstrations. First row contains discrete actions environments, second row corresponds to continuous control.}
    \label{fig:results_toy}
\end{figure}

\subsection{Experiments on classic control and Box2D tasks (discrete and continuous)}
\label{sec:classic_control_results}

Figure~\ref{fig:results_toy} shows that ASAF and its approximate variations ASAF-1 and ASAF-\textit{w} quickly converge to expert's performance (here \textit{w} was tuned to values between 32 to 200, see Appendix~\ref{app:ASAF:hyperparameters} for selected window-sizes). This indicates that the practical benefits of using shorter trajectories or even just transitions does not hinder performance on these simple tasks. Note that for Box2D and classic control environments, we retrain the best configuration of each algorithm for twice as long than was done in the hyperparameter search, which allows to uncover unstable learning behaviors. Figure~\ref{fig:results_toy} shows that our methods display much more stable learning: their performance rises until they match the expert's and does not decrease once it is reached. This is a highly desirable property for an Imitation Learning algorithm since in practice one does not have access to a reward function and thus cannot monitor the performance of the learning algorithm to trigger early-stopping. The baselines on the other hand experience occasional performance drops. For GAIL and AIRL, this is likely due to the concurrent RL and IRL loops, whereas for SQIL, it has been noted that an effective reward decay can occur when accurately mimicking the expert \cite{reddy2019sqil}. This instability is particularly severe in the continuous control case. 
In practice, all three baselines use early stopping to avoid performance decay \cite{reddy2019sqil}.

\subsection{Experiments on MuJoCo (continuous control)}
\label{sec:mujoco_results}

\begin{figure}[t]
    \centering
    \includegraphics[width=1.\textwidth]{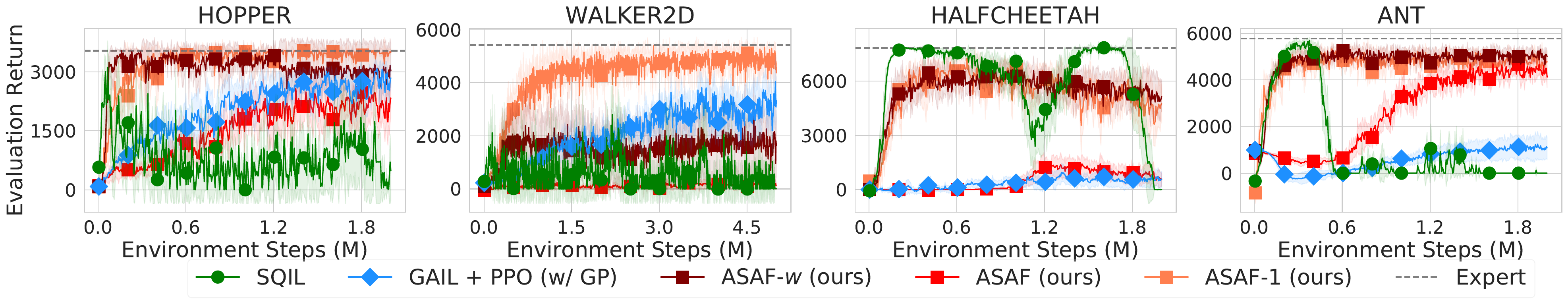}
    \caption[Results on MuJoCo tasks]{Results on MuJoCo tasks for 25 expert demonstrations.}
    \label{fig:mujoco_results}
\end{figure}

To scale up our evaluations in continuous control we use the popular MuJoCo benchmarks. In this domain, the trajectory length is either fixed at a large value (1000 steps on HalfCheetah) or varies a lot across episodes due to termination when the character falls down (Hopper, Walker2d and Ant). Figure~\ref{fig:mujoco_results} shows that these trajectory characteristics hinder ASAF's learning as ASAF requires collecting multiple episodes for every update, while ASAF-1 and ASAF-\textit{w} perform well and are more sample-efficient than ASAF in these scenarios.
We focus on GAIL since \cite{fu2017learning} claim that AIRL performs on par with it on MuJoCo environments. In Figure~\ref{fig:gail_gradient_penalty} in Appendix~\ref{app:ASAF:additional_experiments} we evaluate GAIL both with and without gradient penalty (GP) on discriminator updates \cite{gulrajani2017improved, kostrikov2018discriminatoractorcritic} and while GAIL was originally proposed without GP \cite{ho2016generative},
we empirically found that GP prevents the discriminator to overfit and enables RL to exploit dense rewards, which highly improves its sample efficiency. Despite these ameliorations, GAIL proved to be quite inconsistent across environments despite substantial efforts on hyperparameter tuning. On the other hand, ASAF-1 performs well across all environments. Finally, we see that SQIL's instability is exacerbated on MuJoCo.

\subsection{Experiments on Pommerman (discrete control)}
\label{sec:pommerman_results}

\begin{figure}[t]
    \centering
    \begin{tabular}{ccc}
        \includegraphics[width=.2\textwidth]{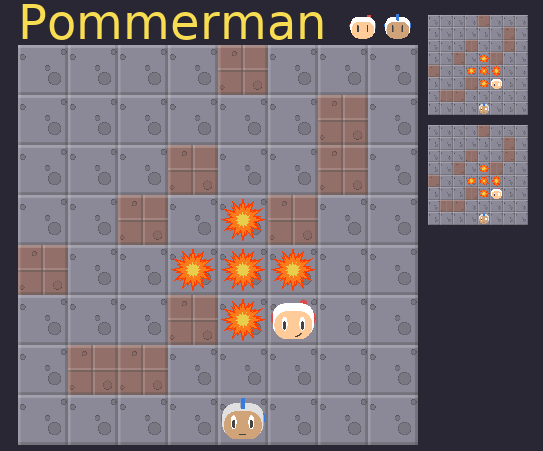} &
        \includegraphics[width=.45\textwidth]{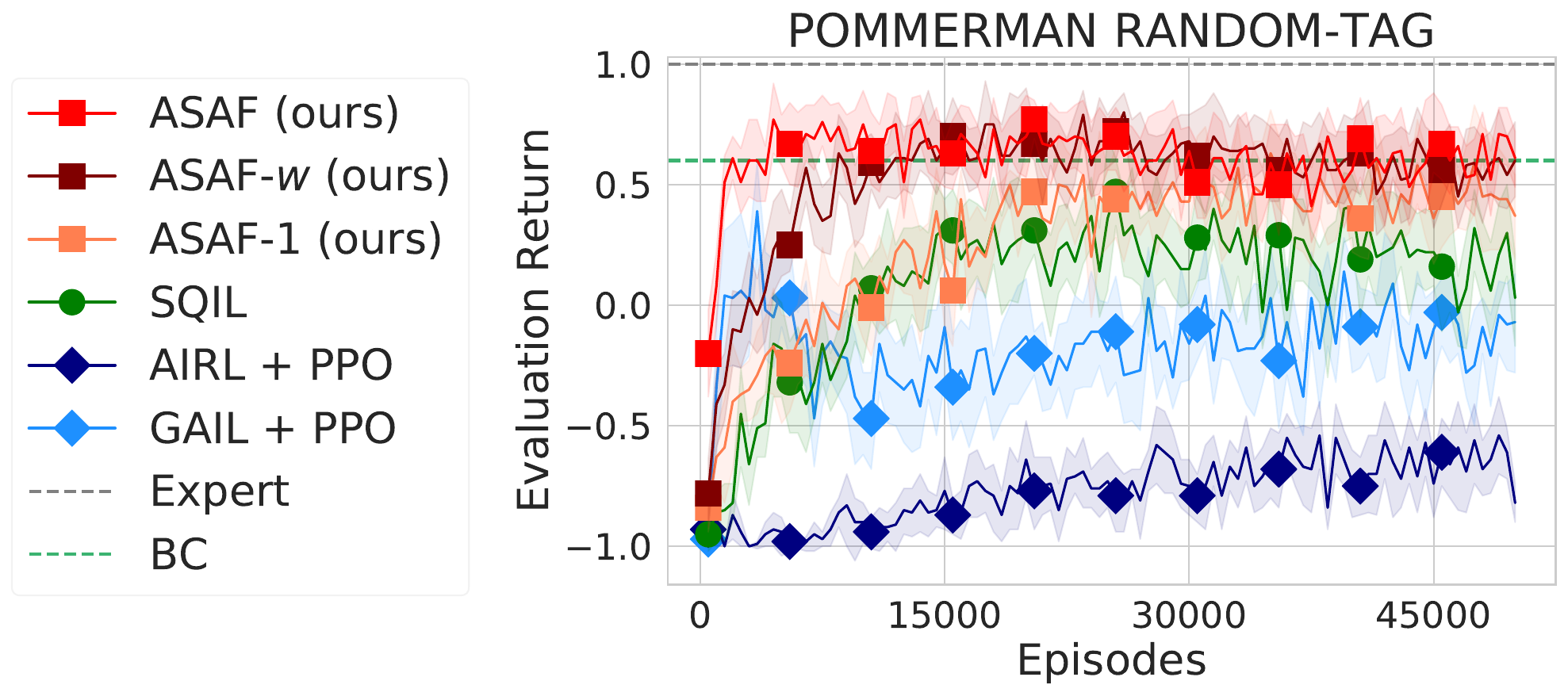} &
        \includegraphics[width=.26\textwidth]{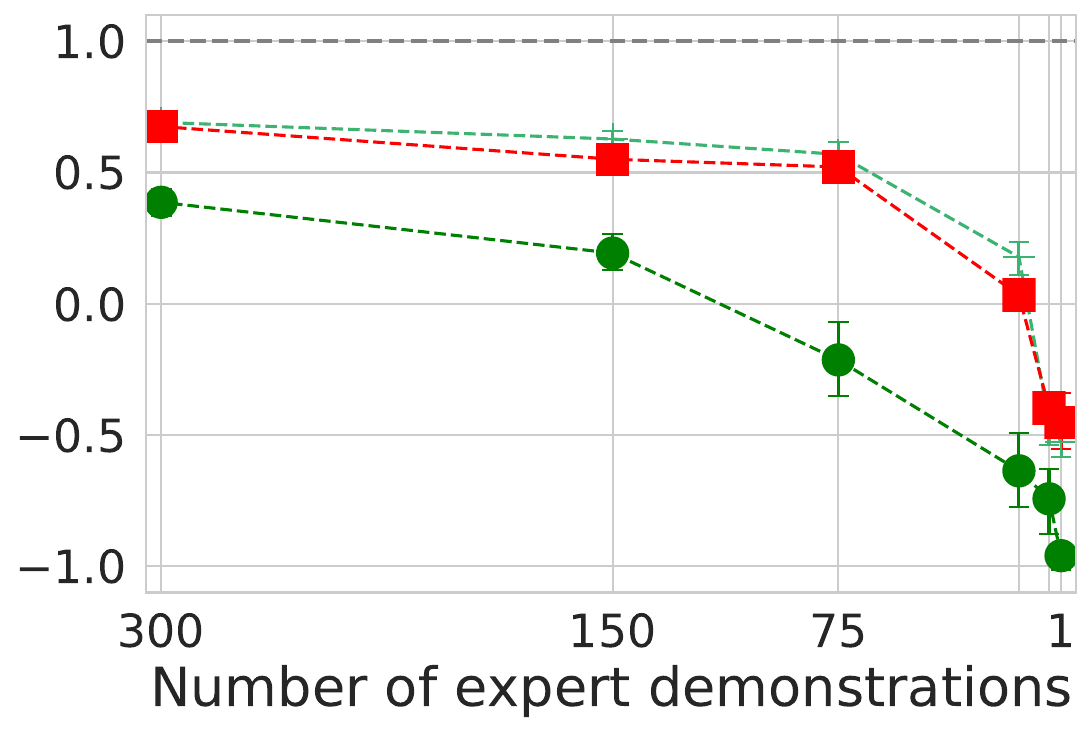} 
    \end{tabular}
    \caption[Results on Pommerman Random-Tag]{Results on Pommerman Random-Tag: (Left) Snapshot of the environment. (Center) Learning measured as evaluation return over episodes for 150 expert trajectories  (Right) Average return on last 20\% of training for decreasing number of expert trajectories [300, 150, 75, 15, 5, 1].}
    \label{fig:pommerman_results_randomTag}
\end{figure}

Finally, to scale up our evaluations in discrete control environments, we consider the domain of Pommerman \cite{resnick2018pommerman}, a challenging and very dynamic discrete control environment that uses rich and high-dimensional observation spaces (see Appendix~\ref{app:ASAF:environments}). We perform evaluations of all of our methods and baselines on a 1 vs 1 task where a learning agent plays against a random agent, the opponent. The goal for the learning agent is to navigate to the opponent and eliminate it using expert demonstrations provided by the champion algorithm of the FFA 2018 competition \cite{zhou2018hybrid}. We removed the ability of the opponent to lay bombs so that it doesn't accidentally eliminate itself. Since it can still move around, it is however surprisingly tricky to eliminate: the expert has to navigate across the whole map, lay a bomb next to the opponent and retreat to avoid eliminating itself. This entire routine has then to be repeated several times until finally succeeding since the opponent will often avoid the hit by chance. We refer to this task as \textit{Pommerman Random-Tag}. Note that since we measure success of the imitation task with the win-tie-lose outcome (sparse performance metric), a learning agent has to truly reproduce the expert behavior until the very end of trajectories to achieve higher scores. Figure~\ref{fig:pommerman_results_randomTag} shows that all three variations of ASAF as well as Behavioral Cloning (BC) outperform the baselines.

\section{Conclusion}

We propose an important simplification to the Adversarial Imitation Learning framework by removing the Reinforcement Learning optimisation loop altogether. We show that, by using a particular form for the discriminator, our method recovers a policy that matches the expert's trajectory distribution. We evaluate our approach against prior works on many different benchmarking tasks and show that our method (ASAF) compares favorably to the predominant Imitation Learning algorithms. The approximate versions, ASAF-\textit{w} and ASAF-1, that use sub-trajectories yield a flexible algorithms that work well both on short and long time horizons. Finally, our approach still involves a reward learning module through its discriminator, and it would be interesting in future work to explore how ASAF can be used to learn robust rewards, along the lines of \citet{fu2017learning}.

\section*{Broader Impact}
Our contributions are mainly theoretical and aim at simplifying current Imitation Learning methods. We do not propose new applications nor use sensitive data or simulator. Yet our method can ease and promote the use, design and development of Imitation Learning algorithms and may eventually lead to applications outside of simple and controlled simulators. We do not pretend to discuss the ethical implications of the general use of autonomous agents but we rather try to investigate what are some of the differences in using Imitation Learning rather than reward oriented methods in the design of such agents.

Using only a scalar reward function to specify the desired behavior of an autonomous agent is a challenging task as one must weight different desiderata and account for unsuspected behaviors and situations. Indeed, it is well known in practice that Reinforcement Learning agents tend to find bizarre ways of exploiting the reward signal without solving the desired task. The fact that it is difficult to specify and control the behavior of an RL agents is a major flaw that prevent current methods to be applied to risk sensitive situations. On the other hand, Imitation Learning proposes a more natural way of specifying nuanced preferences by demonstrating desirable ways of solving a task. Yet, IL also has its drawbacks. First of all one needs to be able to demonstrate the desired behavior and current methods tend to be only as good as the demonstrator. Second, it is a challenging problem to ensure that the agent will be able to adapt to new situations that do not resemble the demonstrations. For these reasons, it is clear for us that additional safeguards are required in order to apply Imitation Learning (and Reinforcement Learning) methods to any application that could effectively have a real world impact. 

\section*{Acknowledgments}
We thank Eloi Alonso, Olivier Delalleau, Félix G. Harvey, Maxim Peter and the entire research team at Ubisoft Montreal's La Forge R\&D laboratory. Their feedback and comments contributed significantly to this work.
Christopher Pal and Derek Nowrouzezahrai acknowledge funding from the
Fonds de Recherche Nature et
Technologies (FRQNT), Ubisoft Montreal and Mitacs’ Accelerate Program in
support of our work,
as well as Compute Canada for providing computing resources. Derek and Paul also acknowledge support from the NSERC Industrial Research Chair program.

\chapter[ARTICLE 2: PROMOTING COORDINATION THROUGH POLICY REGULARIZATION IN MULTI-AGENT DEEP REINFORCEMENT LEARNING]{\\ARTICLE 2: PROMOTING COORDINATION\\THROUGH POLICY REGULARIZATION IN MULTI-AGENT DEEP REINFORCEMENT LEARNING}\label{chap:article2_cmaddpg}

\vspace{-3mm}
\begin{center}
Co-authors\\Paul Barde, Félix Harvey, Derek Nowrouzezahrai \& Christopher Pal
\end{center}
\vspace{-5mm}
\begin{center}
Published in\\Advances in Neural Information Processing Systems, December 12, 2020
\end{center}

\begin{abstract}
In multi-agent reinforcement learning, discovering successful collective behaviors is challenging as it requires exploring a joint action space that grows exponentially with the number of agents. While the tractability of independent agent-wise exploration is appealing, this approach fails on tasks that require elaborate group strategies. We argue that coordinating the agents' policies can guide their exploration and we investigate techniques to promote such an inductive bias. We propose two policy regularization methods: TeamReg, which is based on inter-agent action predictability and CoachReg that relies on synchronized behavior selection. We evaluate each approach on four challenging continuous control tasks with sparse rewards that require varying levels of coordination as well as on the discrete action Google Research Football environment. Our experiments show improved performance across many cooperative multi-agent problems. Finally, we analyze the effects of our proposed methods on the policies that our agents learn and show that our methods successfully enforce the qualities that we propose as proxies for coordinated behaviors.
\end{abstract}

\section{Introduction}
Multi-Agent Reinforcement Learning (MARL) refers to the task of training an agent to maximize its expected return by interacting with an environment that contains other learning agents. It represents a challenging branch of Reinforcement Learning (RL) with interesting developments in recent years \citep{hernandez2018multiagent}. A popular framework for MARL is the use of a Centralized Training and a Decentralized Execution (CTDE) procedure \citep{lowe2017multi, foerster2018counterfactual, iqbal2018actor, foerster2018bayesian, rashid2018qmix}. Typically, one leverages centralized critics to approximate the value function of the aggregated observations-actions pairs and train actors restricted to the observation of a single agent. Such critics, if exposed to coordinated joint actions leading to high returns, can steer the agents' policies toward these highly rewarding behaviors. However, these approaches depend on the agents luckily stumbling on these collective actions in order to grasp their benefit. Thus, it might fail in scenarios where such behaviors are unlikely to occur by chance. We hypothesize that in such scenarios, coordination-promoting inductive biases on the policy search could help discover successful behaviors more efficiently and supersede task-specific reward shaping and curriculum learning. To motivate this proposition we present a simple Markov Game in which agents forced to coordinate their actions learn remarkably faster. For more realistic tasks in which coordinated strategies cannot be easily engineered and must be learned, we propose to transpose this insight by relying on two coordination proxies to bias the policy search. The first avenue, TeamReg, assumes that an agent must be able to predict the behavior of its teammates in order to coordinate with them. The second, CoachReg, supposes that coordinated agents collectively recognize different situations and synchronously switch to different sub-policies to react to them.\footnote{Source code for the algorithms and environments will be made public upon publication of this work. Visualisations of CoachReg are available here: \url{https://sites.google.com/view/marl-coordination/}}.

Our contributions are threefold. First, we show that coordination can crucially accelerate multi-agent learning for cooperative tasks. Second, we propose two novel approaches that aim at promoting such coordination by augmenting CTDE MARL algorithms through additional multi-agent objectives that act as policy regularizers and are optimized jointly with the main return-maximization objective. Third, we design two new sparse-reward cooperative tasks in the multi-agent particle environment \citep{mordatch2018emergence}. We use them along with two standard multi-agent tasks to present a detailed evaluation of our approaches' benefits when they extend the reference CTDE MARL algorithm MADDPG \cite{lowe2017multi}. We validate our methods' key components by performing an ablation study and a detailed analysis of their effect on agents' behaviors. Finally, we verify that these benefits hold on the more complex, discrete action, Google Research Football environment \cite{kurach2019google}.

Our experiments suggest that our TeamReg objective provides a dense learning signal that can help guiding the policy towards coordination in the absence of external reward, eventually leading it to the discovery of higher performing team strategies in a number of cooperative tasks. However we also find that TeamReg does not lead to improvements in every single case and can even be harmful in environments with an adversarial component. For CoachReg, we find that enforcing synchronous sub-policy selection enables the agents to concurrently learn to react to different agreed upon situations and consistently yields significant improvements on the overall performance.

\section{Background}
\subsection{Markov Games}\label{section:markov_games}
We consider the framework of Markov Games \citep{littman1994markov}, a multi-agent extension of Markov Decision Processes (MDPs). A Markov Game of $N$ agents is defined by the tuple $\langle \mathcal{S}, \mathcal{T}, \mathcal{P}, \{ \mathcal{O}^i, \mathcal{A}^i, \mathcal{R}^i \}_{i=1}^N \rangle$ where $\mathcal{S}$, $\mathcal{T}$, and $\mathcal{P}$ are respectively the set of all possible states, the transition function and the initial state distribution. While these are global properties of the environment, $\mathcal{O}^i$, $\mathcal{A}^i$ and $\mathcal{R}^i$ are individually defined for each agent $i$. They are respectively the observation functions, the sets of all possible actions and the reward functions. At each time-step $t$, the global state of the environment is given by $s_t \in \mathcal{S}$ and every agent's individual action vector is denoted by $a_t^i \in \mathcal{A}^i$. To select their action, each agent $i$ only has access to its own observation vector $o_t^i$ which is extracted by the observation function $\mathcal{O}^i$ from the global state $s_t$. The initial state $s_0$ is sampled from the initial state distribution $\mathcal{P}: \mathcal{S} \rightarrow [0, 1]$ and the next state $s_{t+1}$ is sampled from the probability distribution over the possible next states given by the transition function $\mathcal{T}: \mathcal{S} \times \mathcal{S} \times \mathcal{A}^1 \times ... \times \mathcal{A}^N \rightarrow [0,1]$. Finally, at each time-step, each agent receives an individual scalar reward $r_t^i$  from its reward function $\mathcal{R}^i: \mathcal{S} \times \mathcal{S} \times \mathcal{A}^1 \times ... \times \mathcal{A}^N \rightarrow \mathbb{R}$. Agents aim at maximizing their expected discounted return $\mathbb{E}\left[\sum_{t=0}^T \gamma^t r_t^i\right]$ over the time horizon $T$, where $\gamma \in [0, 1]$ is a discount factor.
\subsection{Multi-Agent Deep Deterministic Policy Gradient}
MADDPG~\citep{lowe2017multi} is an adaptation of the Deep Deterministic Policy Gradient algorithm~\citep{lillicrap2015continuous} to the multi-agent setting. It allows the training of cooperating and competing decentralized policies through the use of a centralized training procedure. In this framework, each agent $i$ possesses its own deterministic policy $\mu^i$ for action selection and critic $Q^i$ for state-action value estimation, which are respectively parametrized by $\theta^i$ and $\phi^i$. All parametric models are trained off-policy from previous transitions $\zeta_t \coloneqq (\mathbf{o}_t, \mathbf{a}_t, \mathbf{r}_t, \mathbf{o}_{t+1})$ uniformly sampled from a replay buffer $\mathcal{D}$. Note that $\mathbf{o}_t \coloneqq [o_t^1, ..., o_t^N]$ is the joint observation vector and $\mathbf{a}_t \coloneqq [a_t^1, ..., a_t^N]$ is the joint action vector, obtained by concatenating the individual observation vectors $o_t^i$ and action vectors $a_t^i$ of all $N$ agents. Each centralized critic is trained to estimate the expected return for a particular agent $i$ from the Q-learning loss~\citep{watkins1992q}:
\begin{equation}
\label{eq:Lcritic}
\begin{split}
\mathcal{L}^i(\mathbf{\phi}^i) 
&= \mathbb{E}_{\zeta_t \sim \mathcal{D}} 
\left[ 
    \frac{1}{2} \left(Q^i (\mathbf{o}_t, \mathbf{a}_t; \phi^i)
    - y^i_t \right)^2\right]
\\
y^i_t &= r_t^i + \gamma Q^i (\mathbf{o}_{t+1}, \mathbf{a}_{t+1}; \bar{\phi}^i)\left|_{a_{t+1}^j = \mu^j (o_{t+1}^j; \bar{\theta}^j)\, \forall j} \right.
\end{split}
\end{equation}
For a given set of weights $w$, we define its target counterpart $\bar{w}$, updated from $\bar{w}\leftarrow \tau w + (1-\tau) \bar{w}$ where $\tau$ is a hyperparameter.
Each policy is updated to maximize the expected discounted return of the corresponding agent $i$ :
\begin{align}
\label{eq:JPG}
    J_{PG}^i(\mathbf{\theta}^i) &= \mathbb{E}_{\mathbf{o}_t \sim \mathcal{D}} \left[Q^i (\mathbf{o}_t, \mathbf{a}_t) \bigg|_{\subalign{&a_t^i = \mu^i(o_t^i;\,\theta^i), \\ &a_t^j = \mu^j(o_t^j;\,\bar{\theta}^j)\, \forall j\neq i}} \right]
\end{align}
By taking into account \textit{all} agents' observation-action pairs when guiding an agent's policy, the value-functions are trained in a centralized, stationary environment, despite taking place in a multi-agent setting. This mechanism can allow to learn coordinated strategies that can then be deployed in a decentralized way. However, this procedure does not encourage the \textit{discovery} of coordinated strategies since high-return behaviors have to be randomly experienced through unguided exploration.

\section{Motivation}
\label{sec:CMADDPG:toy_experiment}
\FloatBarrier

In this section, we aim to answer the following question: can coordination help the discovery of effective policies in cooperative tasks? Intuitively, coordination can be defined as an agent's behavior being informed by the behavior of another agent, i.e. structure in the agents' interactions. Namely, a team where agents behave independently of one another would not be coordinated. 

Consider the simple Markov Game consisting of a chain of length $L$ leading to a termination state as depicted in Figure~\ref{fig:toy_experiment_MDP}. At each time-step, both agents receive $r_t=-1$. The joint action of these two agents in this environment is given by $\mathbf{a}\in\mathcal{A}=\mathcal{A}^1 \times \mathcal{A}^2$, where $\mathcal{A}^1 = \mathcal{A}^2 = \{0,1\}$. Agent $i$ tries to go right when selecting $a^i=0$ and left when selecting $a^i=1$. However, to transition to a different state both agents need to perform the same action at the same time (two lefts or two rights).
Now consider a slight variant of this environment with a different joint action structure $\mathbf{a}'\in\mathcal{A}'$. The former action structure is augmented with a hard-coded coordination module which maps the joint primitive $a^i$ to $a^{i\prime}$ like so:
\begin{equation*}
    \mathbf{a}' = \begin{pmatrix*}[l] a^{1 \prime} = a^1 \\ a^{2 \prime} = a^1 a^2 + (1 - a^1)(1 - a^2)\end{pmatrix*},\,\begin{pmatrix*}[l]a^1 \\ a^2\end{pmatrix*} \in \mathcal{A}
\end{equation*}
While the second agent still learns a state-action value function $Q^2(s,a^2)$ with $a^2 \in \mathcal{A}^2$, the coordination module builds $a^{2\prime}$ from $(a^1, a^2)$ so that $a^{2\prime}$ effectively determines whether the second agent acts in \textit{agreement} or in \textit{disagreement} with the first agent. In other words, if $a^2=1$, then $a^{2\prime}=a^1$ (agreement) and if $a^2=0$, then $a^{2\prime}=1 - a^1$ (disagreement).

\WFclear
\clearpage
\begin{minipage}{0.48\textwidth}
While it is true that this additional structure does not modify the action space nor the independence of the action selection, it reduces the stochasticity of the transition dynamics as seen by agent 2. In the first setup, the outcome of an agent’s action is conditioned on the action of the other agent. In the second setup, if agent 2 decides to disagree, the transition becomes deterministic as the outcome is independent of agent 1. This suggests that by reducing the entropy of the transition distribution, this mapping reduces the variance of the Q-updates and thus makes online tabular Q-learning agents learn much faster (Figure~\ref{fig:toy_experiment_MDP}).

This example uses a handcrafted mapping in order to demonstrate the effectiveness of exploring in the space of coordinated policies rather than in the unconstrained policy space. Now, the following question remains: how can one softly learn the same type of constraint throughout training for any multi-agent cooperative tasks? In the following sections, we present two algorithms that tackle this problem.
\end{minipage}
\hfill
\begin{minipage}{0.48\textwidth}
    \centering
    \begin{tabular}{@{}c@{}}  
    \includegraphics[width=\textwidth]{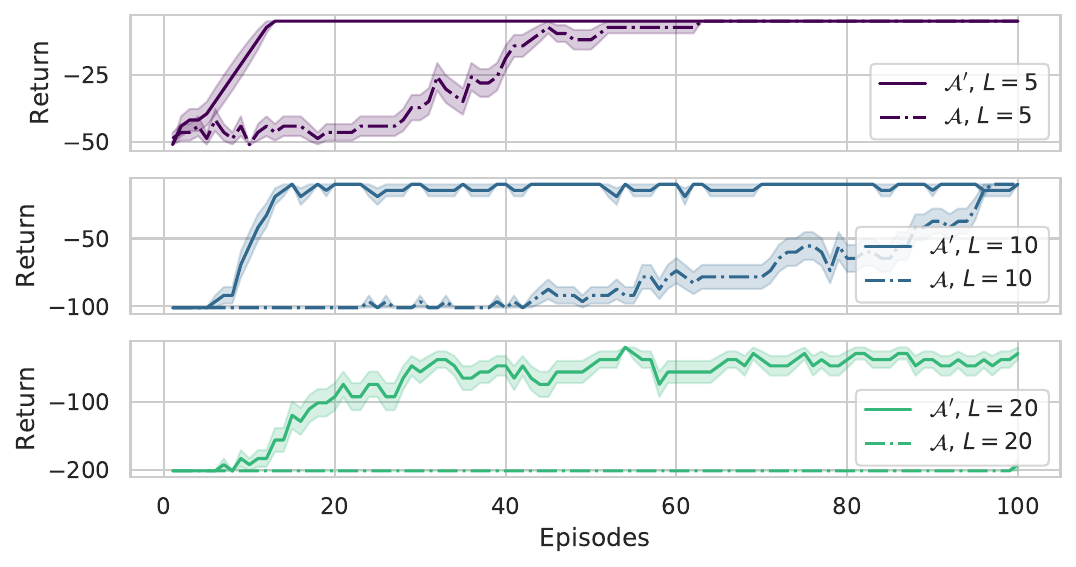} \\
    \includegraphics[width=\textwidth]{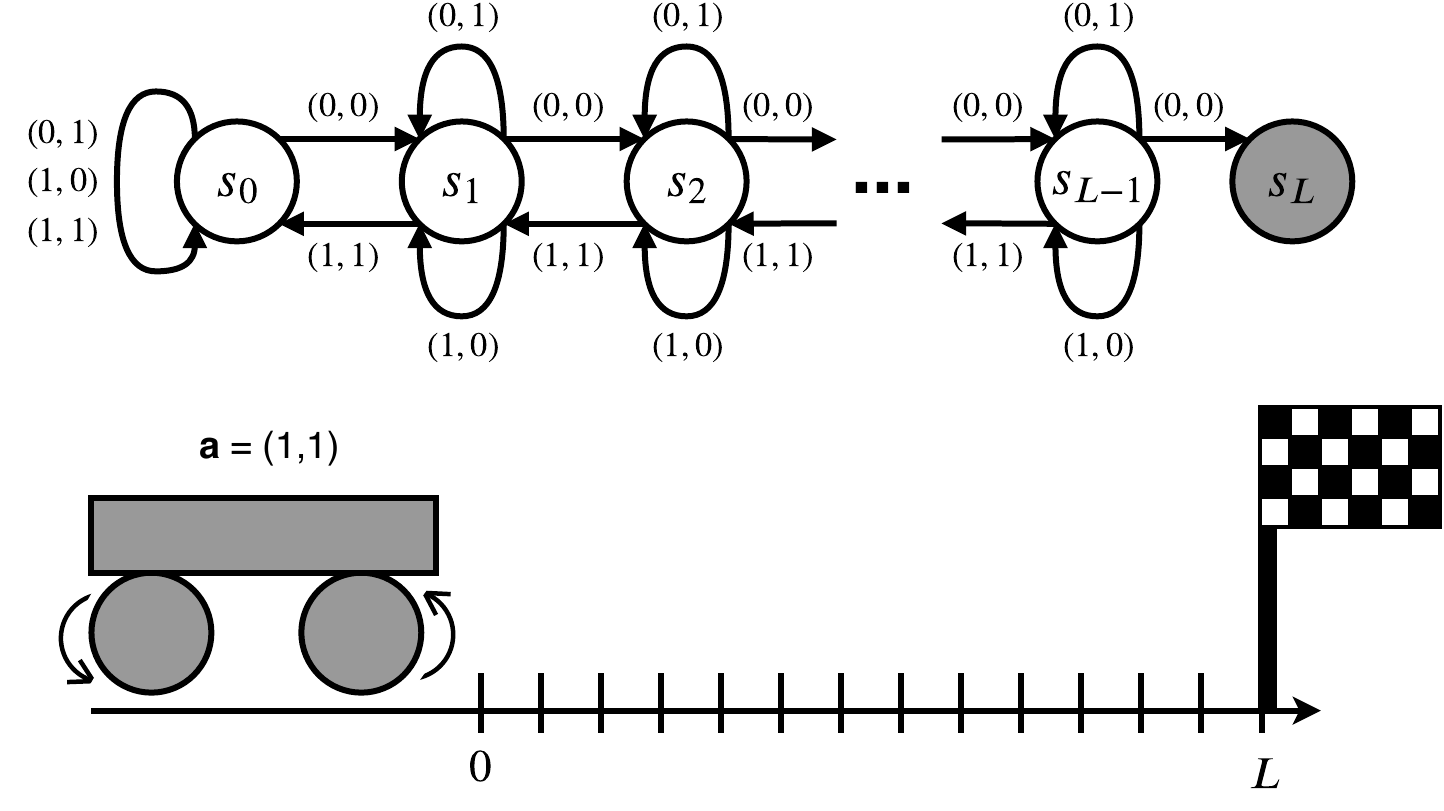}
    \end{tabular}
    \captionof{figure}[Simple coordination MDP]{(Top) The tabular Q-learning agents learn much more efficiently when constrained to the space of coordinated policies (solid lines) than in the original action space (dashed lines). (Bottom) Simple Markov Game consisting of a chain of length $L$ leading to a terminal state (in grey). Agents can be seen as the two wheels of a vehicle so that their actions need to be in agreement for the vehicle to move. The detailed experimental setup is reported in Appendix~\ref{ap:toy}.}
    \label{fig:toy_experiment_MDP}
\end{minipage}

\section{Coordination and Policy regularization}

Pseudocodes of our implementations are provided in Appendix~\ref{app:CMADDPG:algorithms} (see Algorithms~\ref{alg:teammaddpg}~and~\ref{alg:coachmaddpg}).

\subsection{Team regularization}\label{sec:CMADDPG:ts}

This first approach aims at exploiting the structure present in the joint action space of coordinated policies to attain a certain degree of predictability of one agent's behavior with respect to its teammate(s). It is based on the hypothesis that the reciprocal also holds i.e. that promoting agents' predictability could foster such team structure and lead to more coordinated behaviors. This assumption is cast into the decentralized framework by training agents to predict their teammates' actions given only their own observation. For continuous control, the loss is the mean squared error (MSE) between the predicted and true actions of the teammates, yielding a teammate-modelling secondary objective. For discrete action spaces, we use the KL-divergence ($D_\text{KL}$) between the predicted and real action distributions of an agent pair.

While estimating teammates' policies can be used to enrich the learned representations, we extend this objective to also drive the teammates' behaviors towards the predictions by leveraging a differentiable action selection mechanism. We call \textit{team-spirit} this objective pair $J^{i,j}_{TS}$ and $J^{j,i}_{TS}$ between agents $i$ and $j$:
\begin{equation}
    \label{eq:J_TScont}
    J^{i,j}_{TS\text{-continuous}}(\theta^i, \theta^j) 
    = -\mathbb{E}_{\mathbf{o}_t \sim \mathcal{D}} 
    \left[
    \text{MSE}(\mu^j(o_t^j;\theta^j),\hat{\mu}^{i,j}(o_t^i;\theta^i))
    \right]
\end{equation}
\begin{equation}
    \label{eq:J_TSdisc}
    J^{i,j}_{TS\text{-discrete}}(\theta^i, \theta^j) 
    = -\mathbb{E}_{\mathbf{o}_t\sim \mathcal{D}}
    \left[ 
    \text{D}_{\text{KL}} \left( \pi^j(\cdot|o^j_t; \theta^j)||\hat{\pi}^{i,j}(\cdot|o^i_t; \theta^i)\right)
    \right]
\end{equation}
where $\hat{\mu}^{i,j}$ (or $\hat{\pi}^{i,j}$ in the discrete case) is the policy head of agent $i$ trying to predict the action of agent $j$. The total objective for a given agent $i$ becomes:
\begin{equation}
    \label{eq:team_update}
    J_{total}^i(\theta^i) 
    = J_{PG}^i(\theta^i)
    + \lambda_1 \sum_j J^{i,j}_{TS}(\theta^i, \theta^j)
    + \lambda_2 \sum_j J^{j,i}_{TS}(\theta^j, \theta^i)
\end{equation}
where $\lambda_1$ and $\lambda_2$ are hyperparameters that respectively weigh how well an agent should predict its teammates' actions, and how predictable an agent should be for its teammates. We call TeamReg this dual regularization from team-spirit objectives. Figure~\ref{fig:diagram-teamMADDPG} summarizes these interactions.

\subsection{Coach regularization}
In order to foster coordinated interactions, this method aims at teaching the agents to recognize different situations and synchronously select corresponding sub-behaviors. 
%
\paragraph{Sub-policy selection}
Firstly, to enable explicit sub-behavior selection, we propose the use of \textit{policy masks} as a means to modulate the agents' policies. A policy mask $u^j$ is a one-hot vector of size $K$ (a fixed hyperparameter) with its $j^{\text{th}}$ component set to one. In practice, we use these masks to perform dropout \citep{srivastava2014dropout} in a structured manner on $\tilde{h}_1\in \mathbb{R}^{M}$, the pre-activations of the first hidden layer $h_1$ of the policy network $\pi$.  To do so, we construct the vector $\boldsymbol{u}^j$, which is the concatenation of $C$ copies of $u^j$, in order to reach the dimensionality $M = C * K$. The element-wise product $\boldsymbol{u}^j \odot \tilde{h}_1$ is performed and only the units of $\tilde{h}_{1}$ at indices $m\, \text{modulo}\, K = j$ are kept for $m =0,\ \ldots, M-1$. Each agent $i$ generates $e_t^i$, its own policy mask from its observation $o_t^i$, to modulate its policy network. Here, a simple linear layer $l^i$ is used to produce a categorical probability distribution $p^i(e_t^i|o_t^i)$ from which the one-hot vector is sampled:
\vspace{-3mm}
\begin{equation}
    p^i(e_t^i=u^j|o_t^i) 
    = \frac{\text{exp}\left(l^i(o_t^i;\theta^i)_j\right)}{\sum_{k=0}^{K-1} \text{exp}\left(l^i(o_t^i;\theta^i)_k\right)}
\end{equation}

\vspace{-5mm}
\paragraph{Synchronous sub-policy selection}
Although the policy masking mechanism enables the agent to swiftly switch between sub-policies it does not encourage the agents to synchronously modulate their behavior. To promote synchronicity we introduce the \textit{coach} entity, parametrized by $\psi$, which learns to produce policy-masks $e^c_t$ from the joint observations, i.e. $p^c(e_t^c|\mathbf{o}_t; \psi)$. The coach is used at training time only and drives the agents toward synchronously selecting the same behavior mask. Specifically, the coach is trained to output masks that (1) yield high returns when used by the agents and (2) are predictable by the agents. Similarly, each agent is regularized so that (1) its private mask matches the coach's mask and (2) it derives efficient behavior when using the coach's mask. At evaluation time, the coach is removed and the agents only rely on their own policy masks. The policy gradient objective when agent $i$ is provided with the coach's mask is given by:
\begin{equation}
\label{eq:JEPG}
    J_{EPG}^i(\theta^i, \psi)
    = \mathbb{E}_{\mathbf{o}_t, \mathbf{a}_t \sim \mathcal{D}} \left[ Q^i (\mathbf{o}_t, \mathbf{a}_t) \left|_{\subalign{&a_t^i = \mu(o^i_t, e^c_t;\theta^i)\\&e^c_t \sim p^c(\cdot|\mathbf{o}_t; \psi)}} \right.\right]
\end{equation}
The difference between the mask distribution of agent $i$ and the coach's is measured from the Kullback–Leibler divergence:
\begin{equation}
\label{eq:JE}
J_{E}^i(\mathbf{\theta}^i, \psi) = -\mathbb{E}_{\mathbf{o}_t\sim \mathcal{D}}\left[ \text{D}_{\text{KL}} \left(\left. p^c(\cdot|\mathbf{o}_t; \psi)\right||p^i(\cdot|o^i_t; \theta^i)\right)\right]
\end{equation}
The total objective for agent $i$ is:
\begin{equation}
\label{eq:Jtoti}
    J^i_{total}(\theta^i) 
    = J_{PG}^i(\theta^i)
    + \lambda_1 J_{E}^i(\theta^i, \psi)
    + \lambda_2 J_{EPG}^i(\theta^i, \psi)
\end{equation}
with $\lambda_1$ and $\lambda_2$ the regularization coefficients.
Similarly, the coach is trained with the following dual objective, weighted by the $\lambda_3$ coefficient:
\begin{equation}
\label{eq:Jtotc}
J^c_{total}(\psi) = \frac{1}{N}\sum_{i=1}^N\left( J_{EPG}^i(\theta^i, \psi) + \lambda_3J_{E}^i(\theta^i, \psi)\right)
\end{equation}
In order to propagate gradients through the sampled policy mask we reparameterized the categorical distribution using the Gumbel-softmax \citep{jang2016categorical}.
We call this coordinated sub-policy selection regularization CoachReg and illustrate it in Figure~\ref{fig:diagram-coachMADDPG}.

\begin{minipage}{0.48\textwidth}
    \centering
    \includegraphics[width=0.67\textwidth]{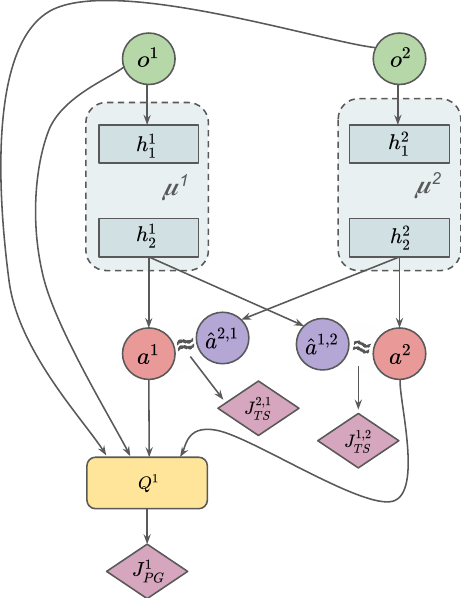}
    \vspace{0mm}
    \captionof{figure}[Illustration of TeamReg with two agents]{Illustration of TeamReg with two agents. Each agent's policy is equipped with additional heads that are trained to predict other agents' actions and every agent is regularized to produce actions that its teammates correctly predict. The method is depicted for agent 1 only to avoid cluttering}
    \label{fig:diagram-teamMADDPG}
\end{minipage}
\hfill
\begin{minipage}{0.49\textwidth}
    \centering
    \vspace{-5mm}
    \includegraphics[width=\textwidth]{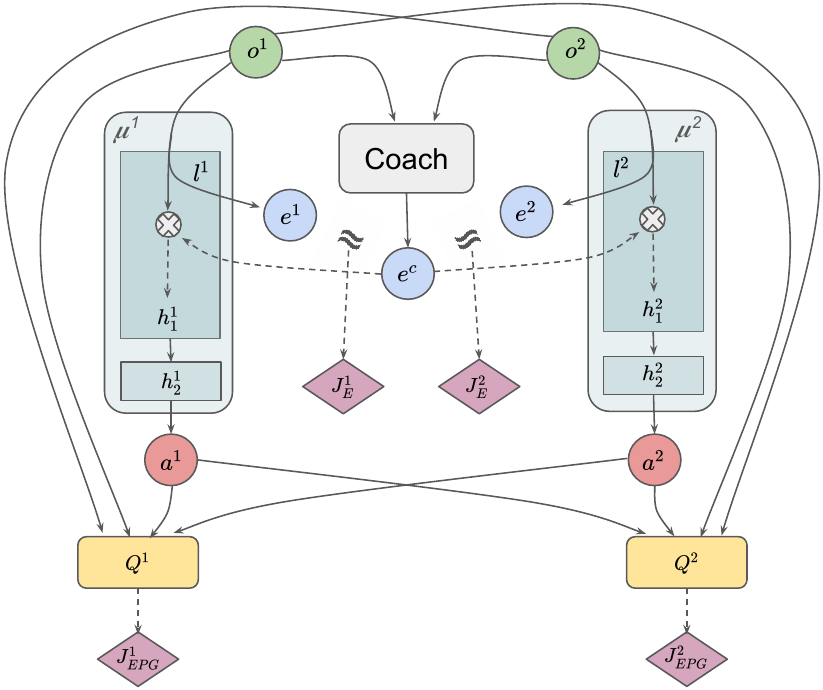}
    \vspace{0mm}
    \captionof{figure}[Illustration of CoachReg with two agents]{Illustration of CoachReg with two agents. A central model, the coach, takes all agents' observations as input and outputs the current mode (policy mask). Agents are regularized to predict the same mask from their local observations and optimize the corresponding sub-policy.}
    \label{fig:diagram-coachMADDPG}
\end{minipage}

\FloatBarrier
\section{Related Work}
\label{section:related_work}
Several works in MARL consider explicit communication channels between the agents and distinguish between communicative actions (e.g. broadcasting a given message) and physical actions (e.g. moving in a given direction) \citep{foerster2016learning,  mordatch2018emergence, lazaridou2016multi}. Consequently, they often focus on the emergence of language, considering tasks where the agents must discover a common communication protocol to succeed. Deriving a successful communication protocol can already be seen as coordination in the communicative action space and can enable, to some extent, successful coordination in the physical action space~\citep{ahilan2019feudal}. Yet, explicit communication is not a necessary condition for coordination as agents can rely on physical communication \citep{mordatch2018emergence, gupta2017cooperative}.

TeamReg falls in the line of work that explores how to shape agents' behaviors with respect to other agents through auxiliary tasks. \citet{strouse2018learning} use the mutual information between the agent's policy and a goal-independent policy to shape the agent's behavior towards hiding or spelling out its current goal. However, this approach is only applicable for tasks with an explicit goal representation and is not specifically intended for coordination. \citet{jaques2018intrinsic} approximate the direct causal effect between agent's actions and use it as an intrinsic reward to encourage social empowerment. This approximation relies on each agent learning a model of other agents' policies to predict its effect on them. In general, this type of behavior prediction can be referred to as \textit{agent modelling} (or opponent modelling) and has been used in previous work to enrich representations~\citep{hern2019agent, hong2017deep}, to stabilise the learning dynamics \citep{he2016opponent} or to classify the opponent's play style \citep{schadd2007opponent}.

With CoachReg, agents learn to unitedly recognize different modes in the environment and adapt by jointly switching their policy. This echoes with the hierarchical RL literature and in particular with the single agent options framework \citep{bacon2017option} where the agent switches between different sub-policies, the options, depending on the current state. To encourage cooperation in the multi-agent setting, \citet{ahilan2019feudal} proposed that an agent, the ``manager'', is extended with the possibility of setting other agents' rewards in order to guide collaboration. CoachReg stems from a similar idea: reaching a consensus is easier with a central entity that can asymmetrically influence the group. Yet, \citet{ahilan2019feudal} guides the group in terms of ``ends'' (influences through the rewards) whereas CoachReg constrains it in terms of ``means'' (the group must synchronously switch between different strategies). Hence, the interest of CoachReg does not just lie in training sub-policies (which are obtained here through a simple and novel masking procedure) but rather in co-evolving synchronized sub-policies across multiple agents. \citet{mahajan2019maven} also looks at sub-policies co-evolution to tackle the problem of joint exploration, however their selection mechanism occurs only on the first timestep and requires duplicating random seeds across agents at test time. On the other hand, with CoachReg the sub-policy selection is explicitly decided by the agents themselves at each timestep without requiring a common sampling procedure since the mode recognition has been learned and grounded on the state throughout training.

Finally, \citet{barton2018measuring} propose convergent cross mapping (CCM) to measure the degree of effective coordination between two agents. Although this represents an interesting avenue for behavior analysis, it fails to provide a tool for effectively enforcing coordination as CCM must be computed over long time series making it an impractical learning signal for single-step temporal difference methods. 

To our knowledge, this work is the first to extend agent modelling to derive an inductive bias towards team-predictable policies or to introduce a collective, agent induced, modulation of the policies without an explicit communication channel. Importantly, these coordination proxies are enforced throughout training only, which allows to maintain decentralised execution at test time.

\section{Training environments}
\label{section:training_environments}
Our continuous control tasks are built on OpenAI's multi-agent particle environment \citep{mordatch2018emergence}. SPREAD and CHASE were introduced by \citep{lowe2017multi}. We use SPREAD as is but with sparse rewards. CHASE is modified with a prey controlled by repulsion forces so that only the predators are learnable, as we wish to focus on coordination in cooperative tasks. Finally we introduce COMPROMISE and BOUNCE where agents are physically tied together. While positive return can be achieved in these tasks by selfish agents, they all benefit from coordinated strategies and maximal return can only be achieved by agents working closely together. Figure~\ref{fig:CMADDPG:envs} presents a visualization and a brief description. In all tasks, agents receive as observation their own global position and velocity as well as the relative position of other entities. A more detailed description is provided in Appendix~\ref{app:CMADDPG:environments}. Note that work showcasing experiments on these environments often use discrete action spaces and dense rewards (e.g. the proximity with the objective) \citep{iqbal2018actor, lowe2017multi, jiang2018learning}. In our experiments, agents learn with continuous action spaces and from sparse rewards which is a far more challenging setting.

\section{Results and Discussion}
The proposed methods offer a way to incorporate new inductive biases in CTDE multi-agent policy search algorithms. We evaluate them by extending MADDPG, one of the most widely used algorithm in the MARL literature. We compare against vanilla MADDPG as well as two of its variants in the four cooperative multi-agent tasks described in Section~\ref{section:training_environments}. The first variant (DDPG) is the single-agent counterpart of MADDPG (decentralized training). The second (MADDPG + sharing) shares the policy and value-function models across agents. Additionally to the two proposed algorithms and the three baselines, we present results for two ablated versions of our methods. The first ablation (MADDPG + agent modelling) is similar to TeamReg but with $\lambda_2=0$, which results in only enforcing agent modelling and not encouraging agent predictability. The second ablation (MADDPG + policy mask) uses the same policy architecture as CoachReg, but with $\lambda_{1,2,3}=0$, which means that agents still predict and apply a mask to their own policy, but synchronicity is not encouraged.

\begin{figure}[htbp]
    \centering
    \includegraphics[width=0.9\textwidth]{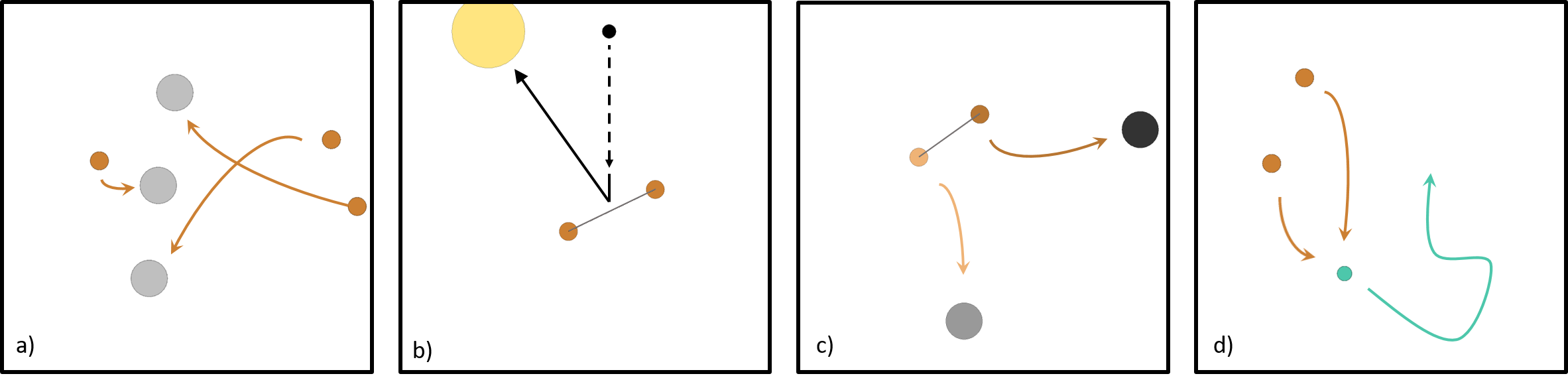}
    \caption[Depiction of multi-agent coordination environments]{Multi-agent tasks we employ. (a) SPREAD: Agents must spread out and cover a set of landmarks. (b) BOUNCE: Two agents are linked together by a spring and must position themselves so that the falling black ball bounces towards a target. (c) COMPROMISE: Two linked agents must compete or cooperate to reach their own assigned landmark. (d) CHASE: Two agents chase a (non-learning) prey (turquoise) that moves w.r.t repulsion forces from predators and walls.}
    \label{fig:CMADDPG:envs}
\end{figure}

To offer a fair comparison between all methods, the hyperparameter search routine is the same for each algorithm and environment (see Appendix~\ref{app:CMADDPG:hyperparameter_search_ranges}). For each search-experiment (one per algorithm per environment), 50 randomly sampled hyperparameter configurations each using 3 random seeds are used to train the models for $15,000$ episodes. For each algorithm-environment pair, we then select the best hyperparameter configuration for the final comparison and retrain them on 10 seeds for twice as long. This thorough evaluation procedure represents around 3 CPU-year. We give all details about the training setup and model selection in Appendix~\ref{app:CMADDPG:training_details} and \ref{app:CMADDPG:model_selection}. The results of the hyperparameter searches are given in Appendix~\ref{app:CMADDPG:hyperparameter_search_results}. Interestingly, Figure~\ref{fig:hyperparam_search_boxplots} shows that our proposed coordination regularizers improve robustness to hyperparameters despite having more hyperparameters to tune.

\begin{table}[t]
\caption[Final performance in coordination environments]{Final performance reported as mean return over agents averaged across 10 episodes and 10 seeds ($\pm$ SE).}
\vspace{-6mm}
\begin{center}
\scalebox{.67}{
\begin{tabular}{l|c|c|c|c|c|c|c|}
 \backslashbox[68pt]{env}{alg} & DDPG & MADDPG & \makecell{MADDPG\\+sharing} & \makecell{MADDPG\\+agent modelling} & \makecell{MADDPG\\+policy mask} & \makecell{MADDPG\\+TeamReg (ours)} & \makecell{MADDPG\\+CoachReg (ours)} \\
 \hline
 SPREAD & $133 \pm 12$ & $159 \pm 6$ & $47 \pm 8$ & $183 \pm 10$ & $\boldsymbol{221 \pm 11}$ & $\boldsymbol{216 \pm 12}$ & $\boldsymbol{210 \pm 12}$ \\ 
 BOUNCE & $3.6 \pm 1.4$ & $4.0 \pm 1.6$ & $0.0 \pm 0.0$ & $3.8 \pm 1.5$ & $3.7 \pm 1.1$ & $\boldsymbol{5.8 \pm 1.3}$ & $\boldsymbol{7.4 \pm 1.2}$ \\  
 COMPROMISE & $19.1 \pm 1.2$ & $18.1 \pm 1.1$ & $19.6 \pm 1.5$ & $12.9 \pm 0.9$ & $18.4 \pm 1.3$ & $8.8 \pm 0.9$ & $\boldsymbol{31.1 \pm 1.1}$ \\
 CHASE & $727 \pm 87$ & $834 \pm 80$ & $\boldsymbol{980 \pm 64}$ & $\boldsymbol{946 \pm 69}$ & $722 \pm 82$ & $\boldsymbol{917 \pm 90}$ & $\boldsymbol{949 \pm 54}$
\end{tabular}
}
\end{center}
\label{tab:final_performance}
\end{table}
\begin{figure}[t]
    \centering
    \makebox[\textwidth][c]{\includegraphics[width=1.\textwidth]{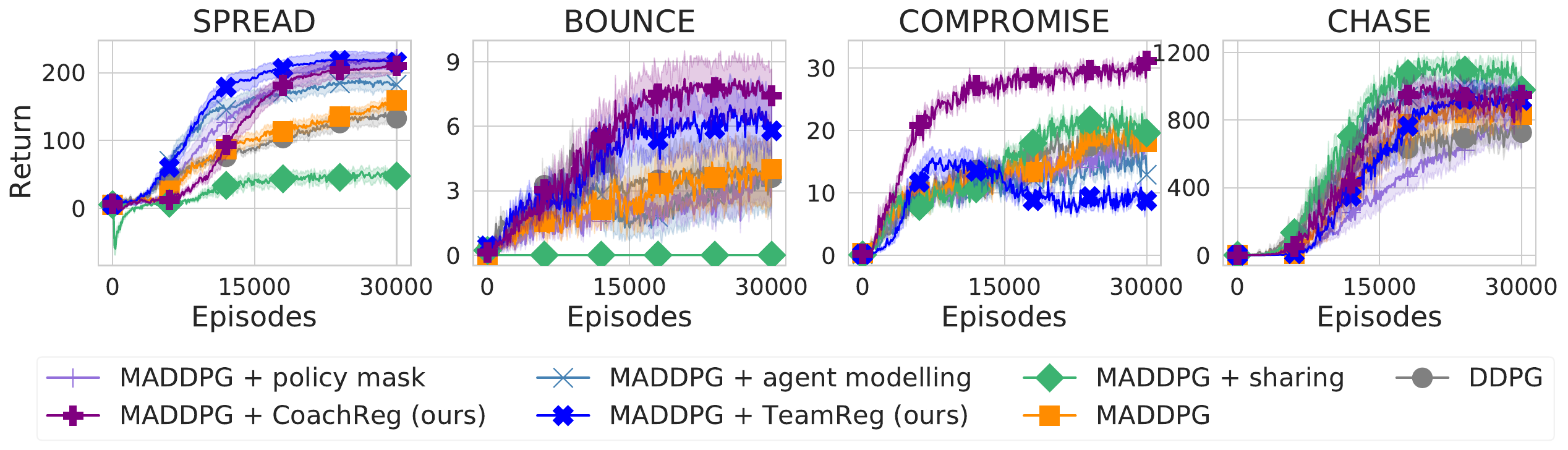}}
    \caption[Learning curves on all four coordination environments]{Learning curves (mean return over agents) for our two proposed algorithms, two ablations and three baselines on all four environments. Solid lines are the mean and envelopes are the Standard Error (SE) across the 10 training seeds.}
    \label{fig:learning_curves}
\end{figure}

\subsection{Asymptotic Performance}

Figure~\ref{fig:learning_curves} reports the average learning curves and Table~\ref{tab:final_performance} presents the final performance. CoachReg is the best performing algorithm considering performance across all tasks. TeamReg also significantly improves performance on two tasks (SPREAD and BOUNCE) but shows unstable behavior on COMPROMISE, the only task with an adversarial component. This result reveals one limitation of this approach and is dicussed in details in Appendix~\ref{app:CMADDPG:effects_enforcing_predictability}. Note that all algorithms perform similarly well on CHASE, with a slight advantage to the one using parameter sharing; yet this superiority is restricted to this task where the optimal strategy is to move symmetrically and squeeze the prey into a corner. Contrary to popular belief, we find that MADDPG almost never significantly outperforms DDPG in these sparse reward environments, supporting the hypothesis that while CTDE algorithms can in principle identify and reinforce highly rewarding coordinated behavior, they are likely to fail to do so if not incentivized to coordinate.

Regarding the ablated versions of our methods, the use of unsynchronized policy masks might result in swift and unpredictable behavioral changes and make it difficult for agents to perform together and coordinate. Experimentally, ``MADDPG~+~policy mask'' performs similarly or worse than MADDPG on all but one environment, and never outperforms the full CoachReg approach. However, policy masks alone seem sufficient to succeed on SPREAD, which is about selecting a landmark from a set.  Finally ``MADDPG~+~agent modelling'' does not drastically improve on MADDPG apart from one environment, and is always outperformed by the full TeamReg (except on COMPROMISE, see Appendix~\ref{app:CMADDPG:effects_enforcing_predictability}) which supports the importance of enforcing predictability alongside agent modeling.

\subsection{Effects of enforcing predictable behavior}\label{section:TeamReg_analysis}

Here we validate that enforcing predictability makes the agent-modelling task more successful. To this end, we compare, on the SPREAD environment, the team-spirit losses between TeamReg and its ablated versions. Figure~\ref{fig:ts_analysis} shows that initially, due to the weight initialization, the predicted and actual actions both have relatively small norms yielding small values of team-spirit loss. As training goes on ($\sim$1000 episodes), the norms of the action-vector increase and the regularization loss becomes more important. As expected, MADDPG leads to the worst team-spirit loss as it is not trained to predict the actions of other agents. When using only the agent-modelling objective ($\lambda_1 > 0$), the agents significantly decrease the team-spirit loss, but it never reaches values as low as when using the full TeamReg objective ($\lambda_1 > 0$ and $\lambda_2 > 0$). Note that the team-spirit loss increases when performance starts to improve i.e. when agents start to master the task ($\sim$8000 episodes). Indeed, once the return maximisation signal becomes stronger, the relative importance of the auxiliary objective is reduced. Being predictable with respect to one-another may push agents to explore in a more structured and informed manner in the absence of reward signal, as similarly pursued by intrinsic motivation approaches \citep{chentanez2005intrinsically}.

\subsection{Analysis of synchronous sub-policy selection}
\label{subsec:CMADDPG:analysis_coach}

\begin{figure}[b!]
\begin{minipage}{.35\textwidth}
\vspace{5mm}
  \centering
  \includegraphics[width=1.\textwidth, height=0.5\textwidth]{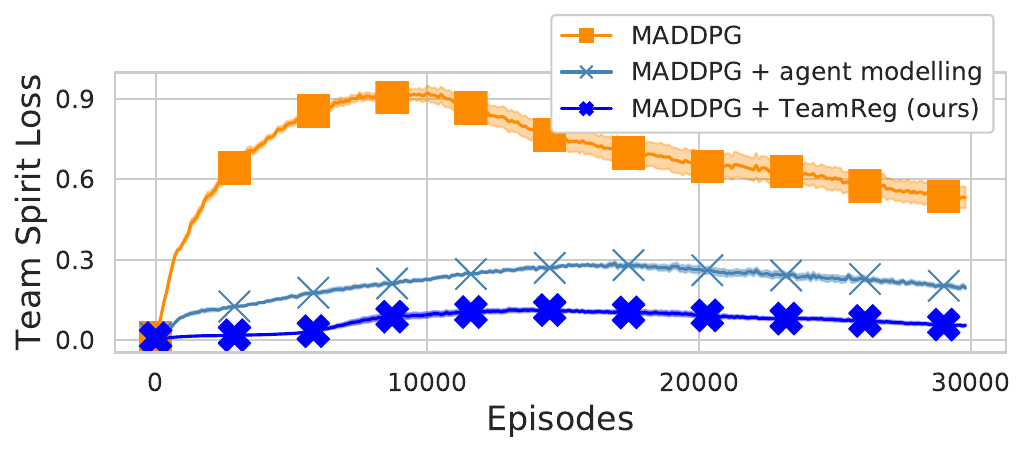}
  \vspace{-1mm}
  \captionof{figure}[TeamReg ablation study]{Effect of enabling and disabling the coefficients $\lambda_1$ and $\lambda_2$ on the ability of agents to predict their teammates behavior. Solid lines and envelope are average and SE on 10 seeds on SPREAD.
  }
  \label{fig:ts_analysis}
\end{minipage}
\hfill
\begin{minipage}{.60\textwidth}
  \centering
  \begin{tabular}{@{}cc@{}}
    \includegraphics[width=0.47\textwidth]{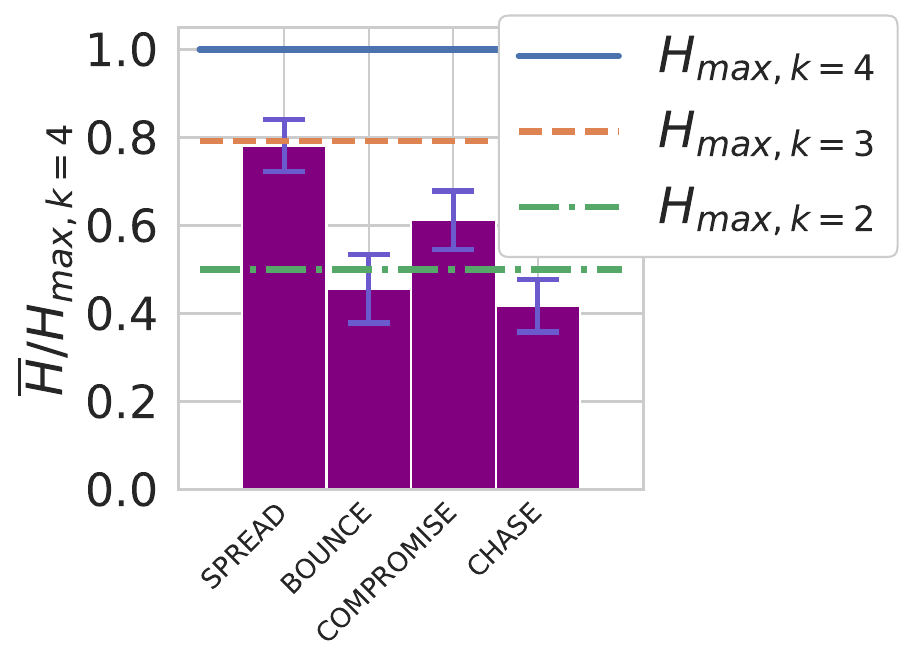} & \includegraphics[width=0.47\textwidth]{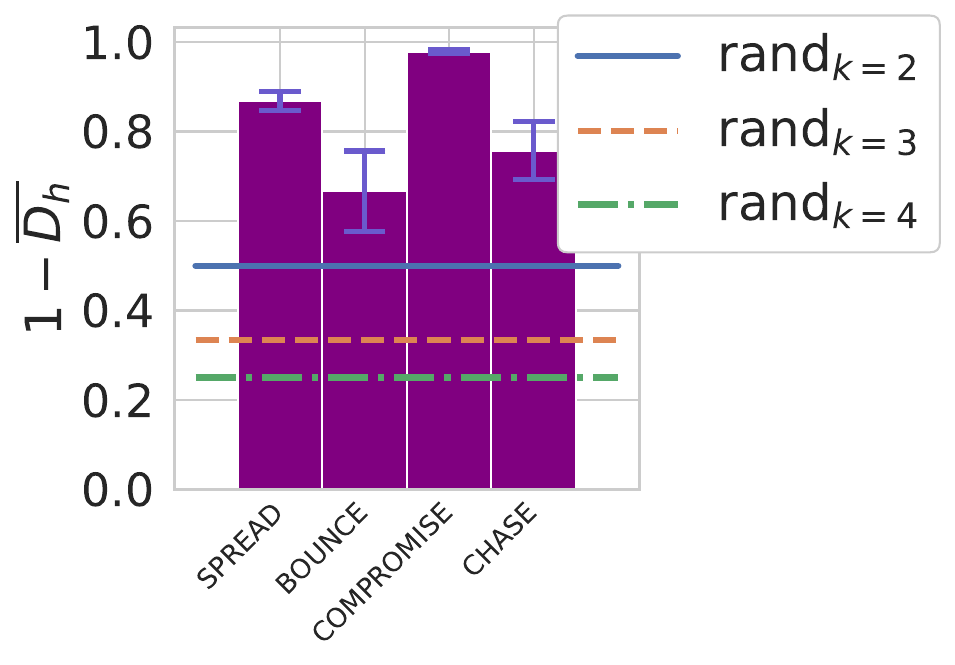}
    \end{tabular}
    \captionof{figure}[CoachReg synchronicity analysis]{(Left) Average entropy of the policy mask distributions for each task. $H_{max, k}$ is the entropy of a $k$-CUD. (Right) Average Hamming Proximity between the policy mask sequence of agent pairs. rand$_k$ stands for agents independently sampling their masks from $k$-CUD. Error bars are the SE on 10 seeds.
    }
    \label{fig:masks_metrics}
\end{minipage}
\vspace{-5mm}
\end{figure}
In this section we confirm that CoachReg yields the desired behavior: agents \textit{synchronously} alternating between \textit{varied} sub-policies.

Figure~\ref{fig:masks_metrics} shows the average entropy of the mask distributions for each environment compared to the entropy of Categorical Uniform Distributions of size $k$ ($k$-CUD). On all the environments, agents use several masks and tend to alternate between masks with more variety (close to uniformly switching between 3 masks) on SPREAD (where there are 3 agents and 3 goals) than on the other environments (comprised of 2 agents). Moreover, the Hamming proximity between the agents' mask sequences, $1 - D_h$ where $D_h$ is the Hamming distance (i.e. the ratio of timesteps for which the two sequences are different) shows that agents are synchronously selecting the same policy mask at test time (without a coach). Finally, we observe that some settings result in the agents coming up with interpretable strategies, like the one depicted in Figure~\ref{fig:BD_bounce} in Appendix~\ref{app:CMADDPG:coach_subpolicy_selection} where the agents alternate between two sub-policies depending on the position of the target\footnote{See animations at \url{https://sites.google.com/view/marl-coordination/}}.

\subsection{Experiments on discrete action spaces}

\begin{minipage}{.48\textwidth}
We evaluate our techniques on the more challenging task of 3vs2 Google Research football environment \cite{kurach2019google}. In this environment, each agent controls an offensive player and tries to score against a defensive player and a goalkeeper controlled by the engine's rule-based bots. Here agents have discrete action spaces of size 21, with actions like moving direction, dribble, sprint, short pass, high pass, etc. We use as observations 37-dimensional vectors containing players' and ball's coordinates, directions, etc.
\end{minipage}
\hfill
\begin{minipage}{.48\textwidth}
    \centering
    \scalebox{0.75}{
    \begin{tabular}{c}
        Table 2: Average Returns for 3v2 football\\
        \setcounter{table}{1}
        \begin{tabular}{|c|c|}
        \hline
        \small MADDPG & \small 0.004 $\pm$ 0.002 \\
        \small MADDPG + sharing & \small 0.005 $\pm$ 0.003 \\
        \small MADDPG + TeamReg (ours) & \small 0.006 $\pm$ 0.003 \\
        \small MADDPG + CoachReg (ours) & \small \textbf{0.088 $\pm$ 0.017} \\
        \hline
        \end{tabular}
        \vspace{3mm}
        \\
        \includegraphics[width=\textwidth]{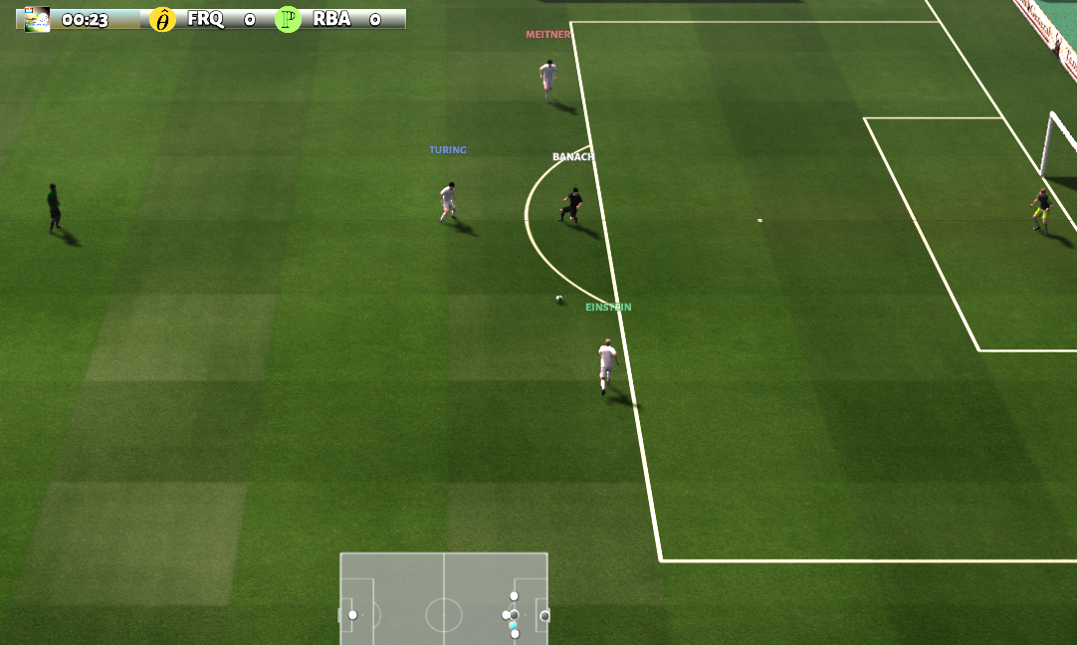}
    \end{tabular}
    }
    \captionof{figure}[Snapshot of the google research football environment]{Snapshot of the google research football environment \textit{3vs1-with-keeper}.}
    \label{fig:football_results}
\end{minipage}

The algorithms presented in Table~2 were trained using 25 randomly sampled hyperparameter configurations. The best configuration was retrained using 10 seeds for 80,000 episodes of 100 steps. Table~2 shows the mean return ($\pm$ standard error across seeds) on the last 10,000 episodes. All algorithms but MADDPG + CoachReg fail to reliably learn policies that achieve positive return (i.e. scoring goals).

\section{Conclusion}
In this work we motivate the use of coordinated policies to ease the discovery of successful strategies in cooperative multi-agent tasks and propose two distinct approaches to promote coordination for CTDE multi-agent RL algorithms. While the benefits of TeamReg appear task-dependent -- we show for example that it can be detrimental on tasks with a competitive component -- CoachReg significantly improves performance on almost all presented environments. Motivated by the success of this single-step coordination technique, a promising direction is to explore model-based planning approaches to promote coordination over long-term multi-agent interactions.

\section*{Broader Impact}
In this work, we present and study methods to enforce coordination in MARL algorithms. It goes without saying that multi-agent systems can be employed for positive and negative applications alike. We do not propose methods aimed at making new applications possible or improving a particular set of applications.  We instead propose methods that allow to better understand and improve multi-agent RL algorithms in general. Therefore, we do not aim in this section at discussing the impact of Multi-Agent Reinforcement Learning applications themselves but focus on the impact of our contribution: promoting multi-agent behaviors that are coordinated. 

We first observe that current Multi-Agent Reinforcement Learning (MARL) algorithms may fail to train agents that leverage information about the behavior of their teammates and that even when explicitly given their teammates observations, action and current policy during the training phase.
We believe that this is an important observation worth raising some concern among the community since there is a widespread belief that centralized training (like MADDPG) should always outperform decentralize training (DDPG). Not only is this belief unsupported by empirical evidence (at least in our experiments) but it also prevents the community from investigating and tackling this flaw that is an important limitation for learning safer and more effective multi-agent behavior. By not accounting for the behavior of its teammates, an agent could not adapt to a new teammate or even a change in the teammates behavior. This prevents current methods to be applied in the real world where there is external perturbations and uncertainties and where an artificial agent may need to interact with various different individuals. 

We propose to focus on coordination and sketch a definition of coordination: an agent behavior should be predictable given its teammate behavior. While this definition is restrictive, we believe that it is a good starting point to consider. Indeed, enforcing that criterion should make learning agents more aware of their teammates in order to coordinate with them. Yet, coordination alone does not ensure success, as agents could be coordinated in an unproductive manner. More so, coordination could have detrimental effects if it enables an attacker to influence an agent through taking control of a teammate or using a mock-up teammate. For these reasons, when using multi-agent RL algorithms (or even single-agent RL for that matter) for real world applications, additional safeguards are absolutely required to prevent the system from misbehaving, which is highly probable if out-of-distribution states are to be encountered.

\section*{Acknowledgements}
We thank Olivier Delalleau for his insightful feedback and comments. We also acknowledge funding in support of this work from the Fonds de Recherche Nature et Technologies (FRQNT) and Mitacs, as well as Compute Canada for supplying computing resources.

\chapter[ARTICLE 3: DIRECT BEHAVIOR SPECIFICATION VIA\\CONSTRAINED REINFORCEMENT LEARNING]{\\ARTICLE 3: DIRECT BEHAVIOR SPECIFICATION VIA CONSTRAINED REINFORCEMENT LEARNING}\label{chap:article3_dbs}

\vspace{-3mm}
\begin{center}
Co-authors\\Roger Girgis, Joshua Romoff, Pierre-Luc Bacon \& Christopher Pal
\end{center}
\vspace{-5mm}
\begin{center}
Published in\\Proceedings of Machine Learning Research, July 28, 2023
\end{center}

\begin{abstract}
    The standard formulation of Reinforcement Learning lacks a practical way of specifying what are admissible and forbidden behaviors. Most often, practitioners go about the task of behavior specification by manually engineering the reward function, a counter-intuitive process that requires several iterations and is prone to reward hacking by the agent. In this work, we argue that constrained RL, which has almost exclusively been used for safe RL, also has the potential to significantly reduce the amount of work spent for reward specification in applied RL projects. To this end, we propose to specify behavioral preferences in the CMDP framework and to use Lagrangian methods to automatically weigh each of these behavioral constraints. Specifically, we investigate how CMDPs can be adapted to solve goal-based tasks while adhering to several constraints simultaneously. We evaluate this framework on a set of continuous control tasks relevant to the application of Reinforcement Learning for NPC design in video games.
\end{abstract}

\section{Introduction}

Reinforcement Learning (RL) has shown rapid progress and lead to many successful applications over the past few years \cite{mnih2013playing,silver2017mastering,andrychowicz2020learning}. The RL framework is predicated on the simple idea that all tasks could be defined as a single scalar function to maximise, an idea generally referred to as the reward hypothesis \cite{sutton2018rlTextbook,silver2021reward,abel2021expressivity}. This idea has proven very useful to develop the theory and concentrate research on a single theoretical framework. However, it can be significantly limiting when translating a real-life problem into an RL problem, since the question of where the reward function comes from is completely ignored \cite{singh2009rewards}. In practice, human-designed reward functions often lead to unforeseen behaviors and represent a serious obstacle to the reliable application of RL in the industry \cite{amodei2016concrete}.

Concretely, for an engineer working on applying RL methods to an industrial problem, the task of reward specification implies to: (1) characterise the desired behavior that the system should exhibit, (2) write in a computer program a reward function for which the optimal policy corresponds to that desired behavior, (3) train an RL agent on that task using one of the methods available in the literature and (4) evaluate whether the agent exhibits the expected behavior. Multiple design iterations of that reward function are generally required, each time accompanied by costly trainings of the policy \cite{hadfield2017inverse,dulac2019challenges}. This inefficient design loop is exacerbated by the fact that current Deep RL algorithms cannot be guaranteed to find the optimal policy \cite{sutton2018rlTextbook}, meaning that the reward function could be correctly specified but still fail to lead to the desired behavior. The design problem thus becomes ``What reward function would lead SAC \cite{haarnoja2018soft} or PPO \cite{schulman2017proximal} to give me a policy that I find satisfactory?'', a difficult puzzle that every RL practitioner has had to deal with.

\begin{figure}[t]
    \centering
    \begin{minipage}{.5\textwidth}
        \centering
        \includegraphics[width=.98\textwidth,height=4.3cm]{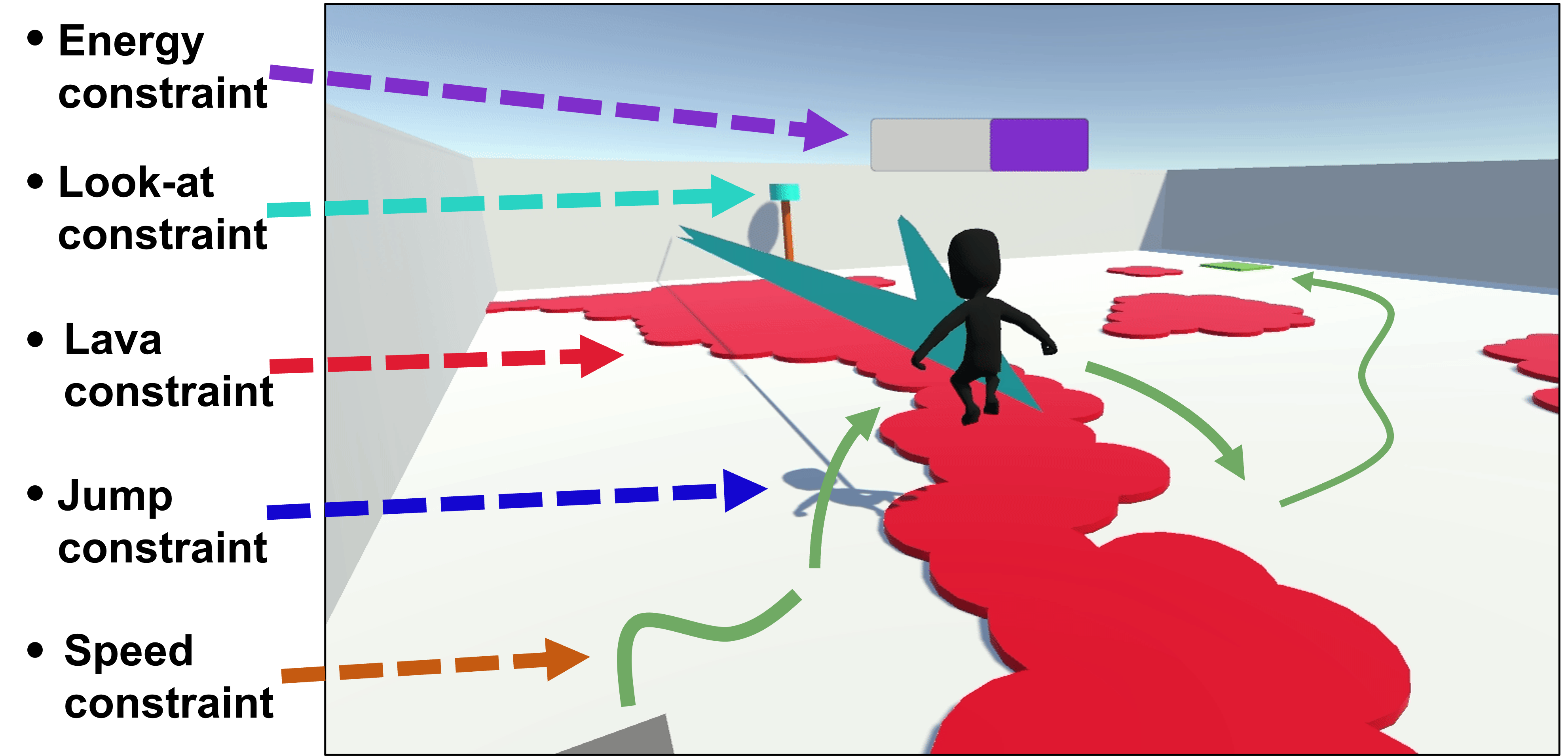}
    \end{minipage}%
    \begin{minipage}{.5\textwidth}
        \centering
        \includegraphics[width=.98\textwidth,height=4.3cm,frame]{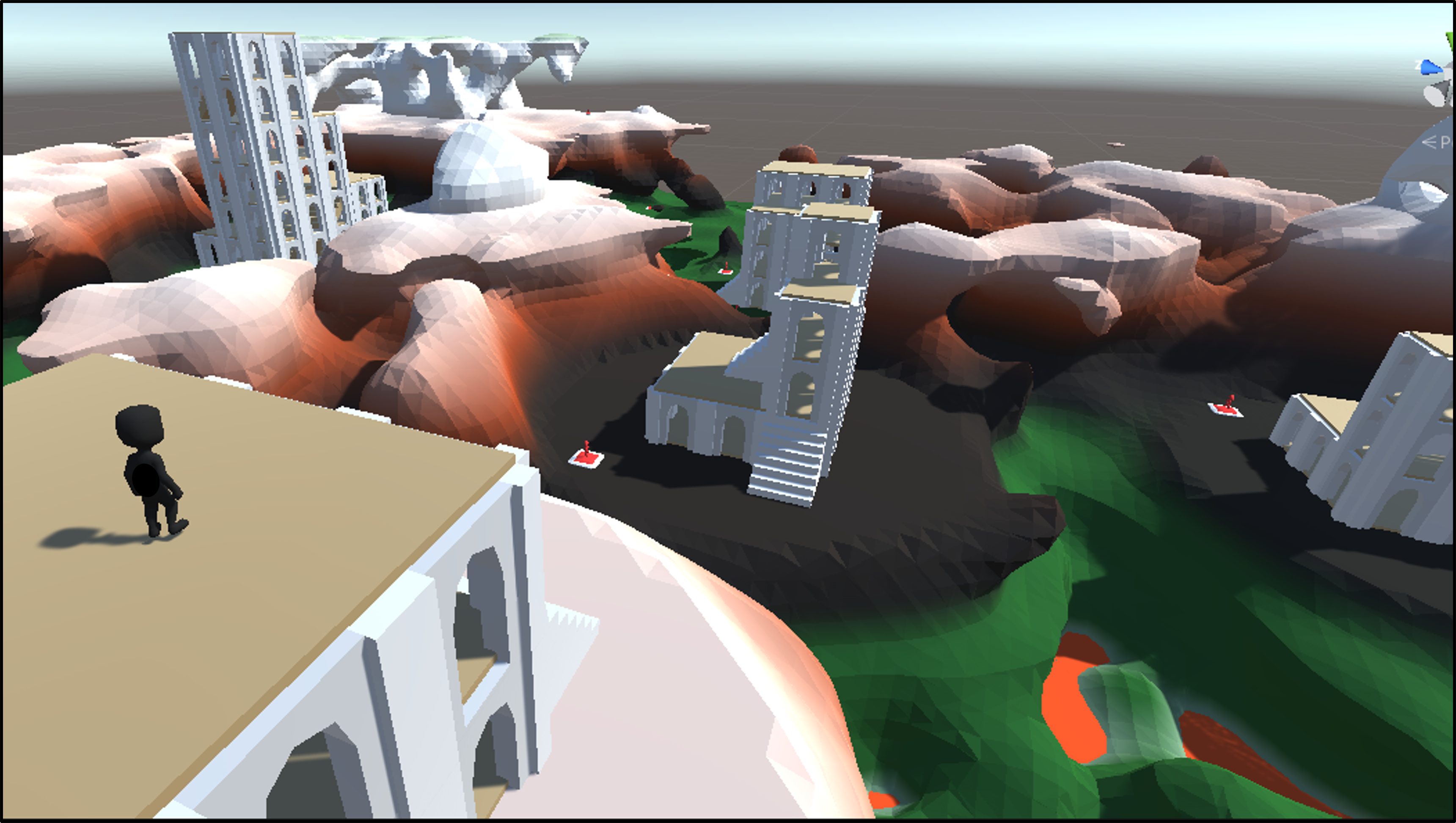}
    \end{minipage}
    \caption[Environments for direct behavior specification using constrained RL]{Depictions of our setup to evaluate direct behavior specification using constrained RL; Arena environment (left); OpenWorld environment (right). For videos see: \href{https://sites.google.com/view/behaviorspecificationviacrl/home}{https://sites.google.com/view/behaviorspecificationviacrl/home}.}
    \label{fig:DBS:envs}
\end{figure}

Most published work on Reinforcement Learning focuses on point (3) i.e. improving the reliability and efficiency with which these algorithms can yield a near-optimal policy for a \textit{given} reward function. This line of work is crucial to allow RL to tackle difficult problems. However, as agents become more and more capable of solving the tasks we present them with, our ability to (2) correctly specify these reward functions will only become more critical \cite{dewey2014reinforcement}.

Constrained Markov Decision Processes \cite{altman1999constrained} offer an alternative framework for sequential decision making. The agent still seeks to maximise a single reward function, but must do so while respecting a set of constraints defined by additional cost functions. While it is generally recognised that this formulation has the potential to allow for an easier task definition from the end user \cite{ray2019benchmarking}, most work on CMDPs focuses on the safety aspect of this framework i.e. that the constraint-satisfying behavior be maintained throughout the entire exploration process \cite{achiam2017constrained,zhang2020first,turchetta2020safe,marchesini2022exploring}. In this paper we specifically focus on the benefits of CMDPs relating to behavior specification. We make the following contributions: (1) we show experimentally that reward engineering poorly scales with the complexity of the target behavior, (2) we propose a solution where a designer can directly specify the desired frequency of occurrence of some events, (3) we develop a novel algorithmic approach capable of jointly satisfying many more constraints and (4) we evaluate this framework on a set of constrained tasks illustrative of the development cycle required for deploying RL in video games.

\begin{figure}[t]
    \centering
    \begin{tabular}{lc}
        \textbf{a)} & 
        \includegraphics[width=0.66\textwidth]{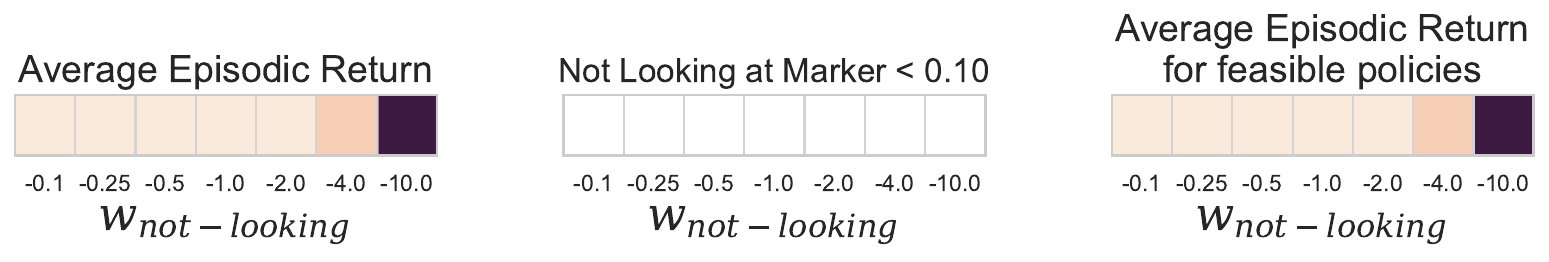}\\
        \textbf{b)} & 
        \includegraphics[width=0.9\textwidth]{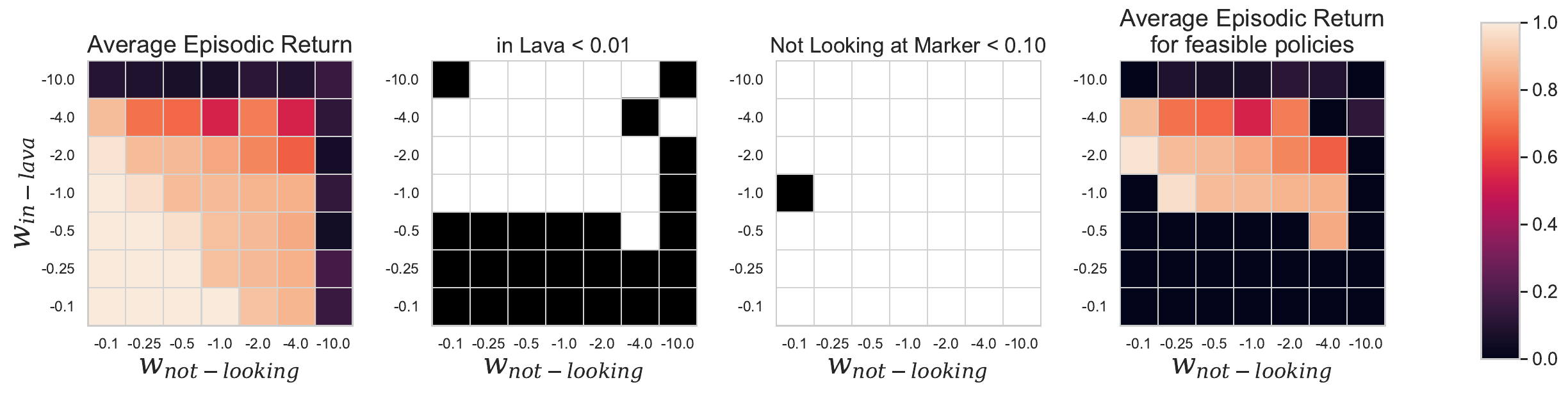}
    \end{tabular}
    \caption[Enforcing behavioral constraints using reward engineering]{Enforcing behavioral constraints using reward engineering. Each grid represents a different metric. Within each grid, each square represents the final performance (according to that metric) of an agent trained for 3M steps using the reward function in Equation~\ref{eq:reward_engineering} parameterised as given by the grid coordinates. Performance is obtained by evaluating the agent on 1000 episodes. The leftmost column indicates the episodic return of the trained policies, the middle columns indicates whether or not the agent respects the behavioral constraint(s) and the rightmost column indicates the average return for these feasible policies only. \textbf{a)} The ``looking-at marker" behavior does not affect too much the main task and, consequently, all chosen weights allow to satisfy the constraint (looking at marker 90\% of the time) and many of them also lead to good performance on the main navigation task ($-0.1 \geq w \geq -2$). \textbf{b)} When also enforcing the ``Not in Lava" behavior, which is much more in the way of the main task, most of the resulting policies do not respect the constraint or perform poorly on the navigation task, highlighting the difficulty of choosing the correct penalty weights ahead of time. On 49 experiments, only two yielded good performing feasible policies: $(-0.10, -2.0)$ and $(-0.25, -1.0)$. On the largest search with 3 behavioral constraints, none of the 343 experiments found a good performing feasible policy (see Figure~\ref{fig:reward_engineering_3constraints} in Appendix~\ref{sec:DBS:additional_experiments}).}
    \label{fig:reward_engineering}
\end{figure}

\section{The problem with reward engineering}
\label{sec:DBS:problem_with_reward_engineering}

In this section, we motivate the impracticality of using reward engineering to shape behavior. 
We consider a navigation task in which the agent has to reach a goal location while being subject to additional behavioral constraints. These constraints are (1) looking at a visible marker $90\%$ of the time, (2) avoiding forbidden terrain $99\%$ of the time and (3) avoiding to run out of energy also $99\%$ of the time. The environment is depicted in Figure~\ref{fig:DBS:envs} (left) and the details are presented in Appendix~\ref{sec:DBS:arena_env_details}. The reward function for this task is of the form:
\begin{equation}
\label{eq:reward_engineering}
R'(s,a) = R(s,a) - \mathbf{1}*w_{\text{not-looking}} - \mathbf{1} * w_{\text{in-lava}} - \mathbf{1} * w_{\text{no-energy}}
\end{equation}
where $R(s,a)$ gives a small shaping reward for progressing towards the goal and a terminal reward for reaching the goal, and the $\mathbf{1}$'s are indicator functions which are only active if their corresponding behavior is exhibited.

The main challenge for an RL practitioner is to determine the correct values of the weights $w_{\text{not-looking}}$, $w_{\text{in-lava}}$ and $w_{\text{no-energy}}$ such that the agent maximises its performance on the main task while respecting the behavioral requirements, a problem often referred to as reward engineering. Setting these weights too low results in an agent that ignores these requirements while setting them too high distracts the agent from completing the main task. In general, knowing how to scale these components relatively to one another is not intuitive and is often performed by trial and error across the space of reward coefficients $w_k$. To illustrate where the desired solutions can be found for this particular problem, we perform 3 grid searches on 7 different values for each of these weights, ranging from 0.1 to 10 times the scale of the main reward function, for the cases of 1, 2 and 3 behavioral constraints. The searches thus respectively must go through 7, 49 and 343 training runs. Figure~\ref{fig:reward_engineering} (and Figure~\ref{fig:reward_engineering_3constraints} in Appendix~\ref{sec:DBS:additional_experiments}) show the results of these experiments. We can see that a smaller and smaller proportion of these trials lead to successful policies as the number of behavioral constraints grows. For an engineer searching to find the right trade-off, they find themselves cornered between two undesirable solutions: an ad-hoc manual approach guided by intuition or to run a computationally demanding grid-search. While expert knowledge or other search strategies can partially alleviate this burden, the approach of reward engineering clearly does not scale as the control problem grows in complexity.

It is important to note that whether or not it is the case that all tasks can in principle be defined as a single scalar function to maximise i.e. the reward hypothesis \cite{sutton2018rlTextbook}, this notion should not be seen as a restrictive design principle when translating a real-life problem into an RL problem. That is because it does not guarantee that this reward function admits a simple form. Rich and multi-faceted behaviors may only be specifiable through a complex reward function \citep{abel2021expressivity} beyond the reach of human intuition. In the next sections we present a practical framework in which CMDPs can be used to provide a more intuitive and human-centric interface for behavioral specification.


\section{Background }
\label{sec:DBS:background}

\paragraph{Markov Decision Processes (MDPs)} \cite{sutton2018rlTextbook} are formally defined through the following four components: $(\mathcal{S}, \mathcal{A}, P, R)$. At timestep $t$, an agent finds itself in state $s_t \in \mathcal{S}$ and picks an action $a_t \in \mathcal{A}(s_t)$. The transition probability function $P$ encodes the conditional probability  $P(s_{t+1}|s_t, a_t)$ of transitioning to the next state $s_{t+1}$. Upon entering the next state, an immediate reward is generated through a reward function $R: \mathcal{S} \times \mathcal{A} \to \mathbb{R}$. In this paper, we restrict our attention to stationary randomized policies of the form $\pi(a|s)$ -- which are sufficient for optimality in both MDPs and CMDPs \cite{altman1999constrained}. The interaction of a policy within an MDP gives rise to trajectories $(s_0, a_0, r_0, \hdots, s_T, a_T, R_T)$ over which can be computed the sum of rewards which we call the \textit{return}. Under the Markov assumption, the probability distribution over trajectories is of the form:
\begin{equation}
    p_\pi(\tau) := P_0(s_0)\prod_{t=0}^T P(s_{t+1}|s_t, a_t) \pi(a_t|s_t)
\end{equation}
where $P_0$ is some initial state distribution. Furthermore, any such policy induces a marginal distribution over state-action pairs referred to as the \textit{visitation distribution} or state-action \textit{occupation measure}:
\begin{equation}
    x_{\pi}(s,a) := \frac{1}{Z(\gamma, T)} \sum_{t=0}^T \gamma^t p_{\pi, t}(S_t=s, A_t=a) \enspace ,
\end{equation}
where $Z(\gamma, T) = \sum_{t=0}^T \gamma^t$ is a normalising constant.

In this paper, it is useful to extend the notion of return to any function $f:\mathcal{S}\times \mathcal{A}\rightarrow \mathbb{R}$ over states and actions other than the reward function of the MDP itself. The expected discounted sum of $f$ then becomes:
\begin{equation}
    J_f(\pi) := \E_{\tau \sim p_{\pi}}\left[ \sum_{t=0}^T \gamma^t f(s_t,a_t) \right]
\end{equation}
where $\gamma \in [0,1]$ is a discount factor. While this idea is the basis for much of the work on General Value Functions (GVFs) \cite{white2015} for predictive state representation \citep{sutton2011horde}, our focus here is on problem of behavior specification and not that of prediction. 

Finally, in the MDP setting, a policy is said to be optimal under the expected discounted return criterion if $\pi^* = \argmax_{\pi \in \Pi} J_R(\pi)$, where $\Pi$ is the set of possible policies. 

\paragraph{Constrained MDPs (CMDPs)} \cite{altman1999constrained} 
is a framework that extends the notion of optimality in MDPs to a scenario where multiple cost constraints need to be satisfied in addition to the main objective. We write $C_k: \mathcal{S}\times\mathcal{A}\rightarrow \mathbb{R}$ to denote such a cost function whose expectation must remain bounded below a specified threshold $d_k \in \mathbb{R}$. 
The set of feasible policies is then:
\begin{equation}
\Pi_C = \{ \pi  \in \Pi: J_{C_k}(\pi) \leq d_k, \, k=1,\dots,K \}.
\end{equation}
Optimal policies in the CMDP framework are those of maximal expected return among the set of feasible policies:
\begin{equation}
\pi^* = \argmax_{\pi \in \Pi} \, J_R(\pi), \, \text{ s.t. } \, \,  J_{C_k}(\pi) \leq d_k \text{ , } \, \, k=1,\dots,K
\end{equation}
While it is sufficient to consider the space of stationary deterministic policies in searching for optimal policies in the MDP setting, this is no longer true in general with CMDPs \citep{altman1999constrained} and we must consider the larger space of stationary randomized policies. 

\paragraph{Lagrangian methods for CMPDs.}
Several recent works have found that the class of Lagrangian methods for solving CMDPs is capable of finding good feasible solutions at convergence  \cite{achiam2017constrained,ray2019benchmarking,stooke2020responsive,zhang2020first}. The basis for this line of work stems from the saddle-point characterisation of the optimal solutions in nonlinear programs with inequality constraints \citep{uzawa, polyak, korpelevich1976extragradient}. Intuitively, these methods combine the main objective $J_R$ and the constraints into a single function $\mathcal{L}$ called the Lagrangian. The relative weight of the constraints are determined by additional variables $\lambda_k$ called the Lagrange multipliers. Applied in our context, this idea leads to the following min-max formulation:
\begin{align}
\label{eq:cmdp_lagrangian}
&\max_\pi \, \, \min_{\lambda\geq 0} \, \mathcal{L}(\pi, \lambda) \notag \\ &\mathcal{L}(\pi, \lambda) = J_R(\pi) - \sum_{k=1}^K \lambda_k (J_{C_k}(\pi) - d_k)
\end{align}
where we denoted $\lambda:=\{\lambda_k\}_{k=1}^K$ for conciseness. \citet{uzawa} proposed to find a solution to this problem iteratively by taking gradient ascent steps of the Lagrangian $\mathcal{L}$ in the variable $\pi$ and descent ones in $\lambda$. This is also the same gradient ascent-descent 
\cite{gda} procedure underpinning many learning algorithms for Generative Adversarial Networks \citep{goodfellow2014generative}.

The maximization of the Lagrangian over the policy variables can be carried out by applying any existing unconstrained policy optimization methods to the new reward function $L: \mathcal{S} 
\times \mathcal{A} \to \mathbb{R}$ where:
\begin{equation}
L(s,a) = R(s,a) - \sum_{k=1}^K \lambda_k C_k(s,a).
\end{equation}
For the gradient w.r.t. the Lagrange multipliers $\lambda$, the term depending on $\pi$ cancels out and we are left with $\nabla_{\lambda_k} \mathcal{L}(\pi, \lambda) = -(J_{C_k}(\pi) - d_k)$. The update is followed by a projection onto $\lambda_k \geq 0$ using the max-clipping operator. If the constraint is violated $(J_{C_k}(\pi) > d_k)$, taking a step in the opposite direction of the gradient will increase the corresponding multiplier $\lambda_k$, thus increasing the relative importance of this constraint in $J_L(\pi)$. Inversely, if the constraint is respected $(J_{C_k}(\pi) < d_k)$, the update will decrease $\lambda_k$, allowing the optimisation process to focus on the other constraints and the main reward function $R$. 


\section{Proposed Framework}

In Reinforcement Learning, the reward function is often assumed to be provided apriori. For example, in most RL benchmarking environments this is indeed the case and researchers can focus on improving current algorithms at finding better policies, faster and more reliably. In industrial applications however, several desiderata are often required for the agent's behavior, and balancing these components into a single reward function is highly non-trivial. In the next sections, we describe a framework in which CMDPs can be used for efficient behavior specification.

\subsection{Indicator cost functions}

The difficulty of specifying the desired behavior of an agent using a single reward function stems from the need to tune the relative scale of each reward component. Moreover, finding the most appropriate ratio becomes more challenging as the number of reward components increases (see Section~\ref{sec:DBS:problem_with_reward_engineering}). While the prioritisation and saturation characteristics of CMDPs help factoring the behavioral specification problem \cite{ray2019benchmarking}, there remains important design challenges. First, the CMDP framework allows for arbitrary forms of cost functions, again potentially leading to unforeseen behaviors. Second, specifying the appropriate thresholds $d_k$ can be difficult to do solely based on intuition. For example, in the mujoco experiments performed by \citet{zhang2020first}, the authors had to run an unconstrained version of PPO \citep{schulman2017proximal} to first estimate the typical range of values for the cost infringements and then run their constrained solver over the appropriately chosen thresholds.

We show here that this separate phase of threshold estimation can be avoided completely if we consider a subclass of CMDPs that allows for a more intuitive connection between the chosen cost functions $C_k$ and their expected returns $J_{C_k}$. More specifically, we restrict our attention to CMDPs where the cost functions are defined as indicators of the form:
\begin{equation}
\label{eq:indicator_cost_functions}
    C_k(s,a) = I(\text{behavior $k$ is met in $(s,a)$})
\end{equation}
which simply expresses whether an agent showcases some particular behavior $k$ when selecting action $a$ in state $s$. An interesting property of this design choice is that, by rewriting the expected discounted sum of these indicator cost functions as an expectation over the visitation distribution of the agent, we can interpret this quantity as a re-scaled probability that the agent exhibits behavior $k$ at any given time during its interactions with the environment:
\begin{align}
    &J_{C_k}(\pi) 
    = \E_{\tau \sim p_{\pi}}\left[ \sum_{t=0}^T \gamma^t C_k(s_t,a_t) \right] \\
    &= Z(\gamma, T) \mathbb{E}_{(s,a)\sim x_{\pi}(s,a)}[C_k(s,a)] \\
    &= Z(\gamma, T) \mathbb{E}_{(s,a)\sim x_{\pi}(s,a)}[I(\text{behavior $k$ met in $(s,a)$})] \\
    &= Z(\gamma, T) Pr\big(\text{behavior $k$ met in $(s,a)$}\big) , \, \, (s,a)\sim x_{\pi}
\end{align}
Dividing each side of $J_{C_k}(\pi) \leq d_k$ by $Z(\gamma, T)$, we are left with $\tilde{d}_k$, a normalized constraint threshold for the constraint $k$ which represents the desired rate of encountering the behavior designated by the indicator cost function $C_k$. In practice, we simply compute the average cost function across the batch to give equal weighting to all state-action pairs regardless of their position $t$ in the trajectory:
\begin{equation}
    \tilde{J}_{C_k}(\pi) := \frac{1}{N}\sum_{i=1}^N C_k(s_i, a_i)
\end{equation}
where $i$ is the sample index from the batch. We also train the corresponding critic $Q^{(k)}$ using a discount factor $\gamma_k<1$ for numerical stability.

While the class of cost functions defined in Equation~\ref{eq:indicator_cost_functions} still allows for modelling a large variety of behavioral preferences, it has the benefit of informing the user on the range of appropriate thresholds -- a probability $\tilde{d}_k\in[0,1]$ -- and the semantics is clear regarding its effect on the agent's behavior (assuming that the constraint is binding and that a feasible policy is found). This effectively allows for minimal to no tuning behavior specification (or ``zero-shot'' behavior specification). 

Finally, indicator cost functions also have the practical advantage of allowing to capture both desired and undesired behaviors without affecting the termination tendencies of the agent. Indeed, when using an arbitrary cost function, it could be tempting to simply flip its sign to enforce the opposite behavior. However, as noted in previous work \cite{kostrikov2018discriminator}, the choice of whether to enforce behaviors through bonuses or penalties should instead be thought about with the termination conditions in mind. A positive bonus could cause the agent to delay termination in order to accumulate more bonuses while negative penalties could shape the agent behavior such that it seeks to trigger the termination of the episode as soon as possible. Indicator cost functions are thus very handy in that they offer a straightforward way to enforce the opposite behavior by simply inverting the indicator function $Not\big(I(s,a)\big) = 1 - I(s,a)$ without affecting the sign of the constraint (penalties v.s. bonuses).

\subsection{Multiplier normalisation}
\label{sec:DBS:mult_norm}

When the constraint $k$ is violated, the multiplier $\lambda_k$ associated with that constraint increases to put more emphasis on that aspect of the overall behavior. While it is essential for the multipliers to be able to grow sufficiently compared to the main objective, a constraint that enforces a behavior which is long to discover can end up reaching very large multiplier values. It then leads to very large policy updates and destabilizes the learning dynamics.

To maintain the ability of one constraint to dominate the policy updates when necessary while keeping the scale of the updates bounded, we propose to normalize the multipliers. This can be readily implemented by using a softmax layer:
\begin{equation}
    \lambda_k = \frac{\exp(z_k)}{\exp(a_0) + \sum_{k'=1}^K \exp(z_{k'})} \, , \quad k=1,\dots,K
\end{equation}
where $z_k$ are the base parameters for each one of the multipliers and $a_0$ is a dummy variable used to obtain a normalized weight $\lambda_0 := 1 - \sum_{k=1}^K \lambda_k$ for the main objective $J_R(\pi)$. The corresponding min-max problem becomes:
\begin{align}
\label{eq:cmdp_lagrangian_normalised_multipliers}
&\max_\pi \, \, \min_{z_{1:K}\geq 0} \, \mathcal{L}(\pi, \lambda) \notag \\
&\mathcal{L}(\pi, \lambda) = \lambda_0 J_R(\pi) - \sum_{k=1}^K \lambda_k (J_{C_k}(\pi) - d_k)
\end{align}
\subsection{Bootstrap Constraint}
In the presence of many constraints, one difficulty that emerges with the above multiplier normalisation is that the coefficient of the Lagrangian function that weighs the main objective is constrained to be  $\lambda_0 = 1 - \sum_{k=1}^K \lambda_k$, which leaves very little to no traction to improve on the main task while the process is looking for a feasible policy. Furthermore, as more constraints are added, the optimisation path becomes discontinuous between regions of feasible policies, preventing learning progress on the main task objective.

A possible solution is to grant the main objective the same powers as the behavioral constraints that we are trying to enforce. This can be done by defining an additional function $S_{K+1}(s,a)$ which captures some measure of success on the main task. Indeed, many RL tasks are defined in terms of such sparse, clearly defined success conditions, and then often only augmented with a dense reward function to guide the agent toward these conditions \cite{ng1999policy}. A so-called \textit{success constraint} of the form $J_{S_{K+1}}(\pi) \geq \tilde{d}_{K+1}$ can thus be implemented using an indicator cost function as presented above and added to the existing constraint set $\{J_{C_k}(\pi) \leq \tilde{d}_{k}\}_{k=1}^K$. While the use of a success constraint alone can be expected to aid learning of the main task, it is only a sparse signal and could be very difficult to discover if the main task is itself challenging. Since the success function $S_{K+1}$ is meant to be highly correlated with the reward function $R$, by going a step further and using the success constraint multiplier $\lambda_{K+1}$ in place of the reward multiplier $\lambda_0$, we can take full advantage of the density of the main reward function when enforcing that constraint. However, to maintain a true maximisation objective over the main reward function, we still need to keep using $\lambda_0$  when other constraints are satisfied, so that the most progress can be made on $J_R(\pi)$. We thus take the largest of these two coefficients for weighing the main objective $\tilde{\lambda}_0 := \max\big(\lambda_0, \lambda_{K+1}\big)$ and replace $\lambda_0$ with $\tilde{\lambda}_0$ in Equation~\ref{eq:cmdp_lagrangian_normalised_multipliers}. Here we say that constraint $K$~$+$~$1$ is used as a \textit{bootstrap constraint}.

Our method of encoding a success criterion in the constraint set can be seen as a way of relaxing the behavioral constraints during the optimisation process without affecting the convergence requirements. For example, in previous work, \citet{calian2020balancing} tune the learning rate of the Lagrange multipliers to automatically turn some constraints into soft-constraints when the agent is not able to satisfy them after a given period of time. Instead, the bootstrap constraint allows to start making some progress on the main task without turning our hard constraints into soft constraints.


\section{Related Work}
\label{sec:DBS:related_work}

\paragraph{Constrained Reinforcement Learning.} CMDPs \cite{altman1999constrained} have been the focus of several previous work in Reinforcement Learning. Lagrangian methods \cite{borkar2005actor,tessler2018reward,stooke2020responsive} combine the constraints and the main objective into a single function and seek to find a saddle point corresponding to feasible solutions to the maximisation problem. Projection-based methods \cite{achiam2017constrained,chow2019lyapunov,yang2020projection,zhang2020first} instead use a projection step to try to map the policy back into a feasible region after the reward maximisation step. While most of these works focus on the single-constraint case \cite{zhang2020first,dalal2018safe,calian2020balancing,stooke2020responsive} and seek to minimize the total regret over the cost functions throughout training \cite{ray2019benchmarking}, we focus on the potential of CMDPs for precise and intuitive behavior specification and work on satisfying many constraints simultaneously.

\paragraph{Reward Specification.} Imitation Learning \cite{zheng2021imitation} is largely motivated by the difficulty of designing reward functions and instead seeks to use expert data to define the task. Other approaches introduce a human in the loop to either guide the agent towards the desired behavior \cite{christiano2017deep} or to prevent it from making catastrophic errors while exploring the environment \cite{saunders2017trial}. While our approach of using CMDPs for behavior specification also seeks to make better use of human knowledge, we focus on the idea of providing this knowledge by simply specifying thresholds and indicator functions rather than requiring expert demonstrations or constant human feedback. Another line of work studies whether natural language can be used as a more convenient interface to specify the agent's desired behavior \cite{goyal2019using,macglashan2015grounding}. While this idea presents interesting perspectives, natural language is inherently ambiguous and prone to reward hacking by the agent. Moreover such approaches generally come with the added complexity of having to learn a language-to-reward model. Finally, others seek to solve reward mis-specification through Inverse Reward Design \cite{hadfield2017inverse,mindermann2018active,ratner2018simplifying} which treats the provided reward function as a single observation of the true intent of the designer and seeks to learn a probabilistic model that explains it. While this approach is interesting for adapting to environmental changes, we focus on behavior specification in fixed-distribution environments.

\paragraph{RL in video games.} Video games have been used as a benchmark for Deep RL for several years \cite{shao2019survey,berner2019dota,vinyals2019grandmaster}. However, examples of RL being used in a video game production are limited due to a variety of factors which include the difficulty of shaping behavior, interpretability, and compute limitations at run-time \cite{jacob2020s,alonso2020deep}. Still, there has been a recent push in the video game industry to build NPCs (Non Player Characters) using RL, for applications including navigation \cite{alonso2020deep,devlin2021navigation}, automated testing \cite{bergdahl2020augmenting,gordillo2021improving}, play-style modeling \cite{de2021configurable} and content generation \cite{gisslen2021adversarial}.


\section{Experiments}

To evaluate the proposed framework, we train SAC agents \cite{haarnoja2018soft} to solve navigation tasks with up to 5 constraints imposed on their behavior. Many of these constraints interact with the main task and with one another which significantly restricts the space of admissible policies. We conduct most of our experiments in the Arena environment (see Figure~\ref{fig:DBS:envs}, left)\footnote{The algorithm is presented in Appendix~\ref{sec:DBS:algorithm}. The code for the Arena environment experiments is available at:\newline\scriptsize{\url{https://github.com/ubisoft/DirectBehaviorSpecification}}} where we seek to verify the capacity of the proposed framework to allow for easy specification of the desired behavior and the ability of the algorithm to deal with a large number of constraints simultaneously. We also perform an experiment in the OpenWorld environment (see Figure~\ref{fig:DBS:envs}, right), a much larger and richer map generated using the GameRLand map generator \cite{beeching2021graph}, where we seek to verify the scalability of that approach and whether it fits the needs of agent behavior specification for the video game industry. See Appendices~\ref{sec:DBS:arena_env_details}~and~\ref{sec:DBS:openworld_env_details} for a detailed description of both experimental setups.

\subsection{Experiments in the Arena environment}

\paragraph{Multiplier Normalization}

Our first set of experiments showcases the effect of normalizing the Lagrange multipliers. For illustrative purposes, we designed a simple scenario where one of the constraints is not satisfied for a long period of time. Specifically, the agent is attempting to satisfy an impossible constraint of never touching the ground. Figure~\ref{fig:multiplier_normalisation_experiment} (in red) shows that the multiplier on the unsatisfied constraint endlessly increases in magnitude, eventually harming the entire learning system; the loss on the critic diverges and the performance collapses. When using our normalization technique, Figure~\ref{fig:multiplier_normalisation_experiment} (in blue) shows that the multiplier and critic losses remain bounded, avoiding such instabilities.

\begin{figure}[htb]
    \centering
    \includegraphics[width=0.9\textwidth]{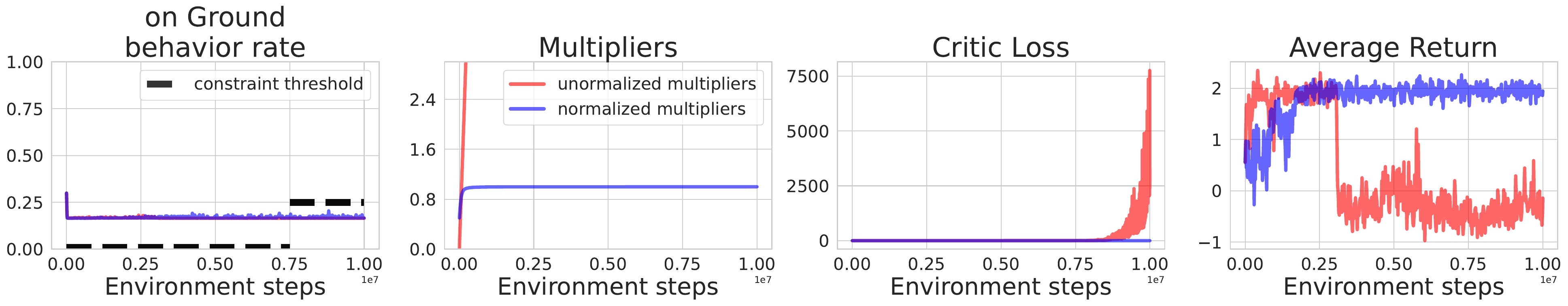}
    \caption[Effect of the multiplier normalization]{The multiplier normalisation keeps the learning dynamics stable when discovering a constraint-satisfying behavior takes a large amount of time. To simulate such a case, an impossible constraint is set for 7.5M steps and then replaced by a feasible one for the last 2.5M steps. The method using unormalized multipliers (red) keeps taking larger and larger steps in policy space leading to the divergence of its learning dynamics and complete collapse of its performance.}
    \label{fig:multiplier_normalisation_experiment}
\end{figure}

\begin{figure}[htb]
    \centering
    \includegraphics[width=0.9\textwidth]{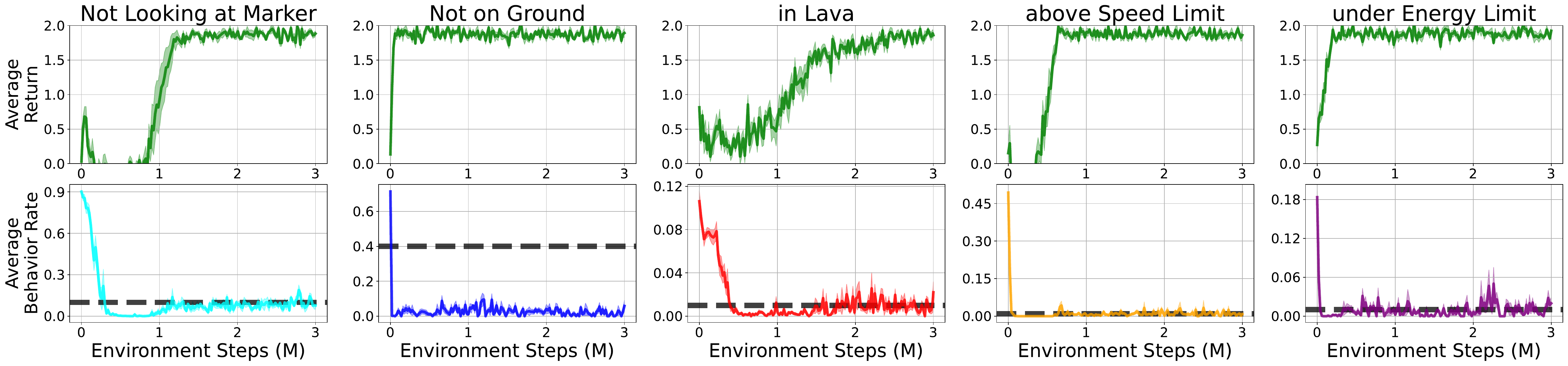}
    \caption[Results with single constraints in the Arena environment]{Each \textit{column} presents the results for an experiment in which the agent is trained for 3M steps with a \textit{single constraint} enforced on its behavior. Training is halted after every $20,000$ environment steps and the agent is evaluated for 10 episodes. All curves show the average over 5 seeds and envelopes show the standard error around that mean. The top row shows the average return, the bottom row shows the average behavior rate on which the constraint is enforced. The black doted lines mark the constraint thresholds.}
    \label{fig:single_constraint_experiments}
\end{figure}

\paragraph{Single Constraint satisfaction}

We use our framework to encode the different behavioral preferences into indicator functions and specify their respective thresholds. Figure \ref{fig:single_constraint_experiments} shows that our SAC-Lagrangian with multiplier normalisation can solve the task while respecting the behavioral requirements when imposed with constraints individually. We note that the different constraints do not affect the main task to the same extent; while some still allow to quickly solve the navigation task, like the behavioral requirement to avoid jumping, others make the navigation task significantly more difficult to solve, like the requirement to avoid certain types of terrain (lava).

\begin{figure}[t]
    \centering
    \includegraphics[width=\textwidth]{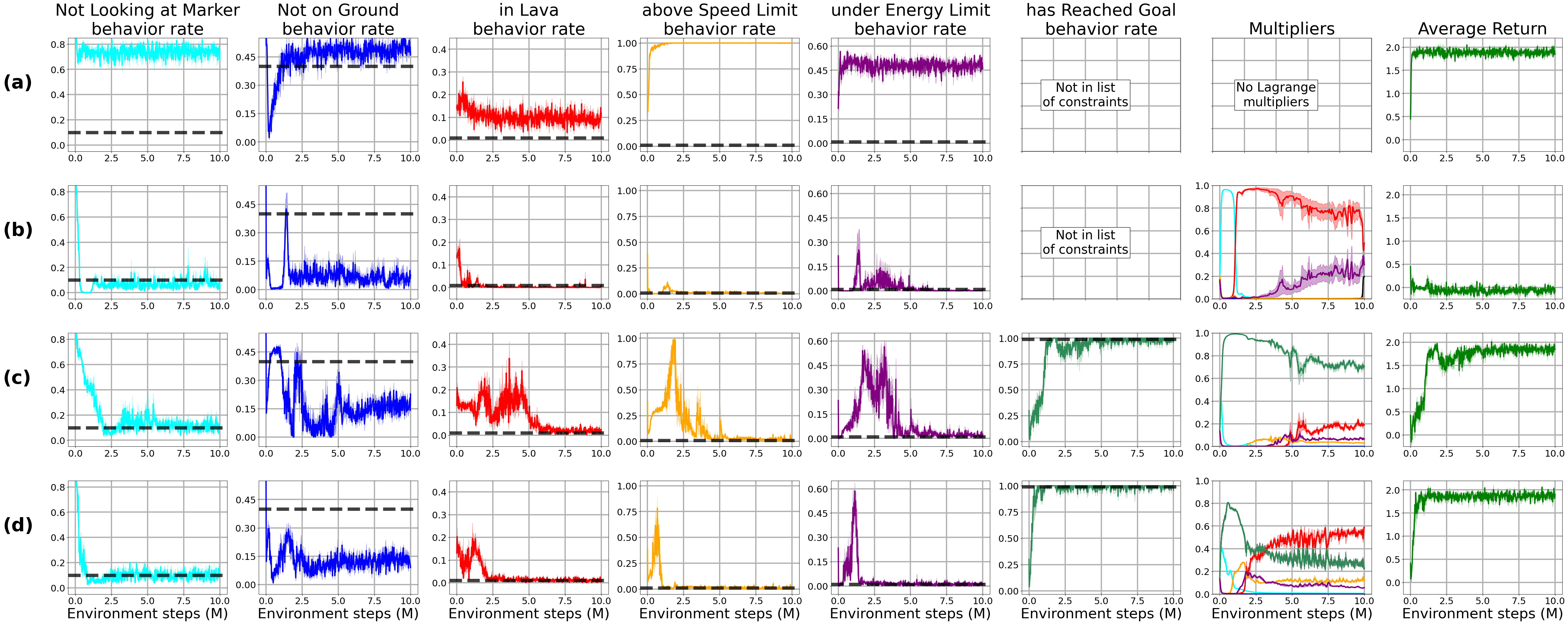}
    \caption[Results with multiple constraints in the Arena environment]{Each \textit{row} presents the results of an experiment in which an agent is trained for 10M steps. Training is halted after every $20,000$ environment steps and the agent is evaluated for 10 episodes.  All curves show the average over 5 seeds and envelopes show the standard error around that mean. \textbf{(a)} Unconstrained SAC agent; none of the behavioral preferences are enforced and consequently improvement on performance is very fast but none of the constraints are satisfied. \textbf{(b)} SAC-Lagrangian with the 5 behavioral constraints enforced. While each constraint was successfully dealt with when imposed one by one (see Figure~\ref{fig:single_constraint_experiments}), maximising the main objective when subject to all the constraints \textit{simultaneously} proves to be much harder. The agent does not find a policy that improves on the main task while keeping the constraints in check. \textbf{(c)} By using an additional success constraint (that the agent should reach its goal in 99\% of episodes), the agent can cut through infeasible policy space to start improving on the main task and optimise the remaining constraints later on. \textbf{(d)} By using the success constraint as a bootstrap constraint (bound to the main reward function) improvement on the main task is much faster as the agent benefits from the dense reward function to improve on the goal-reaching task.}
    \label{fig:many_constraints_experiments}
\end{figure}

\paragraph{Multiple Constraints Satisfaction}

In Figure \ref{fig:many_constraints_experiments} we see that when imposed with all of the constraints simultaneously, the agent learns a feasible policy but fails at solving the main task entirely. The agent effectively settles on a trivial behavior in which it only focuses on satisfying the constraints, but from which it is very hard to move away without breaking the constraints. By introducing a success constraint, the agent at convergence is able to satisfy all of the constraints as well as succeeding in the navigation task. This additional incentive to traverse infeasible regions of the policy space allows to find feasible but better performing solutions. Our best results are obtained when using the success constraint as a bootstrap constraint, effectively lending $\lambda_{K+1}$ to the main reward while the agent is still looking for a feasible policy.

\subsection{Experiment in the OpenWorld environment}

\begin{figure}[t]
    \centering
    \includegraphics[width=\textwidth]{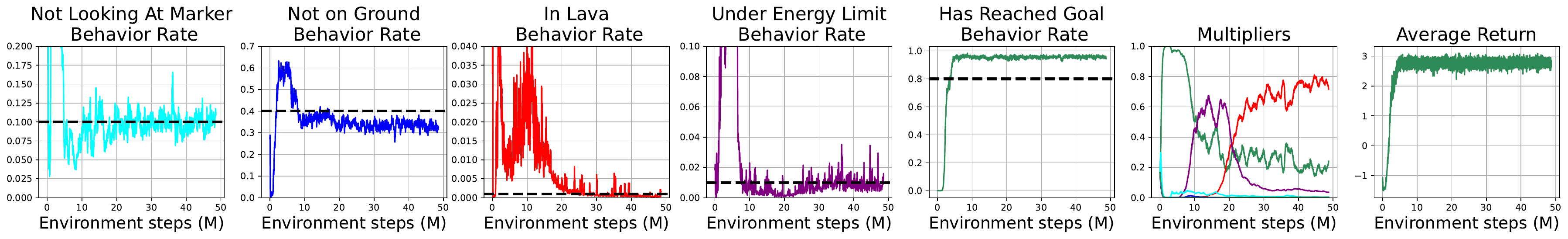}
    \caption[Results in the OpenWorld environment]{A SAC-Lagrangian agent trained to solve the navigation problem in the OpenWorld environment while respecting four constraints and imposing the bootstrap constraint. Results suggest that our SAC-Lagrangian method using indicator cost functions, normalised multipliers and bootstrap constraint scales well to larger and more complex environments.}
    \label{fig:open_world_exp}
\end{figure}

In the OpenWorld environment, we seek to verify that the proposed solution scales well to more challenging and realistic tasks. Contrarily to the Arena environment, the OpenWorld contains uneven terrain, buildings, and interactable objects like jump-pads, which brings this evaluation setting much closer to an actual RL application in the video game industry. For this experiment, we trained a SAC-Lagrangian agent to solve the navigation problem with four constraints on its behavior: \textit{On-Ground}, \textit{Not-In-Lava}, \textit{Looking-At-Marker} and \textit{Above-Energy-Limit}. The SAC component uses the same hyperparameters as in \citet{alonso2020deep}. The results are shown in Figure~\ref{fig:open_world_exp}. While training the agent in this larger and more complex environment now requires up to 50M environment steps, the agent still succeeds at completing the task and respecting the constraints, favourably supporting the scalability of the proposed framework for direct behavior specification. 


\section{Discussion}

Our work showed that CMDPs offer compelling properties when it comes to task specification in RL. More specifically, we developed an approach where the agent's desired behavior is defined by the frequency of occurrence for given indicator events, which we view as constraints in a CMDP formulation. We showed through experiments that this methodology is preferable over the reward engineering alternative where we have to do an extensive hyperparameter search over possible reward functions. We evaluated this framework on the many constraints case in two different environments. Our experiments showed that simultaneously satisfying a large number of constraints is difficult and can systematically prevent the agent from improving on the main task. We addressed this problem by normalizing the constraint multipliers, which resulted in improved stability during training and proposed to bootstrap the learning on the main objective to avoid getting trapped by the composing constraint set. 
This bootstrap constraint becomes a way for practitioners to incorporate prior knowledge about the task and desired result -- if the threshold is strenuous, a high success is prioritized -- if the threshold is lax, it will simply be used to exit the initialisation point and the other constraints will quickly takeover.
Our overall method is easy to implement over existing policy gradient code bases and can scale across domains  easily. 

We hope that these insights can contribute to a wider use of Constrained RL methods in industrial application projects, and that such adoption can be mutually beneficial to the industrial and research RL communities.


\section*{Acknowledgments}

We wish to thank Philippe Marcotte, Maxim Peter, Rémi Labory, Pierre Le Pelletier De Woillemont, Julien Varnier, Pierre Falticska, Gabriel Robert, Vincent Martineau, Olivier Pomarez, Tristan Deleu and Paul Barde as well as the entire research team at Ubisoft La Forge for providing technical support and insightful comments on this work. We also acknowledge funding in support of this work from Fonds de Recherche Nature et Technologies (FRQNT), Mitacs Accelerate Program, Institut de valorisation des données (IVADO) and Ubisoft La Forge.

\chapter[ARTICLE 4: GOAL-CONDITIONED GFLOWNETS FOR\\CONTROLLABLE MULTI-OBJECTIVE MOLECULAR DESIGN]{\\ARTICLE 4: GOAL-CONDITIONED GFLOWNETS FOR CONTROLLABLE MULTI-OBJECTIVE \\MOLECULAR DESIGN}\label{chap:article4_gcgfn}

\begin{center}
Co-authors\\Pierre-Luc Bacon, Christopher Pal \& Emmanuel Bengio
\end{center}
\vspace{-5mm}
\begin{center}
Presented at\\Workshop on Challenges in Deployable Generative AI\\at the International Conference on Machine Learning, June 23, 2023
\end{center}

\begin{abstract}
In recent years, \textit{in-silico} molecular design has received much attention from the machine learning community. When designing a new compound for pharmaceutical applications, there are usually multiple properties of such molecules that need to be optimized: binding energy to the target, synthesizability, toxicity, EC50, and so on. While previous approaches have employed a scalarization scheme to turn the multi-objective problem into a \textit{preference-conditioned} single objective, it has been established that this kind of reduction may produce solutions that tend to slide towards the extreme points of the objective space when presented with a problem that exhibits a concave Pareto front. In this work we experiment with an alternative formulation of \textit{goal-conditioned} molecular generation to obtain a more controllable conditional model that can uniformly explore solutions along the entire Pareto front.
\end{abstract}

\section{Introduction}
\label{sec:GCGFN:introduction}

Modern Multi-Objective optimization (MOO) is comprised of a large number of paradigms \citep{keeney1993decisions,miettinen2012nonlinear} intended to solve the problem of trading off between different objectives; a setting particularly relevant to molecular design \citep{jin2020multi,jain2022multi}. One particular paradigm that integrates well with recent discrete deep-learning based MOO is \emph{scalarization} \citep{ehrgott2005multicriteria,pardalos2017non}, which transforms the problem of discovering the Pareto front of a problem into a \emph{family} of problems, each defined by a set of coefficients over the objectives. One notable issue with such approaches is that the solution they give tends to depend on the \emph{shape} of the Pareto front in objective space \citep{emmerich2018tutorial}.

To tackle this problem, we propose to train models which explicitly target specific regions in \emph{objective space}. Taking inspiration from goal-conditional reinforcement learning \citep{schaul2015universal}, we condition GFlowNet \citep{bengio2021flow,bengio2023gflownet} models on a description of such \textit{goal regions}. Through the choice of distribution over these goals, we enable users of these models to have more fine-grained control over trade-offs. We also find that assuming proper coverage of the goal distribution, goal-conditioned models discover a more complete and higher entropy approximation of the Pareto front than the scalarization approach.

\section{Background \& Related Work}
\label{sec:GCGFN:related_work}

The \textbf{Multi-Objective optimisation} problem can be broadly described as the desire to maximize a set of $K$ objectives over $\mathcal{X}$, $\mathbf{R}(x)\in\mathbb{R}^K$. In typical MOO problems, there is no single optimal solution $x$ such that $R_k(x)>R_k(x') \, \forall \, k, x'$. Instead, the solution set is generally composed of \emph{Pareto optimal} points, which are points $x$ that are not \emph{dominated} by any other point, i.e. $\nexists \, x'\mbox{ s.t. }R_k(x)\geq R_k(x') \, \forall \, k$. In other words, a point is Pareto optimal if it cannot be locally improved. The projection in objective space of the set of Pareto optimal points forms the so-called \emph{Pareto front}. 

As graph-based models improve \citep{rampavsek2022recipe} and more molecular data become available \citep{wu2018moleculenet}, molecular design has become an active field of research within the deep learning community \citep{brown2019guacamol,huang2021therapeutics}, and core to this research is the fact that molecular design is a fundamentally multi-objective search problem \citep{papadopoulos2006multiobjective,brown2006novel}. The advent of such tools has led to various important work at the intersection of these two fields \citep{zhou2019optimization,staahl2019deep,jin2020multi,jain2022multi}.

The \textbf{Generative Flow Network} (GFlowNet, GFN) framework is a recently introduced method to train energy-based generative models  \citep[i.e. models that learn $p_\theta(x) \propto R(x)$; ][]{bengio2021flow}. They have now been successfully applied to a variety of settings such as biological sequences \citep{jain2022biological}, causal discovery \citep{deleu2022bayesian,atanackovic2023dyngfn}, discrete latent variable modeling \citep{hu2023gflownet}, and computational graph scheduling \citep{zhang2023robust}. The framework itself has also received theoretical attention \citep{bengio2023gflownet}, for example, highlighting its connections to variational methods \citep{zhang2022unifying,malkin2022gflownets}, and several objectives to train GFNs have been proposed \citep{malkin2022trajectory,madan2022learning,pan2023better} including extensions to continuous domains \citep{lahlou2023theory}. 

In the context of molecular design, GFlowNets have several important properties that make them an interesting method for this task. Notably, they are naturally well-suited for discrete compositional object generation, and their multi-modal modeling capabilities allow them to induce greater state space diversity in the solutions they find than previous methods. A recent GFN-based approach to multi-objective molecular design, which we call \textit{preference-conditioning} \cite{jain2022multi}, amounts to scalarizing the objective function by using a set of weights (or preferences) $w$:
\begin{equation}
R_w(x) = \sum_k w_k r_k \quad , \quad \sum_k w_k = 1 \quad , \quad w_k \geq 0
\end{equation}
and then passing this preference vector $w$ as input to the model. By sampling various $w$'s from a distribution such as Dirichlet's during training, one can obtain a model that can be conditioned to emphasize some preferred dimensions of the reward function. \citet{jain2022multi} also find that such a method finds diverse candidates in both state and objective spaces.

\section{Methods}
\label{sec:GCGFN:methods}

\subsection{Goal-conditioned GFlowNets}
\label{sec:GCGFN:goal-conditioned-gfn}

\begin{figure}[t]
    \centering
    \centering
    \begin{minipage}{.33\textwidth}
        \centering
        \includegraphics[height=4cm]{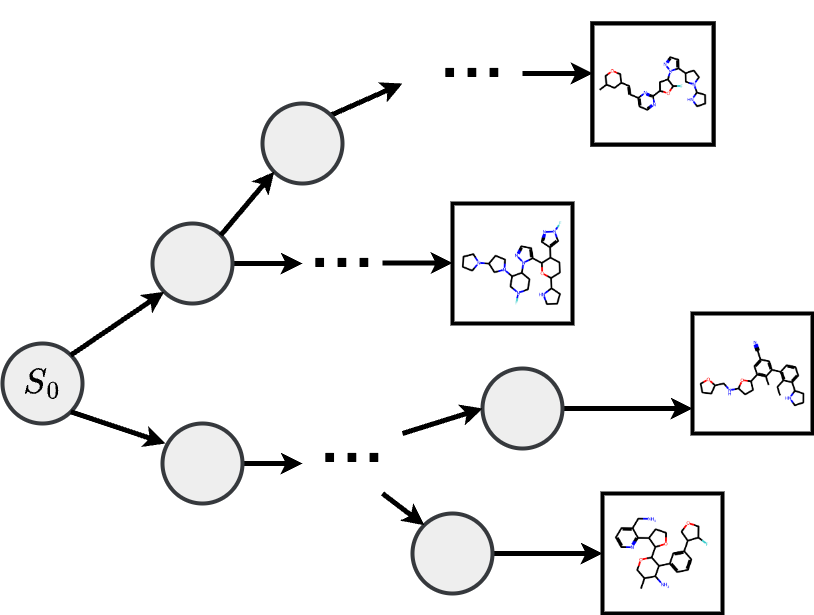}
    \end{minipage}%
    \begin{minipage}{0.66\textwidth}
        \centering
        \includegraphics[height=4cm]{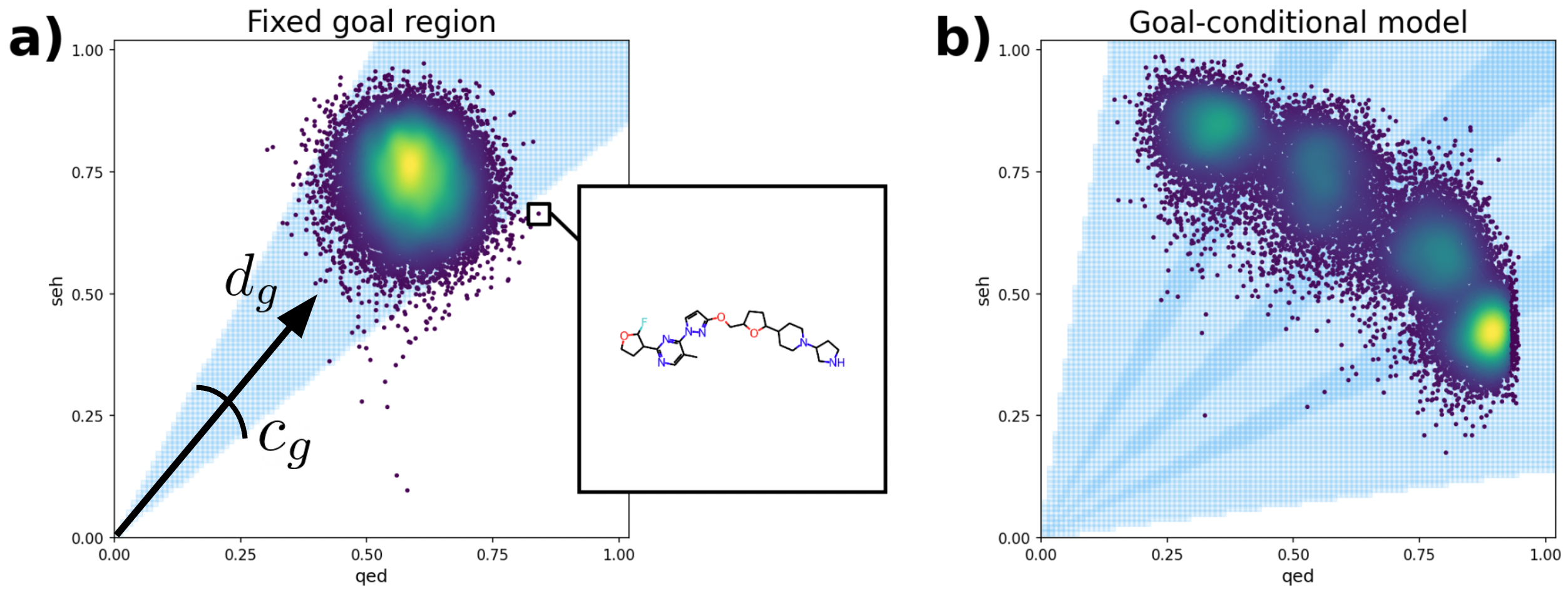}
    \end{minipage}
    \caption[Goal-Conditioned GFlowNets for molecular design]{The diagram on the left depicts the state space of a GFlowNet molecule generator which learns a forward policy that sequentially builds diverse molecules. \textbf{a)} The sampling distribution learned by such a model on a two-objective problem (seh, qed). Each dot represents a molecule's image in the objective space. The focus region (see Section~\ref{sec:GCGFN:goal-conditioned-gfn}) is depicted as a light blue cone, and the colors represent the density of the distribution. The model learns to produce molecules that mostly belong within the focus region. \textbf{b)} By training a goal-conditioned GFlowNet and sampling from several focus regions (here showing 4 distinct regions), we can cover a wider section of the objective space and increase the diversity of proposed candidates.}
    \label{fig:method_overview}
\end{figure}

Building on the method of \citet{jain2022multi}, our approach also formulates the problem as a conditional generative task but now imposes a hard constraint on the model: the goal is to generate samples for which the image in objective space falls into the specified goal region. While many different goal-design strategies could be employed, we take inspiration from \citet{lin2019pareto} and state that a sample $x$ meets the specified goal $g$ if the cosine similarity between its reward vector $r$ and the goal direction $d_g$ is above the threshold $c_g$: $g := \{r \in \mathbb{R}^K: \frac{r \cdot d_g}{||r||\cdot||d_g||} \geq c_g \}$. We call such a goal a \textit{focus region}, which represents a particular choice of trade-off in the objective space (see Figure~\ref{fig:method_overview}). The method can be considered a form of goal-conditional reinforcement learning \cite{schaul2015universal}, where the reward function $R_g$ depends on the current goal $g$. In our case we have:
\begin{equation}
\label{eq:goal-reward}
    R_g(x) = 
    \begin{cases}
    \sum_k r_k ,& \text{if } r \in g\\
    0,              & \text{otherwise}
    \end{cases}
\end{equation}
To alleviate the effects of the now increased sparsity of the reward function $R_g$, we use a replay buffer which proved to stabilise the learning dynamics of our models (see Appendix~\ref{app:GCGFN:replay_buffer}). Notably, by explicitly formulating a goal, we can measure the \textit{goal-reaching accuracy} of our model, which refers to the proportion of samples that successfully landed in their prescribed region. This measurement enables us to employ hindsight experience replay \citep{andrychowicz2017hindsight}, which lets the model learn from the sampled trajectories that didn't meet their goal. Finally, to further increase the goal-reaching accuracy we sharpen the reward function's profile to help the model generate samples closer to the center of the focus region (see Appendix~\ref{app:GCGFN:focus_limit_coef}). 

\subsection{Learned Goal Distribution}
\label{sec:GCGFN:learned_goal_distribution}

Preference conditioning uses soft constraints to steer the model in some regions of the objective space. While hard constraints provide a more explicit way of incorporating the user's intentions in the model \citep{amodei2016concrete, roy2021direct}, they come with the unique challenge that not every goal may be feasible. In such cases, the model will only observe samples with a reward of 0 and thus return molecules of little interest drawn uniformly across the state space. These ``bad samples'' are not harmful in themselves and can easily be filtered out. Still, their prominence will affect the sampling efficiency of goal-conditioned approaches compared to their soft-constrained counterpart. Moreover, the number of infeasible regions will likely multiply as the number of objectives grows, further aggravating this disparity. To cope with this challenge, we propose to use a simple \textit{tabular goal-sampler} (Tab-GS) which maintains a belief about whether any particular goal direction $d_g$ is feasible. Once learned, we can start drawing new goals from it with a much lower likelihood on the goals that are believed to be infeasible, thus restoring most of the lost sample efficiency. We give more details on this approach in Appendix~\ref{app:GCGFN:learned_goal_model} and use it in our experiments in Section~\ref{sec:GCGFN:increasing_number_of_objectives}.

\subsection{Evaluation Metrics}
\label{sec:GCGFN:evaluation_metrics}

While there exists many multi-objective scoring functions to choose from, any single metric only partially captures the desirable properties of the learned generative distribution \citep{audet2021performance}. In this work, we focus on sampling high-performing molecules across the entire Pareto front in a controllable manner at test time. With that in mind, we propose combining three metrics to evaluate our solution. The first one, the Inverted Generational Distance (IGD) \citep{coello2005solving}, uses a set of reference points $P$ (the \textit{true} Pareto front) and takes the average of the distance to the closest generated sample for each of these points:
$\text{IGD}(S, P) := \frac{1}{|P|} \sum_{p \in P} \min_{s \in S} ||s - p||_2$ where $S=\{s_i\}_{i=1}^N$ is the image in objective space of a set of $N$ generated molecules $s_i$. When the true Pareto front is unknown, we use a discretization of the extreme faces of the objective space hypercube as reference points. IGD thus captures the width and depth at which our Pareto front approximation reaches out in the objective space. The second metric, which we call the Pareto-Clusters Entropy (PC-ent), measures how uniformly distributed the samples are along the true Pareto front. To accomplish this, we use the same reference points $P$ as for IGD, and cluster together in the subset $S_j$ all of the samples $s_i$ located closer to the reference point $p_j$ than any other reference point. PC-ent computes the entropy of the histogram of each counts $|S_j|$, reaching its maximum value of $-\log \frac{1}{|P|}$ when all the samples are uniformly distributed relative to the true Pareto front:
$\text{PC-ent}(S, P) := - \sum_{j} \frac{|S_j|}{|P|} \log \frac{|S_j|}{|P|}$. Finally, to report on the \textit{controllability} of the compared methods, we measure the Pearson correlation coefficient (PCC) between the conditional vector $c$ (goal or preference) and the resulting reward vector $s$, averaged across objectives $k$:
$\text{Avg-PCC}(S, C) := \frac{1}{K}\sum_{k=1}^K\text{PCC}(s_{\cdot,k}, c_{\cdot,k})$.

\section{Results}
\label{sec:GCGFN:results}

\subsection{Evaluation Tasks}
\label{sec:GCGFN:task}

\begin{table}[b!]
    \centering
    \caption[Evaluation of conditioning methods with 2 objectives]{Comparisons according to IGD, Avg-PCC and PC-ent between preference-conditioned and goal-conditioned GFNs on a set of increasingly difficult objective landscapes, metrics reported on 3 seeds (mean $\pm$ sem).}
    \vspace{2mm}
    \renewcommand{\arraystretch}{1.3}
    \resizebox{\textwidth}{!}{ 
    \begin{tabular}{|r|c|c c c c c c c|}
        \cline{2-9}
         \multicolumn{1}{c|}{} & \textbf{algorithm} & \textbf{unrestrained} & \textbf{restrained-convex} & \textbf{concave} & \textbf{concave-sharp} & \textbf{multi-concave} & \textbf{4-dots} & \textbf{16-dots}\\
        \hline
        \multirow{2}{*}{IGD $(\downarrow)$} & pref-cond & \cellcolor{lightgray}$\mathbf{0.087 \pm 0.001}$ & $0.316 \pm 0.002$ & $0.272 \pm 0.001$ & \cellcolor{lightgray}$\mathbf{0.180 \pm 0.002}$ & \cellcolor{lightgray}$\mathbf{0.152 \pm 0.006}$ & $0.130 \pm 0.011$ & $0.109 \pm 0.009$ \\
        \cline{2-9}
         & goal-cond & $0.095 \pm 0.002$ & \cellcolor{lightgray}$\mathbf{0.310 \pm 0.001}$ & \cellcolor{lightgray}$\mathbf{0.266 \pm 0.001}$ & $0.197 \pm 0.002$ & $0.173 \pm 0.004$ & $0.134 \pm 0.002$ & $0.115 \pm 0.004$ \\
        \hline
        \hline
        \multirow{2}{*}{Avg-PCC $(\uparrow)$} & pref-cond & $0.905 \pm 0.001$ & $0.673 \pm 0.009$ & $0.830 \pm 0.002$ & $0.855 \pm 0.004$ & $0.700 \pm 0.009$ & $0.768 \pm 0.038$ & $0.770 \pm 0.011$ \\
        \cline{2-9}
         & goal-cond & \cellcolor{lightgray}$\mathbf{0.967 \pm 0.002}$ & \cellcolor{lightgray}$\mathbf{0.953 \pm 0.001}$ & \cellcolor{lightgray}$\mathbf{0.926 \pm 0.002}$ & \cellcolor{lightgray}$\mathbf{0.915 \pm 0.001}$ & \cellcolor{lightgray}$\mathbf{0.946 \pm 0.004}$ & \cellcolor{lightgray}$\mathbf{0.928 \pm 0.002}$ & \cellcolor{lightgray}$\mathbf{0.948 \pm 0.001}$ \\
        \hline
        \hline
        \multirow{2}{*}{PC-ent $(\uparrow)$} & pref-cond & $2.170 \pm 0.004$ & $1.913 \pm 0.019$ & $1.563 \pm 0.009$ & $1.629 \pm 0.002$ & $1.867 \pm 0.015$ & $1.521 \pm 0.022$ & $1.610 \pm 0.019$ \\
        \cline{2-9}
         & goal-cond & \cellcolor{lightgray}$\mathbf{2.472 \pm 0.006}$ & \cellcolor{lightgray}$\mathbf{2.242 \pm 0.013}$ & \cellcolor{lightgray}$\mathbf{1.997 \pm 0.002}$ & \cellcolor{lightgray}$\mathbf{1.918 \pm 0.001}$ & \cellcolor{lightgray}$\mathbf{2.380 \pm 0.020}$ & \cellcolor{lightgray}$\mathbf{2.270 \pm 0.025}$ & \cellcolor{lightgray}$\mathbf{2.262 \pm 0.014}$ \\
        \hline
    \end{tabular}
    }
\label{tab:difficult_pareto_fronts}
\end{table}

We primarily experiment on a two-objective task, the well-known drug-likeness heuristic QED \citep{bickerton2012quantifying}, which is already between 0 and 1, and the sEH binding energy prediction of a pre-trained publicly available model \citep{bengio2021flow}; we divide the output of this model by 8 to ensure it will likely fall between 0 and 1 (some training data goes past values of 8). For 3 and 4 objective tasks, we use a standard heuristic of synthetic accessibility \citep{ertl2009estimation} and a penalty for compounds exceeding a molecular weight of 300. See Appendix~\ref{app:GCGFN:training_details} for all task and training details.

\subsection{Comparisons in Difficult Objective Landscapes}
\label{sec:GCGFN:difficult_landscapes}

To simulate the effect of complexifying the objective landscape while keeping every other parameter of the evaluation fixed, we incorporate \textit{unreachable regions}, depicted in dark in Figures~\ref{fig:difficult_pareto_fronts_alignment}~\&~\ref{fig:difficult_pareto_fronts_density}, by simply setting to \textit{null} the reward function of any molecule whose image in the objective space would fall into these dark regions. We can see that the preference-conditioned approach can effectively solve problems exhibiting a convex pareto-front (Figure~\ref{fig:difficult_pareto_fronts_alignment}~\&~\ref{fig:difficult_pareto_fronts_density}, columns 1-2). However, it is far less effective on problems exhibiting more complex objective landscapes. When faced with a concave Pareto front, the algorithm favours solutions towards the extreme ends (Figure~\ref{fig:difficult_pareto_fronts_alignment}~\&~\ref{fig:difficult_pareto_fronts_density}, columns 3-7). In contrast, by explicitly forcing the algorithm to sample from each trade-off direction in the objective space, our goal-conditioned method learns a sampling distribution that spans the entire space diagonally, no matter how complex we make the objective landscape. Table~\ref{tab:difficult_pareto_fronts} reports the performance of both methods on these objective landscapes in terms of IGD, Avg-PCC and PC-ent (mean $\pm$ sem, over 3 seeds). 

We see in Table~\ref{tab:difficult_pareto_fronts} that according to IGD, preference-conditioning and goal-conditioning perform similarly in terms of pushing the empirical Pareto front forward. While the two learned distributions are in many cases very different (Figure~\ref{fig:difficult_pareto_fronts_alignment}, columns 3-7), the preference-conditioning method still manages to produce a few samples in the middle areas of the Pareto front, which satisfies IGD as it only looks for the single closest sample to each reference point. However, the two algorithms differ drastically in terms of controllability of the distribution (color-coded in  Figure~\ref{fig:difficult_pareto_fronts_alignment}) and uniformity of the distribution along the Pareto front (color-coded in  Figure~\ref{fig:difficult_pareto_fronts_density}), which are highlighted by the Avg-PCC and PC-ent criteria in Table~\ref{tab:difficult_pareto_fronts}.

\begin{figure}[t!]
    \centering
    \includegraphics[width=1\linewidth, height=2.75cm]{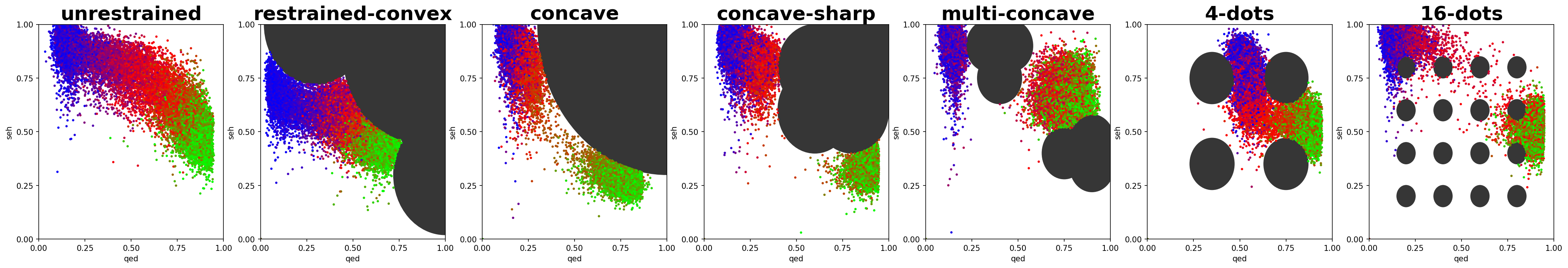}
    \includegraphics[width=1\linewidth, height=2.75cm]{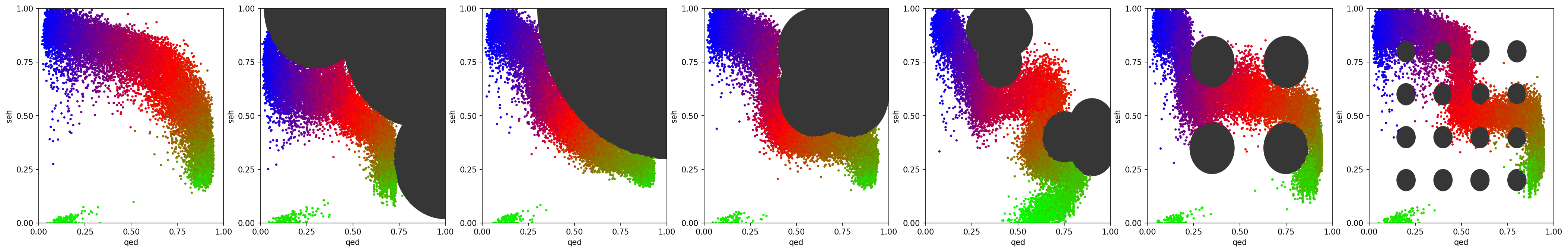}
    \caption[Control comparison of conditioning methods with 2 objectives]{Comparisons between a preference-conditioned GFN (top row) and a goal-conditioned GFN (bottom row) on a set of increasingly complex modifications of a two-objective (seh, qed) fragment-based molecule generation task \cite{jain2022multi}. The BRG colors represent the angle between the vector $[1, 0]$ and either the preference-vector $w$ (top) or the goal direction $d_g$ (bottom), respectively. For example, in the case of preference-conditioning, a green dot means that such samples were produced with a strong preference for the qed-objective, while in the goal-conditioning case, a green dot means that  the model \textit{intended} to produce a sample alongside the qed-axis. We see that goal-conditioning allows to span the entire objective space even in very challenging landscapes (columns 3-7) and in a more controllable way.}
    \label{fig:difficult_pareto_fronts_alignment}
\end{figure}
\begin{figure}[b!]
    \centering
    \includegraphics[width=1\linewidth, height=2.75cm]{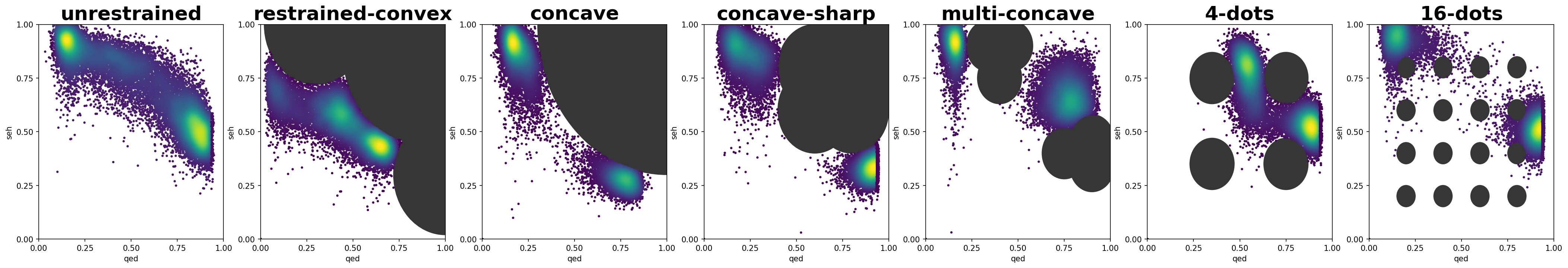}
    \includegraphics[width=1\linewidth, height=2.75cm]{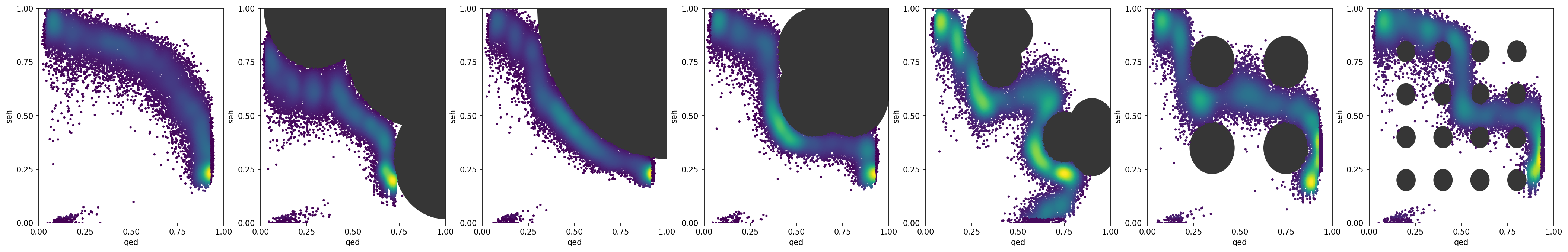}
    \caption[Density comparison of conditioning methods with 2 objectives]{Comparisons of the same sampling distributions depicted in Figure~\ref{fig:difficult_pareto_fronts_alignment}. Now the colors indicate how densely populated a particular area of the objective space is (brighter is more populated). We can see that by explicitly targeting different trade-off regions in objective space, our goal-conditioning approach (bottom row) produces far more evenly distributed samples along the Pareto front than with preference-conditioning (top row).}
    \label{fig:difficult_pareto_fronts_density}
\end{figure}

\subsection{Comparisons for Increasing Number of Objectives}
\label{sec:GCGFN:increasing_number_of_objectives}

Using the same metrics, we also evaluate the performance of both methods when the number of objectives increases. As described in Section~\ref{sec:GCGFN:learned_goal_distribution}, to maintain the sample efficiency of our goal-conditioned approach we sample the goal directions $d_g$ from a learned tabular goal-sampler (Tab-GS) rather than uniformly across the objective space (Uniform-GS). We can see in Table~\ref{tab:GCGFN:growing_number_of_objectives} (and in the ablation in Appendix~\ref{app:GCGFN:learned_goal_model}) that with this adaptation, our goal-conditioned approach maintains its advantages in terms of controllability and uniformity of the learned distribution as the number of objectives increases, proving to be an effective method for probing large, high-dimensional objective spaces for diverse solutions.

\begin{table}[h!]
    \centering
    \caption[Evaluation of conditioning methods with 3 and 4 objectives]{Comparisons according to IGD, Avg-PCC and PC-ent between preference- and goal-conditioned GFNs faced with increasing objectives (3 seeds, mean $\pm$ sem).}
    \vspace{1mm}
    \renewcommand{\arraystretch}{1.3}
    \scalebox{0.8}{
    \begin{tabular}{|r|c|c c c|}
        \cline{2-5}
         \multicolumn{1}{c|}{} & \textbf{algorithm} & \textbf{2 objectives} & \textbf{3 objectives} & \textbf{4 objectives}\\
        \hline
        \multirow{2}{*}{IGD $(\downarrow)$} & pref-cond & \cellcolor{lightgray}$\mathbf{0.088 \pm 0.001}$ & $0.218 \pm 0.003$ & $0.370 \pm 0.000$ \\
        \cline{2-5}
         & goal-cond & $0.094 \pm 0.004$ & \cellcolor{lightgray}$\mathbf{0.199 \pm 0.002}$ & \cellcolor{lightgray}$\mathbf{0.303 \pm 0.001}$ \\
        \hline
        \hline
        \multirow{2}{*}{Avg-PCC $(\uparrow)$} & pref-cond & $0.904 \pm 0.002$ & $0.775 \pm 0.004$ & $0.612 \pm 0.002$ \\
        \cline{2-5}
         & goal-cond & \cellcolor{lightgray}$\mathbf{0.961 \pm 0.001}$ & \cellcolor{lightgray}$\mathbf{0.909 \pm 0.001}$ & \cellcolor{lightgray}$\mathbf{0.893 \pm 0.002}$ \\
        \hline
        \hline
        \multirow{2}{*}{PC-ent $(\uparrow)$} & pref-cond & $2.166 \pm 0.007$ & $3.775 \pm 0.016$ & $4.734 \pm 0.004$ \\
        \cline{2-5}
         & goal-cond & \cellcolor{lightgray}$\mathbf{2.471 \pm 0.001}$ & \cellcolor{lightgray}$\mathbf{4.571 \pm 0.008}$ & \cellcolor{lightgray}$\mathbf{6.320 \pm 0.009}$ \\
        \hline
    \end{tabular}
    }
\label{tab:GCGFN:growing_number_of_objectives}
\end{table}

\section{Future Work}
\label{sec:GCGFN:discussion}
\begin{figure}[t]
    \centering
    \vspace{-4mm}
    \includegraphics[width=0.7\textwidth]{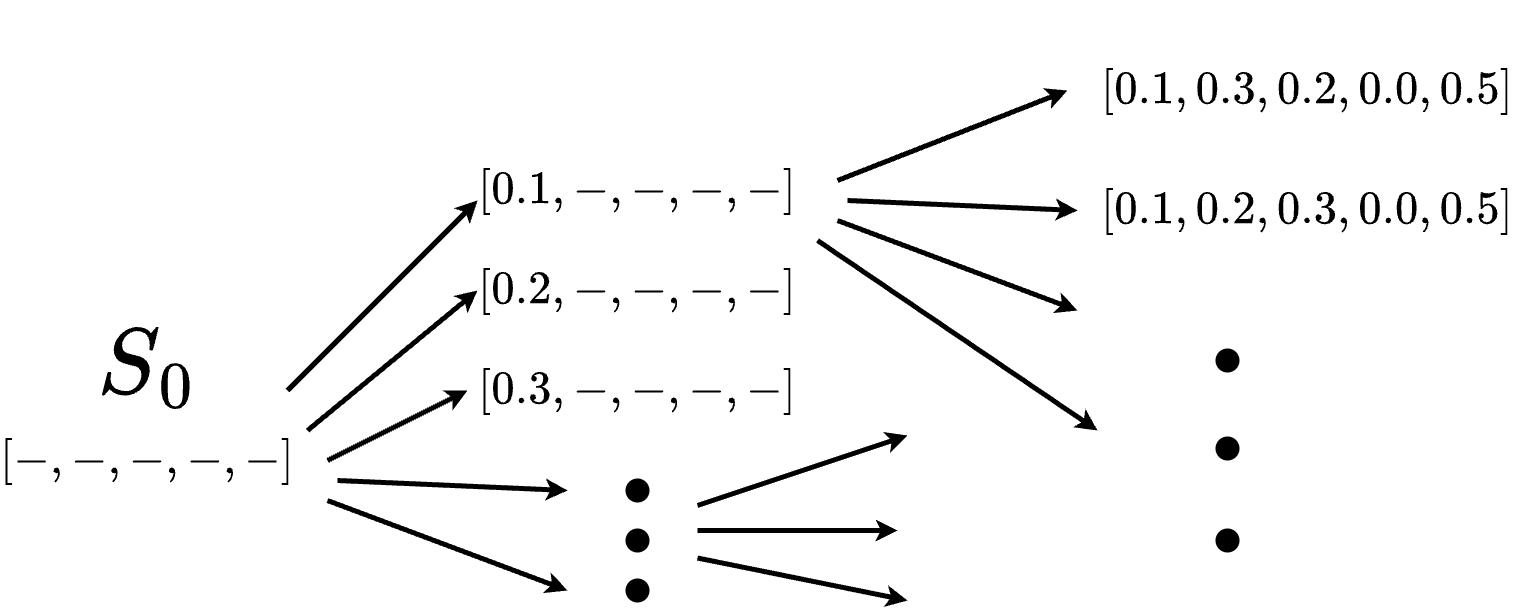}
    \caption[Depiction of a GFlowNet Goal Sampler]{Depiction of a GFlowNet Goal Sampler (GFN-GS) gradually building goal directions $d_g$ one coordinate at a time, as a sequence of $K$ steps.}
    \label{fig:hirarchical_goal_sampler}
\end{figure}
In this work, we proposed goal-conditioned GFlowNets for multi-objective molecular design. We showed that they are an effective solution to give practitioners more control over their generative models, allowing them to obtain a large set of more widely and more uniformly distributed molecules across the objective space. An important limitation of the proposed approach was the reduced sample efficiency of the method due to the existence of \textit{a priori unknown} infeasible goals. We proposed a tabular approach to gradually discredit these fruitless goal regions as we explore the objective space. However, this set of parameters, one for every goal direction, grows exponentially with the number of objectives $K$, eventually leading to statistical and memory limitations. As future steps, we plan to experiment with a GFlowNet-based Goal Sampler (GFN-GS) which would learn to sample feasible goal directions dimension by dimension, thus benefiting from parameter sharing and the improved statistical efficiency of its hierarchical structure.

\section*{Acknowledgements}

We wish to thank Berton Earnshaw, Paul Barde and Tristan Deleu as well as the entire research team at Recursion's Emerging ML Lab for providing insightful comments on this work. We also acknowledge funding in support of this work from Fonds de Recherche Nature et Technologies (FRQNT), Institut de valorisation des données (IVADO) and Recursion Pharmaceuticals.

\addtocontents{toc}{\protect\vspace{\baselineskip}}
\chapter[GENERAL DISCUSSION]{\\GENERAL DISCUSSION}\label{chap:discussion}

Reinforcement learning has shown significant potential in tackling complex sequential decision-making challenges in various real-world domains. More than a decade ago, the potential applications of this technology were already being demonstrated for production scheduling \citep{wang2005application}, aerobatic helicopter flight \citep{abbeel2006application, abbeel2010autonomous} and patient-prosthetic interfaces \citep{pilarski2011online}. Today, Deep RL algorithms are actively being used to improve energy efficiency  \citep{luo2022controlling},
are investigated as a novel approach for plasma control in nuclear fusion reactors \citep{degrave2022magnetic}
and hold great potential in molecular design \citep{popova2018deep}, addressing critical issues such as climate change and disease treatment. However, the core principle of reward maximization, crucial to uncovering innovative solutions, also poses a major challenge in its effective deployment. Designing predictable and efficient reward functions is a complex task, and attempts at reward engineering often result in misaligned solutions or inefficient learning due to incomplete reward specification. In this chapter, we summarize the contributions of this thesis to the field of reward specification in reinforcement learning and discuss their successes and limitations. We then cover additional considerations to alleviate this challenge and conclude by touching on some of the fundamental difficulties that make reward specification so persistent.

\section{Sucesses and Limitations} 

Over the last decades, several families of strategies have been developed to guide and assist the task of reward specification in reinforcement learning. In \textbf{Chapter~\ref{chap:lit_review}}, we divide them into two distinct categories: reward composition and reward modeling. Our contributions, summarized in \textbf{Chapter~\ref{chap:preamble}}, span both of these paradigms. 

Reward modeling aims to bypass the challenge of reward design by learning a model of the reward function using human supervision. In \textbf{Chapter~\ref{chap:article1_asaf}}, we present Adversarial Soft-Advantage Fitting (ASAF), a method which takes advantage of the analytical solution of the adversarial imitation learning problem to parameterize the discriminator in a way that allows to learn a near-optimal policy without performing any policy improvement step. This approach allows to accelerate the learning process and drastically simplify the implementation of adversarial imitation learning algorithms, making it easier to use and deploy. However, reward modeling has important limitations. First, supervision sources need to be available for the task at hand. In the case of imitation learning and inverse reinforcement learning, the algorithm is provided with expert demonstrations. Such demonstrations can be expensive to obtain. For example, in the case of robotics, data collection may require the design of virtual reality simulations, haptic interfaces, or direct robot-manipulation by a human \citep{calinon2009learning, zhang2018deep}. In other cases, such as in molecular design, existing compounds may represent interesting examples to serve as a starting point but may not be deemed optimal, thus restricting the ability of the model to discover new solutions with optimal properties. Secondly, even when available, these demonstrations generally cannot cover all the scenarios of interest. Overcoming this limitation leads to the challenges of generalizing beyond the sampling distribution. The limited availability of human supervision and brittle generalization affect all reward modeling families, including those relying on cheaper input sources such as preference-based methods. Therefore, while demonstration data represent an important asset in accelerating the early phases of learning, in many applications, defining explicity a numerical reward function remains the preferred solution.

The second paradigm, reward composition, regroups reward design strategies which aim at building a reward function from multiple components. Some of these approaches are now widely spread and well understood. For example, sparse rewards can often be augmented with dense potential-based shaping, and this approach can drastically accelerate learning while theoretically preserving the set of optimal policies. However, this type of composition is limited in the type of reward functions that can be captured, and does not explicitly use all of the available information from the environment to shape the learned representation of the policy. The use of auxiliary tasks aims to bridge that gap. Although not guaranteeing the preservation of the set of optimal policies, reasonable assumptions on the necessary properties of the solution set can be made to accelerate learning. For example, a navigation robot operating in a dynamic environment should be able to estimate its own velocity. Auxiliary objectives are an attempt to make use of such additional learning signal to enrich the learned representation of the policy. In \textbf{Chapter~\ref{chap:article2_cmaddpg}}, we present two auxiliary task approaches named Team regularization and Coach regularization, which promote coordination between agents in cooperative multi-agent scenarios. When properly designed and well calibrated, such auxiliary tasks can yield significant gains in performance and sample efficiency by guiding the exploration process towards more promising regions of the policy space. The limitations of auxiliary tasks are two-fold. First, if the proposed task is not sufficiently correlated with the main objective, these additional learning signals can prove detrimental to the performance of the agent. We demonstrate a case of such conflicting signals in the \textit{Compromise} environment in Appendix~\ref{app:CMADDPG:effects_enforcing_predictability} where an agent trained to behave in a predictable way in an adversarial environment becomes subservient to its teammate and neglects to pursue its own goals. While some families of auxiliary tasks are less at risk of being detrimental, for example, when simply enforcing object detection in a vision-based model, such cases of conflicting auxiliary tasks are important reminders of the necessary prior knowledge required to craft effective inductive biases. Moreover, approaches based on auxiliary tasks generally focus on learning efficiency and do not address the problem of alignment. We confront this problem more directly in our two last contributions using multi-objective paradigms.

In \textbf{Chapter~\ref{chap:article3_dbs}}, we present a framework leveraging constrained RL for reward composition. Constrained reinforcement learning defines some components of a task as constraints to satisfy. By restricting cost functions in the constrained MDP framework to identity functions, the expected discounted cost becomes probabilities of events on the agent's visitation distribution, forming a natural interface between the designer's intentions and the agent's behavior. With this framework, for each aspect of behavior that we seek to enforce, a task designer simply needs to write a detection function and specify a target threshold. It may be the case that not every aspect of behavior can be conveniently captured in the form of an indicator function, and this framework comes at the cost of longer training time since the weighting coefficient of each constraint need to be adapted in a bi-level optimization procedure. However, in practice, we find this particular family of cost functions to still retain a lot of expressivity in capturing very diverse disederata, and the additional training time is compensated by saving several design iterations of the reward function to a task designer, thus coming out as a much more effective solution in the overall project development. By enforcing a more thorough and intentional specification procedure, this framework reduces the risk of emergence of exploitative behaviors. These results highlight the benefits of limited optimization focusing on goal-satisfaction, rather than unbounded goal-maximization \citep{vamplew2022scalar}. We have found the use of constraints to be much more intuitive than manually weighting the reward components, and because the CMDP framework naturally reduces to the unconstrained case when only one objective is pursued, we suggest employing a constrained approach as the default paradigm for reward specification over the traditional scalarized MDP approach.

Finally, in \textbf{Chapter~\ref{chap:article4_gcgfn}}, we turn our attention to a fundamentally multi-objective problem: molecular design. We present Goal-Conditioned GFlowNets which leverage goal-conditioning as a way to specify the task in a controllable and flexible set of policies that can be adapted at test-time. While requiring additional computation to explore the set of Pareto-optimal policies, conditional approaches allow to defer the final decision on user preferences to deployment time, effectively postponing the problem of reward specification. This method leverages the compression capabilities of neural parameterizations for solving the dilemma of reward design. Moreover, since rewards are strictly terminal, enforcing hard-constraints can be done efficiently by leveraging the proportional sampling property of GFlowNets using reward manipulation, without requiring a bi-level optimization procedure as in the case on return-based constraints. These constraints allow to cover the most complex of Pareto fronts, thus retaining a broad coverage of solutions independently of the properties of the objective space.

In summary, this thesis explores the landscape of reward specification in reinforcement learning, delving into the realms of reward modeling and reward composition. Our contributions navigate through the challenges and opportunities inherent in these approaches. From human demonstrations to the use of auxiliary tasks, we have demonstrated the potential to accelerate learning using prior knowledge about the task to perform. Subsequently, our exploration of constrained reinforcement learning and goal-conditioned molecular design underlines the importance of intentional design and explicit goals. All of these methods present unique strengths and limitations, and the use of demonstrations, auxiliary tasks and multi-objectivization should be combined to tackle complex real-world challenges most efficiently.

\section{Additional considerations}

This thesis focuses on reward specification. However, there are other considerations that should be taken into account to ensure efficient exploration and aligned behavior. In this section, we briefly discuss the role of environment design, monitoring strategies, and human-in-the-loop to guide policy learning and address misaligned behavior.

\subsubsection*{Environment design}

The reward function in great part captures the task to be accomplished -- for a fixed MDP, different reward functions can lead to completely different behaviors. However, in real-world RL applications, engineers generally have agency over the design of the entire MDP, not just its reward function \citep{taylor2023reinforcement}. 

One of the most fundamental decisions is the representation of the state space $\mathcal{S}$. A correct state specification must include all the information relevant to making the proper action selection and respect the Markov property. As a design decision, it can be difficult to choose between a preprocessed state vector which may omit some environment details but allow for efficient policy learning, and a rich higher-dimensional state representation such as images which makes all of the information available to the agent but requires significantly more training samples. The action space $\mathcal{A}$ also has a great influence on the difficulty of learning the task. Reframing the problem to remove potentially harmful actions from the agent's control and narrowing the learned policy on a low-dimensional control vector can both accelerate policy learning and reduce the risk of misalignment. In some cases, fixed subroutines can be used for finer control while the agent would be in charge of a more distant decision making procedure. Finally, the definition of time-step $t$ has great implications for both controllability and sample efficiency. An agent that acts more frequently will have more precise control over the environment at the cost of more difficult credit assignment. 

Any given real-world task can typically be formulated as several different MDP instantiations, and the design choices behind these MDP components greatly impact the nature of the problem to solve from the RL agent's perspective. For example, starting from a well-defined objective such as teaching a robot arm to play table tennis, \citet{d2023robotic} make a myriad of design decisions to define the MDP, essentially transforming the engineering challenge from hard-coding the robot behavior to designing an environment that can be solved by reinforcement learning.

The challenge of environment design intertwines with different research areas. Hierarchical reinforcement learning aims at learning a hierarchy of agents where higher-level decision-makers receive the true reward function and take temporally-extended action delegating finer control to lower-level effectors \citep{dayan1992feudal}. Higher-level managers provide learned rewards to lower-level effectors, essentially bringing together the concepts of environment design and reward modeling by both specifying multiple MDPs at different levels while learning a model of the reward function to cascade the true signal down. Curriculum learning in RL instead seeks to build a sequence of MDPs terminating with the real task to solve \citep{chevalier2018babyai}. The first MDPs are meant to be easier to solve and their optimal policies are used as starting point for the next problem, allowing to gradually build up the control problem in its full complexity. This approach brings together notions of environment design and auxiliary tasks, decomposing the true problem into several MDPs of increasing complexity.

\subsubsection*{Monitoring}

To find the best configuration for the algorithm, the environment and the reward function, RL projects often involve running hundreds of experiments in parallel. Due to this necessary practice, identifying misaligned behaviors becomes a complex task. Manually inspecting each agent's trajectory is not only labor-intensive but also impractical for large-scale searches. Therefore, automated methods to pinpoint deviations from expected behaviors are essential for the effective deployment of RL systems. 

A notable approach by \citet{pan2022effects} suggests detecting significant shifts in agent behavior by evaluating the divergence between the agent's policy and a pre-examined, \textit{trusted} policy. This method illustrates the utility of expert policies or demonstrations not just in learning but also in monitoring the progression of an agent's behavior. Subsequently, when misalignment is observed, dissecting the reward components individually can shed light on the influences driving the agent's behavior. As discussed by \citet{vamplew2018human}, this is particularly pertinent in multi-objective frameworks. Some works have even explored decomposing a black-boxed scalar rewards into interpretable components to elucidate the agent's behavior in terms of trade-offs \citep{juozapaitis2019explainable}. \citet{anderson2019explaining} further experimented with graphical interfaces that display the anticipated returns for each reward component to allow human observers to analyze surprising decisions from the agent.

Overall, monitoring is a fundamental aspect in detecting, understanding, and rectifying issues related to reward misspecification. It presents significant research opportunities to develop more interpretable agents to inform and guide the design of the reward function. With the increasing adoption of RL in industry settings, this aspect of development will likely take a central place in more standardized development practices.

\subsubsection*{Human-in-the-loop}

Perfectly specifying a reward function on the first attempt for a complex task is highly unlikely. On the other hand, an iterative process is very costly and time consuming, as the agent needs to be trained to convergence between each attempt. To address this dilemma, one idea consists in intervening \textit{during} the agent training, which can be seen as a way to edit the reward function on the fly, to re-specify it as the agent's behavior starts taking form. This paradigm can be referred to as Human-in-the-Loop (HiL) \citep{mosqueira2023human}.

HiL can take the form of active learning \citep{settles2009active}, where the agent models its own uncertainty about the reward function and is responsible for querying human annotators for more labels when its predictions are too ambiguous. Many reward modeling approaches allow for such a feedback loop. In particular, both learning from demonstrations and learning from preferences have been combined with active learning frameworks to optimize the number of queries to human experts \citep{cui2018active, sadigh2017active}. Further along that spectrum, we find more interactive methods in which humans remain in control of when additional feedback should be provided. For example, \citet{knox2009interactively, knox2012reinforcement} propose a framework where a human supervisor directly provides signals of approval and disapproval, allowing the agent to learn to a policy without being responsible for credit assignment. \citet{saunders2017trial} intervene to prevent the agent from making catastrophic errors while learning to interact with its environment. \citet{bajcsy2017learning} use human physical interventions on a robot arm to learn the parameters of its objective function.

These approaches operationalize the idea of widening the modelization of our relationship with learning agents. In a similar line of thought, \citet{jeon2020reward} propose 
to use the intervention itself as a source of feedback; the agent should take note that the fact that a human intervention was necessary is unacceptable.  \citet{taylor2023reinforcement} zooms back even more and argues that RL as a whole should be seen as a human-in-the-loop procedure, from the MDP design to the deployment of the solution. A general HiL perspective and concrete methods to implement it mitigate the challenges of reward specification by allowing the designers to correct course, both accelerating and re-aligning the agent's learning in a live system.

\section{Reward Specification: A persistent challenge}

Reward specification presents itself as a fundamental challenge in the realm of control algorithms. RL agents have a pervasive tendency to yield exploitative solutions, a characteristic that seems to be inherent to automated learning systems, as this phenomenon has also been frequently observed in digital evolution studies \citep{lehman2020surprising}. Both the progress in this field and the remaining difficulties invite for additional work. However, despite its technical nature, reward specification also has an ethical and psychological component, and the disconcerting fragility of any particular reward design can in part be attributed to the deep connections between the act of defining an objective and some core issues relating to policy making, system design, and uncertainty.

First, among a set of desiderata, objectives not only differ in their respective weights and levels of priority but they also vary in nature. Certain rules exhibit flexibility, allowing a degree of tolerance (e.g. be on time), while others hold axiomatic significance (e.g. do not kill). The balancing of these rules often manifests in societal decisions, where, for example, it might be deemed tolerable to infringe on an individual's rights to own land to facilitate the construction of a bridge or electrical infrastructure that can benefit the life of millions. For an agent to understand in which scenarios such trade-offs are acceptable requires significant knowledge about human cultures and values.

Secondly, our capacity to measure aspects of a system is often limited. For instance, the objective of training an agent to ``play a game in an entertaining manner'' is conceptually simple yet presents considerable challenges in execution. Certain attributes, such as the concept of \textit{entertainment}, cannot be directly quantified. Moreover, there is often a lack of consensus among individuals regarding what they consider entertaining. System designers thus have to make significant efforts to distill this overarching goal into a set of quantifiable targets. Despite our best efforts, the agent tends to develop behaviors that diverge from the intended path, in a manner akin to how individuals and corporations invariably find loopholes to maximize profits, regardless how complex tax regulations might become. It is this gap from a conceptual goal to a set of enforceable heuristics that creates the opportunity for exploitative behavior and misalignment. 

Finally, users themselves are often uncertain about their own preferences. Humans regularly change their minds, and external circumstances evolve, making yesterday's targets unfit to meet today's needs. Consequently, a system that was effective under one set of preferences may become less effective or even counterproductive as preferences evolve, necessitating continuous adaptation and reevaluation of the algorithm and its objectives. 

These challenges suggest that the problem of creating robust and adaptable reward functions will persist as a continual and evolving aspect of the field, and mandate for multi-displinary approaches to designing such systems, favouring methods that allow non-technical stakeholders to take part in the reward specification process of RL agents.

\chapter[CONCLUSION]{\\CONCLUSION}\label{sec:Conclusion}

Reward specification is the process of providing a reinforcement learning agent with a reward function. It is designed such that, in maximizing its return, the agent is accomplishing a task for us. This reward function is often engineered by hand through trial-and-error, iteratively refined after each training cycle to speed up the agent's learning and better align its behavior with our objectives. However, due to the accumulation of the reward through time, the difficulty to capture human intentions and the need to balance conflicting signals, many attempts lead to ineffective training and the emergence of undesireable behaviors. Crucially, there is no one-size-fits-all solution to these challenges, and the choice of appropriate methods depends on the availability of data and on the specific task at hand. 

In this thesis, we have presented contributions to several important paradigms addressing these issues, starting from the use of auxiliary tasks and demonstrations to accelerate agent learning, to multi-objective formulations that incorporate several requirements in an agent's behavior. We also surveyed complementary approaches such as preference-based and language-guided reward modeling and discussed important considerations regarding environment design, monitoring and human-in-the-loop. All of these tools can work together. Deep reinforcement learning has the potential to tackle significant real-world problems and advance human knowledge. To deliver these benefits and function effectively within larger systems, RL frameworks must employ controllable and transparent strategies for reward specification, enable interpretable analysis and monitoring, and allow for rapid adaptations in the face of evolving circumstances and shifting goals.

\ifthenelse{\equal{\Langue}{english}}{
	\renewcommand\bibname{REFERENCES}
        \addcontentsline{toc}{compteur}{REFERENCES}
	\bibliography{references}
}{
	\renewcommand\bibname{RÉFÉRENCES}
	\bibliography{references}

\begin{thebibliography}{}

\bibitem[Abadi et~al., 2016]{abadi2016tensorflow}
Abadi, M., Barham, P., Chen, J., Chen, Z., Davis, A., Dean, J., Devin, M.,
  Ghemawat, S., Irving, G., Isard, M., et~al. (2016).
\newblock $\{$TensorFlow$\}$: a system for $\{$Large-Scale$\}$ machine
  learning.
\newblock In {\em 12th USENIX symposium on operating systems design and
  implementation (OSDI 16)}, pages 265--283.

\bibitem[Abbeel et~al., 2010]{abbeel2010autonomous}
Abbeel, P., Coates, A., and Ng, A.~Y. (2010).
\newblock Autonomous helicopter aerobatics through apprenticeship learning.
\newblock {\em The International Journal of Robotics Research},
  29(13):1608--1639.

\bibitem[Abbeel et~al., 2006]{abbeel2006application}
Abbeel, P., Coates, A., Quigley, M., and Ng, A. (2006).
\newblock An application of reinforcement learning to aerobatic helicopter
  flight.
\newblock {\em Advances in neural information processing systems}, 19.

\bibitem[Abbeel and Ng, 2004]{abbeel2004apprenticeship}
Abbeel, P. and Ng, A.~Y. (2004).
\newblock Apprenticeship learning via inverse reinforcement learning.
\newblock In {\em Proceedings of the twenty-first international conference on
  Machine learning}, page~1.

\bibitem[Abbeel and Ng, 2005]{abbeel2005exploration}
Abbeel, P. and Ng, A.~Y. (2005).
\newblock Exploration and apprenticeship learning in reinforcement learning.
\newblock In {\em Proceedings of the 22nd international conference on Machine
  learning}, pages 1--8.

\bibitem[Abel et~al., 2021]{abel2021expressivity}
Abel, D., Dabney, W., Harutyunyan, A., Ho, M.~K., Littman, M., Precup, D., and
  Singh, S. (2021).
\newblock On the expressivity of markov reward.
\newblock {\em Advances in Neural Information Processing Systems},
  34:7799--7812.

\bibitem[Abels et~al., 2019]{abels2019dynamic}
Abels, A., Roijers, D., Lenaerts, T., Now{\'e}, A., and Steckelmacher, D.
  (2019).
\newblock Dynamic weights in multi-objective deep reinforcement learning.
\newblock In {\em International conference on machine learning}, pages 11--20.
  PMLR.

\bibitem[Achiam et~al., 2017]{achiam2017constrained}
Achiam, J., Held, D., Tamar, A., and Abbeel, P. (2017).
\newblock Constrained policy optimization.
\newblock In {\em International conference on machine learning}, pages 22--31.
  PMLR.

\bibitem[Adams et~al., 2022]{adams2022survey}
Adams, S., Cody, T., and Beling, P.~A. (2022).
\newblock A survey of inverse reinforcement learning.
\newblock {\em Artificial Intelligence Review}, 55(6):4307--4346.

\bibitem[Ahilan and Dayan, 2019]{ahilan2019feudal}
Ahilan, S. and Dayan, P. (2019).
\newblock Feudal multi-agent hierarchies for cooperative reinforcement
  learning.
\newblock {\em arXiv preprint arXiv:1901.08492}.

\bibitem[Ahn et~al., 2022]{ahn2022can}
Ahn, M., Brohan, A., Brown, N., Chebotar, Y., Cortes, O., David, B., Finn, C.,
  Fu, C., Gopalakrishnan, K., Hausman, K., et~al. (2022).
\newblock Do as i can, not as i say: Grounding language in robotic affordances.
\newblock {\em arXiv preprint arXiv:2204.01691}.

\bibitem[Aissani et~al., 2008]{aissani2008efficient}
Aissani, N., Beldjilali, B., and Trentesaux, D. (2008).
\newblock Efficient and effective reactive scheduling of manufacturing system
  using sarsa-multi-objective agents.
\newblock In {\em MOSIM’08: 7th Conference Internationale de Modelisation et
  Simulation}, pages 698--707.

\bibitem[Akkaya et~al., 2019]{akkaya2019solving}
Akkaya, I., Andrychowicz, M., Chociej, M., Litwin, M., McGrew, B., Petron, A.,
  Paino, A., Plappert, M., Powell, G., Ribas, R., et~al. (2019).
\newblock Solving rubik's cube with a robot hand.
\newblock {\em arXiv preprint arXiv:1910.07113}.

\bibitem[Akrour et~al., 2011]{akrour2011preference}
Akrour, R., Schoenauer, M., and Sebag, M. (2011).
\newblock Preference-based policy learning.
\newblock In {\em Machine Learning and Knowledge Discovery in Databases:
  European Conference, ECML PKDD 2011, Athens, Greece, September 5-9, 2011.
  Proceedings, Part I 11}, pages 12--27. Springer.

\bibitem[Alonso et~al., 2020]{alonso2020deep}
Alonso, E., Peter, M., Goumard, D., and Romoff, J. (2020).
\newblock Deep reinforcement learning for navigation in aaa video games.
\newblock {\em arXiv preprint arXiv:2011.04764}.

\bibitem[Altman, 1999]{altman1999constrained}
Altman, E. (1999).
\newblock {\em Constrained Markov Decision Processes}.
\newblock Stochastic Modeling Series. Taylor \& Francis.

\bibitem[Amin et~al., 2021]{amin2021survey}
Amin, S., Gomrokchi, M., Satija, H., van Hoof, H., and Precup, D. (2021).
\newblock A survey of exploration methods in reinforcement learning.
\newblock {\em arXiv preprint arXiv:2109.00157}.

\bibitem[Amodei et~al., 2016]{amodei2016concrete}
Amodei, D., Olah, C., Steinhardt, J., Christiano, P., Schulman, J., and
  Man{\'e}, D. (2016).
\newblock Concrete problems in ai safety.
\newblock {\em arXiv preprint arXiv:1606.06565}.

\bibitem[Anderson et~al., 2019]{anderson2019explaining}
Anderson, A., Dodge, J., Sadarangani, A., Juozapaitis, Z., Newman, E., Irvine,
  J., Chattopadhyay, S., Fern, A., and Burnett, M. (2019).
\newblock Explaining reinforcement learning to mere mortals: An empirical
  study.
\newblock {\em arXiv preprint arXiv:1903.09708}.

\bibitem[Andrychowicz et~al., 2017]{andrychowicz2017hindsight}
Andrychowicz, M., Wolski, F., Ray, A., Schneider, J., Fong, R., Welinder, P.,
  McGrew, B., Tobin, J., Abbeel, O.~P., and Zaremba, W. (2017).
\newblock Hindsight experience replay.
\newblock In {\em Advances in neural information processing systems}, pages
  5048--5058.

\bibitem[Andrychowicz et~al., 2020]{andrychowicz2020learning}
Andrychowicz, O.~M., Baker, B., Chociej, M., Jozefowicz, R., McGrew, B.,
  Pachocki, J., Petron, A., Plappert, M., Powell, G., Ray, A., et~al. (2020).
\newblock Learning dexterous in-hand manipulation.
\newblock {\em The International Journal of Robotics Research}, 39(1):3--20.

\bibitem[Arora and Doshi, 2021]{arora2021survey}
Arora, S. and Doshi, P. (2021).
\newblock A survey of inverse reinforcement learning: Challenges, methods and
  progress.
\newblock {\em Artificial Intelligence}, 297:103500.

\bibitem[Arumugam et~al., 2017]{arumugam2017accurately}
Arumugam, D., Karamcheti, S., Gopalan, N., Wong, L.~L., and Tellex, S. (2017).
\newblock Accurately and efficiently interpreting human-robot instructions of
  varying granularities.
\newblock {\em arXiv preprint arXiv:1704.06616}.

\bibitem[Askell et~al., 2021]{askell2021general}
Askell, A., Bai, Y., Chen, A., Drain, D., Ganguli, D., Henighan, T., Jones, A.,
  Joseph, N., Mann, B., DasSarma, N., et~al. (2021).
\newblock A general language assistant as a laboratory for alignment.
\newblock {\em arXiv preprint arXiv:2112.00861}.

\bibitem[Atanackovic et~al., 2023]{atanackovic2023dyngfn}
Atanackovic, L., Tong, A., Hartford, J., Lee, L.~J., Wang, B., and Bengio, Y.
  (2023).
\newblock Dyngfn: Bayesian dynamic causal discovery using generative flow
  networks.
\newblock {\em arXiv preprint arXiv:2302.04178}.

\bibitem[Audet et~al., 2021]{audet2021performance}
Audet, C., Bigeon, J., Cartier, D., Le~Digabel, S., and Salomon, L. (2021).
\newblock Performance indicators in multiobjective optimization.
\newblock {\em European journal of operational research}, 292(2):397--422.

\bibitem[Ba et~al., 2016]{ba2016layer}
Ba, J.~L., Kiros, J.~R., and Hinton, G.~E. (2016).
\newblock Layer normalization.
\newblock {\em arXiv preprint arXiv:1607.06450}.

\bibitem[Bacon et~al., 2017]{bacon2017option}
Bacon, P.-L., Harb, J., and Precup, D. (2017).
\newblock The option-critic architecture.
\newblock In {\em Thirty-First AAAI Conference on Artificial Intelligence}.

\bibitem[Bahdanau et~al., 2018]{bahdanau2018learning}
Bahdanau, D., Hill, F., Leike, J., Hughes, E., Hosseini, A., Kohli, P., and
  Grefenstette, E. (2018).
\newblock Learning to understand goal specifications by modelling reward.
\newblock {\em arXiv preprint arXiv:1806.01946}.

\bibitem[Bai et~al., 2022]{bai2022training}
Bai, Y., Jones, A., Ndousse, K., Askell, A., Chen, A., DasSarma, N., Drain, D.,
  Fort, S., Ganguli, D., Henighan, T., et~al. (2022).
\newblock Training a helpful and harmless assistant with reinforcement learning
  from human feedback.
\newblock {\em arXiv preprint arXiv:2204.05862}.

\bibitem[Bajcsy et~al., 2017]{bajcsy2017learning}
Bajcsy, A., Losey, D.~P., O’malley, M.~K., and Dragan, A.~D. (2017).
\newblock Learning robot objectives from physical human interaction.
\newblock In {\em Conference on Robot Learning}, pages 217--226. PMLR.

\bibitem[Barrett and Narayanan, 2008]{barrett2008learning}
Barrett, L. and Narayanan, S. (2008).
\newblock Learning all optimal policies with multiple criteria.
\newblock In {\em Proceedings of the 25th international conference on Machine
  learning}, pages 41--47.

\bibitem[Barton et~al., 2018]{barton2018measuring}
Barton, S.~L., Waytowich, N.~R., Zaroukian, E., and Asher, D.~E. (2018).
\newblock Measuring collaborative emergent behavior in multi-agent
  reinforcement learning.
\newblock In {\em International Conference on Human Systems Engineering and
  Design: Future Trends and Applications}, pages 422--427. Springer.

\bibitem[Beeching et~al., 2021]{beeching2021graph}
Beeching, E., Peter, M., Marcotte, P., Debangoye, J., Simonin, O., Romoff, J.,
  and Wolf, C. (2021).
\newblock Graph augmented deep reinforcement learning in the gamerland3d
  environment.
\newblock {\em arXiv preprint arXiv:2112.11731}.

\bibitem[Bellemare et~al., 2020]{bellemare2020autonomous}
Bellemare, M.~G., Candido, S., Castro, P.~S., Gong, J., Machado, M.~C., Moitra,
  S., Ponda, S.~S., and Wang, Z. (2020).
\newblock Autonomous navigation of stratospheric balloons using reinforcement
  learning.
\newblock {\em Nature}, 588(7836):77--82.

\bibitem[Bellemare et~al., 2017]{bellemare2017distributional}
Bellemare, M.~G., Dabney, W., and Munos, R. (2017).
\newblock A distributional perspective on reinforcement learning.
\newblock {\em arXiv preprint arXiv:1707.06887}.

\bibitem[Bellemare et~al., 2013]{bellemare2013arcade}
Bellemare, M.~G., Naddaf, Y., Veness, J., and Bowling, M. (2013).
\newblock The arcade learning environment: An evaluation platform for general
  agents.
\newblock {\em Journal of Artificial Intelligence Research}, 47:253--279.

\bibitem[Bellman, 1966]{bellman1966dynamic}
Bellman, R. (1966).
\newblock Dynamic programming.
\newblock {\em science}, 153(3731):34--37.

\bibitem[Bengio et~al., 2021]{bengio2021flow}
Bengio, E., Jain, M., Korablyov, M., Precup, D., and Bengio, Y. (2021).
\newblock Flow network based generative models for non-iterative diverse
  candidate generation.
\newblock {\em Advances in Neural Information Processing Systems},
  34:27381--27394.

\bibitem[Bengio et~al., 2023]{bengio2023gflownet}
Bengio, Y., Lahlou, S., Deleu, T., Hu, E.~J., Tiwari, M., and Bengio, E.
  (2023).
\newblock Gflownet foundations.
\newblock {\em Journal of Machine Learning Research}, 24(210):1--55.

\bibitem[Bergdahl et~al., 2020]{bergdahl2020augmenting}
Bergdahl, J., Gordillo, C., Tollmar, K., and Gissl{\'e}n, L. (2020).
\newblock Augmenting automated game testing with deep reinforcement learning.
\newblock In {\em 2020 IEEE Conference on Games (CoG)}, pages 600--603. IEEE.

\bibitem[Berner et~al., 2019]{berner2019dota}
Berner, C., Brockman, G., Chan, B., Cheung, V., D{\k{e}}biak, P., Dennison, C.,
  Farhi, D., Fischer, Q., Hashme, S., Hesse, C., et~al. (2019).
\newblock Dota 2 with large scale deep reinforcement learning.
\newblock {\em arXiv preprint arXiv:1912.06680}.

\bibitem[Bertsekas, 1997]{bertsekas1997nonlinear}
Bertsekas, D.~P. (1997).
\newblock Nonlinear programming.
\newblock {\em Journal of the Operational Research Society}, 48(3):334--334.

\bibitem[Bickerton et~al., 2012]{bickerton2012quantifying}
Bickerton, G.~R., Paolini, G.~V., Besnard, J., Muresan, S., and Hopkins, A.~L.
  (2012).
\newblock Quantifying the chemical beauty of drugs.
\newblock {\em Nature chemistry}, 4(2):90--98.

\bibitem[Bohez et~al., 2019]{bohez2019value}
Bohez, S., Abdolmaleki, A., Neunert, M., Buchli, J., Heess, N., and Hadsell, R.
  (2019).
\newblock Value constrained model-free continuous control.
\newblock {\em arXiv preprint arXiv:1902.04623}.

\bibitem[Booth et~al., 2023]{booth2023perils}
Booth, S., Knox, W.~B., Shah, J., Niekum, S., Stone, P., and Allievi, A.
  (2023).
\newblock The perils of trial-and-error reward design: misdesign through
  overfitting and invalid task specifications.
\newblock In {\em Proceedings of the AAAI Conference on Artificial
  Intelligence}, volume~37, pages 5920--5929.

\bibitem[Borkar, 2005]{borkar2005actor}
Borkar, V.~S. (2005).
\newblock An actor-critic algorithm for constrained markov decision processes.
\newblock {\em Systems \& control letters}, 54(3):207--213.

\bibitem[Boularias et~al., 2011]{boularias2011relative}
Boularias, A., Kober, J., and Peters, J. (2011).
\newblock Relative entropy inverse reinforcement learning.
\newblock In {\em Proceedings of the fourteenth international conference on
  artificial intelligence and statistics}, pages 182--189. JMLR Workshop and
  Conference Proceedings.

\bibitem[Bowling et~al., 2023]{bowling2023settling}
Bowling, M., Martin, J.~D., Abel, D., and Dabney, W. (2023).
\newblock Settling the reward hypothesis.
\newblock In {\em International Conference on Machine Learning}, pages
  3003--3020. PMLR.

\bibitem[Brockman et~al., 2016]{brockman2016openai}
Brockman, G., Cheung, V., Pettersson, L., Schneider, J., Schulman, J., Tang,
  J., and Zaremba, W. (2016).
\newblock Openai gym.
\newblock {\em arXiv preprint arXiv:1606.01540}.

\bibitem[Brown et~al., 2019a]{brown2019extrapolating}
Brown, D.~S., Goo, W., Nagarajan, P., and Niekum, S. (2019a).
\newblock Extrapolating beyond suboptimal demonstrations via inverse
  reinforcement learning from observations.
\newblock {\em arXiv preprint arXiv:1904.06387}.

\bibitem[Brown et~al., 2019b]{brown2019guacamol}
Brown, N., Fiscato, M., Segler, M.~H., and Vaucher, A.~C. (2019b).
\newblock Guacamol: benchmarking models for de novo molecular design.
\newblock {\em Journal of chemical information and modeling}, 59(3):1096--1108.

\bibitem[Brown et~al., 2006]{brown2006novel}
Brown, N., McKay, B., and Gasteiger, J. (2006).
\newblock A novel workflow for the inverse qspr problem using multiobjective
  optimization.
\newblock {\em Journal of computer-aided molecular design}, 20:333--341.

\bibitem[Calian et~al., 2020]{calian2020balancing}
Calian, D.~A., Mankowitz, D.~J., Zahavy, T., Xu, Z., Oh, J., Levine, N., and
  Mann, T. (2020).
\newblock Balancing constraints and rewards with meta-gradient d4pg.
\newblock {\em arXiv preprint arXiv:2010.06324}.

\bibitem[Calinon et~al., 2009]{calinon2009learning}
Calinon, S., Evrard, P., Gribovskaya, E., Billard, A., and Kheddar, A. (2009).
\newblock Learning collaborative manipulation tasks by demonstration using a
  haptic interface.
\newblock In {\em 2009 International Conference on Advanced Robotics}, pages
  1--6. IEEE.

\bibitem[Castelletti et~al., 2002]{castelletti2002reinforcement}
Castelletti, A., Corani, G., Rizzolli, A., Soncinie-Sessa, R., and Weber, E.
  (2002).
\newblock Reinforcement learning in the operational management of a water
  system.
\newblock In {\em IFAC workshop on modeling and control in environmental
  issues}, pages 325--330. Keio University Yokohama.

\bibitem[Chai et~al., 2021]{chai2021deep}
Chai, J., Zeng, H., Li, A., and Ngai, E.~W. (2021).
\newblock Deep learning in computer vision: A critical review of emerging
  techniques and application scenarios.
\newblock {\em Machine Learning with Applications}, 6:100134.

\bibitem[Cheng et~al., 2011]{cheng2011preference}
Cheng, W., F{\"u}rnkranz, J., H{\"u}llermeier, E., and Park, S.-H. (2011).
\newblock Preference-based policy iteration: Leveraging preference learning for
  reinforcement learning.
\newblock In {\em Machine Learning and Knowledge Discovery in Databases:
  European Conference, ECML PKDD 2011, Athens, Greece, September 5-9, 2011.
  Proceedings, Part I 11}, pages 312--327. Springer.

\bibitem[Chentanez et~al., 2005]{chentanez2005intrinsically}
Chentanez, N., Barto, A.~G., and Singh, S.~P. (2005).
\newblock Intrinsically motivated reinforcement learning.
\newblock In {\em Advances in neural information processing systems}, pages
  1281--1288.

\bibitem[Chevalier-Boisvert et~al., 2018]{chevalier2018babyai}
Chevalier-Boisvert, M., Bahdanau, D., Lahlou, S., Willems, L., Saharia, C.,
  Nguyen, T.~H., and Bengio, Y. (2018).
\newblock Babyai: A platform to study the sample efficiency of grounded
  language learning.
\newblock {\em arXiv preprint arXiv:1810.08272}.

\bibitem[Choi and Kim, 2011]{choi2011map}
Choi, J. and Kim, K.-E. (2011).
\newblock Map inference for bayesian inverse reinforcement learning.
\newblock {\em Advances in neural information processing systems}, 24.

\bibitem[Chow et~al., 2017]{chow2017risk}
Chow, Y., Ghavamzadeh, M., Janson, L., and Pavone, M. (2017).
\newblock Risk-constrained reinforcement learning with percentile risk
  criteria.
\newblock {\em The Journal of Machine Learning Research}, 18(1):6070--6120.

\bibitem[Chow et~al., 2018]{chow2018lyapunov}
Chow, Y., Nachum, O., Duenez-Guzman, E., and Ghavamzadeh, M. (2018).
\newblock A lyapunov-based approach to safe reinforcement learning.
\newblock {\em Advances in neural information processing systems}, 31.

\bibitem[Chow et~al., 2019]{chow2019lyapunov}
Chow, Y., Nachum, O., Faust, A., Duenez-Guzman, E., and Ghavamzadeh, M. (2019).
\newblock Lyapunov-based safe policy optimization for continuous control.
\newblock {\em arXiv preprint arXiv:1901.10031}.

\bibitem[Christiano et~al., 2017]{christiano2017deep}
Christiano, P.~F., Leike, J., Brown, T., Martic, M., Legg, S., and Amodei, D.
  (2017).
\newblock Deep reinforcement learning from human preferences.
\newblock In {\em Advances in Neural Information Processing Systems}, pages
  4299--4307.

\bibitem[Clark and Amodei, 2016]{clark2016faulty}
Clark, J. and Amodei, D. (2016).
\newblock Faulty reward functions in the wild.
\newblock Open AI.

\bibitem[Coello and Cort{\'e}s, 2005]{coello2005solving}
Coello, C. A.~C. and Cort{\'e}s, N.~C. (2005).
\newblock Solving multiobjective optimization problems using an artificial
  immune system.
\newblock {\em Genetic programming and evolvable machines}, 6:163--190.

\bibitem[Cui and Niekum, 2018]{cui2018active}
Cui, Y. and Niekum, S. (2018).
\newblock Active reward learning from critiques.
\newblock In {\em 2018 IEEE international conference on robotics and automation
  (ICRA)}, pages 6907--6914. IEEE.

\bibitem[Cybenko, 1989]{cybenko1989approximation}
Cybenko, G. (1989).
\newblock Approximation by superpositions of a sigmoidal function.
\newblock {\em Mathematics of control, signals and systems}, 2(4):303--314.

\bibitem[Dabney et~al., 2021]{dabney2021value}
Dabney, W., Barreto, A., Rowland, M., Dadashi, R., Quan, J., Bellemare, M.~G.,
  and Silver, D. (2021).
\newblock The value-improvement path: Towards better representations for
  reinforcement learning.
\newblock In {\em Proceedings of the AAAI Conference on Artificial
  Intelligence}, volume~35, pages 7160--7168.

\bibitem[Dabney et~al., 2018]{dabney2018distributional}
Dabney, W., Rowland, M., Bellemare, M.~G., and Munos, R. (2018).
\newblock Distributional reinforcement learning with quantile regression.
\newblock In {\em Thirty-Second AAAI Conference on Artificial Intelligence}.

\bibitem[Dalal et~al., 2018]{dalal2018safe}
Dalal, G., Dvijotham, K., Vecerik, M., Hester, T., Paduraru, C., and Tassa, Y.
  (2018).
\newblock Safe exploration in continuous action spaces.
\newblock {\em arXiv preprint arXiv:1801.08757}.

\bibitem[D'Ambrosio et~al., 2023]{d2023robotic}
D'Ambrosio, D.~B., Abelian, J., Abeyruwan, S., Ahn, M., Bewley, A., Boyd, J.,
  Choromanski, K., Cortes, O., Coumans, E., Ding, T., et~al. (2023).
\newblock Robotic table tennis: A case study into a high speed learning system.
\newblock {\em arXiv preprint arXiv:2309.03315}.

\bibitem[Daniel et~al., 2014]{daniel2014active}
Daniel, C., Viering, M., Metz, J., Kroemer, O., and Peters, J. (2014).
\newblock Active reward learning.
\newblock In {\em Robotics: Science and systems}, volume~98.

\bibitem[Das and Dennis, 1997]{das1997closer}
Das, I. and Dennis, J.~E. (1997).
\newblock A closer look at drawbacks of minimizing weighted sums of objectives
  for pareto set generation in multicriteria optimization problems.
\newblock {\em Structural optimization}, 14:63--69.

\bibitem[David, 1963]{david1963method}
David, H.~A. (1963).
\newblock {\em The method of paired comparisons}, volume~12.
\newblock London.

\bibitem[Dayan and Hinton, 1992]{dayan1992feudal}
Dayan, P. and Hinton, G.~E. (1992).
\newblock Feudal reinforcement learning.
\newblock {\em Advances in neural information processing systems}, 5.

\bibitem[de~Woillemont et~al., 2021]{de2021configurable}
de~Woillemont, P. L.~P., Labory, R., and Corruble, V. (2021).
\newblock Configurable agent with reward as input: A play-style continuum
  generation.
\newblock In {\em 2021 IEEE Conference on Games (CoG)}, pages 1--8. IEEE.

\bibitem[Degrave et~al., 2022]{degrave2022magnetic}
Degrave, J., Felici, F., Buchli, J., Neunert, M., Tracey, B., Carpanese, F.,
  Ewalds, T., Hafner, R., Abdolmaleki, A., de~Las~Casas, D., et~al. (2022).
\newblock Magnetic control of tokamak plasmas through deep reinforcement
  learning.
\newblock {\em Nature}, 602(7897):414--419.

\bibitem[Degris et~al., 2012]{degris2012off}
Degris, T., White, M., and Sutton, R.~S. (2012).
\newblock Off-policy actor-critic.
\newblock {\em arXiv preprint arXiv:1205.4839}.

\bibitem[Deleu et~al., 2022]{deleu2022bayesian}
Deleu, T., G{\'o}is, A., Emezue, C., Rankawat, M., Lacoste-Julien, S., Bauer,
  S., and Bengio, Y. (2022).
\newblock Bayesian structure learning with generative flow networks.
\newblock In {\em Uncertainty in Artificial Intelligence}, pages 518--528.
  PMLR.

\bibitem[Devlin et~al., 2021]{devlin2021navigation}
Devlin, S., Georgescu, R., Momennejad, I., Rzepecki, J., Zuniga, E., Costello,
  G., Leroy, G., Shaw, A., and Hofmann, K. (2021).
\newblock Navigation turing test (ntt): Learning to evaluate human-like
  navigation.
\newblock {\em arXiv preprint arXiv:2105.09637}.

\bibitem[Devlin and Kudenko, 2012]{devlin2012dynamic}
Devlin, S.~M. and Kudenko, D. (2012).
\newblock Dynamic potential-based reward shaping.
\newblock In {\em Proceedings of the 11th international conference on
  autonomous agents and multiagent systems}, pages 433--440. IFAAMAS.

\bibitem[Dewey, 2014]{dewey2014reinforcement}
Dewey, D. (2014).
\newblock Reinforcement learning and the reward engineering principle.
\newblock In {\em 2014 AAAI Spring Symposium Series}.

\bibitem[Ding et~al., 2019]{ding2019goal}
Ding, Y., Florensa, C., Abbeel, P., and Phielipp, M. (2019).
\newblock Goal-conditioned imitation learning.
\newblock In {\em Advances in Neural Information Processing Systems (NeurIPS)},
  pages 15298--15309.

\bibitem[Dosovitskiy et~al., 2017]{dosovitskiy2017carla}
Dosovitskiy, A., Ros, G., Codevilla, F., Lopez, A., and Koltun, V. (2017).
\newblock Carla: An open urban driving simulator.
\newblock In {\em Conference on robot learning}, pages 1--16. PMLR.

\bibitem[Du et~al., 2018]{du2018adapting}
Du, Y., Czarnecki, W.~M., Jayakumar, S.~M., Farajtabar, M., Pascanu, R., and
  Lakshminarayanan, B. (2018).
\newblock Adapting auxiliary losses using gradient similarity.
\newblock {\em arXiv preprint arXiv:1812.02224}.

\bibitem[Dulac-Arnold et~al., 2021]{dulac2021challenges}
Dulac-Arnold, G., Levine, N., Mankowitz, D.~J., Li, J., Paduraru, C., Gowal,
  S., and Hester, T. (2021).
\newblock Challenges of real-world reinforcement learning: definitions,
  benchmarks and analysis.
\newblock {\em Machine Learning}, 110(9):2419--2468.

\bibitem[Dulac-Arnold et~al., 2019]{dulac2019challenges}
Dulac-Arnold, G., Mankowitz, D., and Hester, T. (2019).
\newblock Challenges of real-world reinforcement learning.
\newblock {\em arXiv preprint arXiv:1904.12901}.

\bibitem[Ehrgott, 2005]{ehrgott2005multicriteria}
Ehrgott, M. (2005).
\newblock {\em Multicriteria optimization}, volume 491.
\newblock Springer Science \& Business Media.

\bibitem[Emmerich and Deutz, 2018]{emmerich2018tutorial}
Emmerich, M.~T. and Deutz, A.~H. (2018).
\newblock A tutorial on multiobjective optimization: fundamentals and
  evolutionary methods.
\newblock {\em Natural computing}, 17:585--609.

\bibitem[Erhan et~al., 2010]{erhan2010does}
Erhan, D., Courville, A., Bengio, Y., and Vincent, P. (2010).
\newblock Why does unsupervised pre-training help deep learning?
\newblock In {\em Proceedings of the thirteenth international conference on
  artificial intelligence and statistics}, pages 201--208. JMLR Workshop and
  Conference Proceedings.

\bibitem[Ertl and Schuffenhauer, 2009]{ertl2009estimation}
Ertl, P. and Schuffenhauer, A. (2009).
\newblock Estimation of synthetic accessibility score of drug-like molecules
  based on molecular complexity and fragment contributions.
\newblock {\em Journal of cheminformatics}, 1:1--11.

\bibitem[Evans and Gao, 2016]{deepmind2016cooling}
Evans, R. and Gao, J. (2016).
\newblock Deepmind ai reduces google data centre cooling bill by 40
\newblock DeepMind.

\bibitem[Fedus et~al., 2019]{fedus2019hyperbolic}
Fedus, W., Gelada, C., Bengio, Y., Bellemare, M.~G., and Larochelle, H. (2019).
\newblock Hyperbolic discounting and learning over multiple horizons.
\newblock {\em arXiv preprint arXiv:1902.06865}.

\bibitem[Fernandes et~al., 2023]{fernandes2023bridging}
Fernandes, P., Madaan, A., Liu, E., Farinhas, A., Martins, P.~H., Bertsch, A.,
  de~Souza, J.~G., Zhou, S., Wu, T., Neubig, G., et~al. (2023).
\newblock Bridging the gap: A survey on integrating (human) feedback for
  natural language generation.
\newblock {\em arXiv preprint arXiv:2305.00955}.

\bibitem[Finn et~al., 2016a]{finn2016connection}
Finn, C., Christiano, P., Abbeel, P., and Levine, S. (2016a).
\newblock A connection between generative adversarial networks, inverse
  reinforcement learning, and energy-based models.
\newblock {\em arXiv preprint arXiv:1611.03852}.

\bibitem[Finn et~al., 2016b]{finn2016guided}
Finn, C., Levine, S., and Abbeel, P. (2016b).
\newblock Guided cost learning: Deep inverse optimal control via policy
  optimization.
\newblock In {\em Proceedings of the 33rd International Conference on Machine
  Learning (ICML)}, pages 49--58.

\bibitem[Foerster et~al., 2016]{foerster2016learning}
Foerster, J., Assael, I.~A., de~Freitas, N., and Whiteson, S. (2016).
\newblock Learning to communicate with deep multi-agent reinforcement learning.
\newblock In {\em Advances in Neural Information Processing Systems}, pages
  2137--2145.

\bibitem[Foerster et~al., 2019]{foerster2018bayesian}
Foerster, J., Song, F., Hughes, E., Burch, N., Dunning, I., Whiteson, S.,
  Botvinick, M., and Bowling, M. (2019).
\newblock Bayesian action decoder for deep multi-agent reinforcement learning.
\newblock {\em International Conference on Machine Learning}.

\bibitem[Foerster et~al., 2018]{foerster2018counterfactual}
Foerster, J.~N., Farquhar, G., Afouras, T., Nardelli, N., and Whiteson, S.
  (2018).
\newblock Counterfactual multi-agent policy gradients.
\newblock In {\em Thirty-Second AAAI Conference on Artificial Intelligence}.

\bibitem[Fortnow, 2009]{fortnow2009status}
Fortnow, L. (2009).
\newblock The status of the p versus np problem.
\newblock {\em Communications of the ACM}, 52(9):78--86.

\bibitem[Fu et~al., 2019]{fu2019language}
Fu, J., Korattikara, A., Levine, S., and Guadarrama, S. (2019).
\newblock From language to goals: Inverse reinforcement learning for
  vision-based instruction following.
\newblock {\em arXiv preprint arXiv:1902.07742}.

\bibitem[Fu et~al., 2017]{fu2017learning}
Fu, J., Luo, K., and Levine, S. (2017).
\newblock Learning robust rewards with adversarial inverse reinforcement
  learning.
\newblock {\em arXiv preprint arXiv:1710.11248}.

\bibitem[Fujimoto et~al., 2018]{fujimoto2018addressing}
Fujimoto, S., Van~Hoof, H., and Meger, D. (2018).
\newblock Addressing function approximation error in actor-critic methods.
\newblock {\em arXiv preprint arXiv:1802.09477}.

\bibitem[G{\'a}bor et~al., 1998]{gabor1998multi}
G{\'a}bor, Z., Kalm{\'a}r, Z., and Szepesv{\'a}ri, C. (1998).
\newblock Multi-criteria reinforcement learning.
\newblock In {\em ICML}, volume~98, pages 197--205.

\bibitem[Ghasemipour et~al., 2019]{ghasemipour2019divergence}
Ghasemipour, S. K.~S., Zemel, R., and Gu, S. (2019).
\newblock A divergence minimization perspective on imitation learning methods.
\newblock In {\em Proceedings of the 3rd Conference on Robot Learning (CoRL)}.

\bibitem[Ghasemipour et~al., 2020]{ghasemipour2020divergence}
Ghasemipour, S. K.~S., Zemel, R., and Gu, S. (2020).
\newblock A divergence minimization perspective on imitation learning methods.
\newblock In {\em Conference on Robot Learning}, pages 1259--1277. PMLR.

\bibitem[Gissl{\'e}n et~al., 2021]{gisslen2021adversarial}
Gissl{\'e}n, L., Eakins, A., Gordillo, C., Bergdahl, J., and Tollmar, K.
  (2021).
\newblock Adversarial reinforcement learning for procedural content generation.
\newblock {\em arXiv preprint arXiv:2103.04847}.

\bibitem[Glaese et~al., 2022]{glaese2022improving}
Glaese, A., McAleese, N., Tr{\k{e}}bacz, M., Aslanides, J., Firoiu, V., Ewalds,
  T., Rauh, M., Weidinger, L., Chadwick, M., Thacker, P., et~al. (2022).
\newblock Improving alignment of dialogue agents via targeted human judgements.
\newblock {\em arXiv preprint arXiv:2209.14375}.

\bibitem[Goodfellow et~al., 2016]{goodfellow2016deep}
Goodfellow, I., Bengio, Y., and Courville, A. (2016).
\newblock {\em Deep learning}.
\newblock MIT press.

\bibitem[Goodfellow et~al., 2014]{goodfellow2014generative}
Goodfellow, I., Pouget-Abadie, J., Mirza, M., Xu, B., Warde-Farley, D., Ozair,
  S., Courville, A., and Bengio, Y. (2014).
\newblock Generative adversarial nets.
\newblock {\em Advances in neural information processing systems}, 27.

\bibitem[Gordillo et~al., 2021]{gordillo2021improving}
Gordillo, C., Bergdahl, J., Tollmar, K., and Gissl{\'e}n, L. (2021).
\newblock Improving playtesting coverage via curiosity driven reinforcement
  learning agents.
\newblock In {\em 2021 IEEE Conference on Games (CoG)}, pages 1--8. IEEE.

\bibitem[Goyal et~al., 2021]{goyal2021pixl2r}
Goyal, P., Niekum, S., and Mooney, R. (2021).
\newblock Pixl2r: Guiding reinforcement learning using natural language by
  mapping pixels to rewards.
\newblock In {\em Conference on Robot Learning}, pages 485--497. PMLR.

\bibitem[Goyal et~al., 2019]{goyal2019using}
Goyal, P., Niekum, S., and Mooney, R.~J. (2019).
\newblock Using natural language for reward shaping in reinforcement learning.
\newblock {\em arXiv preprint arXiv:1903.02020}.

\bibitem[Gulrajani et~al., 2017]{gulrajani2017improved}
Gulrajani, I., Ahmed, F., Arjovsky, M., Dumoulin, V., and Courville, A.~C.
  (2017).
\newblock Improved training of {W}asserstein {GAN}s.
\newblock In {\em Advances in Neural Information Processing Systems (NeurIPS)},
  pages 5767--5777.

\bibitem[Gupta et~al., 2022]{gupta2022unpacking}
Gupta, A., Pacchiano, A., Zhai, Y., Kakade, S., and Levine, S. (2022).
\newblock Unpacking reward shaping: Understanding the benefits of reward
  engineering on sample complexity.
\newblock {\em Advances in Neural Information Processing Systems},
  35:15281--15295.

\bibitem[Gupta et~al., 2017]{gupta2017cooperative}
Gupta, J.~K., Egorov, M., and Kochenderfer, M.~J. (2017).
\newblock Cooperative multi-agent control using deep reinforcement learning.
\newblock In {\em AAMAS Workshops}.

\bibitem[Haarnoja et~al., 2017]{haarnoja2017reinforcement}
Haarnoja, T., Tang, H., Abbeel, P., and Levine, S. (2017).
\newblock Reinforcement learning with deep energy-based policies.
\newblock {\em arXiv preprint arXiv:1702.08165}.

\bibitem[Haarnoja et~al., 2018]{haarnoja2018soft}
Haarnoja, T., Zhou, A., Hartikainen, K., Tucker, G., Ha, S., Tan, J., Kumar,
  V., Zhu, H., Gupta, A., Abbeel, P., et~al. (2018).
\newblock Soft actor-critic algorithms and applications.
\newblock {\em arXiv preprint arXiv:1812.05905}.

\bibitem[Hadfield-Menell et~al., 2017]{hadfield2017inverse}
Hadfield-Menell, D., Milli, S., Abbeel, P., Russell, S.~J., and Dragan, A.
  (2017).
\newblock Inverse reward design.
\newblock {\em Advances in neural information processing systems}, 30.

\bibitem[Harutyunyan et~al., 2015]{harutyunyan2015expressing}
Harutyunyan, A., Devlin, S., Vrancx, P., and Now{\'e}, A. (2015).
\newblock Expressing arbitrary reward functions as potential-based advice.
\newblock In {\em Proceedings of the AAAI conference on artificial
  intelligence}, volume~29.

\bibitem[Hayes et~al., 2020]{hayes2020dynamic}
Hayes, C.~F., Howley, E., and Mannion, P. (2020).
\newblock Dynamic thresholded lexicograpic ordering.
\newblock In {\em Adaptive and Learning Agents Workshop (AAMAS 2020)}.

\bibitem[Hayes et~al., 2022]{hayes2022practical}
Hayes, C.~F., R{\u{a}}dulescu, R., Bargiacchi, E., K{\"a}llstr{\"o}m, J.,
  Macfarlane, M., Reymond, M., Verstraeten, T., Zintgraf, L.~M., Dazeley, R.,
  Heintz, F., et~al. (2022).
\newblock A practical guide to multi-objective reinforcement learning and
  planning.
\newblock {\em Autonomous Agents and Multi-Agent Systems}, 36(1):26.

\bibitem[Hazan et~al., 2018]{hazan2018provably}
Hazan, E., Kakade, S.~M., Singh, K., and Van~Soest, A. (2018).
\newblock Provably efficient maximum entropy exploration.
\newblock {\em arXiv preprint arXiv:1812.02690}.

\bibitem[He et~al., 2016]{he2016opponent}
He, H., Boyd-Graber, J., Kwok, K., and Daum{\'e}~III, H. (2016).
\newblock Opponent modeling in deep reinforcement learning.
\newblock In {\em International Conference on Machine Learning}, pages
  1804--1813.

\bibitem[Hernandez-Leal et~al., 2018]{hernandez2018multiagent}
Hernandez-Leal, P., Kartal, B., and Taylor, M.~E. (2018).
\newblock Is multiagent deep reinforcement learning the answer or the question?
  a brief survey.
\newblock {\em arXiv preprint arXiv:1810.05587}.

\bibitem[Hernandez-Leal et~al., 2019a]{hernandez2019agent}
Hernandez-Leal, P., Kartal, B., and Taylor, M.~E. (2019a).
\newblock Agent modeling as auxiliary task for deep reinforcement learning.
\newblock In {\em Proceedings of the AAAI conference on artificial intelligence
  and interactive digital entertainment}, volume~15, pages 31--37.

\bibitem[Hernandez-Leal et~al., 2019b]{hern2019agent}
Hernandez-Leal, P., Kartal, B., and Taylor, M.~E. (2019b).
\newblock {Agent Modeling as Auxiliary Task for Deep Reinforcement Learning}.
\newblock In {\em AAAI Conference on Artificial Intelligence and Interactive
  Digital Entertainment}.

\bibitem[Hessel et~al., 2018]{hessel2018rainbow}
Hessel, M., Modayil, J., Van~Hasselt, H., Schaul, T., Ostrovski, G., Dabney,
  W., Horgan, D., Piot, B., Azar, M., and Silver, D. (2018).
\newblock Rainbow: Combining improvements in deep reinforcement learning.
\newblock In {\em Thirty-Second AAAI Conference on Artificial Intelligence}.

\bibitem[Ho and Ermon, 2016]{ho2016generative}
Ho, J. and Ermon, S. (2016).
\newblock Generative adversarial imitation learning.
\newblock {\em Advances in neural information processing systems},
  29:4565--4573.

\bibitem[Hong et~al., 2017]{hong2017deep}
Hong, Z.-W., Su, S.-Y., Shann, T.-Y., Chang, Y.-H., and Lee, C.-Y. (2017).
\newblock A deep policy inference q-network for multi-agent systems.
\newblock {\em arXiv preprint arXiv:1712.07893}.

\bibitem[Hu et~al., 2023]{hu2023gflownet}
Hu, E., Malkin, N., Jain, M., Everett, K., Graikos, A., and Bengio, Y. (2023).
\newblock Gflownet-em for learning compositional latent variable models.
\newblock {\em arXiv preprint arXiv:2302.06576}.

\bibitem[Hu et~al., 2020]{hu2020learning}
Hu, Y., Wang, W., Jia, H., Wang, Y., Chen, Y., Hao, J., Wu, F., and Fan, C.
  (2020).
\newblock Learning to utilize shaping rewards: A new approach of reward
  shaping.
\newblock {\em Advances in Neural Information Processing Systems},
  33:15931--15941.

\bibitem[Huang et~al., 2021]{huang2021therapeutics}
Huang, K., Fu, T., Gao, W., Zhao, Y., Roohani, Y., Leskovec, J., Coley, C.~W.,
  Xiao, C., Sun, J., and Zitnik, M. (2021).
\newblock Therapeutics data commons: Machine learning datasets and tasks for
  drug discovery and development.
\newblock {\em arXiv preprint arXiv:2102.09548}.

\bibitem[Huang et~al., 2022]{huang2022language}
Huang, W., Abbeel, P., Pathak, D., and Mordatch, I. (2022).
\newblock Language models as zero-shot planners: Extracting actionable
  knowledge for embodied agents.
\newblock In {\em International Conference on Machine Learning}, pages
  9118--9147. PMLR.

\bibitem[Ibarz et~al., 2018]{ibarz2018reward}
Ibarz, B., Leike, J., Pohlen, T., Irving, G., Legg, S., and Amodei, D. (2018).
\newblock Reward learning from human preferences and demonstrations in atari.
\newblock {\em Advances in neural information processing systems}, 31.

\bibitem[Icarte et~al., 2022]{icarte2022reward}
Icarte, R.~T., Klassen, T.~Q., Valenzano, R., and McIlraith, S.~A. (2022).
\newblock Reward machines: Exploiting reward function structure in
  reinforcement learning.
\newblock {\em Journal of Artificial Intelligence Research}, 73:173--208.

\bibitem[Iima and Kuroe, 2014]{iima2014multi}
Iima, H. and Kuroe, Y. (2014).
\newblock Multi-objective reinforcement learning for acquiring all pareto
  optimal policies simultaneously-method of determining scalarization weights.
\newblock In {\em 2014 IEEE International Conference on Systems, Man, and
  Cybernetics (SMC)}, pages 876--881. IEEE.

\bibitem[Ioffe and Szegedy, 2015]{ioffe2015batch}
Ioffe, S. and Szegedy, C. (2015).
\newblock Batch normalization: Accelerating deep network training by reducing
  internal covariate shift.
\newblock {\em arXiv preprint arXiv:1502.03167}.

\bibitem[Iqbal and Sha, 2019]{iqbal2018actor}
Iqbal, S. and Sha, F. (2019).
\newblock Actor-attention-critic for multi-agent reinforcement learning.
\newblock In {\em International Conference on Machine Learning}, pages
  2961--2970.

\bibitem[Jacob et~al., 2020]{jacob2020s}
Jacob, M., Devlin, S., and Hofmann, K. (2020).
\newblock “it’s unwieldy and it takes a lot of time”—challenges and
  opportunities for creating agents in commercial games.
\newblock In {\em Proceedings of the AAAI Conference on Artificial Intelligence
  and Interactive Digital Entertainment}, volume~16, pages 88--94.

\bibitem[Jaderberg et~al., 2016]{jaderberg2016reinforcement}
Jaderberg, M., Mnih, V., Czarnecki, W.~M., Schaul, T., Leibo, J.~Z., Silver,
  D., and Kavukcuoglu, K. (2016).
\newblock Reinforcement learning with unsupervised auxiliary tasks.
\newblock {\em arXiv preprint arXiv:1611.05397}.

\bibitem[Jain et~al., 2013]{jain2013learning}
Jain, A., Wojcik, B., Joachims, T., and Saxena, A. (2013).
\newblock Learning trajectory preferences for manipulators via iterative
  improvement.
\newblock {\em Advances in neural information processing systems}, 26.

\bibitem[Jain et~al., 2022a]{jain2022biological}
Jain, M., Bengio, E., Hernandez-Garcia, A., Rector-Brooks, J., Dossou, B.~F.,
  Ekbote, C.~A., Fu, J., Zhang, T., Kilgour, M., Zhang, D., et~al. (2022a).
\newblock Biological sequence design with gflownets.
\newblock In {\em International Conference on Machine Learning}, pages
  9786--9801. PMLR.

\bibitem[Jain et~al., 2022b]{jain2022multi}
Jain, M., Raparthy, S.~C., Hernandez-Garcia, A., Rector-Brooks, J., Bengio, Y.,
  Miret, S., and Bengio, E. (2022b).
\newblock Multi-objective gflownets.
\newblock {\em arXiv preprint arXiv:2210.12765}.

\bibitem[Jang et~al., 2017]{jang2016categorical}
Jang, E., Gu, S., and Poole, B. (2017).
\newblock Categorical reparametrization with gumble-softmax.
\newblock In {\em International Conference on Learning Representations (ICLR
  2017)}. OpenReview. net.

\bibitem[Jaques et~al., 2019a]{jaques2019way}
Jaques, N., Ghandeharioun, A., Shen, J.~H., Ferguson, C., Lapedriza, A., Jones,
  N., Gu, S., and Picard, R. (2019a).
\newblock Way off-policy batch deep reinforcement learning of implicit human
  preferences in dialog.
\newblock {\em arXiv preprint arXiv:1907.00456}.

\bibitem[Jaques et~al., 2019b]{jaques2018intrinsic}
Jaques, N., Lazaridou, A., Hughes, E., Gulcehre, C., Ortega, P., Strouse, D.,
  Leibo, J.~Z., and De~Freitas, N. (2019b).
\newblock Social influence as intrinsic motivation for multi-agent deep
  reinforcement learning.
\newblock In {\em International Conference on Machine Learning}, pages
  3040--3049.

\bibitem[Jeon et~al., 2020]{jeon2020reward}
Jeon, H.~J., Milli, S., and Dragan, A. (2020).
\newblock Reward-rational (implicit) choice: A unifying formalism for reward
  learning.
\newblock {\em Advances in Neural Information Processing Systems},
  33:4415--4426.

\bibitem[Jiang and Lu, 2018]{jiang2018learning}
Jiang, J. and Lu, Z. (2018).
\newblock Learning attentional communication for multi-agent cooperation.
\newblock In {\em Advances in Neural Information Processing Systems}, pages
  7254--7264.

\bibitem[Jin et~al., 2020]{jin2020multi}
Jin, W., Barzilay, R., and Jaakkola, T. (2020).
\newblock Multi-objective molecule generation using interpretable
  substructures.
\newblock In {\em International conference on machine learning}, pages
  4849--4859. PMLR.

\bibitem[Juliani et~al., 2018]{juliani2018unity}
Juliani, A., Berges, V.-P., Teng, E., Cohen, A., Harper, J., Elion, C., Goy,
  C., Gao, Y., Henry, H., Mattar, M., et~al. (2018).
\newblock Unity: A general platform for intelligent agents.
\newblock {\em arXiv preprint arXiv:1809.02627}.

\bibitem[Juozapaitis et~al., 2019]{juozapaitis2019explainable}
Juozapaitis, Z., Koul, A., Fern, A., Erwig, M., and Doshi-Velez, F. (2019).
\newblock Explainable reinforcement learning via reward decomposition.
\newblock In {\em IJCAI/ECAI Workshop on explainable artificial intelligence}.

\bibitem[Kaelbling et~al., 1996]{kaelbling1996reinforcement}
Kaelbling, L.~P., Littman, M.~L., and Moore, A.~W. (1996).
\newblock Reinforcement learning: A survey.
\newblock {\em Journal of artificial intelligence research}, 4:237--285.

\bibitem[Kallenberg, 2011]{kallenberg2011mdps}
Kallenberg, L. (2011).
\newblock Markov decision processes.
\newblock {\em Lecture Notes. University of Leiden}, pages 65--66.

\bibitem[Kaplan et~al., 2020]{kaplan2020scaling}
Kaplan, J., McCandlish, S., Henighan, T., Brown, T.~B., Chess, B., Child, R.,
  Gray, S., Radford, A., Wu, J., and Amodei, D. (2020).
\newblock Scaling laws for neural language models.
\newblock {\em arXiv preprint arXiv:2001.08361}.

\bibitem[Kaplan et~al., 2017]{kaplan2017beating}
Kaplan, R., Sauer, C., and Sosa, A. (2017).
\newblock Beating atari with natural language guided reinforcement learning.
\newblock {\em arXiv preprint arXiv:1704.05539}.

\bibitem[Karpathy, 2017]{karpathy2017software2}
Karpathy, A. (2017).
\newblock Software 2.0.
\newblock Data Set.

\bibitem[Kartal et~al., 2019]{kartal2019terminal}
Kartal, B., Hernandez-Leal, P., and Taylor, M.~E. (2019).
\newblock Terminal prediction as an auxiliary task for deep reinforcement
  learning.
\newblock In {\em Proceedings of the AAAI Conference on Artificial Intelligence
  and Interactive Digital Entertainment}, volume~15, pages 38--44.

\bibitem[Keeney et~al., 1993]{keeney1993decisions}
Keeney, R., Raiffa, H., L, K., and Meyer, R. (1993).
\newblock {\em Decisions with Multiple Objectives: Preferences and Value
  Trade-Offs}.
\newblock Wiley series in probability and mathematical statistics. Applied
  probability and statistics. Cambridge University Press.

\bibitem[Kingma and Ba, 2014]{kingma2014adam}
Kingma, D.~P. and Ba, J. (2014).
\newblock Adam: A method for stochastic optimization.
\newblock {\em arXiv preprint arXiv:1412.6980}.

\bibitem[Kingma and Welling, 2013]{kingma2013auto}
Kingma, D.~P. and Welling, M. (2013).
\newblock Auto-encoding variational bayes.
\newblock {\em arXiv preprint arXiv:1312.6114}.

\bibitem[Klissarov et~al., 2023]{klissarov2023motif}
Klissarov, M., D'Oro, P., Sodhani, S., Raileanu, R., Bacon, P.-L., Vincent, P.,
  Zhang, A., and Henaff, M. (2023).
\newblock Motif: Intrinsic motivation from artificial intelligence feedback.
\newblock {\em arXiv preprint arXiv:2310.00166}.

\bibitem[Knox et~al., 2023]{knox2023reward}
Knox, W.~B., Allievi, A., Banzhaf, H., Schmitt, F., and Stone, P. (2023).
\newblock Reward (mis) design for autonomous driving.
\newblock {\em Artificial Intelligence}, 316:103829.

\bibitem[Knox and Stone, 2009]{knox2009interactively}
Knox, W.~B. and Stone, P. (2009).
\newblock Interactively shaping agents via human reinforcement: The tamer
  framework.
\newblock In {\em Proceedings of the fifth international conference on
  Knowledge capture}, pages 9--16.

\bibitem[Knox and Stone, 2012]{knox2012reinforcement}
Knox, W.~B. and Stone, P. (2012).
\newblock Reinforcement learning from simultaneous human and mdp reward.
\newblock In {\em AAMAS}, volume 1004, pages 475--482. Valencia.

\bibitem[Korpelevich, 1976]{korpelevich1976extragradient}
Korpelevich, G.~M. (1976).
\newblock The extragradient method for finding saddle points and other
  problems.
\newblock {\em Matecon}, 12:747--756.

\bibitem[Kostrikov et~al., 2018]{kostrikov2018discriminator}
Kostrikov, I., Agrawal, K.~K., Dwibedi, D., Levine, S., and Tompson, J. (2018).
\newblock Discriminator-actor-critic: Addressing sample inefficiency and reward
  bias in adversarial imitation learning.
\newblock {\em arXiv preprint arXiv:1809.02925}.

\bibitem[Kostrikov et~al., 2019]{kostrikov2018discriminatoractorcritic}
Kostrikov, I., Agrawal, K.~K., Dwibedi, D., Levine, S., and Tompson, J. (2019).
\newblock Discriminator-actor-critic: Addressing sample inefficiency and reward
  bias in adversarial imitation learning.
\newblock In {\em Proceedings of the 7th International Conference on Learning
  Representations (ICLR)}.

\bibitem[Kostrikov et~al., 2020]{Kostrikov2020Imitation}
Kostrikov, I., Nachum, O., and Tompson, J. (2020).
\newblock Imitation learning via off-policy distribution matching.
\newblock In {\em Proceedings of the 8th International Conference on Learning
  Representations (ICLR)}.

\bibitem[Kreutzer et~al., 2018]{kreutzer2018can}
Kreutzer, J., Khadivi, S., Matusov, E., and Riezler, S. (2018).
\newblock Can neural machine translation be improved with user feedback?
\newblock {\em arXiv preprint arXiv:1804.05958}.

\bibitem[Kuefler et~al., 2017]{kuefler2017imitating}
Kuefler, A., Morton, J., Wheeler, T., and Kochenderfer, M. (2017).
\newblock Imitating driver behavior with generative adversarial networks.
\newblock In {\em Proceedings of 2017 IEEE Intelligent Vehicles Symposium
  (IV)}, pages 204--211.

\bibitem[Kumar et~al., 2012]{kumar2012fragment}
Kumar, A., Voet, A., and Zhang, K.~Y. (2012).
\newblock Fragment based drug design: from experimental to computational
  approaches.
\newblock {\em Current medicinal chemistry}, 19(30):5128--5147.

\bibitem[Kurach et~al., 2019]{kurach2019google}
Kurach, K., Raichuk, A., Sta{\'n}czyk, P., Zajac, M., Bachem, O., Espeholt, L.,
  Riquelme, C., Vincent, D., Michalski, M., Bousquet, O., et~al. (2019).
\newblock Google research football: A novel reinforcement learning environment.
\newblock {\em arXiv preprint arXiv:1907.11180}.

\bibitem[Kwon et~al., 2023]{kwon2023reward}
Kwon, M., Xie, S.~M., Bullard, K., and Sadigh, D. (2023).
\newblock Reward design with language models.
\newblock {\em arXiv preprint arXiv:2303.00001}.

\bibitem[Lahlou et~al., 2023]{lahlou2023theory}
Lahlou, S., Deleu, T., Lemos, P., Zhang, D., Volokhova, A.,
  Hern{\'a}ndez-Garc{\'\i}a, A., Ezzine, L.~N., Bengio, Y., and Malkin, N.
  (2023).
\newblock A theory of continuous generative flow networks.
\newblock {\em arXiv preprint arXiv:2301.12594}.

\bibitem[Lample and Chaplot, 2017]{lample2017playing}
Lample, G. and Chaplot, D.~S. (2017).
\newblock Playing fps games with deep reinforcement learning.
\newblock In {\em Proceedings of the AAAI Conference on Artificial
  Intelligence}, volume~31.

\bibitem[Laskey et~al., 2017]{laskey2017dart}
Laskey, M., Lee, J., Fox, R., Dragan, A., and Goldberg, K. (2017).
\newblock Dart: Noise injection for robust imitation learning.
\newblock In {\em Conference on robot learning}, pages 143--156. PMLR.

\bibitem[Laskin et~al., 2020]{laskin2020curl}
Laskin, M., Srinivas, A., and Abbeel, P. (2020).
\newblock Curl: Contrastive unsupervised representations for reinforcement
  learning.
\newblock In {\em International Conference on Machine Learning}, pages
  5639--5650. PMLR.

\bibitem[Laud and DeJong, 2003]{laud2003influence}
Laud, A. and DeJong, G. (2003).
\newblock The influence of reward on the speed of reinforcement learning: An
  analysis of shaping.
\newblock In {\em Proceedings of the 20th International Conference on Machine
  Learning (ICML-03)}, pages 440--447.

\bibitem[Lazaridou et~al., 2016]{lazaridou2016multi}
Lazaridou, A., Peysakhovich, A., and Baroni, M. (2016).
\newblock Multi-agent cooperation and the emergence of (natural) language.
\newblock {\em arXiv preprint arXiv:1612.07182}.

\bibitem[LeCun et~al., 2015]{lecun2015deep}
LeCun, Y., Bengio, Y., and Hinton, G. (2015).
\newblock Deep learning.
\newblock {\em nature}, 521(7553):436--444.

\bibitem[LeCun et~al., 1998]{lecun1998gradient}
LeCun, Y., Bottou, L., Bengio, Y., and Haffner, P. (1998).
\newblock Gradient-based learning applied to document recognition.
\newblock {\em Proceedings of the IEEE}, 86(11):2278--2324.

\bibitem[Lee et~al., 2021]{lee2021pebble}
Lee, K., Smith, L., and Abbeel, P. (2021).
\newblock Pebble: Feedback-efficient interactive reinforcement learning via
  relabeling experience and unsupervised pre-training.
\newblock {\em arXiv preprint arXiv:2106.05091}.

\bibitem[Lehman et~al., 2020]{lehman2020surprising}
Lehman, J., Clune, J., Misevic, D., Adami, C., Altenberg, L., Beaulieu, J.,
  Bentley, P.~J., Bernard, S., Beslon, G., Bryson, D.~M., et~al. (2020).
\newblock The surprising creativity of digital evolution: A collection of
  anecdotes from the evolutionary computation and artificial life research
  communities.
\newblock {\em Artificial life}, 26(2):274--306.

\bibitem[Leike et~al., 2018]{leike2018scalable}
Leike, J., Krueger, D., Everitt, T., Martic, M., Maini, V., and Legg, S.
  (2018).
\newblock Scalable agent alignment via reward modeling: a research direction.
\newblock {\em arXiv preprint arXiv:1811.07871}.

\bibitem[Levine et~al., 2011]{levine2011nonlinear}
Levine, S., Popovic, Z., and Koltun, V. (2011).
\newblock Nonlinear inverse reinforcement learning with gaussian processes.
\newblock {\em Advances in neural information processing systems}, 24.

\bibitem[Li and Czarnecki, 2018]{li2018urban}
Li, C. and Czarnecki, K. (2018).
\newblock Urban driving with multi-objective deep reinforcement learning.
\newblock {\em arXiv preprint arXiv:1811.08586}.

\bibitem[Li et~al., 2020]{li2020deep}
Li, K., Zhang, T., and Wang, R. (2020).
\newblock Deep reinforcement learning for multiobjective optimization.
\newblock {\em IEEE transactions on cybernetics}, 51(6):3103--3114.

\bibitem[Liang et~al., 2018]{liang2018accelerated}
Liang, Q., Que, F., and Modiano, E. (2018).
\newblock Accelerated primal-dual policy optimization for safe reinforcement
  learning.
\newblock {\em arXiv preprint arXiv:1802.06480}.

\bibitem[Lillicrap et~al., 2015]{lillicrap2015continuous}
Lillicrap, T.~P., Hunt, J.~J., Pritzel, A., Heess, N., Erez, T., Tassa, Y.,
  Silver, D., and Wierstra, D. (2015).
\newblock Continuous control with deep reinforcement learning.
\newblock {\em arXiv preprint arXiv:1509.02971}.

\bibitem[Lin, 1993]{lin1993reinforcement}
Lin, L.-J. (1993).
\newblock Reinforcement learning for robots using neural networks.
\newblock Technical report, Carnegie-Mellon Univ Pittsburgh PA School of
  Computer Science.

\bibitem[Lin et~al., 2020]{gda}
Lin, T., Jin, C., and Jordan, M. (2020).
\newblock On gradient descent ascent for nonconvex-concave minimax problems.
\newblock In {\em International Conference on Machine Learning}, pages
  6083--6093. PMLR.

\bibitem[Lin et~al., 2019a]{lin2019adaptive}
Lin, X., Baweja, H., Kantor, G., and Held, D. (2019a).
\newblock Adaptive auxiliary task weighting for reinforcement learning.
\newblock {\em Advances in neural information processing systems}, 32.

\bibitem[Lin et~al., 2019b]{lin2019pareto}
Lin, X., Zhen, H.-L., Li, Z., Zhang, Q.-F., and Kwong, S. (2019b).
\newblock Pareto multi-task learning.
\newblock {\em Advances in neural information processing systems}, 32.

\bibitem[Littman, 1994]{littman1994markov}
Littman, M.~L. (1994).
\newblock Markov games as a framework for multi-agent reinforcement learning.
\newblock In {\em Machine learning proceedings 1994}, pages 157--163. Elsevier.

\bibitem[Liu et~al., 2023]{liu2023efficient}
Liu, Y., Datta, G., Novoseller, E., and Brown, D.~S. (2023).
\newblock Efficient preference-based reinforcement learning using learned
  dynamics models.
\newblock {\em arXiv preprint arXiv:2301.04741}.

\bibitem[Liu et~al., 2020]{liu2020ipo}
Liu, Y., Ding, J., and Liu, X. (2020).
\newblock Ipo: Interior-point policy optimization under constraints.
\newblock In {\em Proceedings of the AAAI conference on artificial
  intelligence}, volume~34, pages 4940--4947.

\bibitem[Liu et~al., 2021]{liu2021policy}
Liu, Y., Halev, A., and Liu, X. (2021).
\newblock Policy learning with constraints in model-free reinforcement
  learning: A survey.
\newblock In {\em The 30th International Joint Conference on Artificial
  Intelligence (IJCAI)}.

\bibitem[Lowe et~al., 2017]{lowe2017multi}
Lowe, R., Wu, Y., Tamar, A., Harb, J., Abbeel, O.~P., and Mordatch, I. (2017).
\newblock Multi-agent actor-critic for mixed cooperative-competitive
  environments.
\newblock In {\em Advances in Neural Information Processing Systems}, pages
  6379--6390.

\bibitem[Luketina et~al., 2019]{luketina2019survey}
Luketina, J., Nardelli, N., Farquhar, G., Foerster, J., Andreas, J.,
  Grefenstette, E., Whiteson, S., and Rockt{\"a}schel, T. (2019).
\newblock A survey of reinforcement learning informed by natural language.
\newblock {\em arXiv preprint arXiv:1906.03926}.

\bibitem[Luo et~al., 2022]{luo2022controlling}
Luo, J., Paduraru, C., Voicu, O., Chervonyi, Y., Munns, S., Li, J., Qian, C.,
  Dutta, P., Davis, J.~Q., Wu, N., et~al. (2022).
\newblock Controlling commercial cooling systems using reinforcement learning.
\newblock {\em arXiv preprint arXiv:2211.07357}.

\bibitem[Lyle et~al., 2021]{lyle2021effect}
Lyle, C., Rowland, M., Ostrovski, G., and Dabney, W. (2021).
\newblock On the effect of auxiliary tasks on representation dynamics.
\newblock In {\em International Conference on Artificial Intelligence and
  Statistics}, pages 1--9. PMLR.

\bibitem[Ma et~al., 2023]{ma2023eureka}
Ma, Y.~J., Liang, W., Wang, G., Huang, D.-A., Bastani, O., Jayaraman, D., Zhu,
  Y., Fan, L., and Anandkumar, A. (2023).
\newblock Eureka: Human-level reward design via coding large language models.

\bibitem[MacGlashan et~al., 2015]{macglashan2015grounding}
MacGlashan, J., Babes-Vroman, M., desJardins, M., Littman, M.~L., Muresan, S.,
  Squire, S., Tellex, S., Arumugam, D., and Yang, L. (2015).
\newblock Grounding english commands to reward functions.
\newblock In {\em Robotics: Science and Systems}.

\bibitem[MacGlashan et~al., 2017]{macglashan2017interactive}
MacGlashan, J., Ho, M.~K., Loftin, R., Peng, B., Wang, G., Roberts, D.~L.,
  Taylor, M.~E., and Littman, M.~L. (2017).
\newblock Interactive learning from policy-dependent human feedback.
\newblock In {\em International conference on machine learning}, pages
  2285--2294. PMLR.

\bibitem[Madan et~al., 2022]{madan2022learning}
Madan, K., Rector-Brooks, J., Korablyov, M., Bengio, E., Jain, M., Nica, A.,
  Bosc, T., Bengio, Y., and Malkin, N. (2022).
\newblock Learning gflownets from partial episodes for improved convergence and
  stability.
\newblock {\em arXiv preprint arXiv:2209.12782}.

\bibitem[Madan et~al., 2023]{madan2023learning}
Madan, K., Rector-Brooks, J., Korablyov, M., Bengio, E., Jain, M., Nica, A.~C.,
  Bosc, T., Bengio, Y., and Malkin, N. (2023).
\newblock Learning gflownets from partial episodes for improved convergence and
  stability.
\newblock In {\em International Conference on Machine Learning}, pages
  23467--23483. PMLR.

\bibitem[Mahajan et~al., 2019]{mahajan2019maven}
Mahajan, A., Rashid, T., Samvelyan, M., and Whiteson, S. (2019).
\newblock Maven: Multi-agent variational exploration.
\newblock In {\em Advances in Neural Information Processing Systems}, pages
  7613--7624.

\bibitem[Mahmood et~al., 2014]{mahmood2014weighted}
Mahmood, A.~R., van Hasselt, H.~P., and Sutton, R.~S. (2014).
\newblock Weighted importance sampling for off-policy learning with linear
  function approximation.
\newblock In {\em Advances in Neural Information Processing Systems}, pages
  3014--3022.

\bibitem[Malkin et~al., 2022a]{malkin2022trajectory}
Malkin, N., Jain, M., Bengio, E., Sun, C., and Bengio, Y. (2022a).
\newblock Trajectory balance: Improved credit assignment in gflownets.
\newblock {\em Advances in Neural Information Processing Systems},
  35:5955--5967.

\bibitem[Malkin et~al., 2022b]{malkin2022gflownets}
Malkin, N., Lahlou, S., Deleu, T., Ji, X., Hu, E., Everett, K., Zhang, D., and
  Bengio, Y. (2022b).
\newblock Gflownets and variational inference.
\newblock {\em arXiv preprint arXiv:2210.00580}.

\bibitem[Marchesini et~al., 2022]{marchesini2022exploring}
Marchesini, E., Corsi, D., and Farinelli, A. (2022).
\newblock Exploring safer behaviors for deep reinforcement learning.
\newblock In {\em Proceedings of the AAAI Conference on Artificial
  Intelligence}, volume~36, pages 7701--7709.

\bibitem[Mathewson and Pilarski, 2022]{mathewson2022brief}
Mathewson, K.~W. and Pilarski, P.~M. (2022).
\newblock A brief guide to designing and evaluating human-centered interactive
  machine learning.
\newblock {\em arXiv preprint arXiv:2204.09622}.

\bibitem[Miettinen, 2012]{miettinen2012nonlinear}
Miettinen, K. (2012).
\newblock {\em Nonlinear multiobjective optimization}, volume~12.
\newblock Springer Science \& Business Media.

\bibitem[Mindermann et~al., 2018]{mindermann2018active}
Mindermann, S., Shah, R., Gleave, A., and Hadfield-Menell, D. (2018).
\newblock Active inverse reward design.
\newblock {\em arXiv preprint arXiv:1809.03060}.

\bibitem[Mirowski et~al., 2016]{mirowski2016learning}
Mirowski, P., Pascanu, R., Viola, F., Soyer, H., Ballard, A.~J., Banino, A.,
  Denil, M., Goroshin, R., Sifre, L., Kavukcuoglu, K., et~al. (2016).
\newblock Learning to navigate in complex environments.
\newblock {\em arXiv preprint arXiv:1611.03673}.

\bibitem[Mnih et~al., 2013]{mnih2013playing}
Mnih, V., Kavukcuoglu, K., Silver, D., Graves, A., Antonoglou, I., Wierstra,
  D., and Riedmiller, M. (2013).
\newblock Playing atari with deep reinforcement learning.
\newblock {\em arXiv preprint arXiv:1312.5602}.

\bibitem[Mnih et~al., 2015]{mnih2015human}
Mnih, V., Kavukcuoglu, K., Silver, D., Rusu, A.~A., Veness, J., Bellemare,
  M.~G., Graves, A., Riedmiller, M., Fidjeland, A.~K., Ostrovski, G., et~al.
  (2015).
\newblock Human-level control through deep reinforcement learning.
\newblock {\em nature}, 518(7540):529--533.

\bibitem[Moerland et~al., 2023]{moerland2023model}
Moerland, T.~M., Broekens, J., Plaat, A., Jonker, C.~M., et~al. (2023).
\newblock Model-based reinforcement learning: A survey.
\newblock {\em Foundations and Trends{\textregistered} in Machine Learning},
  16(1):1--118.

\bibitem[Mohamed and Lakshminarayanan, 2016]{mohamed2016learning}
Mohamed, S. and Lakshminarayanan, B. (2016).
\newblock Learning in implicit generative models.
\newblock {\em arXiv preprint arXiv:1610.03483}.

\bibitem[Mordatch and Abbeel, 2018]{mordatch2018emergence}
Mordatch, I. and Abbeel, P. (2018).
\newblock Emergence of grounded compositional language in multi-agent
  populations.
\newblock In {\em Thirty-Second AAAI Conference on Artificial Intelligence}.

\bibitem[Mosqueira-Rey et~al., 2023]{mosqueira2023human}
Mosqueira-Rey, E., Hern{\'a}ndez-Pereira, E., Alonso-R{\'\i}os, D.,
  Bobes-Bascar{\'a}n, J., and Fern{\'a}ndez-Leal, {\'A}. (2023).
\newblock Human-in-the-loop machine learning: A state of the art.
\newblock {\em Artificial Intelligence Review}, 56(4):3005--3054.

\bibitem[Mossalam et~al., 2016]{mossalam2016multi}
Mossalam, H., Assael, Y.~M., Roijers, D.~M., and Whiteson, S. (2016).
\newblock Multi-objective deep reinforcement learning.
\newblock {\em arXiv preprint arXiv:1610.02707}.

\bibitem[Murphy, 2012]{murphy2012machine}
Murphy, K.~P. (2012).
\newblock {\em Machine learning: a probabilistic perspective}.
\newblock MIT press.

\bibitem[Nachum et~al., 2019]{nachum2019dualdice}
Nachum, O., Chow, Y., Dai, B., and Li, L. (2019).
\newblock Dual{DICE}: Behavior-agnostic estimation of discounted stationary
  distribution corrections.
\newblock In {\em Advances in Neural Information Processing Systems (NeurIPS)},
  pages 2318--2328.

\bibitem[Nachum et~al., 2018]{nachum2018trustpcl}
Nachum, O., Norouzi, M., Xu, K., and Schuurmans, D. (2018).
\newblock Trust-{PCL}: An off-policy trust region method for continuous
  control.
\newblock In {\em Proceedings of the 6th International Conference on Learning
  Representations (ICLR)}.

\bibitem[Nair and Hinton, 2010]{nair2010rectified}
Nair, V. and Hinton, G.~E. (2010).
\newblock Rectified linear units improve restricted boltzmann machines.
\newblock In {\em Proceedings of the 27th international conference on machine
  learning (ICML-10)}, pages 807--814.

\bibitem[Ng et~al., 1999]{ng1999policy}
Ng, A.~Y., Harada, D., and Russell, S. (1999).
\newblock Policy invariance under reward transformations: Theory and
  application to reward shaping.
\newblock In {\em ICML}, volume~99, pages 278--287.

\bibitem[Ng et~al., 2000]{ng2000algorithms}
Ng, A.~Y., Russell, S., et~al. (2000).
\newblock Algorithms for inverse reinforcement learning.
\newblock In {\em Icml}, volume~1, page~2.

\bibitem[O'Donoghue et~al., 2016]{o2016combining}
O'Donoghue, B., Munos, R., Kavukcuoglu, K., and Mnih, V. (2016).
\newblock Combining policy gradient and q-learning.
\newblock {\em arXiv preprint arXiv:1611.01626}.

\bibitem[Oord et~al., 2016]{oord2016wavenet}
Oord, A. v.~d., Dieleman, S., Zen, H., Simonyan, K., Vinyals, O., Graves, A.,
  Kalchbrenner, N., Senior, A., and Kavukcuoglu, K. (2016).
\newblock Wavenet: A generative model for raw audio.
\newblock {\em arXiv preprint arXiv:1609.03499}.

\bibitem[Otter et~al., 2020]{otter2020survey}
Otter, D.~W., Medina, J.~R., and Kalita, J.~K. (2020).
\newblock A survey of the usages of deep learning for natural language
  processing.
\newblock {\em IEEE transactions on neural networks and learning systems},
  32(2):604--624.

\bibitem[Ouyang et~al., 2022]{ouyang2022training}
Ouyang, L., Wu, J., Jiang, X., Almeida, D., Wainwright, C., Mishkin, P., Zhang,
  C., Agarwal, S., Slama, K., Ray, A., et~al. (2022).
\newblock Training language models to follow instructions with human feedback.
\newblock {\em Advances in Neural Information Processing Systems},
  35:27730--27744.

\bibitem[Palan et~al., 2019]{palan2019learning}
Palan, M., Landolfi, N.~C., Shevchuk, G., and Sadigh, D. (2019).
\newblock Learning reward functions by integrating human demonstrations and
  preferences.
\newblock {\em arXiv preprint arXiv:1906.08928}.

\bibitem[Pan et~al., 2022]{pan2022effects}
Pan, A., Bhatia, K., and Steinhardt, J. (2022).
\newblock The effects of reward misspecification: Mapping and mitigating
  misaligned models.
\newblock {\em arXiv preprint arXiv:2201.03544}.

\bibitem[Pan et~al., 2023]{pan2023better}
Pan, L., Malkin, N., Zhang, D., and Bengio, Y. (2023).
\newblock Better training of gflownets with local credit and incomplete
  trajectories.
\newblock {\em arXiv preprint arXiv:2302.01687}.

\bibitem[Papadopoulos and Linke, 2006]{papadopoulos2006multiobjective}
Papadopoulos, A.~I. and Linke, P. (2006).
\newblock Multiobjective molecular design for integrated process-solvent
  systems synthesis.
\newblock {\em AIChE Journal}, 52(3):1057--1070.

\bibitem[Pardalos et~al., 2017]{pardalos2017non}
Pardalos, P.~M., {\v{Z}}ilinskas, A., {\v{Z}}ilinskas, J., et~al. (2017).
\newblock {\em Non-convex multi-objective optimization}.
\newblock Springer.

\bibitem[Parisi et~al., 2014]{parisi2014policy}
Parisi, S., Pirotta, M., Smacchia, N., Bascetta, L., and Restelli, M. (2014).
\newblock Policy gradient approaches for multi-objective sequential decision
  making.
\newblock In {\em 2014 International Joint Conference on Neural Networks
  (IJCNN)}, pages 2323--2330. IEEE.

\bibitem[Paszke et~al., 2019]{paszke2019pytorch}
Paszke, A., Gross, S., Massa, F., Lerer, A., Bradbury, J., Chanan, G., Killeen,
  T., Lin, Z., Gimelshein, N., Antiga, L., et~al. (2019).
\newblock Pytorch: An imperative style, high-performance deep learning library.
\newblock {\em Advances in neural information processing systems}, 32.

\bibitem[Pilarski et~al., 2011]{pilarski2011online}
Pilarski, P.~M., Dawson, M.~R., Degris, T., Fahimi, F., Carey, J.~P., and
  Sutton, R.~S. (2011).
\newblock Online human training of a myoelectric prosthesis controller via
  actor-critic reinforcement learning.
\newblock In {\em 2011 IEEE international conference on rehabilitation
  robotics}, pages 1--7. IEEE.

\bibitem[Pineda et~al., 2015]{pineda2015revisiting}
Pineda, L.~E., Wray, K.~H., and Zilberstein, S. (2015).
\newblock Revisiting multi-objective mdps with relaxed lexicographic
  preferences.
\newblock In {\em 2015 AAAI Fall Symposium Series}.

\bibitem[Polyak, 1970]{polyak}
Polyak, B. (1970).
\newblock Iterative methods using lagrange multipliers for solving extremal
  problems with constraints of the equation type.
\newblock {\em USSR Computational Mathematics and Mathematical Physics},
  10(5):42--52.

\bibitem[Pomerleau, 1991]{pomerleau1991efficient}
Pomerleau, D.~A. (1991).
\newblock Efficient training of artificial neural networks for autonomous
  navigation.
\newblock {\em Neural computation}, 3(1):88--97.

\bibitem[Popova et~al., 2018]{popova2018deep}
Popova, M., Isayev, O., and Tropsha, A. (2018).
\newblock Deep reinforcement learning for de novo drug design.
\newblock {\em Science advances}, 4(7):eaap7885.

\bibitem[Puterman, 1990]{puterman1990markov}
Puterman, M.~L. (1990).
\newblock Markov decision processes.
\newblock {\em Handbooks in operations research and management science},
  2:331--434.

\bibitem[Ramachandran and Amir, 2007]{ramachandran2007bayesian}
Ramachandran, D. and Amir, E. (2007).
\newblock Bayesian inverse reinforcement learning.
\newblock In {\em IJCAI}, volume~7, pages 2586--2591.

\bibitem[Ramp{\'a}{\v{s}}ek et~al., 2022]{rampavsek2022recipe}
Ramp{\'a}{\v{s}}ek, L., Galkin, M., Dwivedi, V.~P., Luu, A.~T., Wolf, G., and
  Beaini, D. (2022).
\newblock Recipe for a general, powerful, scalable graph transformer.
\newblock {\em Advances in Neural Information Processing Systems},
  35:14501--14515.

\bibitem[Randl{\o}v and Alstr{\o}m, 1998]{randlov1998learning}
Randl{\o}v, J. and Alstr{\o}m, P. (1998).
\newblock Learning to drive a bicycle using reinforcement learning and shaping.
\newblock In {\em ICML}, volume~98, pages 463--471.

\bibitem[Rashid et~al., 2018]{rashid2018qmix}
Rashid, T., Samvelyan, M., Witt, C.~S., Farquhar, G., Foerster, J., and
  Whiteson, S. (2018).
\newblock Qmix: Monotonic value function factorisation for deep multi-agent
  reinforcement learning.
\newblock In {\em International Conference on Machine Learning}, pages
  4292--4301.

\bibitem[Ratliff et~al., 2006]{ratliff2006maximum}
Ratliff, N.~D., Bagnell, J.~A., and Zinkevich, M.~A. (2006).
\newblock Maximum margin planning.
\newblock In {\em Proceedings of the 23rd international conference on Machine
  learning}, pages 729--736.

\bibitem[Ratliff et~al., 2009]{ratliff2009learning}
Ratliff, N.~D., Silver, D., and Bagnell, J.~A. (2009).
\newblock Learning to search: Functional gradient techniques for imitation
  learning.
\newblock {\em Autonomous Robots}, 27:25--53.

\bibitem[Ratner et~al., 2018]{ratner2018simplifying}
Ratner, E., Hadfield-Menell, D., and Dragan, A.~D. (2018).
\newblock Simplifying reward design through divide-and-conquer.
\newblock {\em arXiv preprint arXiv:1806.02501}.

\bibitem[Ray et~al., 2019]{ray2019benchmarking}
Ray, A., Achiam, J., and Amodei, D. (2019).
\newblock Benchmarking safe exploration in deep reinforcement learning.
\newblock {\em arXiv preprint arXiv:1910.01708}, 7.

\bibitem[Reddy et~al., 2019]{reddy2019sqil}
Reddy, S., Dragan, A.~D., and Levine, S. (2019).
\newblock {SQIL}: Imitation learning via reinforcement learning with sparse
  rewards.

\bibitem[Resnick et~al., 2018]{resnick2018pommerman}
Resnick, C., Eldridge, W., Ha, D., Britz, D., Foerster, J., Togelius, J., Cho,
  K., and Bruna, J. (2018).
\newblock Pommerman: A multi-agent playground.
\newblock {\em arXiv preprint arXiv:1809.07124}.

\bibitem[Reymond et~al., 2023]{reymond2023actor}
Reymond, M., Hayes, C.~F., Steckelmacher, D., Roijers, D.~M., and Now{\'e}, A.
  (2023).
\newblock Actor-critic multi-objective reinforcement learning for non-linear
  utility functions.
\newblock {\em Autonomous Agents and Multi-Agent Systems}, 37(2):23.

\bibitem[Reymond and Now{\'e}, 2019]{reymond2019pareto}
Reymond, M. and Now{\'e}, A. (2019).
\newblock Pareto-dqn: Approximating the pareto front in complex multi-objective
  decision problems.
\newblock In {\em Proceedings of the adaptive and learning agents workshop
  (ALA-19) at AAMAS}.

\bibitem[Rezende and Mohamed, 2015]{rezende2015variational}
Rezende, D. and Mohamed, S. (2015).
\newblock Variational inference with normalizing flows.
\newblock In {\em International conference on machine learning}, pages
  1530--1538. PMLR.

\bibitem[Roijers et~al., 2013]{roijers2013survey}
Roijers, D.~M., Vamplew, P., Whiteson, S., and Dazeley, R. (2013).
\newblock A survey of multi-objective sequential decision-making.
\newblock {\em Journal of Artificial Intelligence Research}, 48:67--113.

\bibitem[Rosenbaum et~al., 2017]{rosenbaum2017routing}
Rosenbaum, C., Klinger, T., and Riemer, M. (2017).
\newblock Routing networks: Adaptive selection of non-linear functions for
  multi-task learning.
\newblock {\em arXiv preprint arXiv:1711.01239}.

\bibitem[Rosenblatt, 1958]{rosenblatt1958perceptron}
Rosenblatt, F. (1958).
\newblock The perceptron: a probabilistic model for information storage and
  organization in the brain.
\newblock {\em Psychological review}, 65(6):386.

\bibitem[Ross and Bagnell, 2010]{ross2010efficient}
Ross, S. and Bagnell, D. (2010).
\newblock Efficient reductions for imitation learning.
\newblock In {\em Proceedings of the 13th International Conference on
  Artificial Intelligence and Statistics (AISTATS)}, pages 661--668.

\bibitem[Ross et~al., 2011]{ross2011reduction}
Ross, S., Gordon, G., and Bagnell, D. (2011).
\newblock A reduction of imitation learning and structured prediction to
  no-regret online learning.
\newblock In {\em Proceedings of the fourteenth international conference on
  artificial intelligence and statistics}, pages 627--635. JMLR Workshop and
  Conference Proceedings.

\bibitem[Roy et~al., 2021]{roy2021direct}
Roy, J., Girgis, R., Romoff, J., Bacon, P.-L., and Pal, C. (2021).
\newblock Direct behavior specification via constrained reinforcement learning.
\newblock {\em arXiv preprint arXiv:2112.12228}.

\bibitem[Rumelhart et~al., 1986]{rumelhart1986learning}
Rumelhart, D.~E., Hinton, G.~E., and Williams, R.~J. (1986).
\newblock Learning representations by back-propagating errors.
\newblock {\em nature}, 323(6088):533--536.

\bibitem[Russell, 1998]{russell1998learning}
Russell, S. (1998).
\newblock Learning agents for uncertain environments.
\newblock In {\em Proceedings of the eleventh annual conference on
  Computational learning theory}, pages 101--103.

\bibitem[Russell and Zimdars, 2003]{russell2003q}
Russell, S.~J. and Zimdars, A. (2003).
\newblock Q-decomposition for reinforcement learning agents.
\newblock In {\em Proceedings of the 20th International Conference on Machine
  Learning (ICML-03)}, pages 656--663.

\bibitem[Rust, 2008]{rust2008dynamic}
Rust, J. (2008).
\newblock Dynamic programming.
\newblock {\em The new Palgrave dictionary of economics}, 1:8.

\bibitem[Sadigh et~al., 2017]{sadigh2017active}
Sadigh, D., Dragan, A.~D., Sastry, S., and Seshia, S.~A. (2017).
\newblock Active preference-based learning of reward functions.

\bibitem[Sasaki et~al., 2018]{sasaki2018sample}
Sasaki, F., Yohira, T., and Kawaguchi, A. (2018).
\newblock Sample efficient imitation learning for continuous control.
\newblock In {\em Proceedings of the 6th International Conference on Learning
  Representations (ICLR)}.

\bibitem[Saunders et~al., 2017]{saunders2017trial}
Saunders, W., Sastry, G., Stuhlmueller, A., and Evans, O. (2017).
\newblock Trial without error: Towards safe reinforcement learning via human
  intervention.
\newblock {\em arXiv preprint arXiv:1707.05173}.

\bibitem[Schadd et~al., 2007]{schadd2007opponent}
Schadd, F., Bakkes, S., and Spronck, P. (2007).
\newblock Opponent modeling in real-time strategy games.
\newblock In {\em GAMEON}, pages 61--70.

\bibitem[Schaul et~al., 2015a]{schaul2015universal}
Schaul, T., Horgan, D., Gregor, K., and Silver, D. (2015a).
\newblock Universal value function approximators.
\newblock In {\em International conference on machine learning}, pages
  1312--1320.

\bibitem[Schaul et~al., 2015b]{schaul2015prioritized}
Schaul, T., Quan, J., Antonoglou, I., and Silver, D. (2015b).
\newblock Prioritized experience replay.
\newblock {\em arXiv preprint arXiv:1511.05952}.

\bibitem[Schulman et~al., 2015]{schulman2015trust}
Schulman, J., Levine, S., Abbeel, P., Jordan, M., and Moritz, P. (2015).
\newblock Trust region policy optimization.
\newblock In {\em Proceedings of the 32nd International Conference on Machine
  Learning (ICML)}, pages 1889--1897.

\bibitem[Schulman et~al., 2017]{schulman2017proximal}
Schulman, J., Wolski, F., Dhariwal, P., Radford, A., and Klimov, O. (2017).
\newblock Proximal policy optimization algorithms.
\newblock {\em arXiv preprint arXiv:1707.06347}.

\bibitem[Septon and Amir, 2022]{septon2022integrating}
Septon, Y. and Amir, O. (2022).
\newblock Integrating policy summaries with reward decomposition explanations.
\newblock In {\em ICAPS 2022 Workshop on Explainable AI Planning}.

\bibitem[Settles, 2009]{settles2009active}
Settles, B. (2009).
\newblock Active learning literature survey.

\bibitem[Shahidinejad and Ghobaei-Arani, 2020]{shahidinejad2020joint}
Shahidinejad, A. and Ghobaei-Arani, M. (2020).
\newblock Joint computation offloading and resource provisioning for e
  dge-cloud computing environment: A machine learning-based approach.
\newblock {\em Software: Practice and Experience}, 50(12):2212--2230.

\bibitem[Shao et~al., 2019]{shao2019survey}
Shao, K., Tang, Z., Zhu, Y., Li, N., and Zhao, D. (2019).
\newblock A survey of deep reinforcement learning in video games.
\newblock {\em arXiv preprint arXiv:1912.10944}.

\bibitem[Shelhamer et~al., 2016]{shelhamer2016loss}
Shelhamer, E., Mahmoudieh, P., Argus, M., and Darrell, T. (2016).
\newblock Loss is its own reward: Self-supervision for reinforcement learning.
\newblock {\em arXiv preprint arXiv:1612.07307}.

\bibitem[Siddique et~al., 2020]{siddique2020learning}
Siddique, U., Weng, P., and Zimmer, M. (2020).
\newblock Learning fair policies in multi-objective (deep) reinforcement
  learning with average and discounted rewards.
\newblock In {\em International Conference on Machine Learning}, pages
  8905--8915. PMLR.

\bibitem[Silver et~al., 2008]{silver2008high}
Silver, D., Bagnell, J., and Stentz, A. (2008).
\newblock High performance outdoor navigation from overhead data using
  imitation learning.
\newblock {\em Robotics: Science and Systems IV, Zurich, Switzerland}, 1.

\bibitem[Silver et~al., 2016]{silver2016mastering}
Silver, D., Huang, A., Maddison, C.~J., Guez, A., Sifre, L., Van Den~Driessche,
  G., Schrittwieser, J., Antonoglou, I., Panneershelvam, V., Lanctot, M.,
  et~al. (2016).
\newblock Mastering the game of go with deep neural networks and tree search.
\newblock {\em nature}, 529(7587):484--489.

\bibitem[Silver et~al., 2014]{silver2014deterministic}
Silver, D., Lever, G., Heess, N., Degris, T., Wierstra, D., and Riedmiller, M.
  (2014).
\newblock Deterministic policy gradient algorithms.
\newblock In {\em International conference on machine learning}, pages
  387--395. Pmlr.

\bibitem[Silver et~al., 2017]{silver2017mastering}
Silver, D., Schrittwieser, J., Simonyan, K., Antonoglou, I., Huang, A., Guez,
  A., Hubert, T., Baker, L., Lai, M., Bolton, A., et~al. (2017).
\newblock Mastering the game of go without human knowledge.
\newblock {\em nature}, 550(7676):354--359.

\bibitem[Silver et~al., 2021]{silver2021reward}
Silver, D., Singh, S., Precup, D., and Sutton, R.~S. (2021).
\newblock Reward is enough.
\newblock {\em Artificial Intelligence}, 299:103535.

\bibitem[Singh et~al., 2009]{singh2009rewards}
Singh, S., Lewis, R.~L., and Barto, A.~G. (2009).
\newblock Where do rewards come from.
\newblock In {\em Proceedings of the annual conference of the cognitive science
  society}, pages 2601--2606. Cognitive Science Society.

\bibitem[Singh et~al., 2010]{singh2010separating}
Singh, S., Lewis, R.~L., Sorg, J., Barto, A.~G., and Helou, A. (2010).
\newblock On separating agent designer goals from agent goals: Breaking the
  preferences--parameters confound.

\bibitem[Skalse et~al., 2022]{skalse2022lexicographic}
Skalse, J., Hammond, L., Griffin, C., and Abate, A. (2022).
\newblock Lexicographic multi-objective reinforcement learning.
\newblock {\em arXiv preprint arXiv:2212.13769}.

\bibitem[Song et~al., 2021]{song2021multimodal}
Song, H., Li, A., Wang, T., and Wang, M. (2021).
\newblock Multimodal deep reinforcement learning with auxiliary task for
  obstacle avoidance of indoor mobile robot.
\newblock {\em Sensors}, 21(4):1363.

\bibitem[Sorg, 2011]{sorg2011optimal}
Sorg, J.~D. (2011).
\newblock {\em The optimal reward problem: Designing effective reward for
  bounded agents}.
\newblock PhD thesis, University of Michigan.

\bibitem[Srivastava et~al., 2014]{srivastava2014dropout}
Srivastava, N., Hinton, G., Krizhevsky, A., Sutskever, I., and Salakhutdinov,
  R. (2014).
\newblock Dropout: a simple way to prevent neural networks from overfitting.
\newblock {\em The journal of machine learning research}, 15(1):1929--1958.

\bibitem[St{\aa}hl et~al., 2019]{staahl2019deep}
St{\aa}hl, N., Falkman, G., Karlsson, A., Mathiason, G., and Bostrom, J.
  (2019).
\newblock Deep reinforcement learning for multiparameter optimization in de
  novo drug design.
\newblock {\em Journal of chemical information and modeling}, 59(7):3166--3176.

\bibitem[Stiennon et~al., 2020]{stiennon2020learning}
Stiennon, N., Ouyang, L., Wu, J., Ziegler, D., Lowe, R., Voss, C., Radford, A.,
  Amodei, D., and Christiano, P.~F. (2020).
\newblock Learning to summarize with human feedback.
\newblock {\em Advances in Neural Information Processing Systems},
  33:3008--3021.

\bibitem[Stooke et~al., 2020]{stooke2020responsive}
Stooke, A., Achiam, J., and Abbeel, P. (2020).
\newblock Responsive safety in reinforcement learning by pid lagrangian
  methods.
\newblock In {\em International Conference on Machine Learning}, pages
  9133--9143. PMLR.

\bibitem[Strouse et~al., 2018]{strouse2018learning}
Strouse, D., Kleiman-Weiner, M., Tenenbaum, J., Botvinick, M., and Schwab,
  D.~J. (2018).
\newblock Learning to share and hide intentions using information
  regularization.
\newblock In {\em Advances in Neural Information Processing Systems}, pages
  10270--10281.

\bibitem[Suddarth and Kergosien, 1990]{suddarth1990rule}
Suddarth, S.~C. and Kergosien, Y. (1990).
\newblock Rule-injection hints as a means of improving network performance and
  learning time.
\newblock In {\em European association for signal processing workshop}, pages
  120--129. Springer.

\bibitem[Sumers et~al., 2021]{sumers2021learning}
Sumers, T.~R., Ho, M.~K., Hawkins, R.~D., Narasimhan, K., and Griffiths, T.~L.
  (2021).
\newblock Learning rewards from linguistic feedback.
\newblock In {\em Proceedings of the AAAI Conference on Artificial
  Intelligence}, volume~35, pages 6002--6010.

\bibitem[Sutton, 2015]{sutton2015introduction}
Sutton, R.~S. (2015).
\newblock Introduction to reinforcement learning with function approximation.
\newblock In {\em Tutorial at the conference on neural information processing
  systems}, volume~33.

\bibitem[Sutton and Barto, 2018]{sutton2018rlTextbook}
Sutton, R.~S. and Barto, A.~G. (2018).
\newblock {\em Reinforcement learning: An introduction}.
\newblock MIT press.

\bibitem[Sutton et~al., 2000]{sutton2000policy}
Sutton, R.~S., McAllester, D.~A., Singh, S.~P., and Mansour, Y. (2000).
\newblock Policy gradient methods for reinforcement learning with function
  approximation.
\newblock In {\em Advances in neural information processing systems}, pages
  1057--1063.

\bibitem[Sutton et~al., 2011]{sutton2011horde}
Sutton, R.~S., Modayil, J., Delp, M., Degris, T., Pilarski, P.~M., White, A.,
  and Precup, D. (2011).
\newblock Horde: A scalable real-time architecture for learning knowledge from
  unsupervised sensorimotor interaction.
\newblock In {\em The 10th International Conference on Autonomous Agents and
  Multiagent Systems-Volume 2}, pages 761--768.

\bibitem[Tallec et~al., 2019]{tallec2019making}
Tallec, C., Blier, L., and Ollivier, Y. (2019).
\newblock Making deep q-learning methods robust to time discretization.
\newblock In {\em International Conference on Machine Learning}, pages
  6096--6104. PMLR.

\bibitem[Taylor, 2023]{taylor2023reinforcement}
Taylor, M.~E. (2023).
\newblock Reinforcement learning requires human-in-the-loop framing and
  approaches.
\newblock In {\em HHAI 2023: Augmenting Human Intellect}, pages 351--360. IOS
  Press.

\bibitem[Teh et~al., 2017]{teh2017distral}
Teh, Y., Bapst, V., Czarnecki, W.~M., Quan, J., Kirkpatrick, J., Hadsell, R.,
  Heess, N., and Pascanu, R. (2017).
\newblock Distral: Robust multitask reinforcement learning.
\newblock {\em Advances in neural information processing systems}, 30.

\bibitem[Tesauro, 1994]{tesauro1994td}
Tesauro, G. (1994).
\newblock Td-gammon, a self-teaching backgammon program, achieves master-level
  play.
\newblock {\em Neural computation}, 6(2):215--219.

\bibitem[Tessler et~al., 2018]{tessler2018reward}
Tessler, C., Mankowitz, D.~J., and Mannor, S. (2018).
\newblock Reward constrained policy optimization.
\newblock {\em arXiv preprint arXiv:1805.11074}.

\bibitem[Todorov et~al., 2012]{todorov2012mujoco}
Todorov, E., Erez, T., and Tassa, Y. (2012).
\newblock Mujoco: A physics engine for model-based control.
\newblock In {\em 2012 IEEE/RSJ international conference on intelligent robots
  and systems}, pages 5026--5033. IEEE.

\bibitem[Treviso et~al., 2023]{treviso2023efficient}
Treviso, M., Lee, J.-U., Ji, T., Aken, B.~v., Cao, Q., Ciosici, M.~R., Hassid,
  M., Heafield, K., Hooker, S., Raffel, C., et~al. (2023).
\newblock Efficient methods for natural language processing: A survey.
\newblock {\em Transactions of the Association for Computational Linguistics},
  11:826--860.

\bibitem[Turchetta et~al., 2020]{turchetta2020safe}
Turchetta, M., Kolobov, A., Shah, S., Krause, A., and Agarwal, A. (2020).
\newblock Safe reinforcement learning via curriculum induction.
\newblock {\em arXiv preprint arXiv:2006.12136}.

\bibitem[Uhlenbeck and Ornstein, 1930]{uhlenbeck1930theory}
Uhlenbeck, G.~E. and Ornstein, L.~S. (1930).
\newblock On the theory of the brownian motion.
\newblock {\em Physical review}, 36(5):823.

\bibitem[Uzawa et~al., 1958]{uzawa}
Uzawa, H., Anow, K., and Hurwicz, L. (1958).
\newblock Studies in linear and nonlinear programming.

\bibitem[Vamplew et~al., 2011]{vamplew2011empirical}
Vamplew, P., Dazeley, R., Berry, A., Issabekov, R., and Dekker, E. (2011).
\newblock Empirical evaluation methods for multiobjective reinforcement
  learning algorithms.
\newblock {\em Machine learning}, 84:51--80.

\bibitem[Vamplew et~al., 2018]{vamplew2018human}
Vamplew, P., Dazeley, R., Foale, C., Firmin, S., and Mummery, J. (2018).
\newblock Human-aligned artificial intelligence is a multiobjective problem.
\newblock {\em Ethics and Information Technology}, 20:27--40.

\bibitem[Vamplew et~al., 2022]{vamplew2022scalar}
Vamplew, P., Smith, B.~J., K{\"a}llstr{\"o}m, J., Ramos, G., R{\u{a}}dulescu,
  R., Roijers, D.~M., Hayes, C.~F., Heintz, F., Mannion, P., Libin, P.~J.,
  et~al. (2022).
\newblock Scalar reward is not enough: A response to silver, singh, precup and
  sutton (2021).
\newblock {\em Autonomous Agents and Multi-Agent Systems}, 36(2):41.

\bibitem[Vamplew et~al., 2008]{vamplew2008limitations}
Vamplew, P., Yearwood, J., Dazeley, R., and Berry, A. (2008).
\newblock On the limitations of scalarisation for multi-objective reinforcement
  learning of pareto fronts.
\newblock In {\em AI 2008: Advances in Artificial Intelligence: 21st
  Australasian Joint Conference on Artificial Intelligence Auckland, New
  Zealand, December 1-5, 2008. Proceedings 21}, pages 372--378. Springer.

\bibitem[Van~den Oord et~al., 2016]{van2016conditional}
Van~den Oord, A., Kalchbrenner, N., Espeholt, L., Vinyals, O., Graves, A.,
  et~al. (2016).
\newblock Conditional image generation with pixelcnn decoders.
\newblock {\em Advances in neural information processing systems}, 29.

\bibitem[Van~Hasselt et~al., 2016]{van2016ddqn}
Van~Hasselt, H., Guez, A., and Silver, D. (2016).
\newblock Deep reinforcement learning with double q-learning.
\newblock In {\em Thirtieth AAAI conference on artificial intelligence}.

\bibitem[Van~Moffaert et~al., 2013a]{van2013hypervolume}
Van~Moffaert, K., Drugan, M.~M., and Now{\'e}, A. (2013a).
\newblock Hypervolume-based multi-objective reinforcement learning.
\newblock In {\em Evolutionary Multi-Criterion Optimization: 7th International
  Conference, EMO 2013, Sheffield, UK, March 19-22, 2013. Proceedings 7}, pages
  352--366. Springer.

\bibitem[Van~Moffaert et~al., 2013b]{van2013scalarized}
Van~Moffaert, K., Drugan, M.~M., and Now{\'e}, A. (2013b).
\newblock Scalarized multi-objective reinforcement learning: Novel design
  techniques.
\newblock In {\em 2013 IEEE symposium on adaptive dynamic programming and
  reinforcement learning (ADPRL)}, pages 191--199. IEEE.

\bibitem[Van~Moffaert and Now{\'e}, 2014]{van2014multi}
Van~Moffaert, K. and Now{\'e}, A. (2014).
\newblock Multi-objective reinforcement learning using sets of pareto
  dominating policies.
\newblock {\em The Journal of Machine Learning Research}, 15(1):3483--3512.

\bibitem[Vincent et~al., 2008]{vincent2008extracting}
Vincent, P., Larochelle, H., Bengio, Y., and Manzagol, P.-A. (2008).
\newblock Extracting and composing robust features with denoising autoencoders.
\newblock In {\em Proceedings of the 25th international conference on Machine
  learning}, pages 1096--1103.

\bibitem[Vinyals et~al., 2019]{vinyals2019grandmaster}
Vinyals, O., Babuschkin, I., Czarnecki, W.~M., Mathieu, M., Dudzik, A., Chung,
  J., Choi, D.~H., Powell, R., Ewalds, T., Georgiev, P., et~al. (2019).
\newblock Grandmaster level in starcraft ii using multi-agent reinforcement
  learning.
\newblock {\em Nature}, 575(7782):350--354.

\bibitem[Vinyals et~al., 2017]{vinyals2017starcraft}
Vinyals, O., Ewalds, T., Bartunov, S., Georgiev, P., Vezhnevets, A.~S., Yeo,
  M., Makhzani, A., K{\"u}ttler, H., Agapiou, J., Schrittwieser, J., et~al.
  (2017).
\newblock Starcraft ii: A new challenge for reinforcement learning.
\newblock {\em arXiv preprint arXiv:1708.04782}.

\bibitem[Wang and Usher, 2005]{wang2005application}
Wang, Y.-C. and Usher, J.~M. (2005).
\newblock Application of reinforcement learning for agent-based production
  scheduling.
\newblock {\em Engineering applications of artificial intelligence},
  18(1):73--82.

\bibitem[Wang et~al., 2016]{wang2016dueling}
Wang, Z., Schaul, T., Hessel, M., Hasselt, H., Lanctot, M., and Freitas, N.
  (2016).
\newblock Dueling network architectures for deep reinforcement learning.
\newblock In {\em International conference on machine learning}, pages
  1995--2003.

\bibitem[Watkins and Dayan, 1992]{watkins1992q}
Watkins, C.~J. and Dayan, P. (1992).
\newblock Q-learning.
\newblock {\em Machine learning}, 8(3-4):279--292.

\bibitem[Waytowich et~al., 2019]{waytowich2019narration}
Waytowich, N., Barton, S.~L., Lawhern, V., and Warnell, G. (2019).
\newblock A narration-based reward shaping approach using grounded natural
  language commands.
\newblock {\em arXiv preprint arXiv:1911.00497}.

\bibitem[White, 2015]{white2015}
White, A. (2015).
\newblock {\em Developing a predictive approach to knowledge}.
\newblock PhD thesis, University of Alberta.

\bibitem[Wiewiora et~al., 2003]{wiewiora2003principled}
Wiewiora, E., Cottrell, G.~W., and Elkan, C. (2003).
\newblock Principled methods for advising reinforcement learning agents.
\newblock In {\em Proceedings of the 20th international conference on machine
  learning (ICML-03)}, pages 792--799.

\bibitem[Williams et~al., 2018]{williams2018learning}
Williams, E.~C., Gopalan, N., Rhee, M., and Tellex, S. (2018).
\newblock Learning to parse natural language to grounded reward functions with
  weak supervision.
\newblock In {\em 2018 ieee international conference on robotics and automation
  (icra)}, pages 4430--4436. IEEE.

\bibitem[Williams, 1992]{williams1992simple}
Williams, R.~J. (1992).
\newblock Simple statistical gradient-following algorithms for connectionist
  reinforcement learning.
\newblock {\em Machine learning}, 8(3-4):229--256.

\bibitem[Wirth et~al., 2017]{wirth2017survey}
Wirth, C., Akrour, R., Neumann, G., F{\"u}rnkranz, J., et~al. (2017).
\newblock A survey of preference-based reinforcement learning methods.
\newblock {\em Journal of Machine Learning Research}, 18(136):1--46.

\bibitem[Wray et~al., 2015]{wray2015multi}
Wray, K., Zilberstein, S., and Mouaddib, A.-I. (2015).
\newblock Multi-objective mdps with conditional lexicographic reward
  preferences.
\newblock In {\em Proceedings of the AAAI Conference on Artificial
  Intelligence}, volume~29.

\bibitem[Wu et~al., 2018]{wu2018moleculenet}
Wu, Z., Ramsundar, B., Feinberg, E.~N., Gomes, J., Geniesse, C., Pappu, A.~S.,
  Leswing, K., and Pande, V. (2018).
\newblock Moleculenet: a benchmark for molecular machine learning.
\newblock {\em Chemical science}, 9(2):513--530.

\bibitem[Wulfmeier et~al., 2015]{wulfmeier2015maximum}
Wulfmeier, M., Ondruska, P., and Posner, I. (2015).
\newblock Maximum entropy deep inverse reinforcement learning.
\newblock {\em arXiv preprint arXiv:1507.04888}.

\bibitem[Wydmuch et~al., 2018]{wydmuch2018vizdoom}
Wydmuch, M., Kempka, M., and Ja{\'s}kowski, W. (2018).
\newblock Vizdoom competitions: Playing doom from pixels.
\newblock {\em IEEE Transactions on Games}, 11(3):248--259.

\bibitem[Xu et~al., 2020]{xu2020prediction}
Xu, J., Tian, Y., Ma, P., Rus, D., Sueda, S., and Matusik, W. (2020).
\newblock Prediction-guided multi-objective reinforcement learning for
  continuous robot control.
\newblock In {\em International conference on machine learning}, pages
  10607--10616. PMLR.

\bibitem[Xu et~al., 2018]{xu2018powerful}
Xu, K., Hu, W., Leskovec, J., and Jegelka, S. (2018).
\newblock How powerful are graph neural networks?
\newblock {\em arXiv preprint arXiv:1810.00826}.

\bibitem[Yang et~al., 2020]{yang2020projection}
Yang, T.-Y., Rosca, J., Narasimhan, K., and Ramadge, P.~J. (2020).
\newblock Projection-based constrained policy optimization.
\newblock {\em arXiv preprint arXiv:2010.03152}.

\bibitem[Yu et~al., 2019]{yu2019deep}
Yu, L., Xie, W., Xie, D., Zou, Y., Zhang, D., Sun, Z., Zhang, L., Zhang, Y.,
  and Jiang, T. (2019).
\newblock Deep reinforcement learning for smart home energy management.
\newblock {\em IEEE Internet of Things Journal}, 7(4):2751--2762.

\bibitem[Yu et~al., 2020]{yu2020gradient}
Yu, T., Kumar, S., Gupta, A., Levine, S., Hausman, K., and Finn, C. (2020).
\newblock Gradient surgery for multi-task learning.
\newblock {\em Advances in Neural Information Processing Systems},
  33:5824--5836.

\bibitem[Yu et~al., 2023]{yu2023language}
Yu, W., Gileadi, N., Fu, C., Kirmani, S., Lee, K.-H., Arenas, M.~G., Chiang,
  H.-T.~L., Erez, T., Hasenclever, L., Humplik, J., et~al. (2023).
\newblock Language to rewards for robotic skill synthesis.
\newblock {\em arXiv preprint arXiv:2306.08647}.

\bibitem[Yu, 2018]{yu2018towards}
Yu, Y. (2018).
\newblock Towards sample efficient reinforcement learning.
\newblock In {\em IJCAI}, pages 5739--5743.

\bibitem[Yun et~al., 2019]{yun2019graph}
Yun, S., Jeong, M., Kim, R., Kang, J., and Kim, H.~J. (2019).
\newblock Graph transformer networks.
\newblock {\em Advances in neural information processing systems}, 32.

\bibitem[Zhang et~al., 2022]{zhang2022unifying}
Zhang, D., Chen, R.~T., Malkin, N., and Bengio, Y. (2022).
\newblock Unifying generative models with gflownets.
\newblock {\em arXiv preprint arXiv:2209.02606}.

\bibitem[Zhang et~al., 2023]{zhang2023robust}
Zhang, D.~W., Rainone, C., Peschl, M., and Bondesan, R. (2023).
\newblock Robust scheduling with gflownets.
\newblock {\em arXiv preprint arXiv:2302.05446}.

\bibitem[Zhang et~al., 2018]{zhang2018deep}
Zhang, T., McCarthy, Z., Jow, O., Lee, D., Chen, X., Goldberg, K., and Abbeel,
  P. (2018).
\newblock Deep imitation learning for complex manipulation tasks from virtual
  reality teleoperation.
\newblock In {\em 2018 IEEE International Conference on Robotics and Automation
  (ICRA)}, pages 5628--5635. IEEE.

\bibitem[Zhang et~al., 2020]{zhang2020first}
Zhang, Y., Vuong, Q., and Ross, K.~W. (2020).
\newblock First order constrained optimization in policy space.
\newblock {\em arXiv preprint arXiv:2002.06506}.

\bibitem[Zhao et~al., 2023]{zhao2023survey}
Zhao, W.~X., Zhou, K., Li, J., Tang, T., Wang, X., Hou, Y., Min, Y., Zhang, B.,
  Zhang, J., Dong, Z., et~al. (2023).
\newblock A survey of large language models.
\newblock {\em arXiv preprint arXiv:2303.18223}.

\bibitem[Zheng et~al., 2021]{zheng2021imitation}
Zheng, B., Verma, S., Zhou, J., Tsang, I., and Chen, F. (2021).
\newblock Imitation learning: Progress, taxonomies and opportunities.
\newblock {\em arXiv preprint arXiv:2106.12177}.

\bibitem[Zhou et~al., 2018]{zhou2018hybrid}
Zhou, H., Gong, Y., Mugrai, L., Khalifa, A., Nealen, A., and Togelius, J.
  (2018).
\newblock A hybrid search agent in pommerman.
\newblock In {\em Proceedings of the 13th International Conference on the
  Foundations of Digital Games (FDG)}, pages 1--4.

\bibitem[Zhou et~al., 2019]{zhou2019optimization}
Zhou, Z., Kearnes, S., Li, L., Zare, R.~N., and Riley, P. (2019).
\newblock Optimization of molecules via deep reinforcement learning.
\newblock {\em Scientific reports}, 9(1):10752.

\bibitem[Ziebart, 2010]{ziebart2010modeling}
Ziebart, B.~D. (2010).
\newblock {\em Modeling purposeful adaptive behavior with the principle of
  maximum causal entropy}.
\newblock Carnegie Mellon University.

\bibitem[Ziebart et~al., 2008]{ziebart2008maximum}
Ziebart, B.~D., Maas, A.~L., Bagnell, J.~A., and Dey, A.~K. (2008).
\newblock Maximum entropy inverse reinforcement learning.
\newblock In {\em Proceedings of the 23rd AAAI Conference on Artificial
  Intelligence}, pages 1433--1438.

\bibitem[Ziegler et~al., 2019]{ziegler2019fine}
Ziegler, D.~M., Stiennon, N., Wu, J., Brown, T.~B., Radford, A., Amodei, D.,
  Christiano, P., and Irving, G. (2019).
\newblock Fine-tuning language models from human preferences.
\newblock {\em arXiv preprint arXiv:1909.08593}.

\end{thebibliography}
}
%


\ifthenelse{\equal{\Langue}{english}}{
	\addcontentsline{toc}{compteur}{APPENDICES}
}{
	\addcontentsline{toc}{compteur}{ANNEXES}
}

\ifthenelse{\equal{\AnnexesPresentes}{O}}{
	\appendix%
	\newcommand{\Annexe}[1]{\annexe{#1}\setcounter{figure}{0}\setcounter{table}{0}\setcounter{footnote}{0}}%


\Annexe{Supplementary Material for Chapter~\ref{chap:article1_asaf}} 

\tocless{\section{Proofs}}{}

\tocless{\subsection{Proof of Lemma~\ref{lem:optim_D_gan}}}{\label{app:ASAF:proof_of_optim_D_gan}}
\begin{proof}
Lemma~\ref{lem:optim_D_gan} states that given $L(\tildep, \pg)$ defined in~Equation~\ref{eq:structured_GAN_obj}:
\begin{enumerate}[label=(\alph*)]
    \item $\displaystyle \tildep^* \triangleq \argmax_{\tildep} L(\tildep, \pg) = \pe$
    \item $\displaystyle \argmin_{\pg} L(\pe, \pg) = \pe$
\end{enumerate}

Starting with (a), we have:
\begin{align*}
    \argmax_{\tildep} L(\tildep, \pg)
    &= \argmax_{\tildep}
    \sum_{x_i} \pe(x_i)\log D_{\tildep, \pg}(x_i) + \pg(x_i)\log (1 - D_{\tildep, \pg}(x_i))\\
    &\triangleq \argmax_{\tildep} \sum_{x_i} L_i
\end{align*}
Assuming infinite discriminator's capacity, $L_i$ can be made independent for all $x_i \in \mathcal{X}$ and we can construct our optimal discriminator $D_{\tildep, \pg}^*$ as a look-up table $D_{\tildep, \pg}^*:\mathcal{X} \rightarrow \, ]0,1[ \, ; \, x_i \mapsto D^*_i$ with $D^*_i$ the optimal discriminator for each $x_i$ defined as:
\begin{equation}
    D^*_i = \argmax_{D_i}L_i = \argmax_{D_i} \pei
    \log D_i + \pgi\log (1 - D_i),
\end{equation}
with $\pgi\triangleq\pg(x_i)$, $\pei\triangleq\pe(x_i)$ and $D_i\triangleq D(x_i)$.

Recall that $D_i \in \, ]0,1[$ and that $\pgi \in \, ]0,1[$. Therefore the function $ \tildep_i \mapsto D_i = \dfrac{\tildep_i}{\tildep_i + \pgi}$ is defined for $\tildep_i \in ]0, +\infty[$. Since it is strictly monotonic over that domain we have that:
\begin{align}
D_i^*=\argmax_{D_i} L_i \, \Leftrightarrow \, \tildep^*_i =\argmax_{\tildep_i} L_i
\end{align}
Taking the derivative and setting to zero, we get:
\begin{align}
    \left.\frac{d L_i}{d \tildep_i}\right|_{\tildep_i} = 0 \, \Leftrightarrow& \,
    \tildep_i = \pei
\end{align}
The second derivative test confirms that we have a maximum, i.e. $\left.\dfrac{d^2 L_i}{d \tildep_i^2}\right|_{\tildep_i^*} < 0$. The values of $L_i$ at the boundaries of the domain of definition of $\tildep_i$ tend to $-\infty$, therefore $L_i(\tildep_i^*=\pei)$ is the global maximum of $L_i$ w.r.t. $\tildep_i$. Finally, the optimal global discriminator is given by:
\begin{equation}
\label{eq:D_optim_GAN_proof}
    D_{\tildep, \pg}^*(x) = \frac{\pe(x)}{\pe(x)+\pg(x)} \quad \forall x \in \mathcal{X}
\end{equation}
This concludes the proof for (a).

The proof for (b) can be found in the work of \citet{goodfellow2014generative}. We reproduce it here for completion. Since from (a) we know that $\tildep^*(x) = \pe(x) \, \forall x \in \mathcal{X}$, we can write the GAN objective for the optimal discriminator as:
\begin{align}
    \argmin_{\pg} L(\tildep^*, \pg)
    &= \argmin_{\pg} L(\pe, \pg) \\
    &= \argmin_{\pg} \E_{x\sim \pe}\left[\log \frac{\pe(x)}{\pe(x)+\pg(x)}\right] + \E_{x\sim \pg}\left[\log \frac{\pg(x)}{\pe(x)+\pg(x)}\right]\label{eq:L_Dopt}
\end{align}
Note that:
\begin{equation}
\label{eq:minuslog4}
    \log4 = \E_{x\sim \pe}\left[\log 2\right] + \E_{x\sim \pg}\left[\log 2\right]
\end{equation}
Adding Equation~\ref{eq:minuslog4} to Equation~\ref{eq:L_Dopt} and subtracting $\log4$ on both sides:
\begin{align}
    \argmin_{\pg} L(\pe, \pg) &= -\log4 + \E_{x\sim \pe}\left[\log \frac{2\pe(x)}{\pe(x)+\pg(x)}\right] + \E_{x \sim \pg}\left[\log \frac{2\pg(x)}{\pe(x)+\pg(x)}\right]\\
    &= -\log4 + \KL\left(\pe\left\|\frac{\pe + \pg}{2}\right.\right)+ \KL\left(\pe\left\|\frac{\pe + \pg}{2}\right.\right)\\
    &= -\log4 + 2\JS\left(\pe\left\|\pg\right.\right)
\end{align}
Where $\KL$ and $\JS$ are respectively the Kullback-Leibler and the Jensen-Shannon divergences. Since the Jensen-Shannon divergence between two distributions is always non-negative and zero if and only if the two distributions are equal, we have that: 
\begin{equation}
\displaystyle \argmin_{\pg} L(\pe, \pg) = \pe
\end{equation}

This concludes the proof for (b). 
\end{proof}

\tocless{\subsection{Proof of Theorem~\ref{thm:traj_GAN}}}{\label{app:ASAF:proof_of_thm_traj_GAN}}
\begin{proof}
Theorem~\ref{thm:traj_GAN} states that given $L(\tildepi, \pig)$ defined in~Equation~\ref{eq:TASAF_obj}:
\begin{enumerate}[label=(\alph*)]
    \item $\displaystyle \tildepi^* \triangleq \argmax_{\tildepi} L(\tildepi, \pig) \text{ satisfies } q_{\tildepi^*} = q_{\pie}$
    \item $\displaystyle \pig^* = \tildepi^* \in \argmin_{\pig}L(\tildepi^*, \pig)$
\end{enumerate}

The proof of (a) is very similar to the one from Lemma~\ref{lem:optim_D_gan}. Starting from Equation~\ref{eq:TASAF_obj} we have:
\begin{align}
    \argmax_{\tildepi}
    L(\tildepi, \pig)
    &=
    \argmax_{\tildepi} \sum_{\tau_i} P_{\pie}(\tau_i) \log D_{\tildepi, \pig}(\tau_i)
    +
    P_{\pig}(\tau_i) \log (1-D_{\tildepi, \pig}(\tau_i)) \\
    &=
    \argmax_{\tildepi} \sum_{\tau_i} \xi(\tau_i) \left( q_{\pie}(\tau_i) \log D_{\tildepi, \pig}(\tau_i)
    +
    q_{\pig}(\tau_i) \log (1-D_{\tildepi, \pig}(\tau_i))
    \right) \\
    &=
    \argmax_{\tildepi} \sum_{\tau_i} L_i
\end{align}
Like for Lemma~\ref{lem:optim_D_gan}, we can optimise for each $L_i$ individually. When doing so, $\xi(\tau_i)$ can be omitted as it is constant w.r.t $\tildepi$. The rest of the proof is identical to the one of but Lemma~\ref{lem:optim_D_gan} with $\pe = q_{\pie}$ and $\pg = q_{\pig}$. It follows that the max of $L(\tildepi, \pig)$ is reached for $q_{\tildepi}^* = q_{\pie}$. From that we obtain that the policy $\tildepi^*$ that makes the discriminator $D_{\tildepi^*, \pig}$ optimal w.r.t $L(\tildepi, \pig)$ is such that $q_{\tildepi^*} = q_{\tildepi}^* = q_{\pie}$ i.e. $\prod_{t=0}^{T-1}\tildepi^*(a_t|s_t) = \prod_{t=0}^{T-1}\pie(a_t|s_t) \, \forall \, \tau$. 

The proof for (b) stems from the observation that choosing $\pig = \tildepi^*$ (the policy recovered by the optimal discriminator $D_{\tildepi^*, \pig}$) minimizes $L(\tildepi^*, \pig)$:
\begin{align}
    \pig(a|s) = \tildepi^*(a|s) \,
    \, \forall \, (s,a) \in \gS\times\gA \quad
    &\Rightarrow \quad \prod_{t=0}^{T-1}\pig(a_t|s_t) = \prod_{t=0}^{T-1}\tildepi^*(a_t|s_t) \,\, \forall \tau \in \gT \\
    &\Rightarrow \quad q_{\pig}(\tau) = q_{\pie}(\tau) \,\, \forall \, \tau \in \gT \\
    &\Rightarrow \quad D_{\tildepi^*, \tildepi^*} = \frac{1}{2} \,\, \forall \,\tau \in \gT \\
    &\Rightarrow \quad L(\tildepi^*, \tildepi^*) = -\log 4
\end{align}
By multiplying the numerator and denominator of $D_{\tildepi^*, \tildepi^*}$ by $\xi(\tau)$ it can be shown in exactly the same way as in Appendix~\ref{app:ASAF:proof_of_optim_D_gan} that $-\log4$ is the global minimum of $L(\tildepi^*, \pig)$.
\end{proof}

\tocless{\section{Adversarial Soft Q-Fitting: transition-wise Imitation Learning without Policy Optimization}}{\label{app:ASAF:ASQF}}
In this section we present Adversarial Soft Q-Fitting (ASQF), a principled approach to Imitation Learning without Reinforcement Learning that relies exclusively on transitions. Using transitions rather than trajectories presents several practical benefits such as the possibility to deal with asynchronously collected data or non-sequential experts demonstrations. We first present the theoretical setting for ASQF and then test it on a variety of discrete control tasks. We show that while it is theoretically sound, ASQF is often outperformed by ASAF-1, an approximation to ASAF that also allows to rely on transitions instead of trajectories.

\paragraph{Theoretical Setting}
We consider the GAN objective of Equation~\ref{eq:GAN_obj} with 
$x=(s,a)$, $
\mathcal{X}=\gS \times \gA$, $\pe=d_{\pie}$, $\quad
\pg=d_{\pig}$ and a discriminator $D_{\tildef, \pig}$ of the form of \citet{fu2017learning}:
\begin{equation}
\begin{aligned}
    \label{eq:ASQF_obj}
    \min_{\pig}\max_{\tildef}
    L(\tildef, \pig)
    \,,\quad L(\tildef, \pig)
    &\triangleq
    \E_{d_{\pie}}
    [\log D_{\tildef, \pig}(s,a)]
    +
    \E_{d_{\pig}}
    [\log (1-D_{\tildef, \pig}(s,a))],\\
    \text{with}\quad D_{\tildef, \pig} &= \frac{\exp \tildef(s, a)}{\exp \tildef(s, a) + \pig(a|s)},
\end{aligned}
\end{equation}
for which we present the following theorem.
\begin{thm}
\label{thm:ASQF_thm}
For any generator policy $\pig$, the optimal discriminator parameter for Equation~\ref{eq:ASQF_obj} is 
\begin{equation*}
   \tildef^* \triangleq \argmax_{\tildef} L(\tildef, \pig) = \log\left(\pie(a|s)\frac{d_{\pie}(s)}{d_{\pig}(s)}\right) \, \forall (s,a) \in \gS \times \gA 
\end{equation*}
Using $\tildef^*$, the optimal generator policy $\pig^*$ is
\begin{align*} \argmin_{\pig}\max_{\tildef}L(\tildef, \pig) = \argmin_{\pig}L(\tildef^*, \pig) = \pie(a|s) = 
 \frac{\exp \tildef^*(s,a)}{\sum_{a'}\exp \tildef^*(s,a')} \,\, \forall (s,a) \in \gS\times \gA.
\end{align*}
\end{thm}
\begin{proof}
The beginning of the proof closely follows the proof of Appendix~\ref{app:ASAF:proof_of_optim_D_gan}.
\begin{equation}
\begin{aligned}
    \argmax_{\tildef} L &(\tildef,\pig)= \\& \argmax_{\tildef} \sum_{s_i, a_i} d_{\pie}(s_i,a_i) \log D_{\tildef, \pig}(s_i,a_i) + d_{\pig}(s_i,a_i)\log(1-D_{\tildef, \pig}(s_i,a_i))
\end{aligned}
\end{equation}
We solve for each individual $(s_i,a_i)$ pair and note that $\tildef_{i} \mapsto D_i = \dfrac{\exp \tildef_{i}}{\exp \tildef_{i} + \pigi}$ is strictly monotonic on $\tildef_{i} \in \sR \, \forall \, \pigi \in ]0,1[$ so,
\begin{align}
    D^*_i=\argmax_{D_i} L_i
    \, \Leftrightarrow& \, \tildef^*_i =\argmax_{\tildef} L_i
\end{align}
Taking the derivative and setting it to 0, we find that
\begin{align}
 \left.\frac{d L_i}{d \tildef_i}\right|_{\tildef_i} = 0 \quad \Leftrightarrow \quad
    \tildef_i =\log\left(\pigi\frac{d_{\pie,i}}{d_{\pig,i}}\right)
\end{align}
We confirm that we have a global maximum with the second derivative test and the values at the border of the domain i.e. $\left.\dfrac{d^2 L_i}{d \tildef_i^2}\right|_{\tildef_i^*} < 0$ and $L_i$ goes to $-\infty$ for $\tildef_i \rightarrow +\infty$ and for $\tildef_i \rightarrow -\infty$.

It follows that
\begin{align}
    \tildef^*(s,a) &=\log\left(\pig(a|s)\frac{d_{\pie}(s,a)}{d_{\pig}(s,a)}\right) \quad \forall (s,a) \in \gS \times \gA\\
    \implies \tildef^*(s,a) &=\log\left(\cancel{\pig(a|s)}\frac{d_{\pie}(s)\pie(a|s)}{d_{\pig}(s)\cancel{\pig(a|s)}}\right) \quad \forall (s,a) \in \gS \times \gA\\
    \label{eq:ASQF_optimal_parameter}
    \implies \tildef^*(s,a) &=\log\left(\pie(a|s)\frac{d_{\pie}(s)}{d_{\pig}(s)}\right) \quad \forall (s,a) \in \gS \times \gA
\end{align}
This proves the first part of Theorem~\ref{thm:ASQF_thm}.

To prove the second part notice that
\begin{equation}
\begin{aligned}
    D_{\tildef^*, \pig}(s,a) &= \dfrac{\pie(a|s)\dfrac{d_{\pie}(s)}{d_{\pig}(s)}}{\pie(a|s)\dfrac{d_{\pie}(s)}{d_{\pig}(s)} + \pig(a|s)}\\
    &= \frac{\pie(a|s)d_{\pie}(s)}{\pie(a|s)d_{\pie}(s) + \pig(a|s)d_{\pig}(s)}\\
    &= \frac{d_{\pie}(s,a)}{d_{\pie}(s,a) + d_{\pig}(s,a)}
\end{aligned}
\end{equation}
This is equal to the optimal discriminator of the GAN objective Equation~\ref{eq:D_optim_GAN_proof} when $x=(s,a)$. For this discriminator we showed in Section~\ref{app:ASAF:proof_of_optim_D_gan} that the optimal generator $\pig^*$ is such that $d_{\pig^*}(s,a) = d_{\pie}(s,a)$ $\forall (s,a) \in \gS \times \gA$, which is satisfied for $\pig^*(a|s) = \pie(a|s)$ $\forall (s,a) \in \gS\times\gA$. 
Using the fact that
\begin{equation}
\label{eq:ASQF_partition_function}
    \sum_{a'}\exp\tildef^*(s,a') = \sum_{a'}\pie(a'|s)\frac{d_{\pie}(s)}{d_{\pig}(s)} = \frac{d_{\pie}(s)}{d_{\pig}(s)}\sum_{a'}\pie(a'|s) = \frac{d_{\pie}(s)}{d_{\pig}(s)}.
\end{equation}
we can combine Equation~\ref{eq:ASQF_optimal_parameter} and Equation~\ref{eq:ASQF_partition_function} to write the expert's policy $\pie$ as a function of the optimal discriminator parameter $\tildef^*$:
\begin{equation}
    \pie(a|s) = 
 \frac{\exp \tildef^*(s,a)}{\sum_{a'}\exp \tildef^*(s,a')} \,\, \forall (s,a) \in \gS\times \gA.
\end{equation}
This concludes the second part of the proof.
\end{proof}

\paragraph{Adversarial Soft-Q Fitting (ASQF) - practical algorithm}
In a nutshell, Theorem~\ref{thm:ASQF_thm} tells us that training the discriminator in Equation~\ref{eq:ASQF_obj} to distinguish between transitions from the expert and transitions from a generator policy can be seen as retrieving $\tildef^*$ which plays the role of the expert's soft Q-function (i.e. which matches Equation~\ref{eq:max_ent_policy} for $\tildef^*=\frac{1}{\alpha}Q_{\text{soft},E}^*$):
\begin{equation}
\label{eq:ASQF_pi}
    \pie(a|s) = \frac{\exp \tildef^*(s,a)}{\sum_{a'}\exp \tildef^*(s,a')} = \exp\left(\tildef^*(s,a) - \log\sum_{a'}\exp \tildef^*(s,a') \right),
\end{equation}
Therefore, by training the discriminator, one simultaneously retrieves the optimal generator policy.

There is one caveat though: the summation over actions that is required in Equation~\ref{eq:ASQF_pi} to go from $\tildef^*$ to the policy is intractable in continuous action spaces and would require an additional step such as a projection to a proper distribution (\citet{haarnoja2018soft} use a Gaussian) in order to draw samples and evaluate likelihoods. Updating in this way the generator policy to match a softmax over our learned state-action preferences ($\tildef^*$) becomes very similar in requirements and computational load to a policy optimization step, thus defeating the purpose of this work which is to get rid of the policy optimization step. For this reason we only consider ASQF for discrete action spaces. 

As explained in Section~\ref{sec:asaf_practical_algorithm}, in practice we optimize $D_{\tildef, \pig}$ only for a few steps before updating $\pig$ by normalizing $\exp \tildef(s,a)$ over the action dimension. See Algorithm~\ref{alg:ASQF_alg} for the pseudo-code. 

\begin{algorithm}[H]
\label{alg:ASQF_alg}
  \begin{algorithmic}
  \caption{Adversarial Soft-Q Fitting (ASQF)}\label{alg:asqf}
    \REQUIRE expert transitions $\mathcal{D}_E = \{(s_i, a_i)\}_{i=1}^{N_E}$
    \STATE Randomly initialize $\tildef$ and get $\pig$ from Equation~\ref{eq:ASQF_pi}
    \FOR{steps $m=0$ to $M$}
      \STATE Collect transitions $\mathcal{D}_G = \{(s_i, a_i)\}_{i=1}^{N_G}$ by executing $\pig$
      \STATE Train $D_{\tildef, \pig}$ using binary cross-entropy on minibatches of transitions from $\mathcal{D}_E$ and $\mathcal{D}_G$
      \STATE Get $\pig$ from Equation~\ref{eq:ASQF_pi}
  \ENDFOR
  \end{algorithmic}
\end{algorithm}

\paragraph{Experimental results}
Figure~\ref{fig:toy_asaf1_asqf} shows that ASQF performs well on small scale environments but struggles and eventually fails on more complicated environments. Specifically, it seems that ASQF does not scale well with the observation space size. Indeed mountaincar, cartpole, lunarlander and pommerman have respectively an observation space dimensionality of 2, 4, 8 and 960. This may be due to the fact that the partition function Equation~\ref{eq:ASQF_partition_function} becomes more difficult to learn. Indeed, for each state, several transitions with different actions are required in order to learn it. Poorly approximating this partition function could lead to assigning too low a probability to expert-like actions and eventually failing to behave appropriately. ASAF on the other hand explicitly learns the probability of an action given the state -- in other word it explicitly learns the partition function -- and is therefore immune to that problem. 
\begin{figure}[h!]
    \centering
    \includegraphics[width=.55\textwidth]{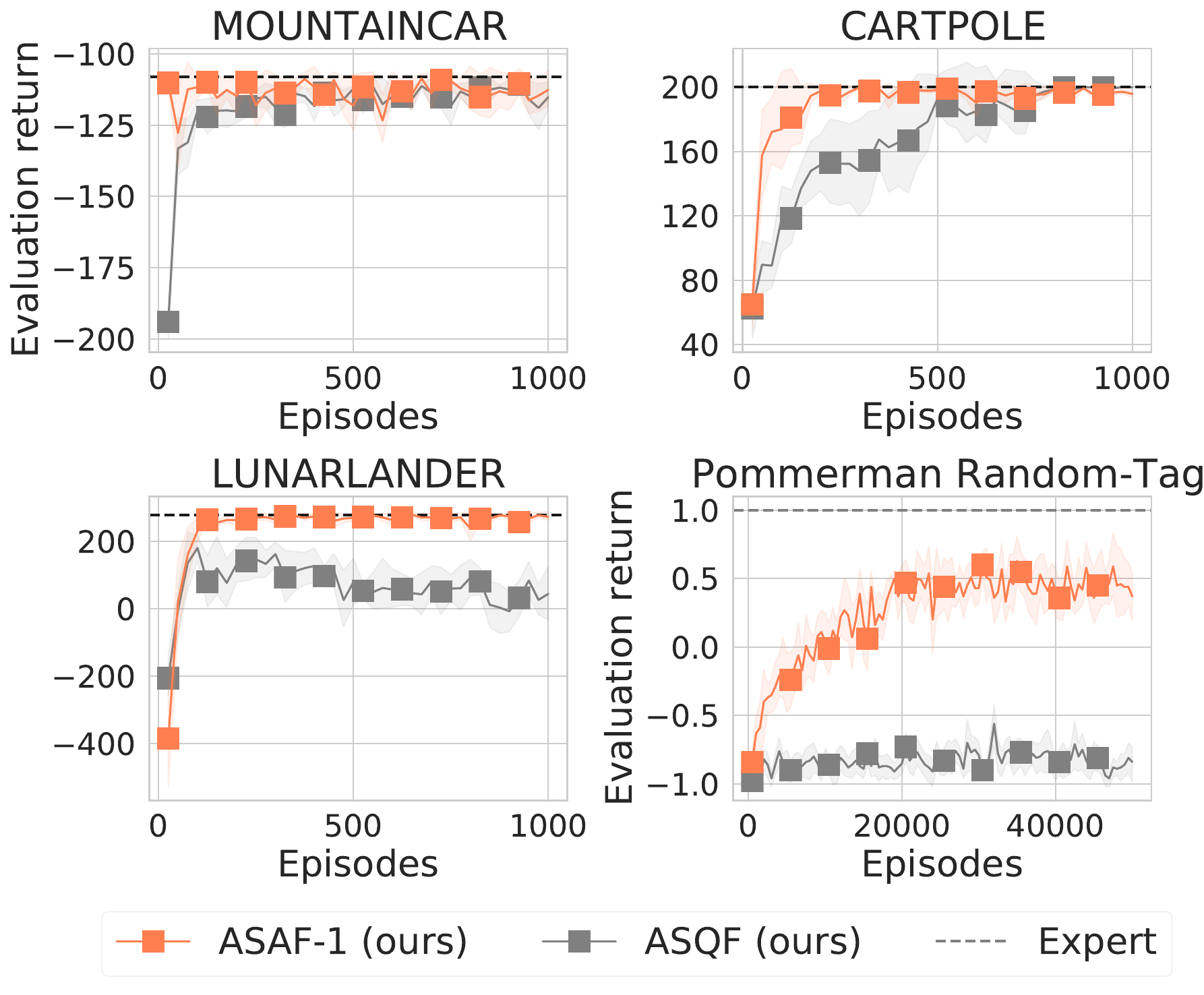}
    \caption[Comparison between ASAF-1 and ASQF]{Comparison between ASAF-1 and ASQF, our two transition-wise methods, on environments with increasing observation space dimensionality}
    \label{fig:toy_asaf1_asqf}
\end{figure}

\clearpage
\tocless{\section{Additional Experiments}}{\label{app:ASAF:additional_experiments}}

\tocless{\subsection{GAIL - Importance of Gradient Penalty}}{}

\begin{figure}[h!]
    \centering
    \includegraphics[width=1.\textwidth]{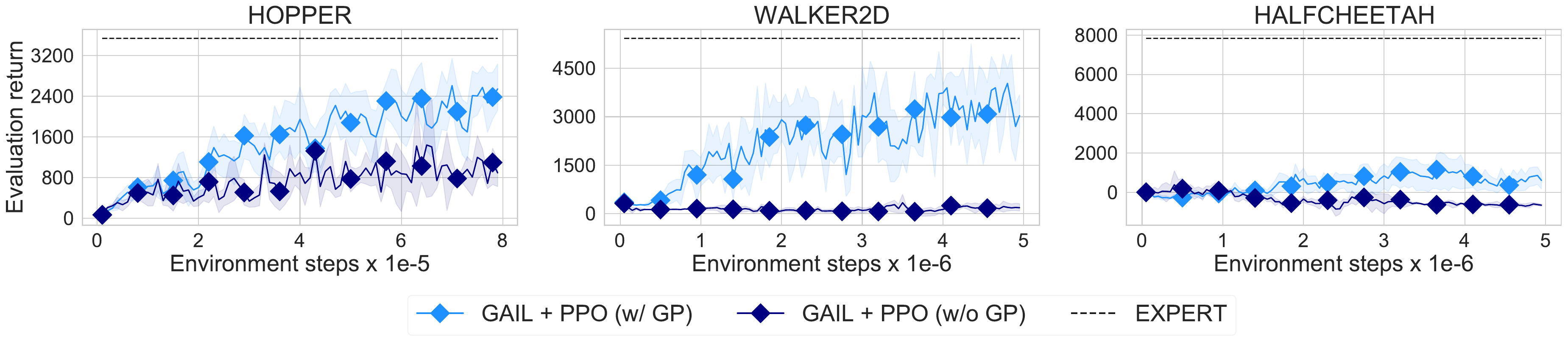}
    \caption[Effect of gradient penalty on GAIL algorithm]{Comparison between original GAIL \cite{ho2016generative} and GAIL with gradient penalty (GP) \cite{gulrajani2017improved,kostrikov2018discriminatoractorcritic}}
    \label{fig:gail_gradient_penalty}
\end{figure}

\clearpage
\tocless{\subsection{Mimicking the expert}}{}
\noindent
\begin{minipage}[r]{0.58\textwidth}
To ensure that our method actually mimics the expert and doesn't just learn a policy that collects high rewards when trained with expert demonstrations, we ran ASAF-1 on the Ant-v2 MuJoCo environment using various sets of 25 demonstrations. These demonstrations were generated from a Soft Actor-Critic agent at various levels of performance during its training.  Since at low-levels of performance the variance of episode's return is high, we filtered collected demonstrations to lie in the targeted range of performance (e.g. return in [800, 1200] for the 1K set).
\end{minipage}\hfill
\begin{minipage}[l]{0.39\textwidth}
    \includegraphics[width=\textwidth]{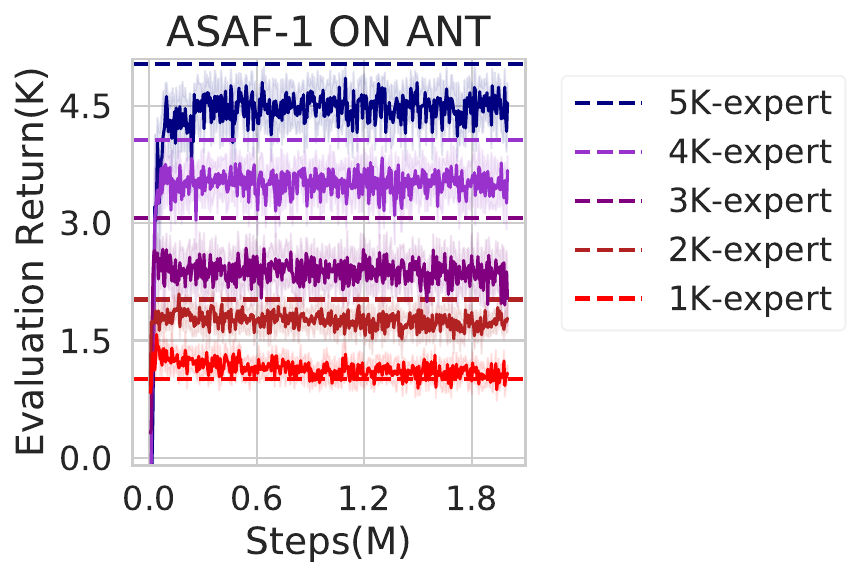}
    \captionof{figure}[ASAF-1 on Ant-v2 for different expert levels of performance]{ASAF-1 on Ant-v2. Colors are 1K, 2K, 3K, 4K, 5K expert's performance.}
    \label{fig:gradual_expert}
\end{minipage}

Results in Figure~\ref{fig:gradual_expert} show that our algorithm succeeds at learning a policy that closely emulates various demonstrators (even when non-optimal).

\tocless{\subsection{Wall Clock Time}}{}
We report training times in Figure~\ref{fig:wall_clock_times} and observe that ASAF-1 is always fastest to learn. Note however that reports of performance w.r.t wall-clock time should always be taken with a grain of salt as they are greatly influenced by hyperparameters and implementation details.
\begin{figure}[h!]
    \centering
    \includegraphics[width=1.\textwidth]{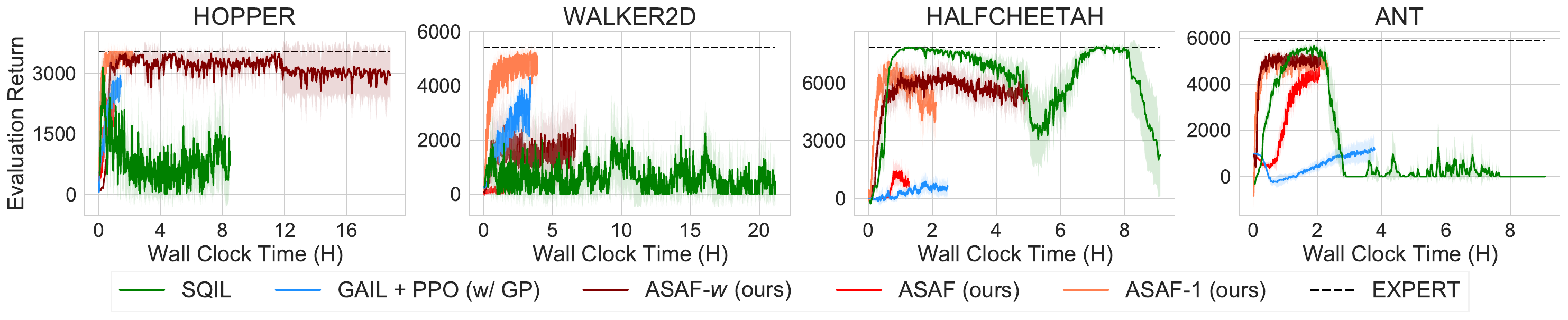}
    \caption[Training times on MuJoCo tasks]{Training times on MuJoCo tasks for 25 expert demonstrations.}
    \label{fig:wall_clock_times}
\end{figure}

\clearpage
\tocless{\section{Hyperparameter tuning and best configurations}}{\label{app:ASAF:hyperparameters}}

\tocless{\subsection{Classic Control}}{}

For this first set of experiments, we use the fixed hyperparameters presented in Table~\ref{table:fixed_hyperparams_classic_control}.

\begin{table}[h!]
\centering
\begin{sc}
\small
\caption{Fixed Hyperparameters for classic control tasks}
\begin{tabular}{lcc}
\label{table:fixed_hyperparams_classic_control}
\textbf{RL component}
\\ \hline
Hyperparameter                                         &Discrete Control       &Continuous Control
\\ \hline
\\
\hspace{5mm} \textbf{SAC} \\
\hspace{5mm} Batch size (in transitions)                &256                    &256 \\
\hspace{5mm} Replay Buffer length $|\mathcal{B}|$       &$10^{6}$               &$10^{6}$ \\
\hspace{5mm} Warmup (in transitions)       &1280                      &10240 \\
\hspace{5mm} Initial entropy weight $\alpha$            &0.4                    &0.4 \\
\hspace{5mm} Gradient norm clipping threshold           &0.2                    &1 \\
\hspace{5mm} Transitions between update                 &40                     &1 \\
\hspace{5mm} Target network weight $\tau$               &0.01                   &0.01 \\
\\
\hspace{5mm} \textbf{PPO} \\
\hspace{5mm} Batch size (in transitions)                &256                    &256 \\
\hspace{5mm} GAE parameter $\lambda$                    &0.95                   &0.95 \\
\hspace{5mm} Transitions between update                 &-                      &2000 \\
\hspace{5mm} Episodes between updates                   &10                     &- \\
\hspace{5mm} Epochs per update                          &10                     &10 \\
\hspace{5mm} Update clipping parameter                  &0.2                    &0.2 \\
\\
\textbf{Reward Learning component}
\\ \hline
Hyperparameter                                         &Discrete Control       &Continuous Control
\\ \hline
\\
\hspace{5mm} \textbf{AIRL, GAIL, ASAF-1} \\
\hspace{5mm} Batch size (in transitions)                &256                    &256 \\
\hspace{5mm} Transitions between update                 &-                      &2000 \\
\hspace{5mm} Episodes between updates                   &10                     &- \\
\hspace{5mm} Epochs per update                          &50                     &50 \\
\hspace{5mm} Gradient value clipping threshold         &-                      &1 \\
\hspace{5mm}(ASAF-1)&&\\
\\
\hspace{5mm} \textbf{ASAF, ASAF-\textit{w}} \\
\hspace{5mm} Batch size (in trajectories)               &10                     &10 \\
\hspace{5mm} Episodes between updates                   &10                     &20 \\
\hspace{5mm} Epochs per update                          &50                     &50 \\
\hspace{5mm} Window size $w$                            &(searched)             &200 \\
\hspace{5mm} Gradient value clipping threshold           &-                      &1 \\
\end{tabular}
\end{sc}
\end{table}

For the most sensitive hyperparameters, the learning rates for the reinforcement learning and discriminator updates ($\epsilon_{\text{RL}}$ and $\epsilon_{\text{D}}$), we perform a random search over 50 configurations and 3 seeds each (for each algorithm on each task) for 500 episodes. We consider logarithmic ranges, i.e. $\epsilon = 10^{u}$ with $u \sim Uniform(-6, -1)$ for $\epsilon_{\text{D}}$ and $u \sim Uniform(-4, -1)$ for $\epsilon_{\text{RL}}$. We also include in this search the critic learning rate coefficient $\kappa$ for PPO also sampled according to a logarithmic scale with $u \sim Uniform(-2, 2)$ so that the effective learning rate for PPO's critic network is $\kappa \cdot \epsilon_{\text{RL}}$. For discrete action tasks, the window-size \textit{w} for ASAF-\textit{w} is sampled uniformly within $\{32, 64, 128\}$. The best configuration for each algorithm is presented in Tables~\ref{table:hyperparams_cartpole}~to~\ref{table:hyperparams_lunarlander_c}. Figure~\ref{fig:results_toy} uses these configurations retrained on 10 seeds and twice as long.

Finally for all neural networks (policies and discriminators) for these experiments we use a fully-connected MLP with two hidden layers and ReLU activation (except for the last layer). We used hidden sizes of 64 for the discrete tasks and of 256 for the continuous tasks.

\begin{table}[h]
\centering
\begin{sc}
\caption{Best found hyperparameters for Cartpole}
\resizebox{\textwidth}{!}{%
\begin{tabular}{lcccccc}
\label{table:hyperparams_cartpole}
Hyperparameter                                     &ASAF           &ASAF-\textit{w}           &ASAF-$1$           &SQIL           &AIRL + PPO         &GAIL + PPO    
\\ \hline
Discriminator update lr $\epsilon_{\text{D}}$       &0.028          &0.039              &0.00046            &-              &2.5*$10^{-6}$        &0.00036 \\
RL update lr $\epsilon_{\text{RL}}$                 &-              &-                  &-                  &0.0067         &0.0052             &0.012 \\
Critic lr coefficient $\kappa$                      &-              &-                  &-                  &-              &0.25               &0.29 \\
window size \textit{w}                                     &-              &64                 &1                  &-              &-                  &- \\
window stride                                     &-              &64                 &1                  &-              &-                  &- \\
\end{tabular}}
\end{sc}
\end{table}

\begin{table}[h!]
\centering
\begin{sc}
\caption{Best found hyperparameters for Mountaincar}
\resizebox{\textwidth}{!}{%
\begin{tabular}{lcccccc}
\label{table:hyperparams_mountaincar}
Hyperparameter                                     &ASAF           &ASAF-\textit{w}           &ASAF-$1$           &SQIL           &AIRL + PPO         &GAIL + PPO    
\\ \hline
Discriminator update lr $\epsilon_{\text{D}}$       &0.059          &0.059              &0.0088             &-              &0.0042             &0.00016 \\
RL update lr $\epsilon_{\text{RL}}$                 &-              &-                  &-                  &0.062          &0.016              &0.0022 \\
Critic lr coefficient $\kappa$                      &-              &-                  &-                  &-              &4.6                &0.018 \\
window size \textit{w}                                     &-              &32                 &1                  &-              &-                  &- \\
window stride                                     &-              &32                 &1                  &-              &-                  &- \\
\end{tabular}}
\end{sc}
\end{table}

\begin{table}[h!]
\centering
\begin{sc}
\caption{Best found hyperparameters for Lunarlander}
\resizebox{\textwidth}{!}{%
\begin{tabular}{lcccccc}
\label{table:hyperparams_lunarlander}
Hyperparameter                                     &ASAF           &ASAF-\textit{w}           &ASAF-$1$           &SQIL           &AIRL + PPO         &GAIL + PPO    
\\ \hline
Discriminator update lr $\epsilon_{\text{D}}$       &0.0055         &0.0015             &0.00045            &-              &0.0002             &0.00019 \\
RL update lr $\epsilon_{\text{RL}}$                 &-              &-                  &-                  &0.0036         &0.0012             &0.0016 \\
Critic lr coefficient $\kappa$                      &-              &-                  &-                  &-              &0.48               &8.5 \\
window size \textit{w}                                     &-              &32                 &1                  &-              &-                  &- \\
window stride                                  &-              &32                 &1                  &-              &-                  &- \\
\end{tabular}}
\end{sc}
\end{table}

\begin{table}[ht!]
\centering
\begin{sc}
\caption{Best found hyperparameters for Pendulum}
\resizebox{\textwidth}{!}{%
\begin{tabular}{lcccccc}
\label{table:hyperparams_pendulum}
Hyperparameter                                     &ASAF           &ASAF-\textit{w}       &ASAF-$1$           &SQIL           &AIRL + PPO         &GAIL + PPO    
\\ \hline
Discriminator update lr $\epsilon_{\text{D}}$       &0.00069        &0.00082        &0.00046            &-              &4.3*$10^{-6}$      &1.6*$10^{-5}$ \\
RL update lr $\epsilon_{\text{RL}}$                 &-              &-              &-                  &0.0001         &0.00038            &0.00028 \\
Critic lr coefficient $\kappa$                      &-              &-              &-                  &-              &0.028              &84 \\
window size \textit{w}                              &-              &200                 &1                  &-              &-                  &- \\
window stride                                       &-              &200                 &1                  &-              &-                  &- \\
\end{tabular}}
\end{sc}
\end{table}

\begin{table}[ht!]
\centering
\begin{sc}
\caption{Best found hyperparameters for Mountaincar-c}
\resizebox{\textwidth}{!}{%
\begin{tabular}{lcccccc}
\label{table:hyperparams_mountaincar_c}
Hyperparameter                                     &ASAF           &ASAF-$w$           &ASAF-$1$           &SQIL           &AIRL + PPO         &GAIL + PPO    
\\ \hline
Discriminator update lr $\epsilon_{\text{D}}$       &0.00021        &3.8*$10^{-5}$      &6.2*$10^{-6}$      &-              &1.7*$10^{-5}$      &1.5*$10^{-5}$ \\
RL update lr $\epsilon_{\text{RL}}$                 &-              &-                  &-                  &0.0079         &0.0012             &0.0052 \\
Critic lr coefficient $\kappa$                      &-              &-                  &-                  &-              &10                 &12 \\
window size \textit{w}                              &-              &200                 &1                  &-              &-                  &- \\
window stride                                       &-              &200                 &1                  &-              &-                  &- \\
\end{tabular}}
\end{sc}
\end{table}

\begin{table}[ht!]
\centering
\begin{sc}
\caption{Best found hyperparameters for Lunarlander-c}
\resizebox{\textwidth}{!}{%
\begin{tabular}{lcccccc}
\label{table:hyperparams_lunarlander_c}
Hyperparameter                                     &ASAF           &ASAF-$w$           &ASAF-$1$           &SQIL           &AIRL + PPO         &GAIL + PPO    
\\ \hline
Discriminator update lr $\epsilon_{\text{D}}$       &0.0051         &0.0022             &0.0003             &-              &0.0045             &0.00014 \\
RL update lr $\epsilon_{\text{RL}}$                 &-              &-                  &-                  &0.0027         &0.00031            &0.00049 \\
Critic lr coefficient $\kappa$                      &-              &-                  &-                  &-              &14                 &0.01 \\
window size \textit{w}                              &-              &200                 &-                  &-              &-                  &- \\
window stride                                       &-              &200                 &-                  &-              &-                  &- \\
\end{tabular}}
\end{sc}
\end{table}

\tocless{\subsection{MuJoCo}}{}

For MuJoCo experiments (Hopper-v2, Walker2d-v2, HalfCheetah-v2, Ant-v2), the fixed hyperparameters are presented in Table~\ref{table:fixed_hyperparams_MuJoCo}. For all exeperiments, fully-connected MLPs with two hidden layers and ReLU activation (except for the last layer) were used, where the number of hidden units is equal to 256. 
\begin{table}[ht!]
\centering
\begin{sc}
\caption{Fixed hyperparameters for MuJoCo environments.}
\scalebox{0.8}{
\begin{tabular}{lcccc}
\label{table:fixed_hyperparams_MuJoCo}
\textbf{RL component}
\\ \hline
Hyperparameter                                         &Hopper, Walker2d, HalfCheetah, Ant
\\ \hline
\\
\hspace{5mm} \textbf{PPO (for GAIL)} \\
\hspace{5mm} GAE parameter $\lambda$                    &{0.98} \\
\hspace{5mm} Transitions between updates                &{2000} \\
\hspace{5mm} Epochs per update                          &{5} \\
\hspace{5mm} Update clipping parameter                  &{0.2} \\
\hspace{5mm} Critic lr coefficient $\kappa$             &{0.25} \\
\hspace{5mm} Discount factor $\gamma$                          &{0.99} \\
\textbf{Reward Learning component}
\\ \hline
Hyperparameter                                         &Hopper, Walker2d, HalfCheetah, Ant
\\ \hline
\\
\hspace{5mm} \textbf{GAIL} \\
\hspace{5mm} Transitions between updates                &{2000} \\
\\
\hspace{5mm} \textbf{ASAF}\\
\hspace{5mm} Episodes between updates                   &{25} \\
\\
\hspace{5mm} \textbf{ASAF-1 and ASAF-\textit{w}}\\
\hspace{5mm} Transitions between updates                &{2000} \\
\end{tabular}
}
\end{sc}
\end{table}

For SQIL we used SAC with the same hyperparameters that were used to generate the expert demonstrations. For ASAF, ASAF-1 and ASAF-\textit{w}, we set the learning rate for the discriminator at 0.001 and ran random searches over 25 randomly sampled configurations and 2 seeds for each task to select the other hyperparameters for the discriminator training. These hyperparameters included the discriminator batch size sampled from a uniform distribution over $\{10, 20, 30\}$ for ASAF and ASAF-\textit{w} (in trajectories) and over $\{100, 500, 1000, 2000\}$ for ASAF-1 (in transitions), the number of epochs per update sampled from a uniform distribution over $\{10, 20, 50\}$, the gradient norm clipping threshold sampled form a uniform distribution over $\{1, 10\}$, the window-size (for ASAF-\textit{w}) sampled from a uniform distribution over $\{100, 200, 500, 1000\}$ and the window stride (for ASAF-\textit{w}) sampled from a uniform distribution over $\{1, 50, w\}$. For GAIL, we obtained poor results using the original hyperparameters from \cite{ho2016generative} for a number of tasks so we ran random searches over 100 randomly sampled configurations for each task and 2 seeds to select for the following hyperparameters: the log learning rate of the RL update and the discriminator update separately sampled from uniform distributions over $[-7, -1]$, the gradient norm clipping for the RL update and the discriminator update separately sampled from uniform distributions over $\{None, 1, 10\}$, the number of epochs per update sampled from a uniform distribution over $\{5, 10, 30, 50\}$, the gradient penalty coefficient sampled from a uniform distribution over $\{1, 10\}$ and the batch size for the RL update and discriminator update separately sampled from uniform distributions over $\{100, 200, 500, 1000, 2000\}$. 

\begin{table}[ht!]
\centering
\begin{sc}
\caption{Best found hyperparameters for the Hopper-v2 environment}
\resizebox{\textwidth}{!}{%
\begin{tabular}{lccccccc}
\label{table:hopper}
Hyperparameter                                 &ASAF           &ASAF-\textit{w}           &ASAF-$1$           &SQIL           &GAIL + PPO  
\\ \hline
RL batch size (in transitions)       &-              &-              &-         &256     &200 \\
Discriminator batch size (in transitions)       &-              &-              &100         &-     &2000 \\
Discriminator batch size (in trajectories)      &10              &10              &-         &-     &- \\
Gradient clipping (RL update)      &-              &-              &-         &-     &1. \\
Gradient clipping (discriminator update)  &10.              &10.              &1.         &-     &1. \\
Epochs per update &50              &50              &30         &-     &5 \\
Gradient penalty (discriminator update) &-              &-              &-         &-     &1. \\
RL update lr $\epsilon_{\text{RL}}$       &-     &-         &-         &$3*10^{-4}$     &$1.8*10^{-5}$        \\
Discriminator update lr $\epsilon_{\text{D}}$        &0.001     &0.001         &0.001         &-     &0.011         \\
window size \textit{w}       &-     &200         &1        &-     &-        \\
window stride       &-     &1         &1        &-     &-        \\
\end{tabular}}
\end{sc}
\end{table}

\begin{table}[ht!]
\centering
\begin{sc}
\caption{Best found hyperparameters for the HalfCheetah-v2 environment}
\resizebox{\textwidth}{!}{%
\begin{tabular}{lccccccc}
\label{table:halfcheetah}
Hyperparameter                                 &ASAF           &ASAF-\textit{w}           &ASAF-$1$           &SQIL           &GAIL + PPO  
\\ \hline
RL batch size (in transitions)       &-              &-              &-         & 256    &1000 \\
Discriminator batch size (in transitions)       &-              &-              &100         &-     &100 \\
Discriminator batch size (in trajectories)      &10              &10              &-         &-     &- \\
Gradient clipping (RL update)      &-              &-              &-         &-     &- \\
Gradient clipping (discriminator update)  &10.              &1              &1         &-     &10 \\
Epochs per update &50              &10              &30         &-     &30 \\
Gradient penalty (discriminator update) &-              &-              &-         &-     &1. \\
RL update lr $\epsilon_{\text{RL}}$       &-     &-         &-         &$3*10^{-4}$     &0.0006        \\
Discriminator update lr $\epsilon_{\text{D}}$        &0.001     &0.001         &0.001         &-     &0.023         \\
window size \textit{w}       &-     &200         &1         &-     &-        \\
window stride       &-     &1         &1        &-     &-        \\
\end{tabular}}
\end{sc}
\end{table}

\begin{table}[ht!]
\centering
\begin{sc}
\caption{Best found hyperparameters for the Walker2d-v2 environment}
\resizebox{\textwidth}{!}{%
\begin{tabular}{lccccccc}
\label{table:walker2D}
Hyperparameter                                 &ASAF           &ASAF-\textit{w}           &ASAF-$1$           &SQIL           &GAIL + PPO  
\\ \hline
RL batch size (in transitions)       &-              &-              &-         & 256    &200 \\
Discriminator batch size (in transitions)       &-              &-              &500         &-     &2000 \\
Discriminator batch size (in trajectories)      &20              &20              &-         &-     &- \\
Gradient clipping (RL update)      &-              &-              &-         &-     &- \\
Gradient clipping (discriminator update)  &10.              &1.              &10.         &-     &- \\
Epochs per update &30              &10              &50         &-     &30 \\
Gradient penalty (discriminator update) &-              &-              &-         &-     &1. \\
RL update lr $\epsilon_{\text{RL}}$       &-     &-         &-         &$3*10^{-4}$     &0.00039        \\
Discriminator update lr $\epsilon_{\text{D}}$        &0.001     &0.001         &0.001         &-     &0.00066         \\
window size \textit{w}       &-     &100         &1         &-     &-        \\
window stride       &-     &1         &1        &-     &-        \\
\end{tabular}}
\end{sc}
\end{table}

\begin{table}[ht!]
\centering
\begin{sc}
\caption{Best found hyperparameters for the Ant-v2 environment}
\resizebox{\textwidth}{!}{%
\begin{tabular}{lccccccc}
\label{table:ant}
Hyperparameter                                 &ASAF           &ASAF-\textit{w}           &ASAF-$1$           &SQIL           &GAIL + PPO  
\\ \hline
RL batch size (in transitions)       &-              &-              &-         &256     &500 \\
Discriminator batch size (in transitions)       &-              &-              &100         &-     &100 \\
Discriminator batch size (in trajectories)      &20              &20              &-         &-     &- \\
Gradient clipping (RL update)      &-              &-              &-         &-     &- \\
Gradient clipping (discriminator update)  &10.              &1.              &1.         &-     &10. \\
Epochs per update &50              &50              &10         &-     &50 \\
Gradient penalty (discriminator update) &-              &-              &-         &-     &10 \\
RL update lr $\epsilon_{\text{RL}}$       &-     &-         &-         &$3*10^{-4}$     &$8.5*10^{-5}$        \\
Discriminator update lr $\epsilon_{\text{D}}$        &0.001     &0.001         &0.001         &-     &0.0016         \\
window size \textit{w}       &-     &200         &1         &-     &-        \\
window stride       &-     &50         &1        &-     &-        \\
\end{tabular}}
\end{sc}
\end{table}

\clearpage
\tocless{\subsection{Pommerman}}{}

For this set of experiments, we use a number of fixed hyperparameters for all algorithms either inspired from their original papers for the baselines or selected through preliminary searches. These fixed hyperparameters are presented in Table~\ref{table:fixed_hyperparams_pommerman_random_tag}.

\begin{table}[h!]
\centering
\begin{sc}
\caption[Fixed Hyperparameters for Pommerman environment]{Fixed Hyperparameters for Pommerman Random-Tag environment.}
\scalebox{0.8}{
\begin{tabular}{lcc}
\label{table:fixed_hyperparams_pommerman_random_tag}
\textbf{RL component}
\\ \hline
Hyperparameter                                         &Pommerman Random-Tag
\\ \hline
\\
\hspace{5mm} \textbf{SAC} \\
\hspace{5mm} Batch size (in transitions)                &256 \\
\hspace{5mm} Replay Buffer length $|\mathcal{B}|$       &$10^{5}$ \\
\hspace{5mm} Warmup (in transitions)       &1280 \\
\hspace{5mm} Initial entropy weight $\alpha$            &0.4 \\
\hspace{5mm} Gradient norm clipping threshold           &0.2 \\
\hspace{5mm} Transitions between update                 &10 \\
\hspace{5mm} Target network weight $\tau$               &0.05 \\
\\
\hspace{5mm} \textbf{PPO} \\
\hspace{5mm} Batch size (in transitions)                &256 \\
\hspace{5mm} GAE parameter $\lambda$                    &0.95 \\
\hspace{5mm} Episodes between updates                   &10 \\
\hspace{5mm} Epochs per update                          &10 \\
\hspace{5mm} Update clipping parameter                  &0.2 \\
\hspace{5mm} Critic lr coefficient $\kappa$             &0.5 \\
\\
\textbf{Reward Learning component}
\\ \hline
Hyperparameter                                         &Pommerman Random-Tag
\\ \hline
\\
\hspace{5mm} \textbf{AIRL, GAIL, ASAF-1} \\
\hspace{5mm} Batch size (in transitions)                &256 \\
\hspace{5mm} Episodes between updates                   &10 \\
\hspace{5mm} Epochs per update                          &10 \\
\\
\hspace{5mm} \textbf{ASAF, ASAF-\textit{w}} \\
\hspace{5mm} Batch size (in trajectories)               &5 \\
\hspace{5mm} Episodes between updates                   &10 \\
\hspace{5mm} Epochs per update                          &10 \\
\end{tabular}
}
\end{sc}
\end{table}

For the most sensitive hyperparameters, the learning rates for the reinforcement learning and discriminator updates ($\epsilon_{\text{RL}}$ and $\epsilon_{\text{D}}$), we perform a random search over 25 configurations and 2 seeds each for all algorithms. We consider logarithmic ranges, i.e. $\epsilon = 10^{u}$ with $u \sim Uniform(-7, -3)$ for $\epsilon_{\text{D}}$ and $u \sim Uniform(-4, -1)$ for $\epsilon_{\text{RL}}$. We also include in this search the window-size \textit{w} for ASAF-\textit{w}, sampled uniformly within $\{32, 64, 128\}$. The best configuration for each algorithm is presented in Table~\ref{table:hyperparams_pommerman}. Figure~\ref{fig:pommerman_results_randomTag} uses these configurations retrained on 10 seeds.

\begin{table}[ht!]
\centering
\begin{sc}
\caption[Best found hyperparameters for the Pommerman environment]{Best found hyperparameters for the Pommerman Random-Tag environment}
\resizebox{\textwidth}{!}{%
\begin{tabular}{lccccccc}
\label{table:hyperparams_pommerman}
Hyperparameter         &ASAF           &ASAF-\textit{w}           &ASAF-$1$           &SQIL           &AIRL + PPO         &GAIL + PPO         &BC   
\\ \hline
Discriminator update lr $\epsilon_{\text{D}}$        &0.0007     &0.0002         &0.0001         &-     &3.1*$10^{-7}$         &9.3*$10^{-7}$         &0.00022\\
RL update lr $\epsilon_{\text{RL}}$       &-     &-         &-         &0.00019     &0.00017         &0.00015         &-\\
window size \textit{w}       &-     &32         &1         &-     &-         &-         &-\\
window stride       &-     &32         &1         &-     &-         &-         &-\\
\end{tabular}}
\end{sc}
\end{table}

Finally for all neural networks (policies and discriminators) we use the same architecture. Specifically, we first process the feature maps (see Section~\ref{app:ASAF:env_pommerman}) using a 3-layers convolutional network with number of hidden feature maps of 16, 32 and 64 respectivelly. Each one of these layers use a kernel size of 3x3 with stride of 1, no padding and a ReLU activation. This module ends with a fully connected layer of hidden size 64 followed by a ReLU activation. The output vector is then concatenated to the unprocessed additional information vector (see Section~\ref{app:ASAF:env_pommerman}) and passed through a final MLP with two hidden layers of size 64 and ReLU activations (except for the last layer).

\tocless{\section{Environments and expert data}}{\label{app:ASAF:environments}}

\tocless{\subsection{Classic Control}}{}
The environments used here are the reference Gym implementations for classic control\footnote{See: \url{http://gym.openai.com/envs/\#classic_control}} and for Box2D\footnote{See: \url{http://gym.openai.com/envs/\#box2d}}.
We generated the expert trajectories for mountaincar (both discrete and continuous version) by hand using keyboard inputs. For the other tasks, we trained our SAC implementation to get experts on the discrete action tasks and our PPO implementation to get experts on the continuous action tasks.

\tocless{\subsection{MuJoCo}}{}
The experts were trained using our implementation of SAC \cite{haarnoja2018soft} the state-of-the-art RL algorithm in MuJoCo continuous control tasks.
Our implementation basically refactors the SAC implementation from Rlpyt\footnote{See: 
\url{https://github.com/astooke/rlpyt}}. We trained SAC agent for 1,000,000 steps for Hopper-v2 and 3,000,000 steps for Walker2d-v2 and HalfCheetah-v2 and Ant-v2. 
We used the default hyperparameters from Rlpyt.

\tocless{\subsection{Pommerman}}{\label{app:ASAF:env_pommerman}}
The observation space that we use for Pommerman domain \cite{resnick2018pommerman} is composed of a set of 15 feature maps as well as an additional information vector. The feature maps whose dimensions are given by the size of the board (8x8 in the case of 1vs1 tasks) are one-hot across the third dimension and represent which element is present at which location. Specifically, these feature maps identify whether a given location is the current player, an ally, an ennemy, a passage, a wall, a wood, a bomb, a flame, fog, a power-up. Other feature maps contain integers indicating bomb blast stength, bomb life, bomb moving direction and flame life for each location. Finally, the additional information vecor contains the time-step, number of ammunition, whether the player can kick and blast strengh for the current player. The agent has an action space composed of six actions: do-nothing, up, down, left, right and lay bomb.

For these experiments, we generate the expert demonstrations using Agent47Agent, the open-source champion algorithm of the FFA 2018 competition \cite{zhou2018hybrid} which uses hardcoded heuristics and Monte-Carlo Tree-Search\footnote{See: \url{https://github.com/YichenGong/Agent47Agent/tree/master/pommerman}}. While this agent occasionally eliminates itself during a match, we only select trajectories leading to a win as being expert demonstrations.

\tocless{\subsection{Demonstrations summary}}{}
Table~\ref{table:demo_summary} provides a summary of the expert data used.

\begin{table}[h]
\centering
\begin{sc}
\caption{Expert demonstrations used for Imitation Learning}
\resizebox{\textwidth}{!}{%
\scalebox{0.8}{
\begin{tabular}{lcc}
\label{table:demo_summary}
Task-Name         &Expert mean return            &Number of expert trajectories  
\\ \hline
Cartpole &200.0     &10 \\
Mountaincar & -108.0      &10\\
Lunarlander & 277.5      &10\\
Pendulum & -158.6     &10\\
Mountaincar-c & 93.92    &10\\
Lunarlander-c & 266.1    &10\\
Hopper &3537        &25\\
Walker2D &5434      &25\\
Halfcheetah &7841       &25\\
Ant &5776       &25\\
Pommerman random-tag        &1 &300, 150, 75, 15, 5, 1\\
\end{tabular} }
}
\end{sc}
\end{table}

\Annexe{Supplementary Material for Chapter~\ref{chap:article2_cmaddpg}}

\tocless{\section{Additional details on Motivation section}}{\label{ap:toy}}

We trained each agent $i$ with online Q-learning \cite{watkins1992q} on the $Q^i(a^i,s)$ table using Boltzmann exploration \cite{kaelbling1996reinforcement}. The Boltzmann temperature is fixed to 1 and we set the learning rate to 0.05 and the discount factor to 0.99. After each learning episode we evaluate the current greedy policy on 10 episodes and report the mean return. Curves are averaged over 20 seeds and the shaded area represents the standard error.

\tocless{\section{Tasks descriptions}}{\label{app:CMADDPG:environments}}

\textbf{SPREAD} (Figure~\ref{fig:CMADDPG:envs}a): 
In this environment, there are 3 agents (small orange circles) and 3 landmarks (bigger gray circles). At every timestep, agents receive a team-reward $r_t=n-c$ where $n$ is the number of landmarks occupied by at least one agent and $c$ the number of collisions occurring at that timestep. To maximize their return, agents must therefore spread out and cover all landmarks. Initial agents' and landmarks' positions are random. Termination is triggered when the maximum number of timesteps is reached.

\textbf{BOUNCE} (Figure~\ref{fig:CMADDPG:envs}b): 
In this environment, two agents (small orange circles) are linked together with a spring that pulls them toward each other when stretched above its relaxation length. At episode's mid-time a ball (smaller black circle) falls from the top of the environment. Agents must position correctly so as to have the ball bounce on the spring towards the target (bigger beige circle), which turns yellow if the ball's bouncing trajectory passes through it. They receive a team-reward of $r_t = 0.1$ if the ball reflects towards the side walls, $r_t = 0.2$ if the ball reflects towards the top of the environment, and $r_t = 10$ if the ball reflects towards the target. At initialisation, the target's and ball's vertical position is fixed, their horizontal positions are random. Agents' initial positions are also random. Termination is triggered when the ball is bounced by the agents or when the maximum number of timesteps is reached. 

\textbf{COMPROMISE} (Figure~\ref{fig:CMADDPG:envs}c):
In this environment, two agents (small orange circles) are linked together with a spring that pulls them toward each other when stretched above its relaxation length. They both have a distinct assigned landmark (light gray circle for light orange agent, dark gray circle for dark orange agent), and receive a reward of $r_t = 10$ when they reach it.  Once a landmark is reached by its corresponding agent, the landmark is randomly relocated in the environment. Initial positions of agents and landmark are random. Termination is triggered when the maximum number of timesteps is reached.

\textbf{CHASE} (Figure~\ref{fig:CMADDPG:envs}d): 
In this environment, two predators (orange circles) are chasing a prey (turquoise circle). The prey moves with respect to a scripted policy consisting of repulsion forces from the walls and predators. At each timestep, the learning agents (predators) receive a team-reward of $r_t = n$ where $n$ is the number of predators touching the prey. The prey has a greater max speed and acceleration than the predators. Therefore, to maximize their return, the two agents must coordinate in order to squeeze the prey into a corner or a wall and effectively trap it there. Termination is triggered when the maximum number of time steps is reached.

\tocless{\section{Training details}}{\label{app:CMADDPG:training_details}}

In all of our experiments, we use the Adam optimizer \citep{kingma2014adam} to perform parameter updates. All models (actors, critics and coach) are parametrized by feedforward networks containing two hidden layers of $128$ units. We use the Rectified Linear Unit (ReLU) \citep{nair2010rectified} as activation function and layer normalization \citep{ba2016layer} on the pre-activations unit to stabilize the learning. We use a buffer-size of $10^6$ entries and a batch-size of $1024$. We collect $100$ transitions by interacting with the environment for each learning update. For all tasks in our hyperparameter searches, we train the agents for $15,000$ episodes of $100$ steps and then re-train the best configuration for each algorithm-environment pair for twice as long ($30,000$ episodes) to ensure full convergence for the final evaluation. The scale of the exploration noise is kept constant for the first half of the training time and then decreases linearly to $0$ until the end of training. We use a discount factor $\gamma$ of $0.95$ and a gradient clipping threshold of $0.5$ in all experiments. Finally for CoachReg, we fixed $K$ to 4 meaning that agents could choose between 4 sub-policies. Since policies' hidden layers are of size 128 the corresponding value for $C$ is 32. All experiments were run on Intel E5-2683 v4 Broadwell (2.1GHz) CPUs in less than 12 hours.

\clearpage

\tocless{\section{Algorithms}}{\label{app:CMADDPG:algorithms}}

\begin{algorithm}[h!]
\scriptsize
  \begin{algorithmic}
  \caption{Team}\label{alg:teammaddpg}
    \STATE Randomly initialize $N$ critic networks $Q^i$ and actor networks $\mu^i$
    \STATE Initialize the target weights
    \STATE Initialize one replay buffer $\mathcal{D}$
    \FOR{episode from 0 to number of episodes}
      \STATE Initialize random processes $\mathcal{N}^i$ for action exploration
      \STATE Receive initial joint observation $\bo_0$
      \FOR{timestep t from 0 to episode length}
        \STATE Select action $a_i = \mu^i(o^i_t) + \mathcal{N}^i_t$ for each agent
        \STATE Execute joint action $\ba_t$ and observe joint reward $\br_t$ and new observation $\bo_{t+1}$
        \STATE Store transition $(\bo_t, \ba_t, \br_t, \bo_{t+1})$ in $\mathcal{D}$
       \ENDFOR
        \STATE Sample a random minibatch of $M$ transitions from $\mathcal{D}$
        \FOR{each agent $i$}
            \STATE Evaluate $\mathcal{L}^i$ and $J_{PG}^i$ from Equations~(\ref{eq:Lcritic})~and~(\ref{eq:JPG})
            \FOR{each other agent ($j \neq i$)}
                \STATE Evaluate $J_{TS}^{i,j}$ from Equations~(\ref{eq:J_TScont}, \ref{eq:J_TSdisc})
                \STATE Update actor $j$ with $\theta^j \leftarrow \theta^j + \alpha_{\theta}\nabla_{\theta^j}\lambda_2 J_{TS}^{i,j}$
            \ENDFOR
            \STATE Update critic with $\phi^i \leftarrow \phi^i - \alpha_{\phi}\nabla_{\phi^i}\mathcal{L}^i$
            \STATE Update actor $i$ with $\theta^i \leftarrow \theta^i + \alpha_{\theta}\nabla_{\theta^i}\left(J_{PG}^i + \lambda_1 \sum_{j=1}^N J_{TS}^{i,j} \right)$
      \ENDFOR
        \STATE Update all target weights
    \ENDFOR
  \end{algorithmic}
\end{algorithm}

\begin{algorithm}[h!]
\scriptsize
  \begin{algorithmic}
  \caption{Coach}\label{alg:coachmaddpg}
    \STATE Randomly initialize $N$ critic networks $Q^i$, actor networks $\mu^i$ and one coach network $p^c$
    \STATE Initialize $N$ target networks $Q^{i\prime}$ and $\mu^{i\prime}$
    \STATE Initialize one replay buffer $\mathcal{D}$
    \FOR{episode from 0 to number of episodes}
      \STATE Initialize random processes $\mathcal{N}^i$ for action exploration
      \STATE Receive initial joint observation $\bo_0$
      \FOR{timestep t from 0 to episode length}
        \STATE Select action $a_i = \mu^i(o^i_t) + \mathcal{N}^i_t$ for each agent
        \STATE Execute joint action $\ba_t$ and observe joint reward $\br_t$ and new observation $\bo_{t+1}$
        \STATE Store transition $(\bo_t, \ba_t, \br_t, \bo_{t+1})$ in $\mathcal{D}$
       \ENDFOR
        \STATE Sample a random minibatch of $M$ transitions from $\mathcal{D}$
        \FOR{each agent $i$}
            \STATE Evaluate $\mathcal{L}^i$ and $J_{PG}^i$ from Equations~(\ref{eq:Lcritic}) and (\ref{eq:JPG})
            \STATE Update critic with $\phi^i \leftarrow \phi^i - \alpha_{\phi}\nabla_{\phi^i}\mathcal{L}^i$
            \STATE Update actor with $\theta^i \leftarrow \theta^i + \alpha_{\theta}\nabla_{\theta^i}J_{PG}^i$
        \ENDFOR
        \FOR{each agent $i$}
            \STATE Evaluate $J_{E}^i$ and $J_{EPG}^i$ from Equations~(\ref{eq:JE}) and (\ref{eq:JEPG})
            \STATE Update actor with $\theta^i \leftarrow \theta^i + \alpha_{\theta}\nabla_{\theta^i}\left(\lambda_1J_{E}^i+\lambda_2J_{EPG}^i\right)$
        \ENDFOR
        \STATE Update coach with $\psi \leftarrow \psi + \alpha_{\psi}\nabla_{\psi}\frac{1}{N}\sum_{i=1}^N \left(J_{EPG}^i + \lambda_3J_{E}^i\right)$
        \STATE Update all target weights
    \ENDFOR
  \end{algorithmic}
\end{algorithm}

\clearpage

\tocless{\section{Hyperparameter search}}{\label{app:CMADDPG:hyperparameter_search}}

\tocless{\subsection{Hyperparameter search ranges}}{\label{app:CMADDPG:hyperparameter_search_ranges}}

We perform searches over the following hyperparameters: the learning rate of the actor $\alpha_{\theta}$, the learning rate of the critic $\omega_{\phi}$ relative to the actor ($\alpha_{\phi} = \omega_{\phi} * \alpha_{\theta}$), the target-network soft-update parameter $\tau$ and the initial scale of the exploration noise $\eta_{noise}$ for the Ornstein-Uhlenbeck noise generating process \citep{uhlenbeck1930theory} as used by \citet{lillicrap2015continuous}. When using TeamReg and CoachReg, we additionally search over the regularization weights $\lambda_1$, $\lambda_2$ and $\lambda_3$. The learning rate of the coach is always equal to the actor's learning rate (i.e. $\alpha_{\theta} = \alpha_{\psi}$), motivated by their similar architectures and learning signals and in order to reduce the search space. Table \ref{table:hyperparamsearch} shows the ranges from which values for the hyperparameters are drawn uniformly during the searches.

\begin{table}[ht]
\caption[Hyperparameter ranges for multi-agent coordination environments]{Ranges for hyperparameter search, the log base is 10}
\begin{sc}
\begin{center}
\begin{tabular}{lc}
Hyperparameter &Range   
\\ \hline
$\log(\alpha_{\theta})$  &$[-8, -3]$\\
$\log(\omega_{\phi})$    &$[-2, \,\,\,\,\,2]$\\
$\log(\tau)$             &$[-3, -1]$\\
$\log(\lambda_1)$        &$[-3\,,\,\,\,\, 0]$\\
$\log(\lambda_2)$        &$[-3\,,\,\,\,\, 0]$\\
$\log(\lambda_3)$        &$[-1\,,\,\,\,\, 1]$\\
$\eta_{noise}$         &$[0.3, 1.8]$\\
\end{tabular}
\end{center}
\end{sc}
\label{table:hyperparamsearch}
\end{table}

\tocless{\subsection{Model selection}}{\label{app:CMADDPG:model_selection}}

During training, a policy is evaluated on a set of 10 different episodes every 100 learning steps. At the end of the training, the model at the best evaluation iteration is saved as the best version of the policy for this training, and is re-evaluated on 100 different episodes to have a better assessment of its final performance. The performance of a hyperparameter configuration is defined as the average performance (across seeds) of the best policies learned using this set of hyperparameter values.

\clearpage

\tocless{\subsection{Selected hyperparameters}}{}

Tables \ref{table:hp_spread}, \ref{table:hp_bounce}, \ref{table:hp_chase}, and \ref{table:hp_compromise} shows the best hyperparameters found by the random searches for each of the environments and each of the algorithms.

\begin{table}[ht]
\centering
\begin{sc}
\caption[Best found hyperparameters for SPREAD]{Best found hyperparameters for the SPREAD environment}
\resizebox{\textwidth}{!}{%
\begin{tabular}{lccccc}
\label{table:hp_spread}
Hyperparameter          &DDPG          &MADDPG         &MADDPG+Sharing       &MADDPG+TeamReg         &MADDPG+CoachReg   
\\ \hline
$\alpha_{\theta}$  &$5.3*10^{-5}$ &$2.1*10^{-5}$  &$9.0*10^{-4}$      &$2.5*10^{-5}$      &$1.2*10^{-5}$ \\
$\omega_{\phi}$    &$53$          &$79$           &$0.71$             &$42$               &$82$ \\
$\tau$             &$0.05$        &$0.083$        &$0.076$            &$0.098$            &$0.0077$\\
$\lambda_1$        &-             &-              &-                  &$0.054$            &$0.13$\\
$\lambda_2$        &-             &-              &-                  &$0.29$             &$0.24$\\
$\lambda_3$         &-                      &-              &-                  &-          &$8.4$\\
$\eta_{noise}$           &$1.0$         &$0.5$          &$0.7$              &$1.2$              &$1.6$\\
\end{tabular}}
\end{sc}
\end{table}

\begin{table}[ht!]
\centering
\begin{sc}
\caption[Best found hyperparameters for BOUNCE]{Best found hyperparameters for the BOUNCE environment}
\resizebox{\textwidth}{!}{%
\begin{tabular}{lccccc}
\label{table:hp_bounce}
Hyperparameter          &DDPG          &MADDPG       &MADDPG+Sharing         &MADDPG+TeamReg         &MADDPG+CoachReg   
\\ \hline
$\alpha_{\theta}$  &$8.1*10^{-4}$ &$3.8*10^{-5}$    &$1.2*10^{-4}$    &$1.3*10^{-5}$      &$6.8*10^{-5}$\\
$\omega_{\phi}$    &$2.4$         &$87$             &$0.47$           &$85$               &$9.4$\\
$\tau$             &$0.089$       &$0.016$          &$0.06$           &$0.055$            &$0.02$\\
$\lambda_1$        &-             &-                &-                &$0.06$             &$0.0066$\\
$\lambda_2$        &-             &-                &-                &$0.0026$           &$0.23$\\
$\lambda_3$         &-                      &-              &-                  &-          &$0.34$\\
$\eta_{noise}$           &$1.2$         &$0.9$            &$1.2$            &$1.0$              &$1.1$\\
\end{tabular}}
\end{sc}
\end{table}

\begin{table}[ht!]
\centering
\begin{sc}
\caption[Best found hyperparameters for CHASE]{Best found hyperparameters for the CHASE environment}
\resizebox{\textwidth}{!}{%
\begin{tabular}{lccccc}
\label{table:hp_chase}
Hyperparameter          &DDPG      &MADDPG             &MADDPG+Sharing   &MADDPG+TeamReg         &MADDPG+CoachReg   
\\ \hline
$\alpha_{\theta}$  &$4.5*10^{-4}$ &$2.0*10^{-4}$  &$9.7*10^{-4}$   &$1.3*10^{-5}$    &$1.8*10^{-4}$\\
$\omega_{\phi}$    &$32$          &$64$           &$0.79$          &$85$             &$90$\\
$\tau$             &$0.031$       &$0.021$        &$0.032$         &$0.055$          &$0.011$\\
$\lambda_1$        &-             &-              &-               &$0.06$           &$0.0069$\\
$\lambda_2$        &-             &-              &-               &$0.0026$         &$0.86$\\
$\lambda_3$         &-                      &-              &-                  &-          &$0.76$\\
$\eta_{noise}$           &$0.6$         &$1.0$          &$1.5$           &$1.0$            &$1.1$\\
\end{tabular}}
\end{sc}
\end{table}

\begin{table}[ht!]
\centering
\begin{sc}
\caption[Best found hyperparameters for COMPROMISE]{Best found hyperparameters for the COMPROMISE environment}
\resizebox{\textwidth}{!}{%
\begin{tabular}{lccccc}
\label{table:hp_compromise}
Hyperparameter          &DDPG                   &MADDPG           &MADDPG+Sharing     &MADDPG+TeamReg     &MADDPG+CoachReg   
\\ \hline
$\alpha_{\theta}$  &$6.1*10^{-5}$         &$3.1*10^{-4}$    &$6.2*10^{-4}$    &$1.5*10^{-5}$  &$3.4*10^{-4}$\\
$\omega_{\phi}$    &$1.7$                  &$0.94$           &$0.58$             &$90$         &$29$\\
$\tau$             &$0.065$                &$0.045$         &$0.007$          &$0.02$        &$0.0037$\\
$\lambda_1$        &-                     &-                &-                &$0.0013$        &$0.65$\\
$\lambda_2$        &-                     &-                &-                &$0.56$       &$0.5$\\
$\lambda_3$         &-                      &-              &-                  &-          &$1.3$\\
$\eta_{noise}$     &$1.1$                   &$0.7$            &$1.3$            &$1.6$          &$1.6$\\
\end{tabular}}
\end{sc}
\end{table}

\begin{table}[ht!]
\centering
\begin{sc}
\caption[Best found hyperparameters for the football environment]{Best found hyperparameters for the \textit{3-vs-1-with-keeper} Google Football environment}
\resizebox{\textwidth}{!}{%
\begin{tabular}{lccccc}
\label{table:hp_soccer}
Hyperparameter         &MADDPG                 &MADDPG+Sharing         &MADDPG+TeamReg         &MADDPG+CoachReg   
\\ \hline
$\alpha_{\theta}$       &$1.6*10^{-6}$          &$3.4*10^{-5}$          &$3.5*10^{-6}$          &$9.4*10^{-5}$\\
$\omega_{\phi}$         &$3.1$                  &$13$                   &$0.96$                 &$2.9$\\
$\tau$                  &$0.004$                &$0.0014$               &$0.0066$               &$0.018$\\
$\lambda_1$             &-                      &-                      &$0.1$                  &$0.027$\\
$\lambda_2$             &-                      &-                      &$0.02$                 &$0.027$\\
$\lambda_3$             &-                      &-                      &-                      &$2.4$\\
\end{tabular}}
\end{sc}
\end{table}


\tocless{\subsection{Selected hyperparameters (ablations)}}{}

Tables \ref{table:hp_spread_ablation}, \ref{table:hp_bounce_ablation}, \ref{table:hp_chase_ablation}, and \ref{table:hp_compromise_ablation} shows the best hyperparameters found by the random searches for each of the environments and each of the ablated algorithms.

\begin{table}[ht!]
\caption[Ablations: best found hyperparameters for SPREAD]{Best found hyperparameters for the SPREAD environment}
\begin{sc}
\begin{center}
\resizebox{0.6\textwidth}{!}{%
\begin{tabular}{lccccc}
\label{table:hp_spread_ablation}
Hyperparameter                &MADDPG+Agent Modelling         &MADDPG+Policy Mask   
\\ \hline
$\alpha_{\theta}$        &$1.3*10^{-5}$      &$6.8*10^{-5}$ \\
$\omega_{\phi}$          &$85$               &$9.4$ \\
$\tau$                  &$0.055$            &$0.02$\\
$\lambda_1$             &$0.06$            &$0$\\
$\lambda_2$            &$0$             &$0$\\
$\lambda_3$          &-          &$0$\\
$\eta_{noise}$      &$1.0$              &$1.1$\\
\end{tabular}}
\end{center}
\end{sc}
\end{table}

\begin{table}[ht!]
\caption[Ablations: best found hyperparameters for BOUNCE]{Best found hyperparameters for the BOUNCE environment}
\begin{sc}
\begin{center}
\resizebox{0.6\textwidth}{!}{%
\begin{tabular}{lccccc}
\label{table:hp_bounce_ablation}
Hyperparameter                &MADDPG+Agent Modelling         &MADDPG+Policy Mask   
\\ \hline
$\alpha_{\theta}$        &$1.3*10^{-5}$      &$2.5*10^{-4}$ \\
$\omega_{\phi}$          &$85$               &$0.52$ \\
$\tau$                  &$0.055$            &$0.0077$\\
$\lambda_1$             &$0.06$            &$0$\\
$\lambda_2$            &$0$             &$0$\\
$\lambda_3$          &-          &$0$\\
$\eta_{noise}$      &$1.0$              &$1.3$\\
\end{tabular}}
\end{center}
\end{sc}
\end{table}

\begin{table}[ht!]
\caption[Ablations: best found hyperparameters for CHASE]{Best found hyperparameters for the CHASE environment}
\begin{sc}
\begin{center}
\resizebox{0.6\textwidth}{!}{%
\begin{tabular}{lccccc}
\label{table:hp_chase_ablation}
Hyperparameter                &MADDPG+Agent Modelling         &MADDPG+Policy Mask   
\\ \hline
$\alpha_{\theta}$        &$2.5*10^{-5}$      &$6.8*10^{-5}$ \\
$\omega_{\phi}$          &$42$               &$9.4$ \\
$\tau$                  &$0.098$            &$0.02$\\
$\lambda_1$             &$0.054$            &$0$\\
$\lambda_2$            &$0$             &$0$\\
$\lambda_3$          &-          &$0$\\
$\eta_{noise}$      &$1.2$              &$1.1$\\
\end{tabular}}
\end{center}
\end{sc}
\end{table}

\begin{table}[ht!]
\caption[Ablations: best found hyperparameters for COMPROMISE]{Best found hyperparameters for the COMPROMISE environment}
\begin{sc}
\begin{center}
\resizebox{0.6\textwidth}{!}{%
\begin{tabular}{lccccc}
\label{table:hp_compromise_ablation}
Hyperparameter                &MADDPG+Agent Modelling         &MADDPG+Policy Mask   
\\ \hline
$\alpha_{\theta}$        &$1.2*10^{-4}$      &$2.5*10^{-4}$ \\
$\omega_{\phi}$          &$0.71$               &$0.52$ \\
$\tau$                  &$0.0051$            &$0.0077$\\
$\lambda_1$             &$0.0075$            &$0$\\
$\lambda_2$            &$0$             &$0$\\
$\lambda_3$          &-          &$0$\\
$\eta_{noise}$      &$1.8$              &$1.3$\\
\end{tabular}}
\end{center}
\end{sc}
\end{table}


\tocless{\subsection{Hyperparameter search results}}{\label{app:CMADDPG:hyperparameter_search_results}}

The performance distributions across hyperparameters configurations for each algorithm on each task are depicted in Figure~\ref{fig:hyperparam_search_boxplots} using box-and-whisker plot. It can be seen that, while most algorithms can perform reasonably well with the correct configuration,  TeamReg, CoachReg as well as their ablated versions boost the performance of the third quartile, suggesting an increase in the robustness across hyperparameter compared to the baselines.

\begin{figure}[h!]
    \centering
    \makebox[\textwidth][c]{\includegraphics[width=0.88\textwidth]{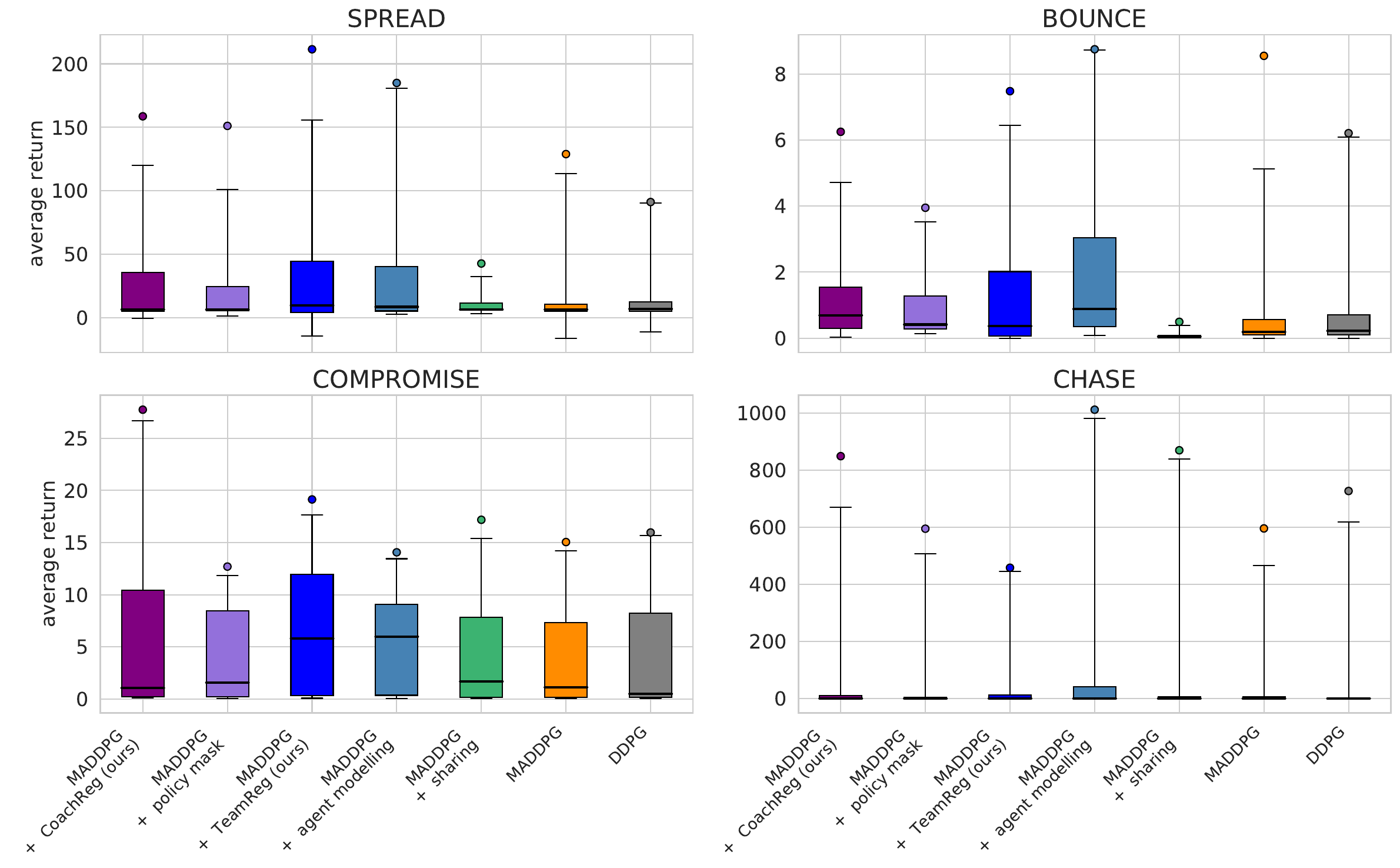}}
    \caption[Hyperparameter tuning results for coordination algorithms]{Hyperparameter tuning results for all algorithms. There is one distribution per \textit{(algorithm, environment)} pair, each one formed of 50 data-points (hyperparameter configuration samples). Each point represents the best model performance averaged over 100 evaluation episodes and averaged over the 3 training seeds for one sampled hyperparameters configuration. The box-plots divide in quartiles the 49 lower-performing configurations for each distribution while the score of the best-performing configuration is highlighted above the box-plots by a single dot.}
    \label{fig:hyperparam_search_boxplots}
\end{figure}

\clearpage
\tocless{\section{The effects of enforcing predictability (additional results)}}{\label{app:CMADDPG:effects_enforcing_predictability}}

The results presented in Figure~\ref{fig:learning_curves} show that MADDPG + TeamReg is outperformed by all other algorithms when considering average return across agents. In this section we seek to further investigate this failure mode. 

Importantly, COMPROMISE is the only task with a competitive component (i.e. the only one in which agents do not share their rewards). The two agents being strapped together, a good policy has both agents reach their landmark successively (e.g. by having both agents navigate towards the closest landmark). However, if one agent never reaches for its landmark, the optimal strategy for the other one becomes to drag it around and always go for its own, leading to a strong imbalance in the return cumulated by both agents. While such scenario doesn't occur for the other algorithms, we found TeamReg to often lead to cases of domination such as depicted in Figure~\ref{fig:performance_difference_compromise_learning_curves}.

\begin{figure}[b!]
    \centering
    \includegraphics[width=0.45\textwidth]{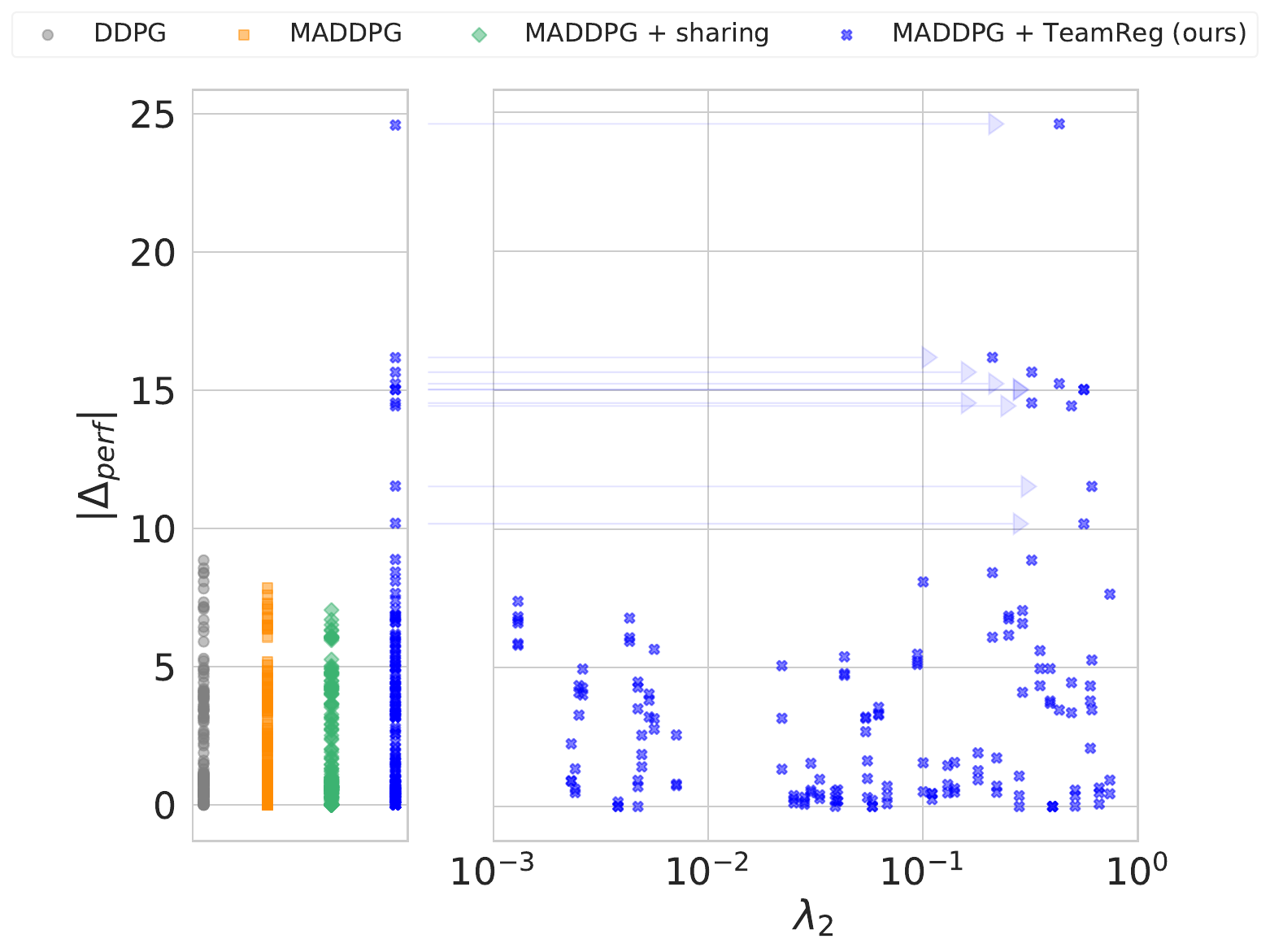}
    \caption[Average performance difference between the two agents]{Average performance difference ($\Delta_{perf}$) between the two agents in COMPROMISE for each 150 runs of the hyperparameter searches (left). All occurrences of abnormally high performance difference are associated with high values of $\lambda_2$ (right).}
    \label{fig:ts_compromise_analysis}
\end{figure}


Figure~\ref{fig:ts_compromise_analysis} depicts the performance difference between the two agents for every 150 runs of the hyperparameter search for TeamReg and the baselines, and shows that (1) TeamReg is the only algorithm that leads to large imbalances in performance between the two agents and (2) that these cases where one agent becomes dominant are all associated with high values of $\lambda_2$, which drives the agents to behave in a predictable fashion to one another. 

Looking back at Figure~\ref{fig:performance_difference_compromise_learning_curves}, while these domination dynamics tend to occur at the beginning of training, the dominated agent eventually gets exposed more and more to sparse reward gathered by being dragged (by chance) onto its own landmark, picks up the goal of the task and starts pulling in its own direction, which causes the average return over agents to drop as we see happening midway during training in Figure~\ref{fig:learning_curves}. These results suggest that using a predictability-based team-regularization in a competitive task can be harmful; quite understandably, you might not want to optimize an objective that aims at making your behavior predictable to your opponent.

\begin{figure}[ht]
    \centering
    \includegraphics[width=0.6\textwidth]{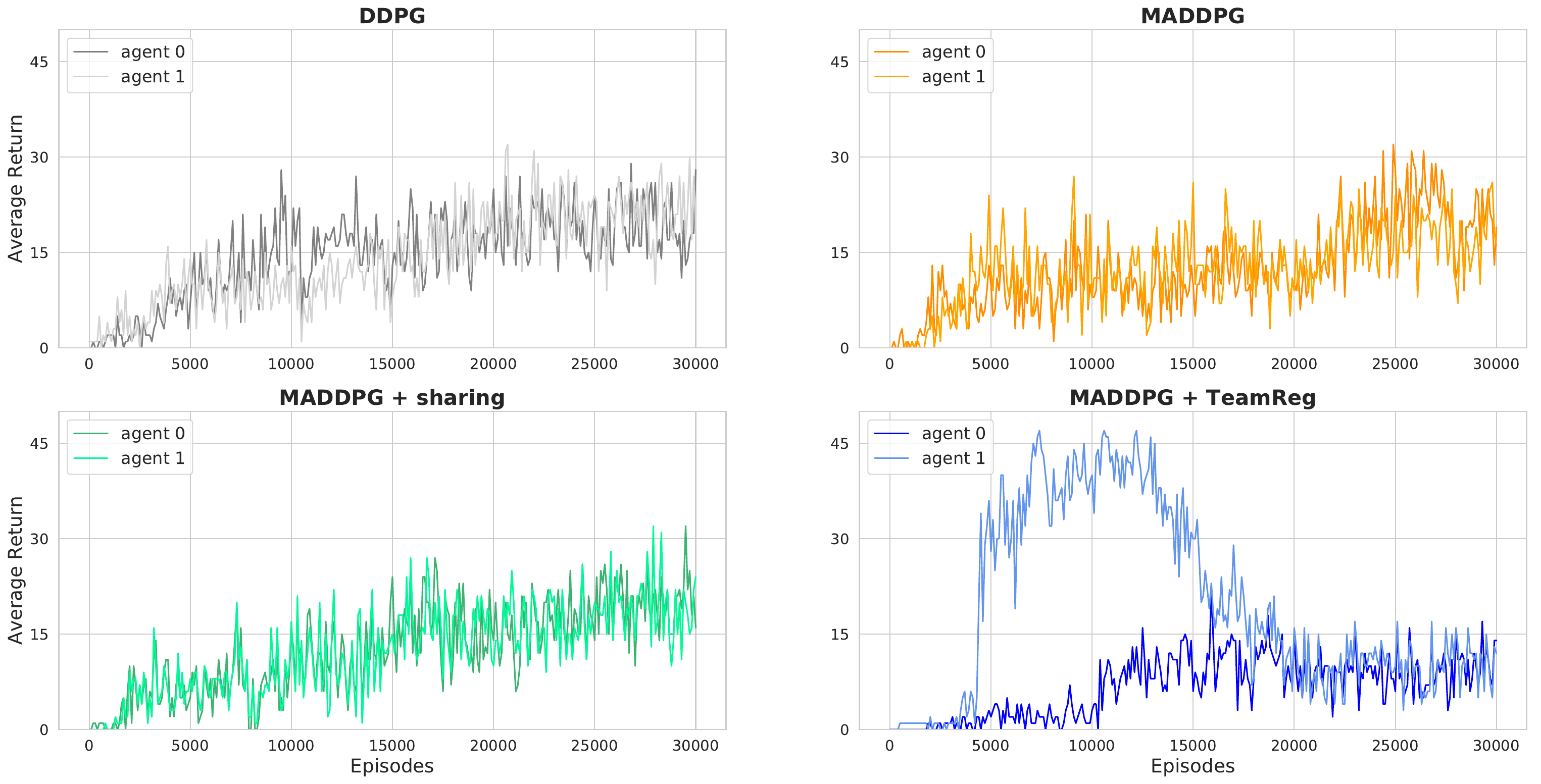}
    \caption[Learning curves for TeamReg and the baselines on COMPROMISE]{Learning curves for TeamReg and the three baselines on COMPROMISE. We see that while both agents remain equally performant as they improve at the task for the baseline algorithms, TeamReg tends to make one agent much stronger than the other one. This domination is optimal as long as the other agent remains docile, as the dominant agent can gather much more reward than if it had to compromise. However, when the dominated agent finally picks up the task, the dominant agent that has learned a policy that does not compromise see its return dramatically go down and the mean over agents overall then remains lower than for the baselines.}
    \label{fig:performance_difference_compromise_learning_curves}
\end{figure}


\tocless{\section{Analysis of sub-policy selection (additional results)}}{}
\tocless{\subsection{Mask densities}}{}
We depict on Figure~\ref{fig:mask_distributions} the mask distribution of each agent for each \textit{(seed, environment)} experiment when collected on a 100 different episodes. Firstly, in most of the experiments, agents use at least 2 different masks. Secondly, for a given experiments, agents' distributions are very similar, suggesting that they are using the same masks in the same situations and that they are therefore synchronized. Finally, agents collapse more to using only one mask on CHASE, where they also display more dissimilarity between one another. This may explain why CHASE is the only task where CoachReg does not improve performance. Indeed, on CHASE, agents do not seem synchronized nor leveraging multiple sub-policies which are the priors to coordination behind CoachReg. In brief, we observe that CoachReg is less effective in enforcing those priors to coordination of CHASE, an environment where it does not boost nor harm performance.
\begin{figure}[t!]
    \centering
    \makebox[\textwidth][c]{\includegraphics[width=0.88\textwidth]{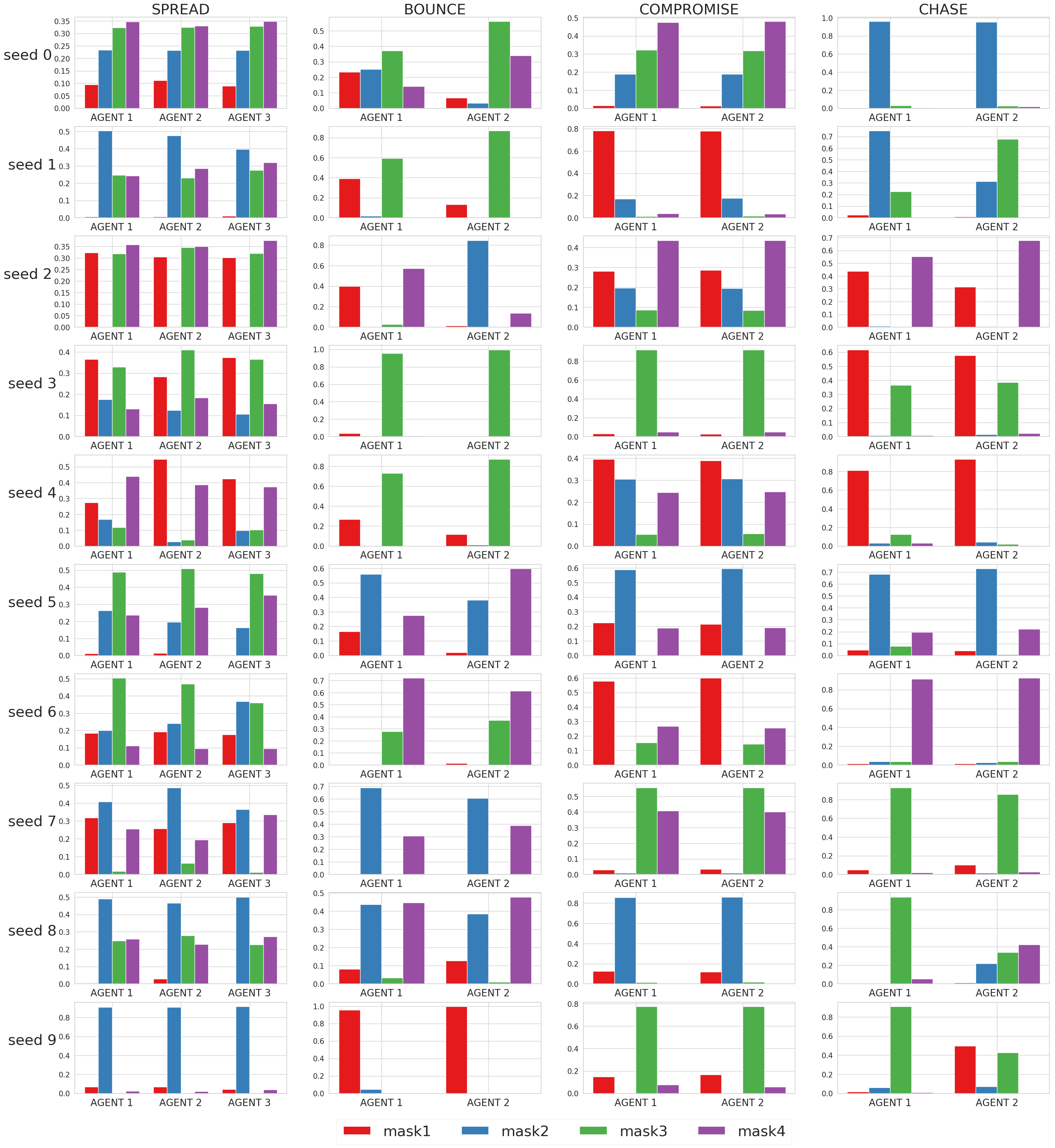}}
    \caption[Agent's policy mask distributions]{Agent's policy mask distributions. For each \textit{(seed, environment)} we collected the masks of each agents on 100 episodes.}
    \label{fig:mask_distributions}
\end{figure}

\FloatBarrier

\tocless{\subsection{Episodes rollouts with synchronous sub-policy selection}}{\label{app:CMADDPG:coach_subpolicy_selection}}
We display here and on \url{https://sites.google.com/view/marl-coordination/} some interesting sub-policy selection strategy evolved by CoachReg agents. 
On Figure~\ref{fig:BD_bounce}, the agents identified two different scenarios depending on the target-ball location and use the corresponding policy mask for the whole episode. Whereas on Figure~\ref{fig:BD_bounce}, the agents synchronously switch between policy masks during an episode. In both cases, the whole group selects the same mask as the one that would have been suggested by the coach.
\begin{figure}[h!]
    \centering
    \begin{tabular}{@{}c@{}}
    \footnotesize
    (a) BOUNCE: The ball is on the left side of the target, agents both select the purple policy mask\\
    \makebox[\textwidth][c]{\includegraphics[width=0.8\textwidth]{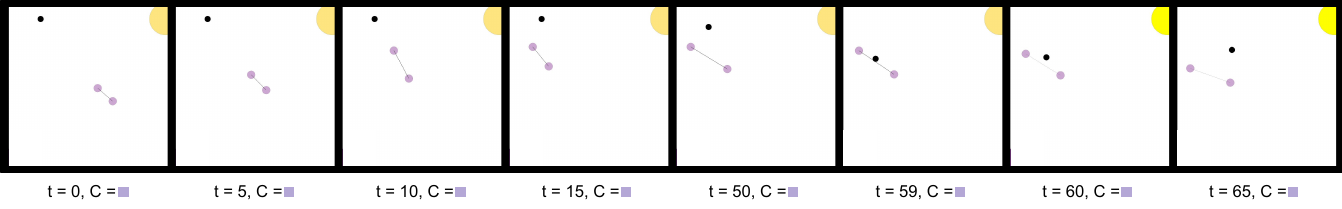}}\\
    \footnotesize
    (b) BOUNCE: The ball is on the right side of the target, agents both select the green policy mask \\
    \makebox[\textwidth][c]{\includegraphics[width=0.8\textwidth]{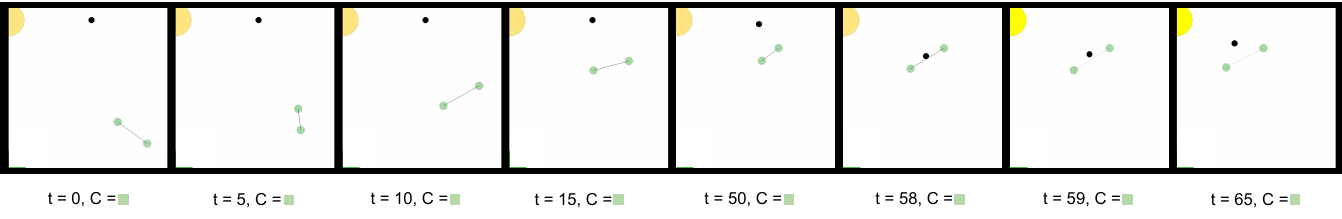}}
    \end{tabular}
    \caption[Visualization of two different BOUNCE evaluation episodes]{Visualization of two different BOUNCE evaluation episodes. Note that here, the agents' colors represent their chosen policy mask. Agents have learned to synchronously identify two distinct situations and act accordingly. The coach's masks (not used at evaluation time) are displayed with the timestep at the bottom of each frame.}
    \label{fig:BD_bounce}
\end{figure}
\begin{figure}[h!]
    \begin{tabular}{@{}c@{}}
    \footnotesize
    (a) SPREAD\\
    \makebox[\textwidth][c]{\includegraphics[width=0.8\textwidth]{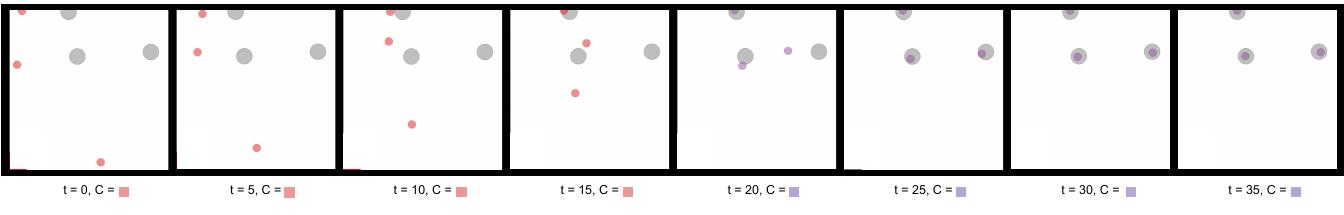}}\\
    (b) COMPROMISE\\
    \makebox[\textwidth][c]{\includegraphics[width=0.8\textwidth]{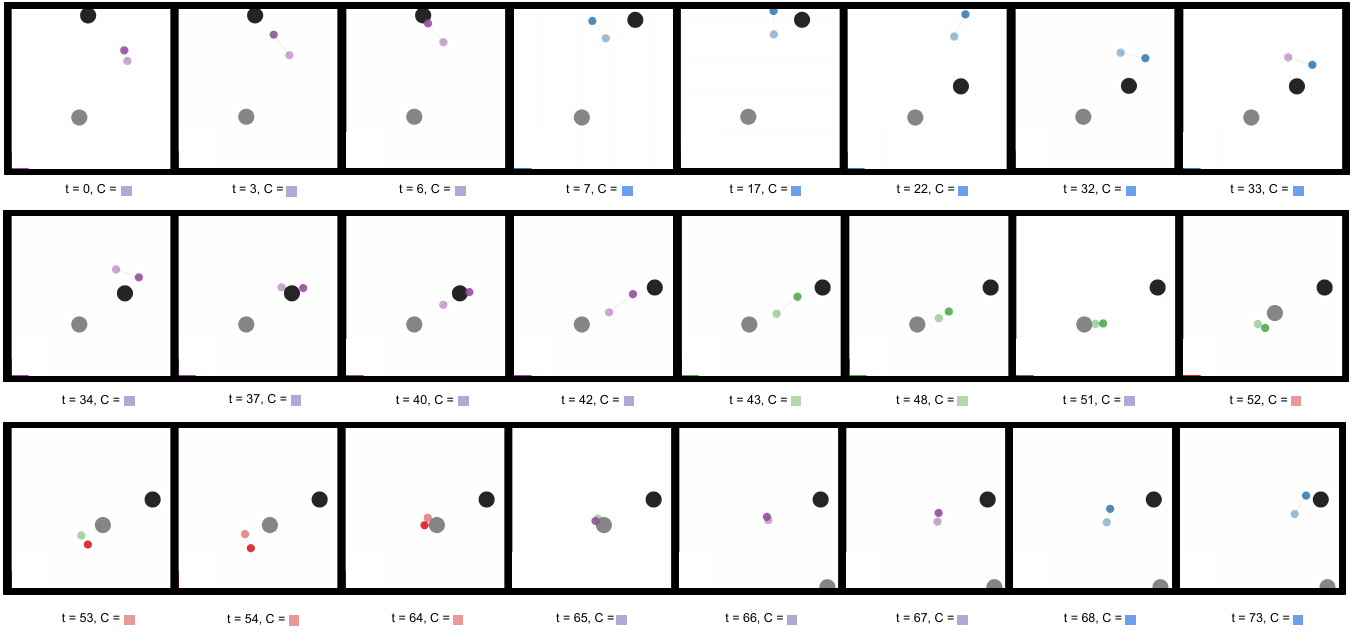}}
    \end{tabular}
    \caption[Visualization of sequences on two different environments]{Visualization of sequences on two different environments. An agent's color represent its current policy mask. The coach's masks (not used at evaluation time) are displayed with the timestep at the bottom of each frame. Agents synchronously switch between the available policy masks.}
    \label{fig:BD_envs}
\end{figure}
\clearpage

\tocless{\subsection{Mask diversity and synchronicity (ablation)}}{}
As in Subsection~\ref{subsec:CMADDPG:analysis_coach} we report the mean entropy of the mask distribution and the mean Hamming proximity for the ablated ``MADDPG + policy mask'' and compare it to the full CoachReg. With ``MADDPG + policy mask'' agents are not incentivized to use the same masks. Therefore, in order to assess if they synchronously change policy masks, we computed, for each agent pair, seed and environment, the Hamming proximity for every possible masks equivalence (mask 3 of agent 1 corresponds to mask 0 of agent 2, etc.) and selected the equivalence that maximised the Hamming proximity between the two sequences. 

We can observe that while ``MADDPG + policy mask'' agents display a more diverse mask usage, their selection is less synchronized than with CoachReg. This is easily understandable as the coach will tend to reduce diversity in order to have all the agents agree on a common mask, on the other hand this agreement enables the agents to synchronize their mask selection. To this regard, it should be noted that ``MADDPG + policy mask'' agents are more synchronized that agents independently sampling their masks from $k$-CUD, suggesting that, even in the absence of the coach, agents tend to synchronize their mask selection.
\begin{figure}[h!]
    \centering
    \begin{tabular}{@{}c@{}}
    \makebox[\textwidth][c]{\includegraphics[width=0.6\textwidth]{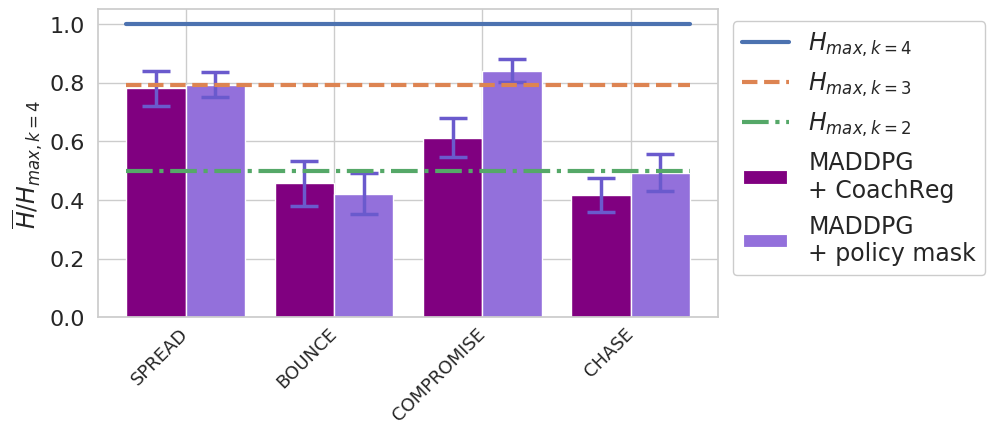}}\\
    \makebox[\textwidth][c]{\includegraphics[width=0.6\textwidth]{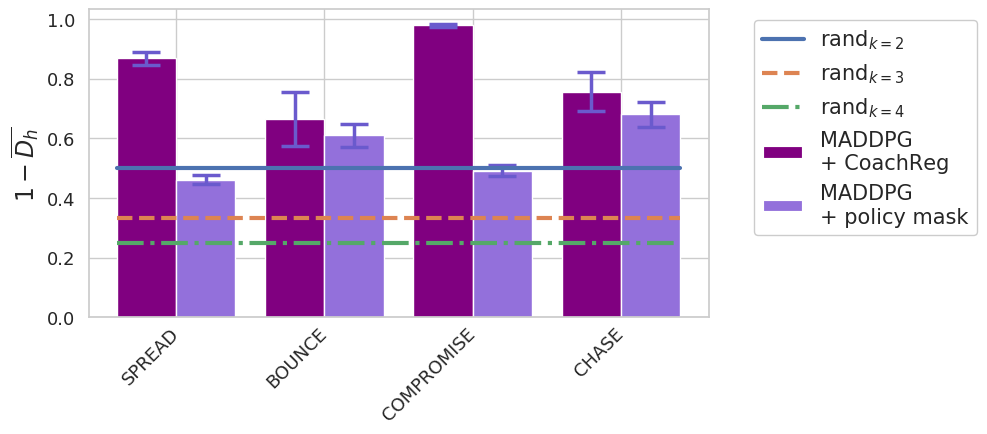}}
    \end{tabular}
    \caption[Analysis of the policy mask distributions]{(Left) Entropy of the policy mask distributions for each task and method, averaged over agents and training seeds. $H_{max,k}$ is the entropy of a $k$-CUD. (Right) Hamming Proximity between the policy mask sequence of each agent averaged across agent pairs and seeds. rand$_k$ stands for agents independently sampling their masks from $k$-CUD. Error bars are SE across seeds.}
    \label{fig:ablated_proxim}
\end{figure}

\clearpage
\tocless{\section{Scalability with the number of agents}}{}
\tocless{\subsection{Complexity}}{}
In this section we discuss the increases in model complexity that our methods entail. In practice, this complexity is negligible compared to the overall complexity of the CTDE framework. To that respect, note that (1) the critics are not affected by the regularizations, so our approaches only increase complexity for the forward and backward propagation of the actor, which consists of roughly half of an agent’s computational load at training time. Moreover, (2) efficient design choices significantly impact real-world scalability and performance: we implement TeamReg by adding only additional heads to the pre-existing actor model (effectively sharing most parameters for the teammates’ action predictions with the agent’s action selection model). CoachReg consists only of an additional linear layer per agent and a unique Coach entity for the whole team (which scales better than a critic since it only takes observations as inputs).
As such, only a small number of additional parameters need to be learned relatively to the underlying base CTDE algorithm. 
For a TeamReg agent, the number of parameters of the actor increases linearly with the number of agents (additional heads) whereas the critic model grows quadratically (since the observation size themselves usually depend on the number of agents). In the limit of increasing the number of agents, the proportion of added parameters by TeamReg compared to the increase in parameters of the centralised critic vanishes to zero.
On the SPREAD task for example, training 3 agents with TeamReg increases the number of parameters by about 1.25\% (with similar computational complexity increase). With 100 agents, this increase is only of 0.48\%.
For CoachReg, the increase in an agent's parameter is independent of the number of agent. 
Finally, any additional heads in TeamReg or the Coach in CoachReg are only used during training and can be safely removed at execution time, reducing the systems computational complexity to that of the base algorithm.

\tocless{\subsection{Robustness}}{}

To assess how the proposed methods scale to greater number of agents, we increase the number of agents in the SPREAD task from three to six agents. The results presented in Figure~\ref{fig:number_of_agents} show that the performance benefits provided by our methods hold when the number of agents is increased. Unsurprisingly, we also note how quickly learning becomes more challenging when the number of agents rises. Indeed, with each new agent, the coordination problem becomes more and more difficult, and that might explain why our methods that promote coordination maintain a higher degree of performance. Nonetheless, in the sparse reward setting, the complexity of the task soon becomes too difficult and none of the algorithms is able to solve it with six agents.

While these results show that our methods do not contribute to a quicker downfall when the number of agents is increased, they are not however aimed at tackling the problem of massively-multi-agent RL. Other approaches that use attention heads~\cite{iqbal2018actor} or restrict one agent perceptual field to its $n$-closest teammates are better suited to these particular challenges and our proposed regularisation schemes could readily be adapted to these settings as well.
\begin{figure}[ht]
    \centering
    \includegraphics[width=0.9\textwidth]{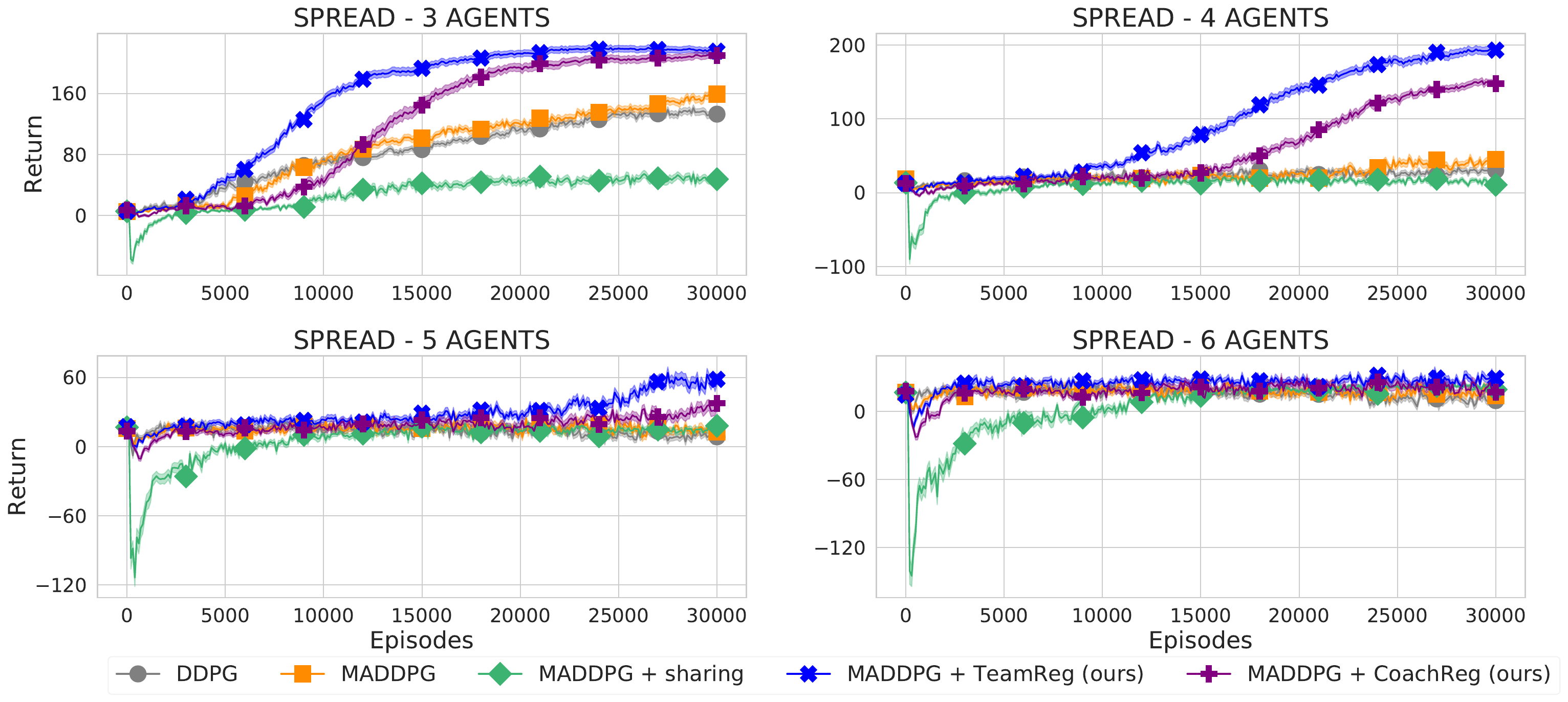}
    \caption[Learning curves for varying number of agents]{Learning curves (mean return over agents) for all algorithms on the SPREAD environment for varying number of agents. Solid lines are the mean and envelopes are the Standard Error (SE) across the 10 training seeds.}
    \label{fig:number_of_agents}
\end{figure}

\Annexe{Supplementary Material for Chapter~\ref{chap:article3_dbs}}

\tocless{\section{Algorithm}}{\label{sec:DBS:algorithm}}

Our implementation of the SAC-Lagrangian algorithm is presented below. The exact values of each hyperparameter for all of our experiments are listed in Tables~\ref{table:hyperparams_arena}~and~\ref{table:hyperparams_openWorld}. One notable difference between an unconstrained Soft-Actor Critic \cite{haarnoja2018soft} and our constrained version is that SAC is typically updated after every environment step to maximise the sample efficiency of the algorithm. In the constrained case however, since the constraints are optimized on-policy, updating the SAC agent at every environment step would only allow for one-sample estimates of the multiplier's objective. On the other hand, freezing the SAC-agent for as many environment steps as the Lagrange multiplier batch-size $N_\lambda$ makes the overall algorithm significantly less sample efficient. One \textit{could} disregard the ``on-policyness'' of the multiplier's objective but in preliminary experiments we found that, unsurprisingly, updating the Lagrange multipliers very frequently while using a large set of samples (many of which were collected using previous versions of the policy) lead to significant overshoot and harms the ability of the multipliers to converge to a stable behavior. There is thus a trade-off to make between the variance of the multiplier's objective estimate, the degree to which the multipliers are updated on-policy and the sample efficiency of the overall algorithm. In practice we found that the values for $M_\theta$ and $M_\lambda$ presented in Tables~\ref{table:hyperparams_arena}~and~\ref{table:hyperparams_openWorld} represented good compromises between these different characteristics. Another important detail is that we use $K+1$ separate critics to model the discounted expected sum of reward and costs. $Q^{(0)}$ is the critic that models the main objective and $Q^{(k)}, k=1,\dots,K+1$ are the critics that model the constraint components of the Lagrangian. Using separate critics allows to avoid fast changes in the scale of the objective, as seen by the critics, when the multipliers $\lambda_k$ get adjusted; they can solely focus on modeling the agent's changing behavior with respect to their respective function (reward or costs).

\begin{algorithm*}
\small
    \begin{algorithmic}
    \caption{\label{alg:SAC-Lagrangian}SAC-Lagrangian with Bootstrap Constraint}
    \REQUIRE learning rate $\beta$, replay buffer $\mathcal{B}$, entropy coefficient $\alpha$ and minibatch sizes $N_\theta$ and $N_\lambda$
    \REQUIRE Initialise the policy $\pi_\theta$ and value-functions $Q^{(k)}_{\phi}$ randomly, $k=0,\dots,K+1$
    \REQUIRE Initialise the Lagrange multiplier parameters $z_k$
    \REQUIRE Collect enough transitions to fill $\mathcal{B}$ with $max(N_\theta, N_\lambda)$ samples
    \FOR{updates $u=1,...$ (until convergence)}
        \STATE \textcolor{gray}{\textbf{\# Data collection}}
        \STATE Sample from the current policy: $a\sim \pi_{\theta}(\cdot|s)$
        \STATE Query next state, reward and indicators $(s', r, \{e\}_{k=1}^{K+1})$  by interacting with the environment
        \STATE Append transition $(s, a, r, s', \{e\}_{k=1}^K+1)$ to the replay buffer $\mathcal{B}$
        \STATE \textcolor{gray}{\textbf{\# Policy Gradient update}}
        \IF{$u \, \% \, M_{\theta} == 0$}
        \STATE Sample a minibatch of $N_\theta$ transitions \textbf{uniformly} from the replay buffer
        \STATE Sample next actions: $\qquad a_i' \sim \pi_\theta(\cdot|s_i') \quad i=1,...,N_{\theta}$
        \FOR{$k=0,\dots,K+1$}
            \STATE Set the ``rewards'' to their corresponding values: $\qquad r^{(0)}_i = r_i\quad$ and $\quad r^{(k)}_i = e_i^{(k)}$
            \STATE Compute the Q-targets: $\qquad y_i^{(k)} = -\alpha \log \pi_\theta(a_i'|s_i') + \min_{j \in \{1,2\}} Q_{\phi_j}^{(k)}(s_i', a_i')$
            \STATE Adam descent on Q-nets with: $\qquad \nabla_{\phi_j}\frac{1}{N_\theta}\sum_{i=1}^{N_\theta}||Q_{\phi_j}^{(k)}(s_i, a_i) - \big(r_i^{(k)} + (1 - done) \gamma y_i^{(k)}\big)||_2$
        \ENDFOR
        \STATE Re-sample the current actions: $\qquad a_i \sim \pi_\theta(\cdot|s_i) \quad i=1,...,N_\theta$
        \STATE Adam ascent on policy with: \begin{align*}\qquad \nabla_{\theta} \frac{1}{N_\theta} &\sum_{i=1}^{N_{\theta}} - \alpha \log \pi_\theta(a_i|s_i) + \max(\lambda_0, \lambda_{K+1})\min_{j}Q_{\phi_j}^{(0)}(s_i, a_i) \\&+\lambda_{K+1} \min_{j}Q_{\phi_j}^{(K+1)}(s_i, a_i) 
        - \sum_{k=1}^{K} \lambda_k \min_{j}Q_{\phi_j}^{(k)}(s_i, a_i)\end{align*}
        \ENDIF
        \STATE \textcolor{gray}{\textbf{\# Multipliers update}}
        \IF{$u \, \% \, M_\lambda == 0$}
        \STATE Draw from the replay buffer a minibatch composed of \textbf{the last} $N_{\lambda}$ transitions
        \FOR{$k=0,\dots,K+1$}
        \STATE Compute average costs: $\qquad \tilde{J}_{C_k}(\pi) = \frac{1}{N_{\lambda}}\sum_{i=1}^{N_\lambda}e_i^{(k)}$
        \STATE Adam descent on multipliers with: $\qquad \nabla_{z_k} \lambda_k (\tilde{J}_{C_k}(\pi) - \tilde{d}_k) \, \, $ if $\, \, k = K+1 \, \,$ else $\, \, \nabla_{z_k} \lambda_k (\tilde{d}_k - \tilde{J}_{C_k}(\pi))$
        \ENDFOR
        \ENDIF
    \ENDFOR
  \end{algorithmic}
\end{algorithm*}
\normalsize

\clearpage
\tocless{\section{Details for experiments in the Arena environment}}{\label{sec:DBS:arena_env_details}}

\tocless{\subsection{Environment details}}{}

In the Arena Environment, the agent's main goal is to navigate to the green tile (see Figure~\ref{fig:DBS:envs}, left). The constraints that we explore in this environment are \{\textit{On-Ground}, \textit{Not-in-Lava}, \textit{Looking-At-Marker}, \textit{Under-Speed-Limit} and \textit{Above-Energy-Limit}\}. It receives as observations its XYZ position, direction and velocity, the relative XZ position of the goal, its distance to the goal, as well as an indicator for whether it is on the ground. For the looking-at constraint, it also receives the XZ vector for the direction it is looking at, its Y-angular velocity, the marker's relative XZ position and distance, the normalised angle between the agent's looking direction and the marker as well as an indicator for whether the marker is within its field of view (a fixed-angle cone in front of the agent). For the energy constraint, the agent receives the normalised value of its energy bar and an indicator for whether it is currently recharging. Finally for the lava constraint, the agent receives an indicator of whether it currently stands in lava as well as an indicator for 25 vertical raycast of its surrounding (0 indicating safe ground and 1 indicating lava). We also add to the agent's observations the per-episode rates of indicator cost functions to the agent observation for each of the constraint as well a normalised representation of the remaining time-steps before reaching the time limit condition, leading to a total dimensionality of 53 for the observation vector. The action space is composed of 5 continuous actions (clamped between -1 and 1) which represent its XZ velocity and Y-angular velocity, a jump action (jump is triggered when the agent outputs a value above 0 for that dimensionality) and a recharge action (also with threshold of 0). The reward function is simply 1 when the agent reaches the goal (causing termination), 0 otherwise, and augmented with a small shaping reward function \cite{ng1999policy} based on whether the agent got closer or further away from the goal location.

\tocless{\subsection{Hyperparameters}}{}

Most of the hyperparameters are the same as in the original unconstrained Soft Actor-Critic (SAC) \cite{haarnoja2018soft}. Some additional hyperparameters emerge from the constraint enforcement aspect of our version of SAC-Lagrangian and are described in the Algorithm section above. We use the Adam optimizer \cite{kingma2014adam} for all parameter updates (policy, critics and Lagrange multipliers). For all experiments taking place in the Arena Environment, the policy is parameterized as a a two layer neural networks that outputs the parameters of a Gaussian distribution with a diagonal covariance matrix. The hidden layers are composed of 256 units and followed by a $tanh$ activation function. The first hidden layer also uses layer-normalisation before the application of the $tanh$ function. We use $K + 1$ fully independent critic models to estimate the expected discount sum of each of the constraint and of the main reward function. The critic models are also parameterized with two-hidden-layers neural networks with the same size for the hidden layers as the policy but instead followed by $relu$ activation functions. Table \ref{table:hyperparams_arena} shows the hyperparameters used in our experiments conducted in the Arena environment.

\begin{table}[!hb]
\small
\centering
\begin{sc}
\caption{Hyperparameters for experiments in the Arena Environment.}
\scalebox{0.9}{
\begin{tabular}{llr}
\label{table:hyperparams_arena}
\\ \toprule
\textbf{General} &
\hspace{5mm} Discount factor $\gamma$                          &{0.9} \\
&\hspace{5mm} Number of random exploration steps     &{10000} \\\vspace{1.2mm}
&\hspace{5mm} Number of buffer warmup steps          &{2560} \\
\textbf{SAC Agent} &
\hspace{5mm} Learning rate $\beta$                    &{0.0003} \\
&\hspace{5mm} Transitions between updates $M_\theta$              &{200} \\
&\hspace{5mm} Batch size $N_\theta$                    &{256} \\
&\hspace{5mm} Replay buffer size                       &{1,000,000} \\
&\hspace{5mm} Initial entropy coefficient $\alpha$             &{0.02} \\\vspace{1.2mm}
&\hspace{5mm} Target networks soft-update coefficient $\tau$             &{0.005} \\
\textbf{Lagrange Multipliers} &
\hspace{5mm} Learning rate $\beta$                    &{0.03} \\
&\hspace{5mm} Initial multiplier parameters value $z_k$                    &{0.02} \\
&\hspace{5mm} Transitions between updates $M_\lambda$               &{2000} \\\vspace{1.2mm}
&\hspace{5mm} Batch size $N_\lambda$                    &{2000} \\
\textbf{Constraint Thresholds} &
\hspace{5mm} Has reached goal (lower-bound)                       &0.99 \\
&\hspace{5mm} \textbf{NOT} looking at marker                       &0.10 \\
&\hspace{5mm} \textbf{NOT} on ground                               &0.40 \\
&\hspace{5mm} In lava                                              &0.01 \\
&\hspace{5mm} Above speed limit                                    &0.01 \\
&\hspace{5mm} Is under the minimum energy level                    &0.01 \\
\bottomrule
\end{tabular}
}
\end{sc}
\end{table}

\clearpage
\tocless{\section{Details for experiments in the OpenWorld environment}}{\label{sec:DBS:openworld_env_details}}

\tocless{\subsection{Environment details}}{}


The OpenWorld environment is a large environment (approximately $30,000$ times larger than the agent) that includes multiple multi-storey buildings with staircases, mountains, tunnels, natural bridges and lava.
In addition, the environment includes $50$ jump-pads that propel the agent into the air when it steps on one of them.
The agent is tasked with navigating towards a goal randomly placed in the environment at the beginning of every episode.
The agent controls include translation in the XY frame ($2$ inputs), a jumping action ($1$ input), a rotation action controlling where the agent is looking independent of its direction of travel ($1$ input), and a recharging action which allows the agent to recharge its energy level ($1$ input).
The recharging action immobilizes the agent, i.e., it does not allow the agent to progress towards its goal.
The environment also includes a look-at marker which we would like the agent to look at while it accomplishes its main navigation task.

At every timestep, the agent receives as observations its XYZ position relative to the goal as well as its normalized velocity and acceleration in the environment.
In addition, it receives its relative position to the nearest jump-pad in the environment.
For looking at the marker, as in the Arena environment, the agent receives the marker's relative XZ position and distance, the normalised angle between the agent's looking direction and the marker, as well as an indicator for whether the marker is within its field of view (a fixed-angle cone in front of the agent).
For the energy-limit constraint, the agent obtains the value of its energy level, a boolean describing if it is currently recharging and a Boolean indicating if it was recharging in the previous timestep.
The agent also receives a series of indicators denoting whether it is currently standing in lava, if it is touching the ground, and if the agent is currently below the minimum energy level.
In order for the agent to observe lava and other elements it can collide with in the environment (e.g., buildings, doors, mountains), the agent receives 2 channels of $8\times8$ raycasts around the agent.

\tocless{\subsection{Hyperparameters}}{}

The SAC agent in the OpenWorld environment uses the same architecture and similar hyperparameters as in \cite{alonso2020deep}.
The raycasts and raw state described above are processed using two separate embedding models.
For the raycasts, we employ a CNN with 3 convolutional layers, each with a corresponding ReLU layer.
The raw state is processed using a separate 3-layer MLP with $1024$ hidden units at each layer.
The two representations are concatenated into a single vector representing the current state.
The policy is parameterized by a 3-layer MLP that receives as input the concatenated representation and outputs the parameters of a Gaussian distribution with a diagonal covariance matrix. 
Each hidden layer is composed of $1024$ hidden units and is followed by a ReLU activation function.
The critic models are also parameterized by 3-layer MLP, are composed of $1024$ hidden units and use ReLU activation functions.
Table \ref{table:hyperparams_openWorld} shows some of these hyperparameters with a focus on the constrained enforcement aspect of our version of SAC-Lagrangian.

\begin{table}[h!]
\small
\centering
\begin{sc}
\caption{Hyperparameters for experiments in the OpenWorld Environment.}
\scalebox{0.9}{
\begin{tabular}{llr}
\label{table:hyperparams_openWorld}
\\ \toprule
\textbf{General} &
\hspace{5mm} Discount factor $\gamma$                          &{0.99} \\
&\hspace{5mm} Number of random exploration steps $\beta$     &{200} \\\vspace{1.2mm}
&\hspace{5mm} Number of buffer warmup steps $\beta$          &{2560} \\
\textbf{SAC Agent} &
\hspace{5mm} Learning rate $\beta$                    &{0.0001} \\
&\hspace{5mm} Batch size $N_\theta$                    &{2560} \\
&\hspace{5mm} Replay buffer size                       &{4,000,000} \\
&\hspace{5mm} Initial entropy coefficient $\alpha$             &{0.005} \\\vspace{1.2mm}
&\hspace{5mm} Target networks soft-update coefficient $\tau$             &{0.005} \\ 
\textbf{Lagrange Multipliers} &
\hspace{5mm} Learning rate $\beta$                    &{0.00005} \\
&\hspace{5mm} Initial multiplier parameters value $z_k$                    &{0.02} \\
&\hspace{5mm} Transitions between updates                & every timestep \\\vspace{1.2mm}
&\hspace{5mm} Batch size $N_\lambda$                    &{5000} \\
\textbf{Constraint Thresholds} &
\hspace{5mm} Has reached goal (lower-bound)                 &0.80 \\
&\hspace{5mm} \textbf{NOT} looking at marker                 &0.10 \\
&\hspace{5mm} \textbf{NOT} on ground                         &0.40 \\
&\hspace{5mm} In lava                                        &0.001 \\
&\hspace{5mm} Is under the minimum energy level              &0.01 \\
\bottomrule
\end{tabular}
}
\end{sc}
\end{table}


\clearpage
\tocless{\section{Additional experiments on reward engineering}}{\label{sec:DBS:additional_experiments}}

See Section~\ref{sec:DBS:problem_with_reward_engineering} for the description of our experiments motivating against the use of reward engineering for behavior specification. Figure~\ref{fig:reward_engineering_3constraints} below shows the results for the biggest of the 3 grid searches performed to showcase the difficulty of finding a reward function that fits the behavioral requirements when the number of requirements grows.
\begin{figure}[b!]
    \centering
    \includegraphics[width=0.8\textwidth]{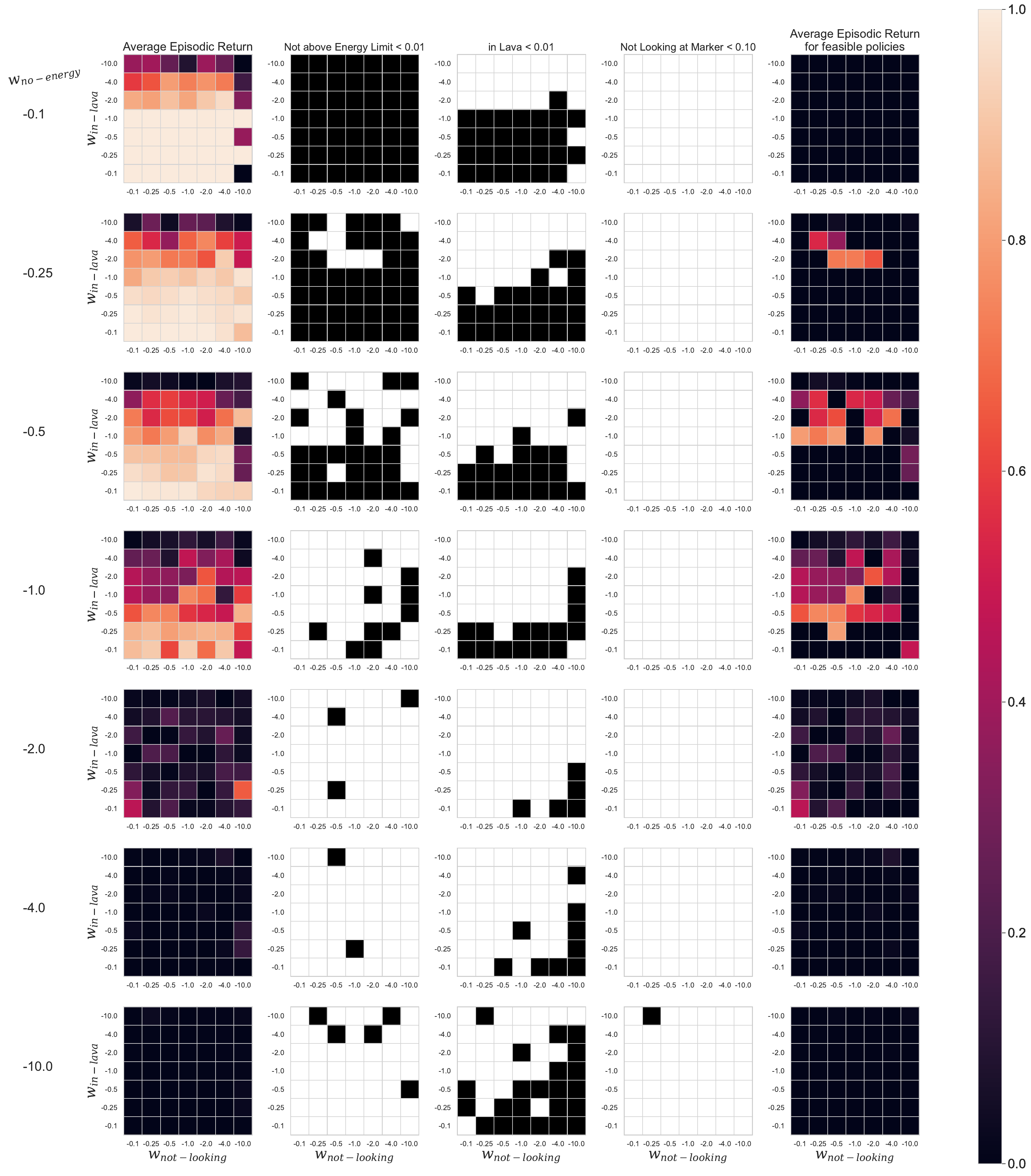}
    \caption[Reward engineering for 3 behavioral requirements]{Also see Figure~\ref{fig:reward_engineering}. When enforcing 3 behavioral requirements with reward engineering, an ever larger proportion of the experiments are wasted finding either low-performing policies or policies that do not satisfy the behavioral constraints. In this case, none of the 343 experiments yielded a feasible policy that also solves the task (success rate near 1.0), showcasing that reward engineering scales poorly with the number of constraints due to the curse of dimensionality and to the composing effect of the multiple constraints in narrowing the space of feasible policies.}
    \label{fig:reward_engineering_3constraints}
    \vspace{-10mm}
\end{figure}


\clearpage
\tocless{\section{Additional experiments on TD3}}{\label{sec:DBS:additional_experiments_td3}}

We validate that our framework can be combined with any policy optimisation algorithm by applying it to the TD3 algorithm \cite{fujimoto2018addressing}. This leads to a TD3-Lagrangian formulation using our indicator cost functions, normalized multipliers and bootstrap constraint. As for our experiments with SAC (Figure~\ref{fig:many_constraints_experiments}-d), our TD3-Lagrangian agent performs well and all constraints are satisfied. The results are presented in Figure~\ref{fig:td3_all_constraints}.

\begin{figure}[h!b]
    \centering
    \includegraphics[trim={3cm 0 0 0},clip,width=\textwidth]{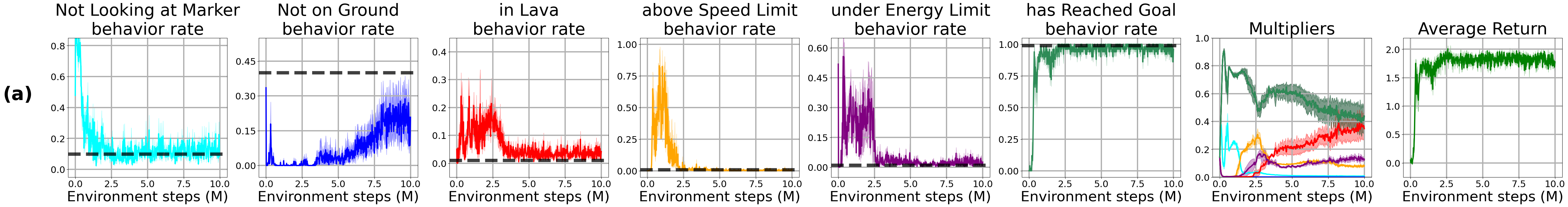}
    \caption[TD3-Lagrangian agent in the Arena environment]{TD3-Lagrangian agent in the Arena environment using normalised multipliers, indicator cost functions and using the success constraint as a bootstrap constraint. Training is halted after every $20,000$ environment steps and the agent is evaluated for 10 episodes. All curves show the average over 5 seeds and envelopes show the standard error around that mean.}
    \label{fig:td3_all_constraints}
\end{figure}

\Annexe{Supplementary Material for Chapter~\ref{chap:article4_gcgfn}}

\tocless{\section{Task and Training Details}}{\label{app:GCGFN:training_details}}

We use the GFlowNet framework \cite{bengio2021flow, bengio2023gflownet} to train discrete distribution samplers over the space of molecules that can be assembled from a set of pre-defined molecular fragments \cite{kumar2012fragment}. A state is represented as a graph in which each node represents a fragment from the fragment library and where each edge has two attributes representing the attachment point of each connected fragment to its neighbor. The state representation is augmented with a fully-connected virtual node, whose features are an embedding of the conditioning information computed from the conditioning vector that represents the preferences $w$ and/or the goal direction $d_g$. To produce the state-conditional distribution over actions, the model processes the state using a graph transformer architecture \cite{yun2019graph} for a predefined number of message-passing steps (number of layers). Our GFlowNet sampler thus starts from the initial state $s_0$ representing an empty graph. It iteratively constructs a molecule by either adding a node or an edge to the current state $s_t$ until it eventually selects the `STOP' action. 

To maintain some amount of exploration throughout training, at each construction step $t$, the model samples a random action with probability $\epsilon$ and otherwise samples from its forward transition distribution. The model is trained using the trajectory balance criterion \cite{malkin2022trajectory} and thus is parameterised by a forward action distribution $P_F$ and an estimation of the partition function $Z:= \sum_x R(x)$. Forbidden actions are masked out from the forward transition distribution (for example, the action of adding an edge to the empty state). We use a uniform distribution for the backward policy $P_B$. To prevent the sampling distribution from changing too abruptly, we collect new trajectories from a sampling model $P_F(\, \cdot \,|\theta_{\text{sampling}})$ which uses a soft update with hyperparameter $\tau$ to track the learned GFN at update $k$: $\theta_{\text{sampling}}^{(k)} \leftarrow \tau \cdot \theta_{\text{sampling}}^{(k-1)} + (1 - \tau) \cdot \theta^{(k)}$. This is akin to the target Q-functions and target policies used in actor-critic frameworks \cite{mnih2015human, fujimoto2018addressing}. 

The hyperparameters used for training both methods are listed in Table~\ref{tab:hyperparameters}.

\begin{table}[ht!]
    \centering
    \caption{Hyperparameters used in our conditional-GFN training pipeline}
    \scriptsize
    \renewcommand{\arraystretch}{1.3}
    \scalebox{0.8}{
    \begin{tabular}{|l|c|c|}
        \hline
        \multirow{2}{*}{\textbf{Hyperparameters}} & \multicolumn{2}{c|}{\textbf{Values}} \\ 
        \cline{2-3} & \textbf{Goal-conditioned GFN} & \textbf{Preference-conditioned GFN} \\ 
        \hline
        Batch size & 64 & 64 \\
        GFN temperature parameter $\beta$ & 60 & 60 \\
        Number of training steps & 40,000 & 40,000 \\
        Number of GNN layers & 2 & 2 \\
        GNN node embedding size & 256 & 256 \\
        Learning rate for GFN's $P_F$ & $10^{-4}$ & $10^{-4}$ \\
        Learning rate for GFN's $Z$-estimator & $10^{-3}$ & $10^{-3}$ \\
        Sampling moving average $\tau$ & 0.95 & 0.95 \\
        Random action probability $\epsilon$ & 0.01 & 0.01 \\
        Focus region cosine similarity threshold $c_g$ & 0.98 & - \\
        Limit reward coefficient $m_g$ & 0.20 & - \\
        Replay buffer length & 100,000 & - \\
        Number of replay buffer trajectory warmups & 1,000 & - \\
        Hindsight ratio & 0.30 & - \\
        Conditioning-vector sampling distribution &
        $d_g \sim 
        \begin{cases}
        \text{Uniform-GS} &\text{ (Sec~\ref{sec:GCGFN:difficult_landscapes})} \\
        \text{Tab-GS} &\text{ (Sec~\ref{sec:GCGFN:increasing_number_of_objectives})}
        \end{cases}$ & $w \sim Dirichlet(1)$ \\
        \rule{0pt}{1pt} & & \\
        \hline
        \end{tabular}
        }
        \label{tab:hyperparameters}
\end{table}

\tocless{\section{Failure Modes and Filtering}}{}

While using goal regions as hard constraints offers a more precise tool for controllable generation, it faces the additional challenge that not all goals may be feasible (or that reaching some goals may be much easier to learn than others). When a model is conditioned with an infeasible goal, all the samples that it will observe will have a reward $R(x)=0$. The proper behavior, in that case, is to sample any possible molecule with equal weight, thus sampling uniformly across the entire molecular state space. Such molecules generally won't be of any interest and can be discarded. Thus, in our experiments, we filter out such \textit{out-of-focus} samples (molecules falling outside the focus region) and evaluate the candidates that were inside their prescribed focus region. Figure~\ref{fig:goal_conditioning} shows the conditional distributions learned by a single model trained on the 2-objective task. The picture on the last row, second column showcases such an occurrence of difficult focus region which results in many samples simply belonging to the uniform distribution over the state space.
\begin{figure}[ht]
    \centering
    \includegraphics[width=0.65\textwidth]{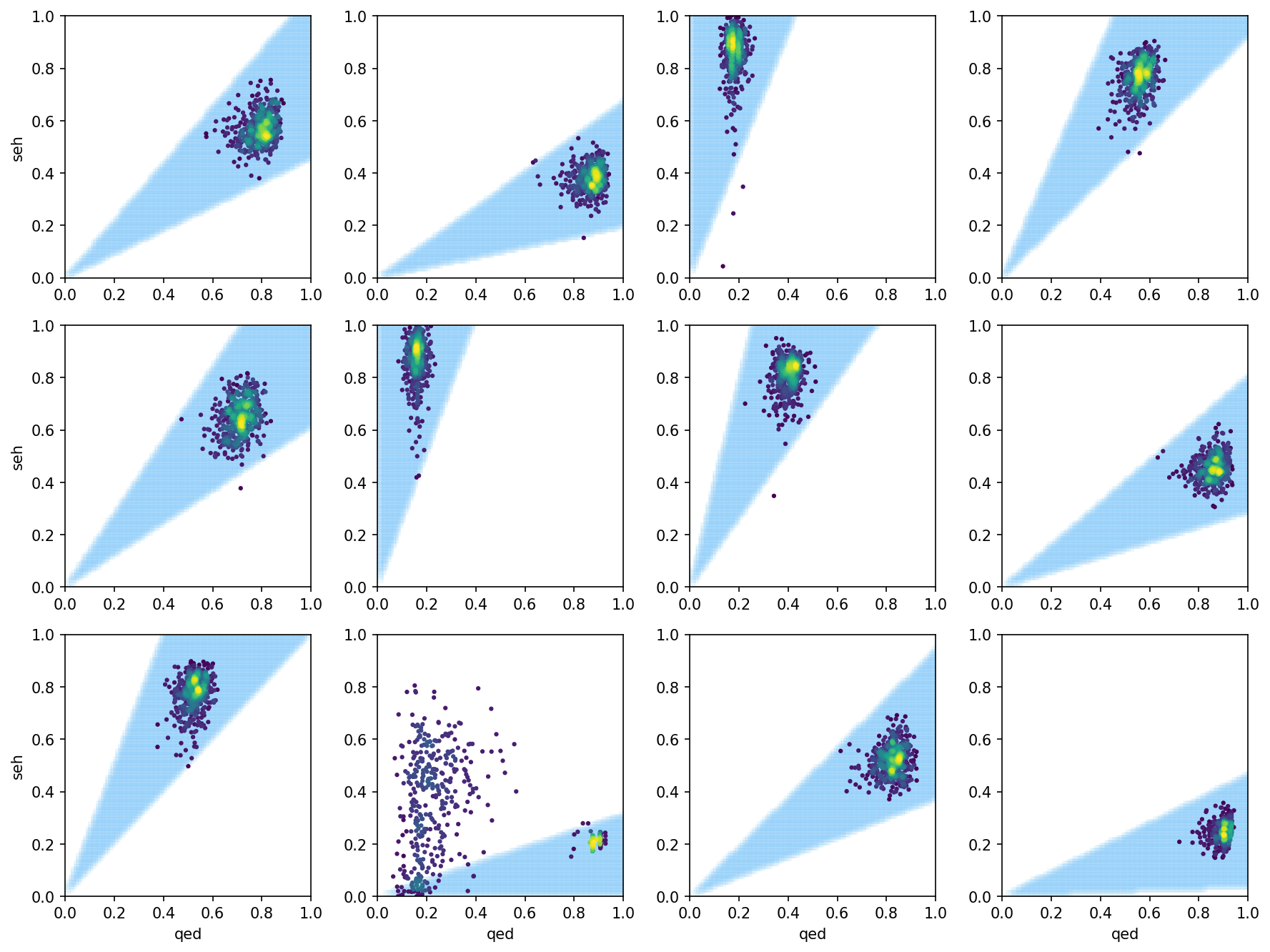}
    \caption[Learned conditional-distributions for different focus regions]{Learned conditional-distributions for different focus regions passed as input to the same model. Each dot marks the image of a generated molecule in the objective space. The colors indicate how densely populated a particular area of the objective space is (brighter is denser). The focus regions (goal regions) are depicted in light blue. The distribution on the last row, second column, showcases a focus region which seems difficult to reach and may not contain as large a population of molecules in the state space. In such cases, the model cannot learn to consistently produce samples from that goal region when conditioned on this goal direction $d_g$ and will instead produce several samples very similar to the sampling distribution of an untrained model (uniform across the state space).}
    \label{fig:goal_conditioning}
\end{figure}

\tocless{\section{Ablations}}{\label{app:GCGFN:ablations}}
\tocless{\subsection{Replay Buffer}}{\label{app:GCGFN:replay_buffer}}

While both the un-conditional and the preference-conditioned GFN models are learning stably even in a purely on-policy setting, we found that the goal-conditioned models were more prone to instabilities and mode-collapse when employed purely on-policy (see Figure~\ref{fig:replay_buffer_effect}). This could be because imposing these hard constraints on the generative behavior of the model drastically changes the reward landscape from one set of goals to another. While larger batches could potentially alleviate this problem, sampling uniformly from a replay buffer of the last trajectories proved effective, as observed in many works stemming from \citet{mnih2015human}. 
As described in Section~\ref{sec:GCGFN:goal-conditioned-gfn}, we also use hindsight experience replay \citep{andrychowicz2017hindsight}. Specifically, for every batch of data, we randomly select a subset of trajectories (hindsight-ratio * batch-size), among which we re-label both the goal direction $d_g$ and the corresponding reward for the examples that didn't reach their goal. 

\begin{figure}[h!]
    \centering
    \includegraphics[width=1.\textwidth]{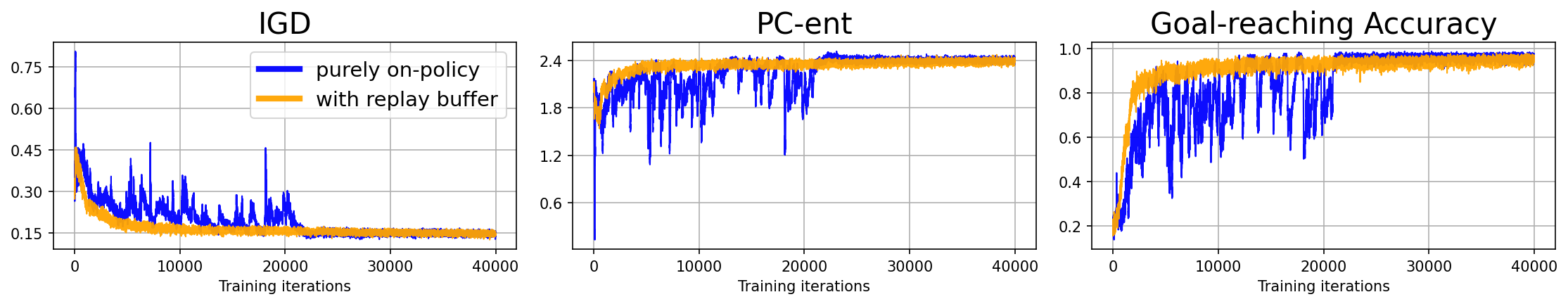}
    \caption[Effect of the replay buffer on goal-conditioned GFNs]{Learning curves for goal-conditioned models either trained purely on-policy (in blue) or using a replay buffer of past trajectories (in orange) on the 2-objective (seh, qed) task.}
    \label{fig:replay_buffer_effect}
\end{figure}

\tocless{\subsection{Limit Reward Coefficient}}{\label{app:GCGFN:focus_limit_coef}}

While the GFN model is given the goal direction $d_g$ as input, the width of the goal region, which depends on the cosine-similarity threshold $c_g$ is fixed, and the model adapts to producing samples within the region over time by trial-and-error. One can trade off the level of controllability of the goal-conditioned model with the difficulty of reaching those goals by increasing or reducing $c_g$. Another way to increase the controllability \textit{and} goal-reaching accuracy without drastically affecting the difficulty of reaching such goals is to make the model preferentially generate samples near the center of the focus region, thus reducing the risk of producing an out-of-focus sample due to epistemic uncertainty. To do so, we modify Equation~\ref{eq:goal-reward} and add a reward-coefficient $\alpha_g$, which further modulates the magnitude of the scalar reward based on how close to the center of the focus region the sample was generated. While many shaping functions could be devised, we choose the following form:
\begin{equation}
\label{eq:goal-reward-extended}
    R_g(x) = 
    \begin{cases}
    \alpha_g \sum_k r_k ,& \text{if } r \in g\\
    0, & \text{otherwise}
    \end{cases}
    \quad,\qquad \quad \alpha_g = \left( \frac{r \cdot d_g}{||r||\cdot||d_g||} \right) ^{\frac{\log m_g}{\log c_g}}
\end{equation}

\begin{figure}[hb!]
    \centering
    \includegraphics[width=1.\textwidth]{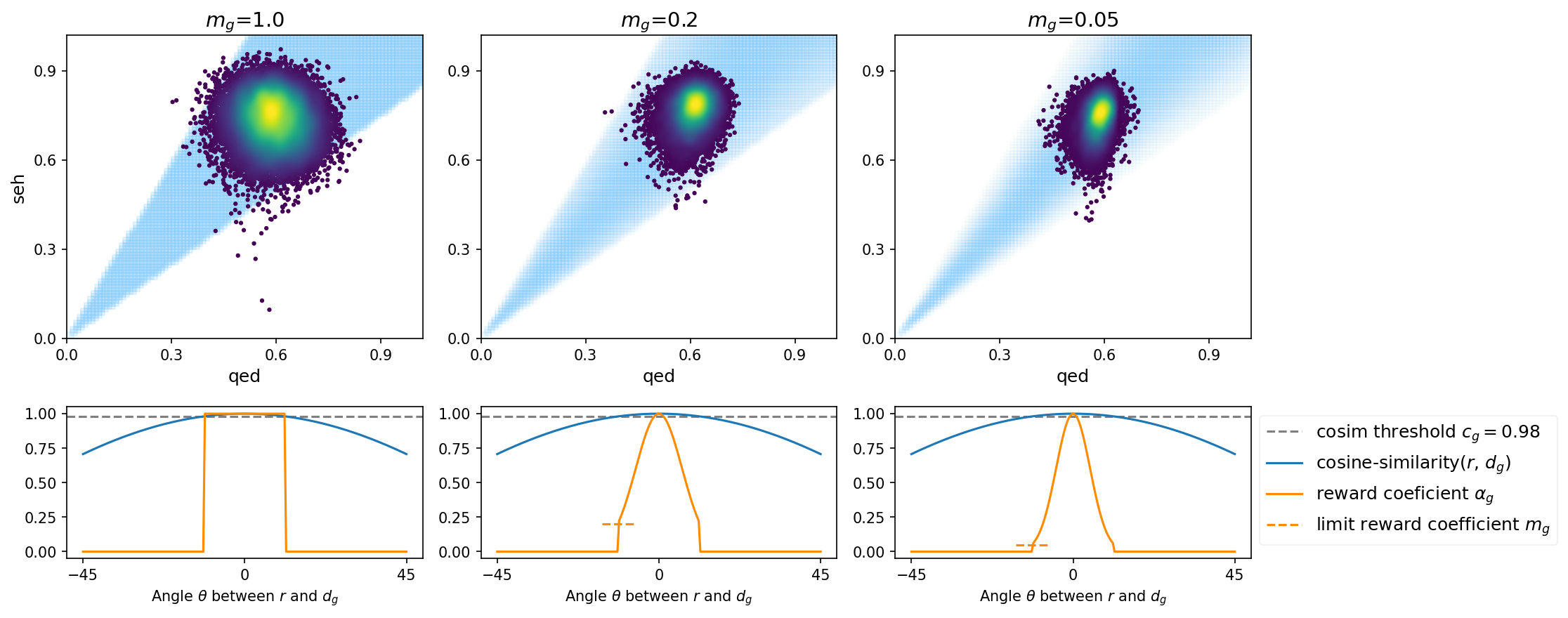}
    \caption[Effect of the hyperparameter $m_g$ on the reward profile]{Effect of the hyperparameter $m_g$ on the profile of the reward coefficient $\alpha_g$ and the learned sampling distribution (top row).}
    \label{fig:focus_limit_coef}
\end{figure}

In words, the reward coefficient $\alpha_g$ is equal to the cosine similarity between the reward vector $r$ and the goal direction $d_g$ exponentiated in such a way that $\alpha_g=m_g$ at the limit of the focus region. So for example, setting $m_g=0.2$ means that the reward is maximal at the center of the focus region, is at 20\% of that magnitude at the limit of the focus region, and follows a sharp sigmoid-like profile in between. Figure~\ref{fig:focus_limit_coef} showcases the reward coefficient as a function of the angle between $r$ and $d_g$ for different values of $m_g$ and the corresponding distributions learned by the model. We can see that a smaller value of $m_g$ encourages the model to produce samples in a more focused way towards the center of the goal region. Importantly, with a large enough value of $m_g$, this design preserves the notion of a well-defined goal region (positive reward inside the region and zero reward outside) and thus also preserves our ability to reason about goal-reaching accuracy, a beneficial concept for monitoring the model, filtering out-of-focus samples, etc.

\tocless{\subsection{Tabular Goal-Sampler}}{\label{app:GCGFN:learned_goal_model}}

To cope with the problem of infeasible goal regions described in Section~\ref{sec:GCGFN:learned_goal_distribution}, we explore the idea of sampling the goal directions $d_g$ from a learned goal distribution rather than sampling all directions uniformly. The idea is that, as the model learns about which goal directions point towards infeasible regions of the objective space, we can attribute a much lower sampling likelihood to these regions in order to focus on more fruitful goals. 

We implement a first version of this idea as a tabular goal-sampler (Tab-GS). We first build a dataset of goal directions $\mathcal{D}_G$. This could be done in many different ways such as sampling a large number of positive vectors at the surface of the unit hypersphere in objective space. In our case, we discretise the extreme faces of the unit hypercube and normalize them. At training time, for each direction vector $d_g \in \mathcal{D}_G$, we keep a count of the number of samples which have landed closest to it (closer than any other direction $d_g'$) and follow this very simple scheme: from the beginning up to 25\% of the training iterations, we sample batches of goal directions $\{d_g\}_{i=1}^N$ uniformly over $\mathcal{D}_G$. Then starting at 25\% of the training iterations, while we keep updating each direction's count, we sample batches of goal directions according to the following (unnormalized) likelihoods:

\begin{equation}
\label{eq:goal-sampling-likelihood}
    f(d_g) = 
    \begin{cases}
    1 & \text{if } d_g \text{ has never been sampled}\\
    1 & \text{if there has been a sample } r \text{ closer to } d_g \text{ than any other goal direction in } \mathcal{D}_G\\
    0.1 & \text{otherwise}
    \end{cases}
\end{equation}

Finally, at 75\% of the training, we stop updating the goal direction counts to allow the model to fine-tune itself to a now stationary goal-distribution Tab-GS($f$). At test time we also sample from that same stationary distribution. 

Figure~\ref{fig:effect_of_tabgs} shows the effect of our learned tabular goal-sampler (Tab-GS) on the model's performance and learning dynamics. While the 2-objective problem does not contain a lot of infeasible goal directions, resulting in very similar behaviors for both methods, we can see that in the case of 3 and 4 objectives, the model experiences an important immediate improvement in goal-reaching accuracy at 25\% of training when we start sampling $d_g$'s according to our learned goal-sampler and that this improved focus helps the model further improves on these more fruitful goal directions, resulting in an increase IGD and PC-ent scores.

\begin{figure}[ht!]
    \centering
    \includegraphics[width=1.\textwidth]{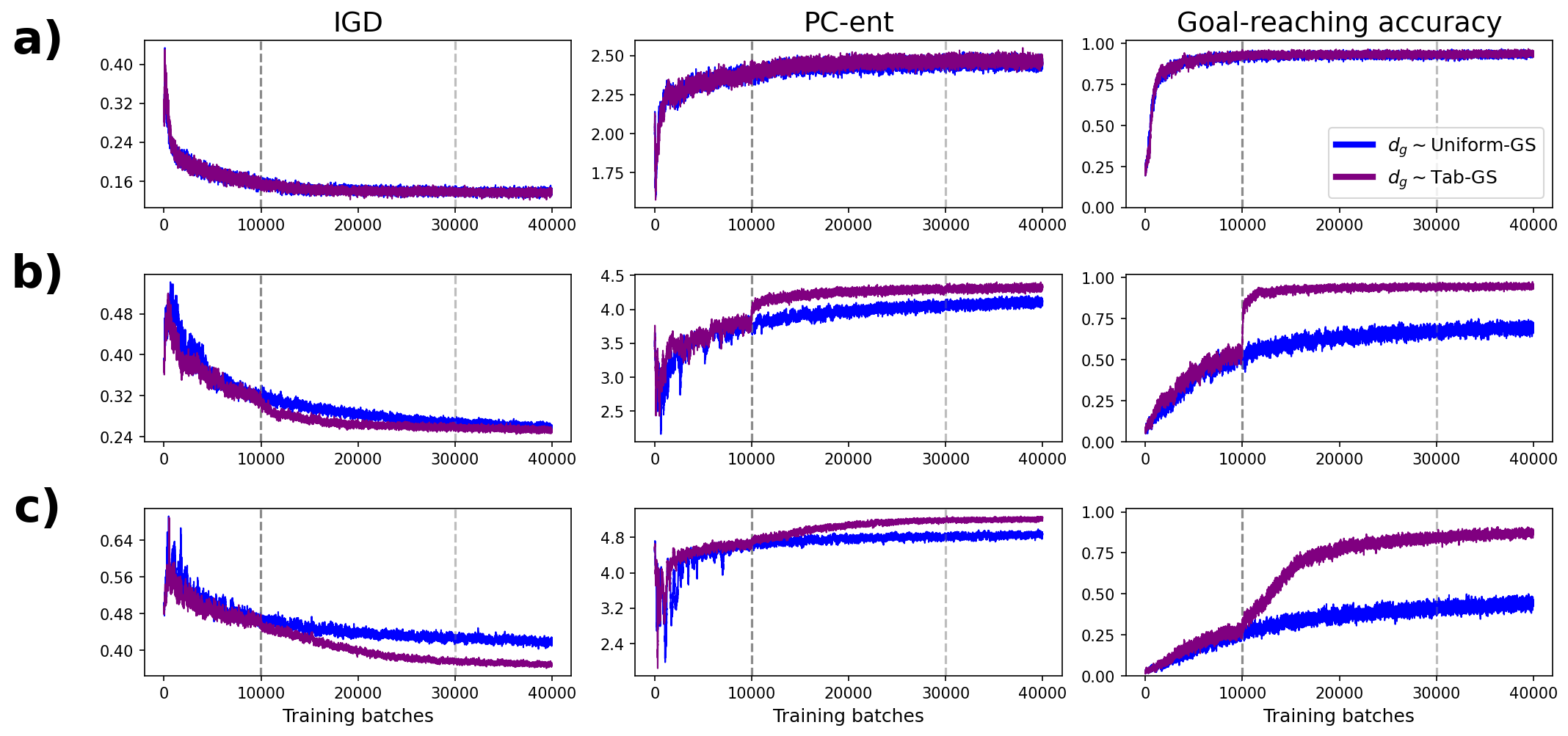}
    \caption[Learning curves with and without a tabular goal sampler]{Learning curves for our goal-conditioned model trained by sampling goal directions $d_g$ either uniformly on the positive quadrant of a $K$-dimensional hypersphere (Uniform-GS) in blue or according to our learned tabular goal-sampler (Tab-GS) in purple on \textbf{a)} 2 objectives, \textbf{b)} 3 objectives and \textbf{c)} 4 objectives. Vertical dotted lines indicate 25\% and 75\% of training when we start sampling goal directions according to Equation~\ref{eq:goal-sampling-likelihood} and when we stop updating the learned goal-sampler, respectively.}
    \label{fig:effect_of_tabgs}
\end{figure}

\clearpage
\tocless{\section{Additional Results}}{\label{app:GCGFN:additional_results}}

In this section, we present additional plots for experiments on 2, 3 and 4 objectives (Figures~\ref{fig:all_plots_2_objectives}, \ref{fig:all_plots_3_objectives} \& \ref{fig:all_plots_4_objectives}).

\begin{figure}[ht!]
    \centering
    \begin{minipage}{.5\textwidth}
        \centering
        \includegraphics[width=\textwidth]{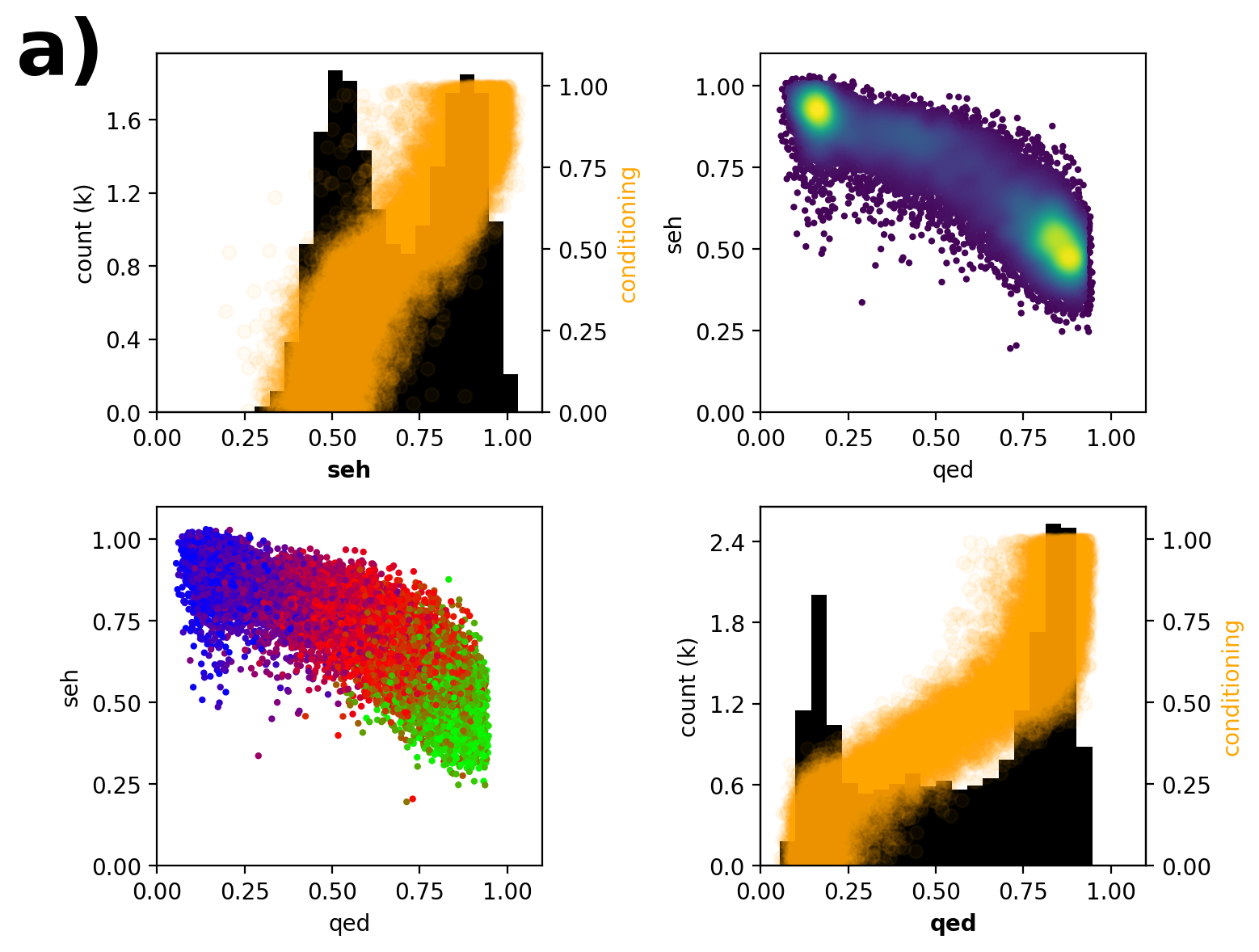}
    \end{minipage}%
    \begin{minipage}{0.5\textwidth}
        \centering
        \includegraphics[width=\textwidth]{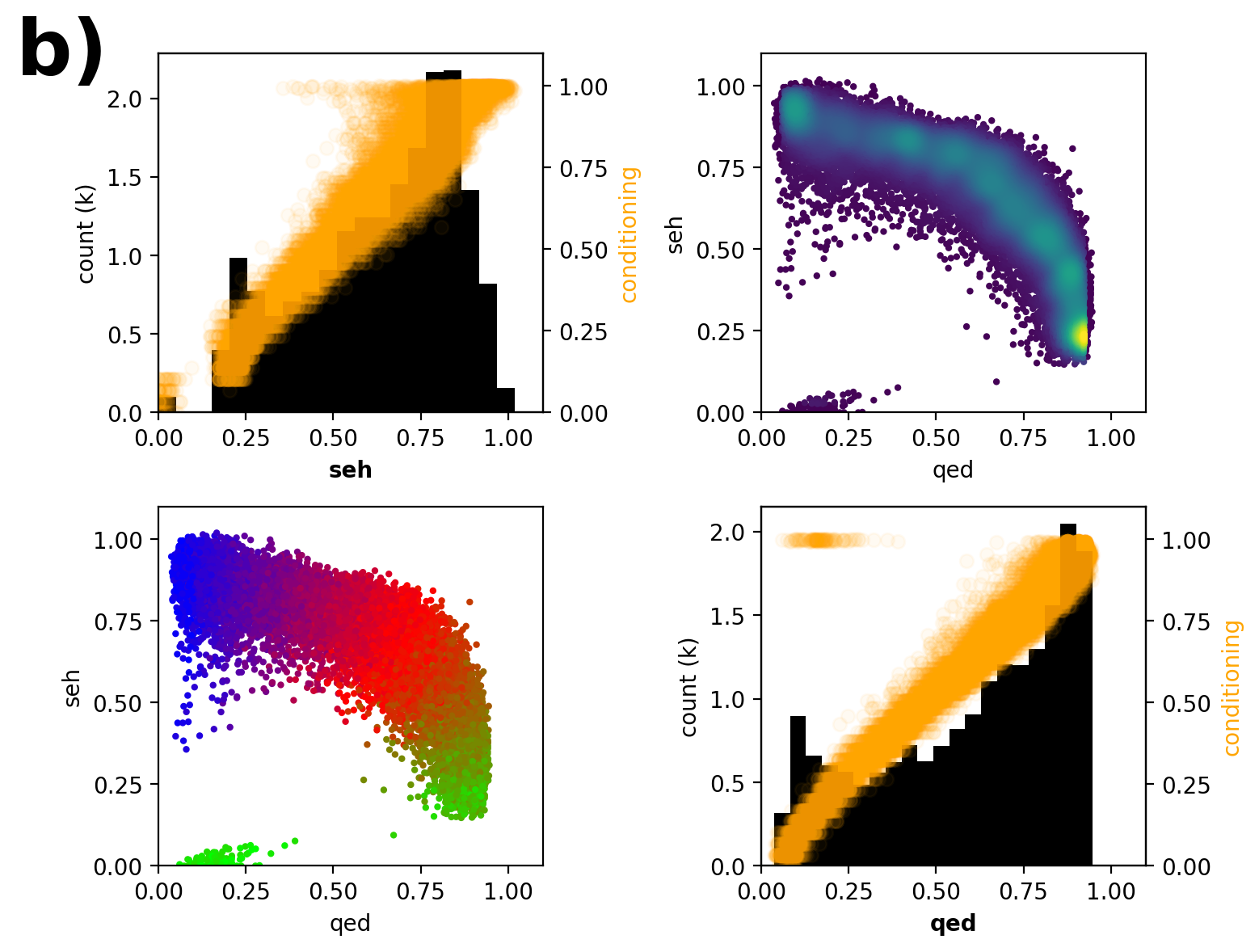}
    \end{minipage}
    \caption[Comparison between conditioning methods for 2 objectives]{Comparison between \textbf{a)} preference-conditioned and \textbf{b)} goal-conditioned models trained on the 2-objective problem (seh, qed). Each panel is an assemblage of $K \times K$ plots where $K$ is the number of objectives. \textbf{On the diagonal}, each plot focuses on a single objective. They each show a histogram (dark) of the samples' scores $r_{\cdot,k}$ for that objective, overlayed with a scatter plot (orange) in which each point is a distinct sample $i$ with coordinates $(x,y)=(r_{i,k},c_{i,k})$, where $r_{i,k}$ is the reward attributed to sample $i$ for objective $k$ and $c_{i,k}$ is the corresponding value of the conditioning vector that was used to generate that sample. The histogram thus showcases the distribution and span of our set of samples for a given dimension in objective space while the scatter plot allows us to visualise the correlation between the conditioning vectors and the resulting rewards for that dimension. \textbf{Above the diagonal}, each plot shows the density of the learned distribution on the plane corresponding to a pair of objectives. Brighter colors indicate that a region is more densely populated. \textbf{Below the diagonal}, each plot shows the controllability of the learned distribution where BRG colors represent the angle between the vector $[1, 0]$ and \textbf{a)} the preference-vector $w$ or \textbf{b)} the goal direction $d_g$ (this is the same as in Figure~\ref{fig:difficult_pareto_fronts_alignment}). Overall, we can see that on the density plot and on the histograms that the goal-conditioned approach produces a more uniformly distributed set of samples while the orange scatter plot and the BRG-colored plots show that they also provide a finer control over the generated samples.}
    \label{fig:all_plots_2_objectives}
\end{figure}
\begin{figure}[ht!]
    \centering
    \begin{minipage}{.5\textwidth}
        \centering
        \includegraphics[width=\textwidth]{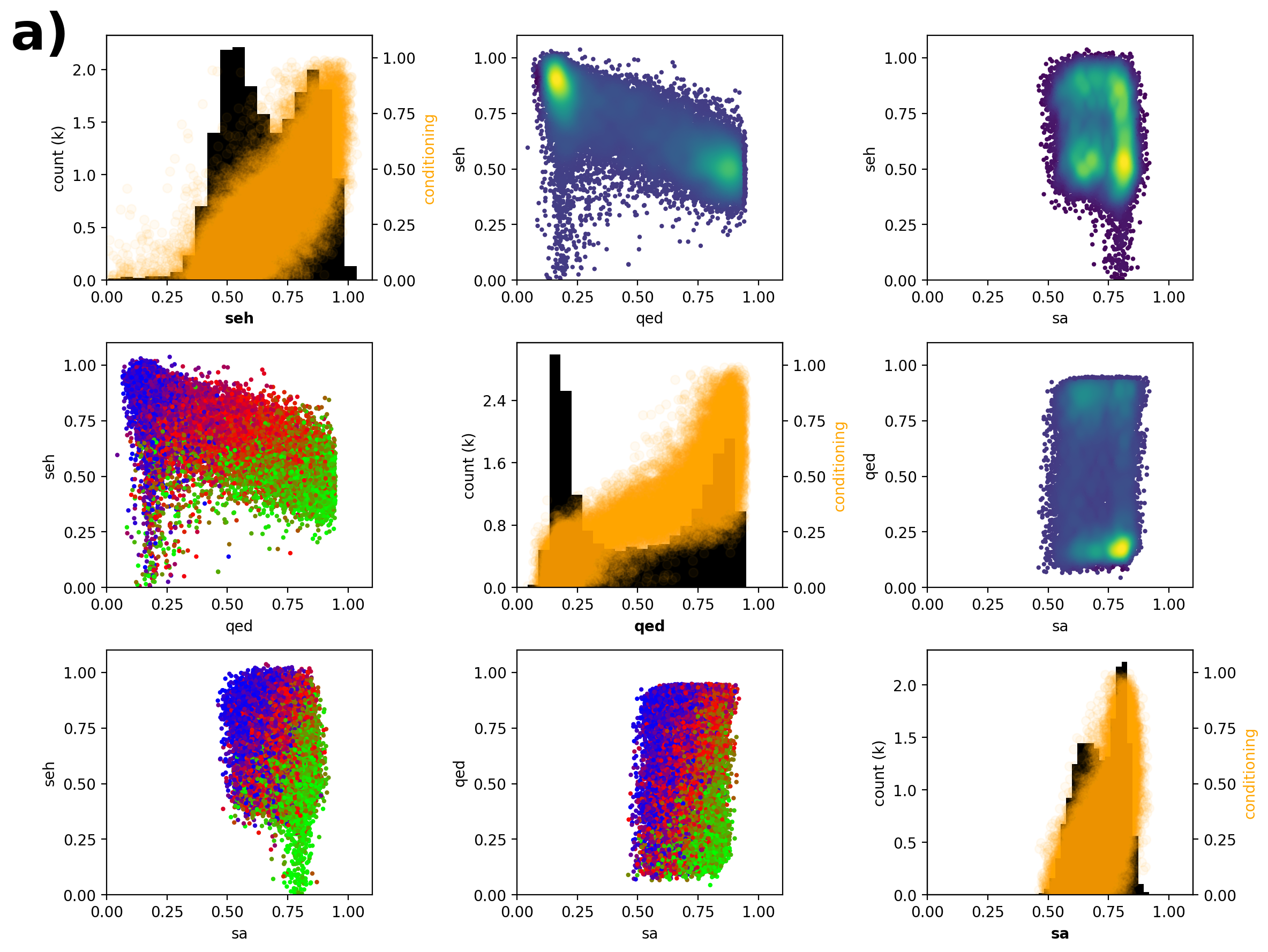}
    \end{minipage}%
    \begin{minipage}{0.5\textwidth}
        \centering
        \includegraphics[width=\textwidth]{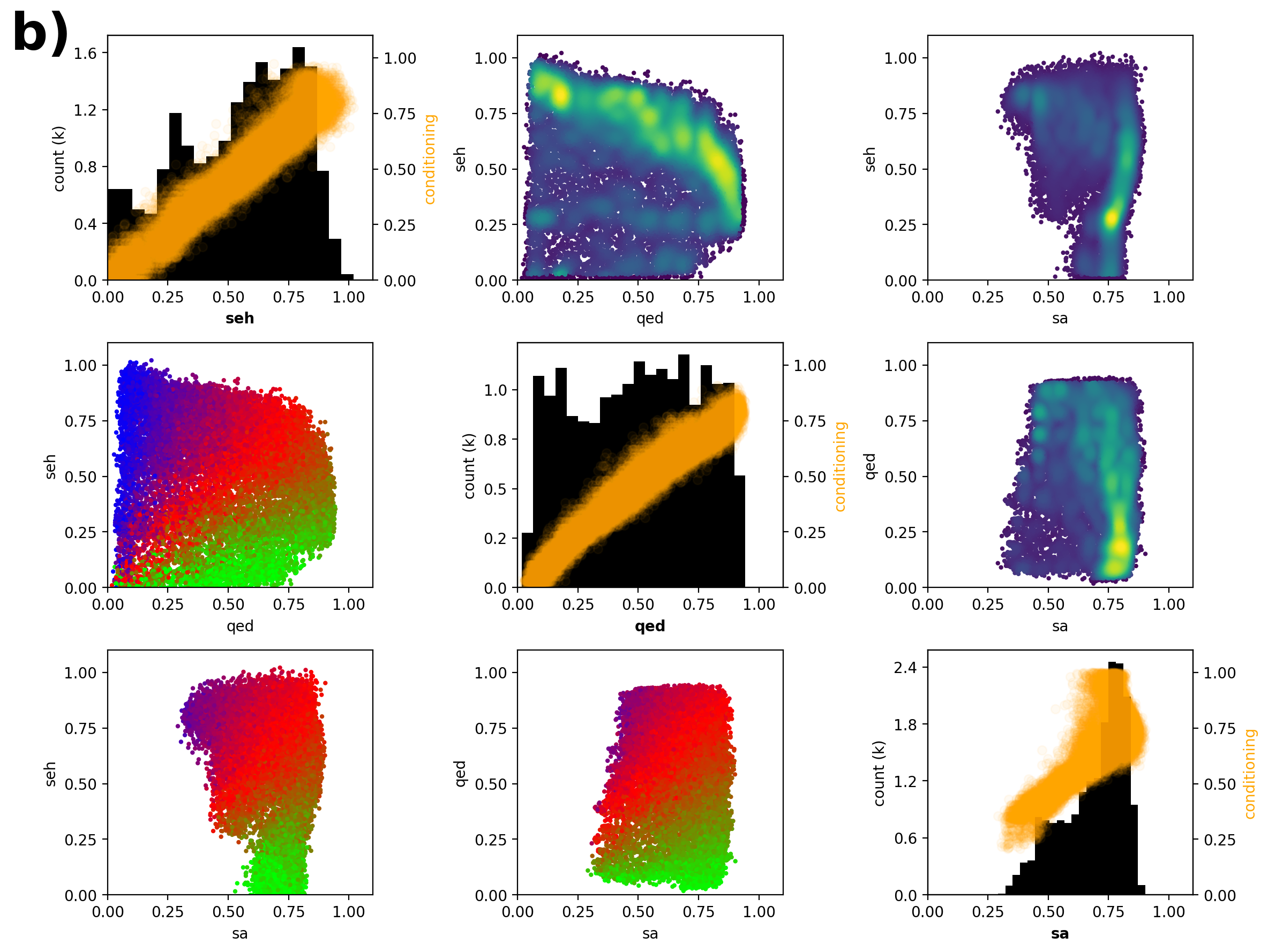}
    \end{minipage}
    \caption[Comparison between conditioning methods for 3 objectives]{Idem to Figure~\ref{fig:all_plots_3_objectives} but with 3 objectives: seh, qed, sa.}
    \label{fig:all_plots_3_objectives}
\end{figure}
\begin{figure}[ht!]
    \centering
    \begin{minipage}{.5\textwidth}
        \centering
        \includegraphics[width=\textwidth]{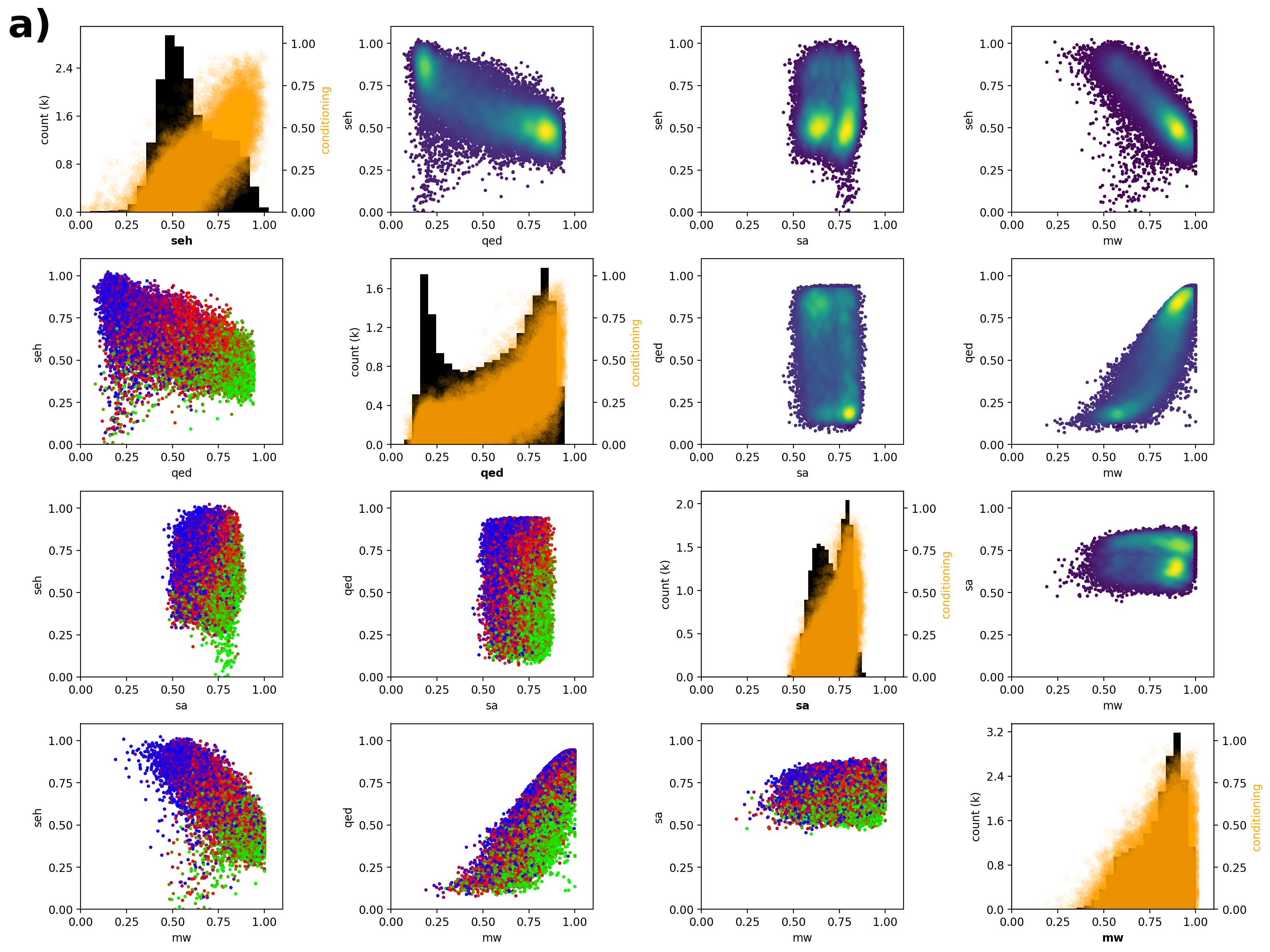}
    \end{minipage}%
    \begin{minipage}{0.5\textwidth}
        \centering
        \includegraphics[width=\textwidth]{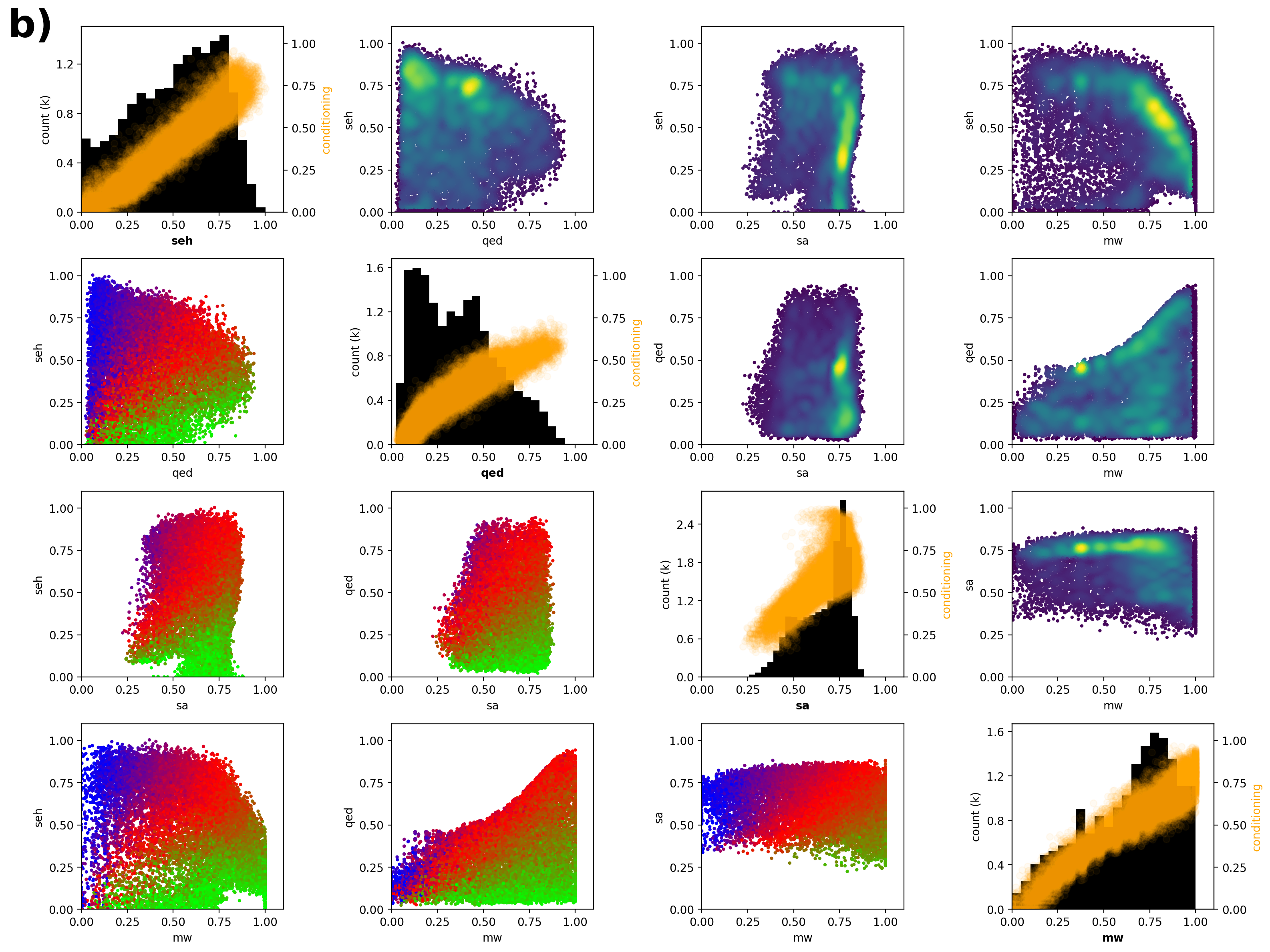}
    \end{minipage}
    \centering
    \caption[Comparison between conditioning methods for 4 objectives]{Idem to Figure~\ref{fig:all_plots_2_objectives} but with 4 objectives: seh, qed, sa, mw.}
    \label{fig:all_plots_4_objectives}
\end{figure}

}  
{}  
\end{document}